\documentclass{article}
\usepackage{algpseudocode}
\usepackage{algorithm}
\usepackage{amsthm}
\usepackage{amstext}
\usepackage{amsmath}
\usepackage{amssymb}
\usepackage{array}
\usepackage{appendix}
\usepackage{booktabs}
\usepackage{caption}
\usepackage{courier}
\usepackage{color,soul}
\usepackage{color,caption}
\usepackage{color,multirow}
\usepackage{epsfig}
\usepackage{epstopdf}
\usepackage{geometry}
\usepackage{graphics}
\usepackage{graphicx}
\usepackage{helvet}
\usepackage{latexsym}
\usepackage{mathrsfs}
\usepackage{multirow}
\usepackage{setspace}
\usepackage{subfigure}
\usepackage[pagewise]{lineno}

\newcommand{\argmin}{\mathop{\mathrm{arg\,min}}}

\usepackage{url}
\usepackage{xcolor}
\DeclareFixedFont{\myfont}{OT1}{ptm}{m}{n}{7pt}
\begin{document}
\title{Tensor completion via nonconvex tensor ring rank minimization with guaranteed convergence}


\author{
Meng Ding \thanks{Meng Ding, Ting-Zhu Huang, and Xi-Le Zhao are with the School of Mathematical Sciences, University of Electronic Science and Technology of China, Chengdu, Sichuan, 611731, P.R.China. Tian-Hui Ma is with the School of Mathematics and Statistics, Xi'an Jiaotong University, Xi'an, Shaanxi, 710049, P.R.China. E-mails: dingmeng56@163.com, tingzhuhuang@126.com, xlzhao122003@163.com, nkmth0307@126.com.}
\and Ting-Zhu Huang\thanks{Corresponding author.} \and Xi-Le Zhao  \and Tian-Hui Ma}

\date{}
\maketitle

\begin{abstract}
In recent studies, the tensor ring (TR) rank has shown high effectiveness in tensor completion due to its ability of capturing the intrinsic structure within high-order tensors. A recently proposed TR rank minimization method is based on the convex relaxation by penalizing the weighted sum of nuclear norm of TR unfolding matrices. However, this method treats each singular value equally and neglects their physical meanings, which usually leads to suboptimal solutions in practice. In this paper, we propose to use the logdet-based function as a nonconvex smooth relaxation of the TR rank for tensor completion, which can more accurately approximate the TR rank and better promote the low-rankness of the solution. To solve the proposed nonconvex model efficiently, we develop an alternating direction method of multipliers algorithm and theoretically prove that, under some mild assumptions, our algorithm converges to a stationary point. Extensive experiments on color images, multispectral images, and color videos demonstrate that the proposed method outperforms several state-of-the-art competitors in both visual and quantitative comparison.

\end{abstract}

\section{Introduction}
Tensor plays an important role in various fields, such as image processing \cite{Jiang2018FastDeRain,Lu2020TRPCA,Wen2008Restoration}, remote sensing \cite{Chang2017Transformed,Fu2016Unmixing,Zhao2013Unmixing,Zheng2019Mixed}, and machine learning \cite{Chang2019DeNet,Yang2017TTRNN}, due to its ability of expressing the complex interactions within high-dimensional data. Tensor completion aims to estimate the missing entries or damaged parts from the observed data, which is a fundamental problem in multidimensional image processing, e.g., color image inpainting \cite{Komodakis2006Image-Completion,Liu2019Image,Zhao2015Bayesian}, video inpainting \cite{Chen2014STDC,Zhang2018Nonconvex}, hyperspectral images recovery \cite{Li2012Coupled,Xing2012Dictionary}, and seismic data reconstruction \cite{Kreimer2012seismic}.

Inspired by the success of rank minimization in matrix completion, many researchers applied the low-tensor-rank constraint to recover high-order tensors with missing entries, named as low-rank tensor completion (LRTC). Unfortunately, unlike the matrix case, characterizing the redundancy of the tensor is much more difficult, and there exists many definitions of the tensor rank, such as CANDECOMP/PARAFAC (CP) rank, Tucker rank, tubal rank, and tensor train (TT) rank. Below we briefly review some related works and introduce our motivation and contributions.

\subsection{Related works}
Three representative works on the tensor low-rankness characterization are CP rank \cite{Chiantini2012CPDIdentifiability,Hitchcock1927CPD}, Tucker rank \cite{Ishteva2009Tucker_rank,Tucker1966Tucker}, and tubal rank \cite{Kilmer2013Third-Order}. As a direct generalization of matrix rank, CP rank is defined as the smallest number of rank-one tensors needed to generate the target tensor. Despite of its theoretical elegance, the computation of CP rank is NP-hard, and thus minimizing CP rank usually suffers from computational issues. Tucker rank is a vector consisting of ranks of unfolding matrices of the target tensor. Some works \cite{Gandy2011Tensor,Liu2013tensor} proposed to minimize Tucker rank using its convex relaxation, i.e., the sum of nuclear norm (SNN) of unfolding matrices. However, Tucker rank can only capture the correction between one mode and the rest modes of the tensor due to its unbalanced unfolding scheme, which is not much suitable for high-order tensor data \cite{Oseledets2011Tensor-Train-Decomposition}. Recently, Kilmer et al. \cite{Kilmer2013Third-Order} developed a new tensor singular value decomposition (tSVD) by treating third-order tensors as operators on matrices and defined the corresponding tubal rank as the nonzero singular tubes under the tSVD of the target tensor. Later, Zhang et al. \cite{Zhang2017tSVD} suggested to minimize tubal rank using tensor nuclear norm (TNN) and established theoretical results of TNN minimization for LRTC; Lu et al. \cite{Lu2018Exact} gave the exact guarantee of TNN minimization for the low-tubal-rank tensor recovery from Gaussian measurements.  Zheng et al. \cite{Zheng2018Ntubal} extended the tubal rank to the $N$-tubal rank for high-order tensors ($\textrm{order} >3$), with better flexibility in depicting the correlations along different modes.

Recently, tensor decompositions based on matrix product states have attracted much attention. Specifically, TT decomposition \cite{Oseledets2011Tensor-Train-Decomposition} represents a $j$th-order tensor $\mathcal{X}\in \mathbb{R}^{m_{1}\times \cdots \times m_{j}}$ by a set of third-order core tensors with two border matrices, i.e.,
\begin{equation}\label{TT decomposition}
x_{i_{1},\ldots,i_{j}}=\textbf{G}_{1}(i_{1},:)\mathcal{G}_{2}(:,i_{2},:) \cdots \mathcal{G}_{j-1}(:,i_{j-1},:)\textbf{G}_{j}(:,i_{j}),
\end{equation}
where $\textbf{G}_{1}\in \mathbb{R}^{m_{1}\times r_{1}}$, $\textbf{G}_{j}\in \mathbb{R}^{r_{j-1}\times m_{j}}$, $\mathcal{G}_{h}\in \mathbb{R}^{r_{h-1} \times m_{h} \times r_{h}}$, $h=2, \cdots, j-1$, and TT rank is defined as $(r_{1},\ldots,r_{j-1})$. TT decomposition and TT rank have been widely studied with theoretical analyses and numerical implementations \cite{Ding2019TTTV,Grasedyck2015TT,Wang2016TT}. Particularly, Bengua et al. \cite{Bengua2017Efficient} relaxed TT rank by tensor train nuclear norm based on a canonical matricization scheme, i.e., matricizing the tensor along permutations of modes. However, TT unfolding scheme also suffers from the unbalanced problem, i.e., matricizing the tensor along permutations makes the sizes of the middle unfolding matrices more balanced than those of the border matrices. To tackle this limitation, Zhao et al. \cite{Zhao2016TR} extended TT to tensor ring (TR) decomposition, which essentially solves the unbalance problem and balances the size of core tensors via the trace operation. More precisely, TR decomposition models each element of $\mathcal{X}$ by
\begin{equation}\label{TR decomposition}
x_{i_{1},\ldots,i_{j}}=\textrm{tr}(\mathcal{G}_{1}(:,i_{1},:) \cdots \mathcal{G}_{j}(:,i_{j},:)),
\end{equation}
where $\mathcal{G}_{h}\in \mathbb{R}^{r_{h-1} \times m_{h} \times r_{h}}$ is the $h$th third-order core tensor ($h=1, \cdots, j$), the boundary condition states that $r_{0} = r_{j}$, and $\textrm{tr}(\cdot)$ denotes the matrix trace. TR rank corresponding to \eqref{TR decomposition} is defined as $(r_{1},\ldots,r_{j})$. The TR model can be viewed as a linear combination of several correlated TT decompositions, leading to a higher representation ability.

\begin{figure}[!t]
\scriptsize\setlength{\tabcolsep}{0.5pt}
\begin{center}
\begin{tabular}{c}
\includegraphics[width=0.94\textwidth]{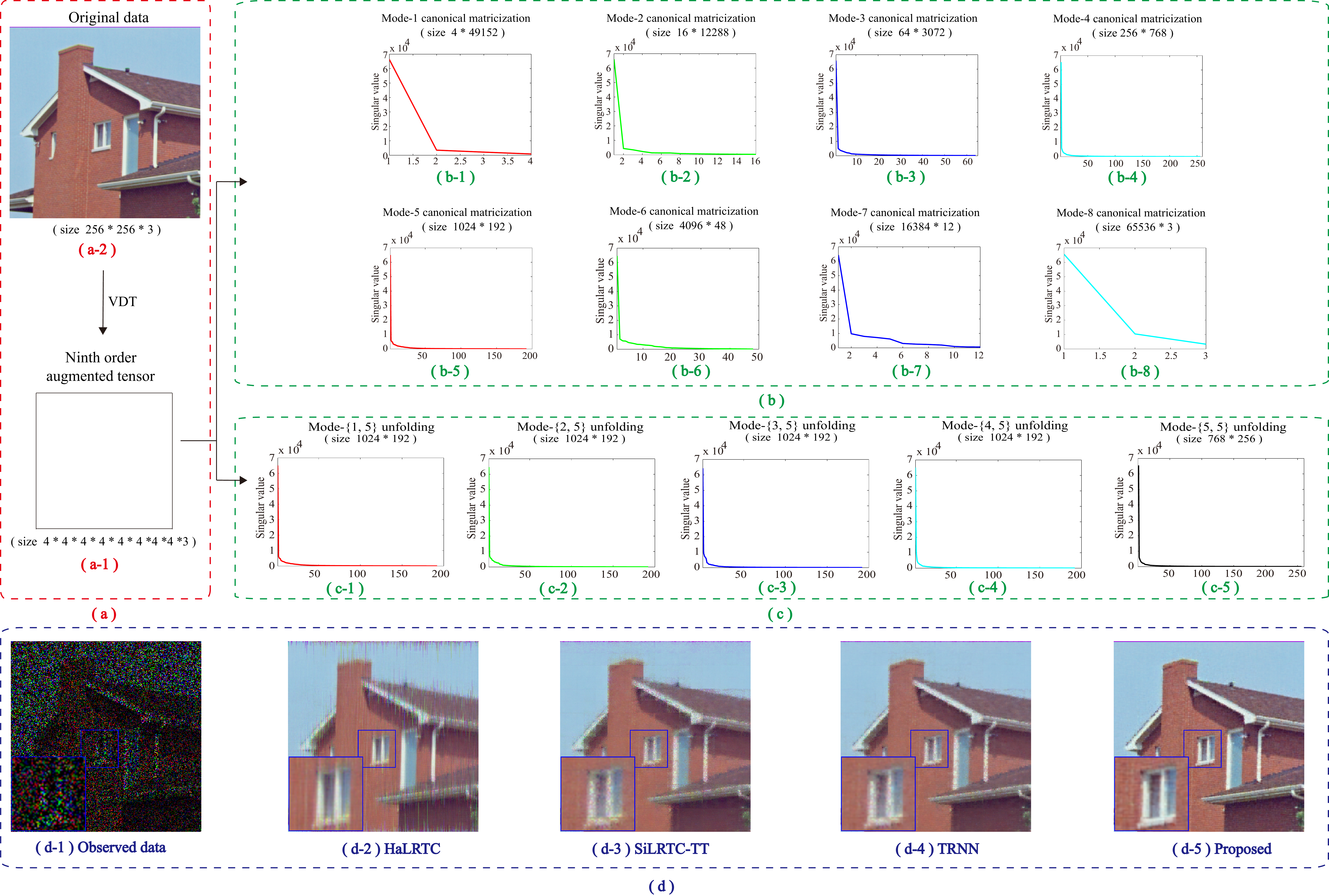}
\end{tabular}
\caption{\small{Comparison of the low-rankness of the canonical matricization scheme and TR unfolding scheme. (a-1, 2) the augmented tensor, and the original data. (b-1) to (b-8) the distribution of singular values of the mode-1 to mode-8 canonical matricizations of the augmented tensor (a-1). (c-1) to (c-5) the distribution of singular values of the mode-$\{1,5\}$ to mode-$\{5,5\}$ unfoldings of the augmented tensor (a-1). The average ratios of singular values larger than $1\%$ of the corresponding largest ones of (b) and (c) are $21.0\%$ and $13.1\%$ respectively.  (d-1) to (d-5) the observed data, the results recovered by HaLRTC \cite{Liu2013tensor},  SiLRTC-TT \cite{Bengua2017Efficient}, TRNN \cite{Huang2019TRNN}, and the proposed method.}}
  \label{fig:motivation}
  \end{center}\vspace{-0.3cm}
\end{figure}

The minimization of TR rank has became a hot research topic. Wang et al. \cite{Wang2017TR} proposed an iterative algorithm by alternatively updating each core tensor, and Yuan et al. \cite{Yuan2019TR} imposed the low-rank regularization on TR core tensors for tensor completion. However, these methods are generally time-consuming and suffer from the problem of optimal rank selection. For more efficiently minimizing TR rank, Yu et al. \cite{Yu2019TRNN} and Huang et al. \cite{Huang2019TRNN} proposed a new circular TR unfolding scheme named mode-$\{n,l\}$ unfolding and relaxed the nonconvex TR rank by the convex tensor ring nuclear norm (TRNN). More precisely, the mode-$\{n,l\}$ unfolding is implemented by first permuting $\mathcal{X}$ with order $[l,\ldots,j,1,\ldots,l-1]$ and then unfolding $\mathcal{X}$ along first $n$ modes, and then the TRNN is defined as the sum of nuclear norm of each TR unfolding matrix, i.e.,
$\|\mathcal{Z}\|_{\textrm{TRNN}}=\sum_{n}^{j}\sum_{l}^{j}\|\textbf{Z}_{\{n,l\}}\|_{\ast}$.
TRNN minimization has shown promising performance in LRTC with lower computational complexity and no need of choosing the optimal TR rank. In addition, compared with TT unfolding, TR unfolding can better capture the global correlation of high-order tensors, since TR unfolding matrices admit more balanced sizes and exhibit more significantly low-rank property than those obtained by TT unfolding; see Figure \ref{fig:motivation} for an illustration.

\subsection{Motivations and contributions}
Despite of the effectiveness of the above TRNN-based methods, TRNN still has two shortcomings in TR rank minimization. First, TRNN is based on the nuclear norm, which is only a biased approximation to the TR rank and can not effectively promote the low-rankness of the solution. Second, TRNN treats each singular value equally and neglects the physical meaning of singular values, which leads to suboptimal solutions and loss the major information. Actually, in practice the singular values have clear physical meanings and should be treated differently \cite{Gu2017Weighted}. For instance, larger singular values often represent low-frequency information such as major edges and cartoons; smaller singular values convey high-frequency information such as tiny structures and textures, which are, however, more likely to be contaminated by noises. Thus, we should shrink less the larger singular values to preserve the major data components while shrink more the smaller ones to suppress random errors.

Summarizing the aforementioned observations, TR decomposition admits a promising representation ability for high-order tensors; TR unfolding operator gives a balanced tensor matricization scheme; and TRNN is the convex relaxation of TR rank, which is easy to minimize. However, computing TR rank is NP-hard and time-consuming; and TRNN treats each singular value equally, which less effectively approximates TR rank and leads to a suboptimal solution. So here comes the question: can we find a new relaxation for TR rank that is tighter than TRNN and easy to optimize?

In this paper, we propose a novel nonconvex approximation to TR rank by using the logdet function \cite{Fazel2003logDet} onto TR unfolding matrices, which is defined as
\begin{equation}\label{LogTR}
\begin{split}
\|\mathcal{Z}\|_{\textrm{LogTR}}= &\ \sum_{n=1}^{j}\sum_{l=1}^{j}\beta_{n,l}\log\det((\textbf{X}_{\{n,l\}}\textbf{X}_{\{n,l\}}^{\top})^{1/2}+\varepsilon\textbf{I}_{n})\\
 =&\ \sum_{n=1}^{j}\sum_{l=1}^{j}\beta_{n,l}\sum_{i=1}^{i_{n,l}}\log(\sigma_{i}(\textbf{X}_{\{n,l\}})+\varepsilon),
\end{split}
\end{equation}
where $\{\beta_{n,l}\}$ are non-negative weighted parameters and $\sigma_{i}(\textbf{X}_{\{n,l\}})$ is the $i$th singular value of $\textbf{X}_{\{n,l\}}$. Here, the proposed LogTR surrogate has three advantages. First, LogTR function does not need to compute TR rank. Second, LogTR function not only retains the strength of TR unfolding (shown in Figure \ref{fig:motivation}), but also provides a tighter approximation to TR rank ($l_{0}$ norm of the singular values) than TRNN. Figure \ref{fig:log} compares the rank, the nuclear norm, and the logdet function for scalars; and Table \ref{table:LogTR} gives the low-rank approximation of TR unfolding matrices\footnote{The TR unfolding matrices are obtained by applying TR unfolding on the augmented tensor.} of the CAVE multispectral images (MSI) database\footnote{http://www1.cs.columbia.edu/CAVE/databases/multispectral} on average. From both visual and numerical comparisons, LogTR surrogate can approximate TR rank much better than TRNN. Third, it is easy to solve the proposed nonconvex LogTR surrogate by the alternating direction method of multipliers (ADMM) method, where the logdet-based subproblem has the closed-form solution using weighted singular value thresholding \cite{Gong2013General,Xie2016Multispectral}.

\begin{figure}[!t]
\scriptsize\setlength{\tabcolsep}{0.5pt}
\begin{center}
\begin{tabular}{c}
\includegraphics[width=0.8\textwidth]{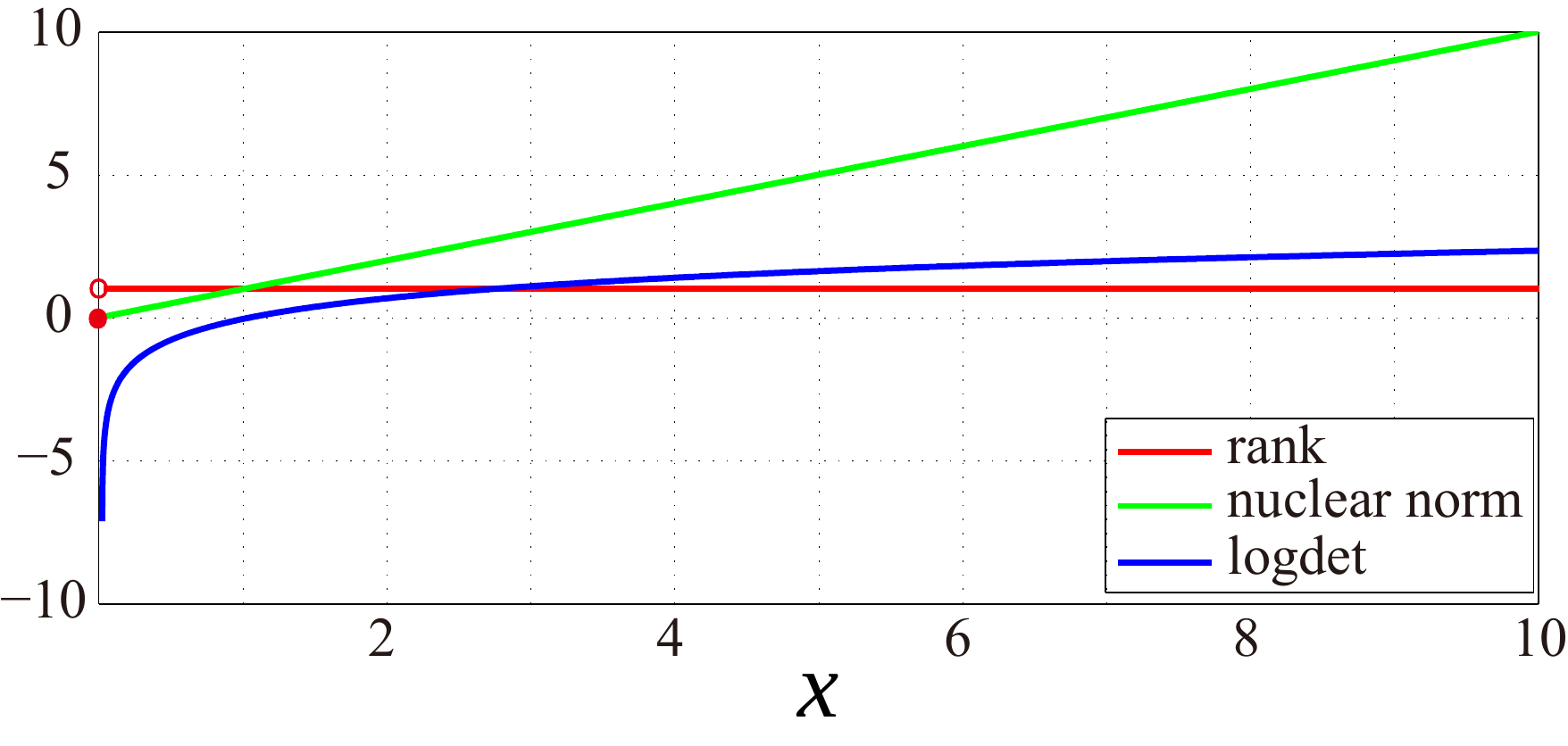}
\end{tabular}
\caption{Comparison of the rank, the nuclear norm, and the logdet function for scalars.}
  \label{fig:log}
  \end{center}\vspace{-0.3cm}
\end{figure}

\begin{table}[!t]
\renewcommand\arraystretch{1.2}
\caption{Average low-rank approximation of TR unfolding. TR unfolding rank, TRNN, and LogTR are calculated by the numbers of the singular values which are larger than 0.01 of the largest one, the weighted sum of nuclear norm, the weighted sum of the logarithmic singular values of TR unfolding matrices, respectively.}
\vspace{-0.3cm}
\begin{center}
\begin{tabular}{c|c|ccccc}
\hline \hline
Data &size & TR unfolding rank & TRNN & LogTR\\ \hline
MSIs & $256\times 256\times 31$ & 57    &$2.8076\times 10^{5}$	  &2610    \\
          \hline
\end{tabular}\label{table:LogTR}
\end{center}
\end{table}

Based on the proposed low-TR-rank approximation \eqref{LogTR}, we formulate the following nonconvex model for tensor completion:
\begin{equation}\label{our model}
\begin{split}
\min_{\mathcal{X}} &\ \sum_{n=1}^{\lceil j/2 \rceil}\beta_{n}\log\det((\textbf{X}_{\{n\}}\textbf{X}_{\{n\}}^{\top})^{1/2}+\varepsilon\textbf{I}_{n})\\
\textrm{s.t.} &\ \mathcal{P}_{\Omega}(\mathcal{X})=\mathcal{P}_{\Omega}(\mathcal{T}),
\end{split}
\end{equation}
where $\{\beta_{n}\}_{n=1}^{\lceil j/2 \rceil}$ are weighted parameters satisfying $\beta_{n} \geq 0$ and $\sum_{n=1}^{\lceil j/2 \rceil}\beta_{n}=1$, $\textbf{X}_{\{n\}}$ is the mode-$\{n,\lceil \frac{j}{2} \rceil\}$ unfolding of $\mathcal{X}$, $\textbf{I}_{n}$ is the identify matrix, $\mathcal{T}$ is the incomplete tensor with order $j$, $\Omega$ is the index of observed entries, $\mathcal{P}_{\Omega}(\cdot)$ is the projection operator that keeps entries in $\Omega$ and zeros out others. In \eqref{our model}, we only consider the first $\lceil j/2 \rceil$ rather all the unfolding matrices, because this setting not only reduces much computational complexity, but also ensures that the balanced unfolding matrices capture the most global correlations of high-order tensors \cite{Huang2019TRNN}. To solve the proposed nonconvex model, we develop the ADMM method and demonstrate that, under some mild assumptions, the sequence generated by the ADMM-based algorithm converges to the stationary point of the augmented Lagrangian function. From Figure \ref{fig:motivation} (d), one can see that the proposed method preserves structures and details better than compared methods.

The contributions of this paper are mainly three folds: (1) we propose a new logdet-based TR rank approximation for tensor completion, which can effectively depicts the global low-rankness of tensors; (2) we solve the proposed model by an efficient ADMM-based algorithm with guaranteed convergence; (3) experiments show that the proposed method achieves better performance than several existing LRTC methods in recovered visual effects and numerical metrics.



The outline of this paper is as follows. In Section \ref{section:Preliminary}, we give some preliminary knowledge about tensors and visual data tensorization. In Section \ref{section:model and algorithm}, we detail the proposed effective ADMM solver with guaranteed convergence. In Section \ref{section:experiments}, we conduct numerical experiments to demonstrate the effectiveness of the proposed algorithm. Finally, we conclude this work in Section \ref{section:Conclusion}.

\section{Preliminary}
\label{section:Preliminary}
\subsection{Tensor basics}
We give some basic notations of tensors, which are listed in Table \ref{table:notation}.

\begin{table}[!ht]
\centering
\caption{\small{Tensor notations.}}
\label{table:notation}
\begin{tabular}{ll}
\hline

\hline
  \vspace{-0.25cm}
  \\
  Notations                                                 & Explanations
  \vspace{0.08cm}
  \\
  \hline
  \vspace{-0.25cm}
  \\
  \vspace{0.08cm}
  $\mathcal{Z}$, $\textbf{Z}$, \textbf{z}, $z$              & tensor, matrix, vector, scalar.
  \\
  \vspace{0.08cm}
  $\langle\mathcal{X}, \mathcal{Y}\rangle$                  & inner product of two same-sized tensors $\mathcal{X}$ and $\mathcal{Y}.$
  \\
  \vspace{0.08cm}
  $\|\mathcal{Z}\|_{F}$                                     & Frobenius norm of $\mathcal{Z}$.
  \\
  \vspace{0.08cm}
  $\textbf{Z}_{(n)}$                                        & mode-$n$ unfolding of $\mathcal{Z}\in \mathbb{R}^{m_{1}\times \ldots \times m_{j}}$ of size $\mathbb{R}^{m_{n}\times \prod_{d\neq n} m_{d}}$.
  \\
  \vspace{0.08cm}
  $\textbf{Z}_{[n]}$                                        & mode-$n$ canonical matricization of $\mathcal{Z}\in \mathbb{R}^{m_{1}\times \ldots \times m_{j}}$ of size
  \\
  \vspace{0.03cm}
                                                            &$ \mathbb{R}^{(\prod_{d=1}^{n}m_{d}) \times (\prod_{d=n+1}^{j}m_{d})}$.
  \\
  \vspace{0.08cm}
  $\textbf{Z}_{\{n,l\}}$, $\textrm{unfold}_{\{n,l\}}(\mathcal{Z})$  & mode-$\{n,l\}$ unfolding of $\mathcal{Z}\in \mathbb{R}^{m_{1}\times \ldots \times m_{j}}$ of size $ \mathbb{R}^{(\prod_{d=l}^{l+n-1}m_{d}) \times (\prod_{d=l+n}^{l-1}m_{d})}$.
  \\
  \vspace{0.1cm}
  $\textrm{fold}_{\{n,l\}} (\textbf{Z}_{\{n,l\}})$                  & inverse operator of mode-${\{n,l\}}$ unfolding satisfying $\mathcal{Z} = \textrm{fold}_{\{n,l\}} (\textbf{Z}_{\{n,l\}})$.
  \\
\hline

\hline
\end{tabular}
\end{table}

A tensor is a high-dimensional array and its order (or mode) is the number of its dimensions. We denote scalars, vectors, matrices, and tensors as lowercase letters ($z$), boldface lowercase letters (\textbf{z}), capital letters, ($\textbf{Z}$), and calligraphic letters ($\mathcal{Z}$), respectively. $\mathcal{Z}\in \mathbb{R}^{m_{1}\times \ldots \times m_{j}}$ is the $j$th-order tensor and its $(i_{1},\ldots,i_{j})$-th component is denoted as $z_{i_{1},\ldots,i_{j}}$.

The inner product of tensors $\mathcal{X}$ and $\mathcal{Y}$ is denoted as
\[
\langle\mathcal{X}, \mathcal{Y}\rangle = \sum_{i_{1},\ldots,i_{j}} x_{i_{1},\ldots,i_{j}}\cdot y_{i_{1},\ldots,i_{j}}.
\]
$\|\mathcal{Z}\|_{F} = \sqrt{\langle\mathcal{Z}, \mathcal{Z}\rangle}$ denotes the Frobenius norm of $\mathcal{Z}$.

$\textbf{Z}_{(n)}\in \mathbb{R}^{m_{n}\times \prod_{d\neq n} m_{d}}$ denotes the mode-$n$ unfolding of $\mathcal{Z}$. The element $(i_{n}, b)$ of matrix $\textbf{Z}_{(n)}$ maps to the tensor element $(i_{1},\ldots,i_{j})$ satisfying
\begin{equation}\label{mode-$k$ unfolding}
b = 1+\sum_{d=1,d\neq n}^{j}\ (i_{d} - 1) j_{d} \quad \textrm{with} \quad j_{d} = \prod_{t = 1, t\neq n}^{d-1} m_{t}.
\end{equation}
This operator can be implemented via the following MATLAB command:
\[
\textbf{Z}_{(n)} = \textrm{reshape}\big(\textrm{shiftdim}(\mathcal{Z}, n-1), \textrm{size}(\mathcal{Z}),[]\big).
\]

$\textbf{Z}_{[n]}\in \mathbb{R}^{(\prod_{d=1}^{n}m_{d}) \times (\prod_{d=n+1}^{j}m_{d})}$ denotes the mode-$n$ canonical matricization of $\mathcal{Z}$. The element $(a, b)$ of matrix $\textbf{Z}_{[n]}$ maps to the tensor element $(i_{1},\ldots,i_{j})$ satisfying
\begin{equation}\label{mode-$k$ canonical matricization}
a = 1+\sum_{d=1}^{n}\big((i_{d}-1)\prod_{t=1}^{d-1}m_{t}\big) \quad \textrm{and} \quad b = 1+\sum_{d=n+1}^{j}\big((i_{d}-1)\prod_{t=n+1}^{d-1}m_{t}\big).
\end{equation}
This operator can be implemented by the function reshape in MATLAB, i.e.,
\[
\textbf{Z}_{[k]} = \textrm{reshape}\big(\mathcal{Z}, \Pi_{d=1}^{n}m_{d},\Pi_{d=n+1}^{j}m_{d}\big).
\]

$\textbf{Z}_{\{ n,l\}}\in \mathbb{R}^{(\prod_{d=l}^{l+n-1}m_{d}) \times (\prod_{d=l+n}^{l-1}m_{d})}$ denotes the mode-$\{n,l\}$ unfolding of $\mathcal{Z}$. The element $(a, b)$ of matrix $\textbf{Z}_{\{ n,l\}}$ maps to the tensor element $(i_{1},\ldots,i_{j})$ satisfying

\begin{equation}\label{shifting-$k$ mode-$l$ matricization}
a = 1+\sum_{d=l}^{l+n-1}\big((i_{d}-1)\prod_{t=l}^{d-1}m_{t}\big) \quad \textrm{and} \quad b = 1+\sum_{d=l+n}^{l-1}\big((i_{d}-1)\prod_{t=l+n}^{d-1}m_{t}\big).
\end{equation}
Using the permutation and reshape operators, we can get $\textbf{Z}_{\{n,l\}}$ as follows:
\[
\textbf{Z}_{\{n,l\}} = \textrm{reshape}\big(\textrm{permute}(\mathcal{Z}, [l,\ldots,j,1,\ldots,l-1]), \Pi_{d=l}^{l+n-1}m_{d},\Pi_{d=l+n}^{l-1}m_{d}\big).
\]
We denote the mode-$\{n,l\}$ unfolding as $\textrm{unfold}_{\{n,l\}}(\cdot)$, and the corresponding inverse operator is denoted as ``$\textrm{fold}_{\{n,l\}}$", i.e., $\text{fold}_{\{n,l\}}(\textbf{Z}_{\{n,l\}})=\mathcal{Z}$.

\subsection{Visual data tensorization}
\label{VDT}
We introduce the visual data tensorization (VDT) \cite{Yuan2018High} as a rearranging method for transforming a low-order tensor to a high-order one. Using VDT, the proposed method can effectively exploit the low-TR-rankness embedded in the underlying data.

Generally, given visual data $\mathcal{Z}\in \mathbb{R}^{m\times n\times p_{1}\ldots \times p_{s}}$, where the first two dimensions are spatial dimensions and the later dimensions represent RGB color channels, time, bands, etc. The details of performing VDT on $\mathcal{Z}$ are as follows. Assuming that $m$ and $n$ have factorizations $m=\Pi_{d=1}^{q} m_{d}$ and $n=\Pi_{d=1}^{q} n_{d}$, we factorize the spatial dimensions $m\times n$ to $m_{1}\times m_{2}\times \ldots \times m_{q}\times n_{1}\times n_{2}\times \ldots \times n_{q}$, then we permute the order of the first $2q$ dimensions to $m_{1}\times n_{1}\times m_{2}\times n_{2}\times \ldots \times m_{q}\times n_{q}$ and reshape to the size $m_{1} n_{1}\times m_{2} n_{2}\times \ldots \times m_{q}n_{q}$, finally the original tensor is transformed into a high-order tensor $\mathcal{\tilde{Z}}\in \mathbb{R}^{m_{1} n_{1}\times m_{2} n_{2}\times \ldots \times m_{q}n_{q}\times p_{1}\ldots \times p_{s}}$. The $d$-th dimension of $\mathcal{\tilde{Z}}$ corresponds to an $m_{d}\times n_{d}$ patch of $\mathcal{Z}$. After applying the completion algorithm on $\mathcal{\tilde{Z}}$, performing the reverse operation of VDT to transform the result into the original size.


\section{Tensor completion via nonconvex TR minimization}
\label{section:model and algorithm}

In this section, we  present the proposed algorithm in detail and establish the convergence of the proposed algorithm.

\subsection{The proposed algorithm}

Recall that the proposed model is
\begin{equation}\label{re our model}
\begin{split}
\min_{\mathcal{X}}&\ \sum_{n=1}^{\lceil j/2 \rceil}\beta_{n}L(\textbf{X}_{\{n\}})\\
\textrm{s.t.}     &\ \mathcal{P}_{\Omega}(\mathcal{X})=\mathcal{P}_{\Omega}(\mathcal{T}),
\end{split}
\end{equation}
where $L(\textbf{X}_{\{n\}})=\log\det((\textbf{X}_{\{n\}}\textbf{X}_{\{n\}}^{\top})^{1/2}+\varepsilon\textbf{I}_{n})$. We formulate the numerical scheme based on ADMM to solve the optimization problem \eqref{re our model}. By introducing auxiliary variables $\mathcal{G} = [\mathcal{G}_{1};\cdots;\mathcal{G}_{\lceil j/2 \rceil}]$, we get the equivalent constrained version of \eqref{re our model} as follows:
\begin{equation}\label{constrained model}
\begin{split}
\argmin_{\mathcal{X},  \mathcal{G}} & \ \mathcal{E}(\mathcal{G})+\mathcal{I}_{\Omega}(\mathcal{X})\\
                      \textrm{s.t.} & \ \mathcal{G}-[\textbf{I};\cdots;\textbf{I}]\mathcal{X}=0,
\end{split}
\end{equation}
where $\mathcal{E}(\mathcal{G})=\sum_{n=1}^{\lceil j/2 \rceil}\beta_{n}L(\textbf{G}_{n\{n\}})$, $\mathcal{I}_{\Omega}(\cdot)$ is the indicator function satisfies $\mathcal{I}_{\Omega}(x)=0$ if $x\in \Omega$ and $\infty$ otherwise, and $\textbf{I}$ denotes the identify operator. By separating the variables in \eqref{constrained model} into two groups $\{\mathcal{G}_{n}\}_{n=1}^{\lceil j/2 \rceil}$ and $\mathcal{X}$, we observe that  \eqref{constrained model} fits the framework of ADMM \cite{Wang2019Global}. The augmented Lagrangian function of \eqref{constrained model} is defined as
\begin{equation}\label{Lagrangian function}
\mathcal{L}_{\eta}(\mathcal{G},\mathcal{X},\mathcal{H})=\mathcal{E}(\mathcal{G})+\mathcal{I}_{\Omega}(\mathcal{X})+\langle \mathcal{G}-[\textbf{I};\cdots;\textbf{I}]\mathcal{X}, \mathcal{H}\rangle+\frac{\eta}{2}\|\mathcal{G}-[\textbf{I};\cdots;\textbf{I}]\mathcal{X}\|_{F}^{2},
\end{equation}
where $\mathcal{H} = [\mathcal{H}_{1};\cdots;\mathcal{H}_{\lceil j/2 \rceil}]$, $\{\mathcal{H}_{n}\}_{n=1}^{\lceil j/2 \rceil}$ are Lagrangian multipliers, and $\eta$ is a penalty parameter. Then, the ADMM procedure for solving \eqref{Lagrangian function} is following:
\begin{equation}\label{iterative scheme}
\left\{
\begin{array}{l}
\begin{split}
\vspace{0.1cm}
&\mathcal{G}^{k+1} =\argmin_{\mathcal{G}}\mathcal{L}_{\eta}(\mathcal{G},\mathcal{X}^{k}, \mathcal{H}^{k}),\\
\vspace{0.2cm}
&\mathcal{X}^{k+1}=\argmin_{\mathcal{G}}\mathcal{L}_{\eta}(\mathcal{G}^{k+1}, \mathcal{X}, \mathcal{H}^{k}),\\
\vspace{0.1cm}
&\mathcal{H}_{n}^{k+1}=\mathcal{H}_{n}^{k}+\eta(\mathcal{G}_{n}^{k+1}-\mathcal{X}^{k+1}).\\
\end{split}
\end{array}
\right.
\end{equation}
Next, we give the details for solving each subproblem.


  (1)\ \textbf{$\mathcal{G}$-subproblem.} It is clear that the minimization with respect to each $\mathcal{G}_{n}$ is decoupled. The optimal $\mathcal{G}_{n}$ is given by
  \begin{equation}\label{Gsubproblem}
  \mathcal{G}_{n}^{k+1} = \argmin_{\mathcal{G}_{n}} \beta_{n}L(\textbf{G}_{n\{n\}})+\frac{\eta^{k}}{2}\Big\|\mathcal{G}_{n}-\mathcal{X}^{k}+\frac{\mathcal{H}_{n}^{k}}{\eta^{k}}\Big\|_{F}^{2}.
  \end{equation}
  By using the equation $\|\textbf{X}_{\{k\}}\|_{F} = \|\mathcal{X}\|_{F}$, we rewrite \eqref{Gsubproblem} as the following problem:
  \begin{equation}
  \begin{split}
  \textbf{G}_{n\{n\}}^{k+1} & = \argmin_{\textbf{G}_{n\{n\}}} \beta_{n}L(\textbf{G}_{n\{n\}})+\frac{\eta^{k}}{2}\Big\|\textbf{G}_{n\{n\}}-\textbf{X}_{\{n\}}^{k}+\frac{\textbf{H}_{n\{n\}}^{k}}{\eta^{k}}\Big\|_{F}^{2},\\
                      & = \argmin_{\textbf{G}_{n\{n\}}} \beta_{n}\sum_{j}\textrm{log}(\sigma_{j}(\textbf{G}_{n\{n\}})+\varepsilon) +\frac{\eta^{k}}{2}\Big\|\textbf{G}_{n\{n\}}-\textbf{X}_{\{n\}}^{k}+\frac{\textbf{H}_{n\{n\}}^{k}}{\eta^{k}}\Big\|_{F}^{2}.
  \end{split}
  \end{equation}
  According to the work \cite{Gong2013General,Xie2016Multispectral}, $\mathcal{G}_{n}$ has the closed-form solution
  \begin{equation}\label{Gsubproblem solver}
  \mathcal{G}_{n}^{k+1} = \textrm{fold}_{\{n\}}\big[ \textbf{U}_{n}\textbf{S}_{\frac{\beta_{n}}{\eta^{k}},\varepsilon}(\Sigma_{n}) \textbf{V}_{n}^{T}\big],
  \end{equation}
  where $\textbf{U}_{n}\Sigma_{n} \textbf{V}_{n}^{T}$ is the singular value decomposition (SVD) of $\textbf{X}_{\{n\}}^{k}-\frac{\textbf{H}_{n\{n\}}^{k}}{\eta}$ and $\textbf{S}_{\frac{\beta_{n}}{\eta^{k}},\varepsilon}(\Sigma_{k})$ is the thresholding operator defined as
  \begin{equation}
  \textbf{S}_{\frac{\beta_{n}}{\eta^{k}},\varepsilon}(x)=
  \left\{
  \begin{array}{ll}
      0, & \textrm{if} \ c_{2}\leq 0, \\
      \textrm{sign}(x)\Big(\frac{c_{1}+\sqrt{c_{2}}}{2}\Big), & \textrm{if} \ c_{2}> 0, \\
  \end{array}
  \right.
  \end{equation}
  with $c_{1}=|x|-\varepsilon$, $c_{2}=(c_{1})^{2}-4\big(\frac{\beta_{n}}{\eta^{k}}-\varepsilon|x|\big)$. This thresholding operator shrinks less the larger singular values while more the smaller ones \cite{Zheng2019Mixed}. The calculation of $\mathcal{G}_{n}$ mainly involves the SVD of the matrix with size $p_{n} \times q_{n}$ ($p_{n}=\prod_{d=l}^{l+n-1}m_{d}$,  $q_{n}=\prod_{d=l+n}^{l-1}m_{d}$, $n=1,\ldots,l$, and $l=\lceil j/2 \rceil$), whose complexity is $O\big(\textrm{min}\big(p_{n}^{2} q_{n},\ p_{n} q_{n}^{2}\big)\big)$.

   (2)\ \textbf{$\mathcal{X}$-subproblem.} \ The optimal $\mathcal{X}$ is the solution of the following quadratic problem:
  \begin{equation}\label{Xsubproblem}
  \mathcal{X}^{k+1} = \argmin \ \mathcal{I}_{\Omega}(\mathcal{X})+ \sum_{n=1}^{\lceil j/2 \rceil}\frac{\eta^{k}}{2}\Big\|\mathcal{G}_{n}^{k+1}-\mathcal{X}+\frac{\mathcal{H}_{n}^{k}}{\eta^{k}}\Big\|_{F}^{2}.
  \end{equation}
  Then $\mathcal{X}$ can be calculated by
  \begin{equation}\label{Xsubproblem solver}
  \mathcal{X}^{k+1} = \mathcal{P}_{\Omega^{c}}\Bigg( \sum_{n=1}^{\lceil j/2 \rceil}(\mathcal{G}_{n}^{k+1}+\mathcal{H}_{n}^{k+1}/\eta^{k})\Bigg)+\mathcal{P}_{\Omega}(\mathcal{T}).
  \end{equation}
  The cost of computing $\mathcal{X}$ is $O(\prod_{n=1}^{\lceil j/2 \rceil}m_{n})$.

The proposed ADMM-based algorithm is summarized in Algorithm \ref{algorithm}. At each iteration, the total cost of Algorithm \ref{algorithm} is

\[
O\bigg(\sum_{n=1}^{\lceil j/2 \rceil}\textrm{min}\big(p_{n}^{2} q_{n},\ p_{n} q_{n}^{2}\big)\bigg),
\]
where $p_{n}=\prod_{d=\lceil j/2 \rceil}^{\lceil j/2 \rceil+n-1}m_{d}$,  $q_{n}=\prod_{d=\lceil j/2 \rceil+n}^{\lceil j/2 \rceil-1}m_{d}$, $n=1,\ldots,\lceil j/2 \rceil$.

\renewcommand{\algorithmicrequire}{\textbf{Input:}} 
\renewcommand{\algorithmicensure}{\textbf{Output:}} 

\begin{algorithm}
\caption{ADMM-based solver for \eqref{re our model}.}\label{algorithm}
\begin{algorithmic}[1]
      \Require the observed tensor $\mathcal{T}$, index set $\Omega$, parameters $\eta$ and $\varepsilon$.
      \State \textbf{Initialization:} $\mathcal{X}=\mathcal{T}$, $\mathcal{H}_{n}=0$, and $k_{max}=500$. \vspace{1mm}
	  \State \textbf{While} not satisfying the stopping condition, \textbf{do}\vspace{1mm}
	  \State  \quad  \textbf{for} $k=1$ to $\lceil j/2 \rceil$ \textbf{do};
	  \State  \quad \quad \quad update $\mathcal{G}_{n}$ via \eqref{Gsubproblem solver};
      \State  \quad  \textbf{end for};
	  \State  \quad  update $\mathcal{X}$ via \eqref{Xsubproblem solver};
	  \State  \quad  update $\mathcal{H}_{n}$ via \eqref{iterative scheme};
      \State  \quad  update $\eta^{k+1} = 1.1 \eta^{k}$;
	  \State \textbf{end while}
	  \Ensure restored tensor $\mathcal{X}$.
\end{algorithmic}
\end{algorithm}

\subsection{Convergence}
In this subsection, we present the convergence of Algorithm \ref{algorithm}. Following, we first briefly review the framework and the convergence of ADMM for solving nonconvex and nonsmooth optimization problems \cite{Wang2019Global}. In \cite{Wang2019Global}, the authors considered the optimization problem:
\begin{equation}\label{gengral ADMM}
\begin{split}
\min_{x, y} & \quad \mathcal{E}(\textbf{x})+\mathcal{F}(\textbf{y})\\
\textrm{s.t.} & \quad \textbf{A}\textbf{x}+\textbf{B}\textbf{y}=0,
\end{split}
\end{equation}
where $\mathcal{E}(\textbf{x})$ is continuous, proper, possibly nonsmooth, $\textbf{x}\in \mathbb{R}^{m_{1}}$ is the variable with the corresponding coefficient $\textbf{A}\in \mathbb{R}^{l\times m_{1}}$, $\mathcal{F}(\textbf{y})$ is proper and differentiable, $\textbf{y}\in \mathbb{R}^{m_{2}}$ is the variable with the corresponding coefficient $\textbf{B}\in \mathbb{R}^{l\times m_{2}}$. $\mathcal{E}$ and $\mathcal{F}$ can be possibly nonconvex functions. By introducing the auxiliary multiplier $\textbf{z}\in \mathbb{R}^{l}$, we obtain the augmented Lagrangian function of \eqref{gengral ADMM} as
\[
\mathcal{L}_{\eta}(\textbf{x},\textbf{y},\textbf{z})=\mathcal{E}(\textbf{x})+\mathcal{F}(\textbf{y})+\langle \textbf{z},\textbf{A}\textbf{x}+\textbf{B}\textbf{y}\rangle+\frac{\eta}{2}\|\textbf{A}\textbf{x}+\textbf{B}\textbf{y}\|_{2}^{2},
\]
where $\eta > 0$ is a penalty parameter. Denoting by $k$ the iteration index, according to ADMM \cite{Wu2010Augmented}, the iterative way to solve \eqref{gengral ADMM} is
\begin{equation}\label{general iter}
\left\{
\begin{array}{ll}
\vspace{0.3cm}
\textbf{x}^{k+1}=\arg\min_{\textbf{x}}\mathcal{L}_{\eta}(\textbf{x},\textbf{y}^{k},\textbf{z}^{k}),\\
\vspace{0.3cm}
\textbf{y}^{k+1}=\arg\min_{\textbf{y}}\mathcal{L}_{\eta}(\textbf{x}^{k+1},\textbf{y},\textbf{z}^{k}),\\
\textbf{z}^{k+1}=\textbf{z}^{k}+\eta(\textbf{A}\textbf{x}^{k+1}+\textbf{B}\textbf{y}^{k+1}).\\
\end{array}
\right.
\end{equation}

The following theorem \cite{Wang2019Global} presents the convergence result of the nonsmooth and nonconvex ADMM.

\newtheorem{theorem}{\textbf{Theorem}}
\begin{theorem} \cite{Wang2019Global} \label{lemma convergence}
Suppose that the following assumptions A1$-$A5 hold. Then, for any initial guess and sufficiently large $\eta$, the sequence $(\textbf{x}^{k}, \textbf{y}^{k}, \textbf{z}^{k})$ generated by \eqref{general iter} converges to the stationary point of $\mathcal{L}_{\eta}$.

A1 (\textbf{coercivity}) Define the nonempty feasible set $\mathcal{D}=\{(\textbf{x},\textbf{y})\in \mathbb{R}^{m_{1}+m_{2}}:\textbf{A}\textbf{x}+\textbf{B}\textbf{y}=0\}$. $\mathcal{E}(\textbf{x})+\mathcal{F}(\textbf{y})$ is coercive over $\mathcal{D}$, i.e., $\mathcal{E}(\textbf{x})+\mathcal{F}(\textbf{y})\rightarrow \infty$ if $(\textbf{x},\textbf{y})\in \mathcal{D}$ and $\|(\textbf{x},\textbf{y})\|_{2}\rightarrow \infty$.

A2 (\textbf{feasibility}) $Im(\textbf{A})\subseteq Im(\textbf{B})$, where $Im(\cdot)$ denotes the image of a matrix;

A3 (\textbf{Lipschitz sub-minimization paths})

\quad (a) For any $\textbf{x}$, $H:Im(\textbf{B})\rightarrow \mathbb{R}^{l}$ obeying $H(\textbf{u})=\arg\min_{\textbf{y}}\{\mathcal{E}(\textbf{x})+\mathcal{F}(\textbf{y}):\textbf{B}\textbf{y}=\textbf{u}\}$ is a Lipschitz continuous map,

\quad (b) For any $\textbf{y}$, $G:Im(\textbf{A})\rightarrow \mathbb{R}^{l}$ obeying $G(\textbf{u})=\arg\min_{\textbf{x}}\{\mathcal{E}(\textbf{x})+\mathcal{F}(\textbf{y}):\textbf{A}\textbf{x}=\textbf{u}\}$ is a Lipschitz continuous map;

A4 (\textbf{objective-$\mathcal{E}$ regularity}) $\mathcal{E}$ is Lipschitz differentiable.

A5 (\textbf{objective-$\mathcal{F}$ regularity}) $\mathcal{F}$ is lower semi-continuous or $\sup\{\|d\|:\textbf{x}\in X, \textbf{d}\in \partial\mathcal{F}(\textbf{x})\}$ is bound for any bound set $X$;
\end{theorem}

Next, we present the convergence of Algorithm \ref{algorithm} by proving that it fits the framework of \cite{Wang2019Global}.

\begin{theorem} \label{convergence}
For sufficiently large $\eta$, the sequence $(\mathcal{G}^{k},\mathcal{X}^{k},\mathcal{H}^{k})$ generated by Algorithm \ref{algorithm} converges to the stationary point of the augmented Lagrangian function \eqref{Lagrangian function}.
\end{theorem}

\begin{proof}
We reformulate (\ref{constrained model}) as the following matrix-vector multiplication form:
\[
\begin{split}
&\argmin_{\mathcal{X},  \mathcal{G}} \quad \mathcal{E}(\mathcal{G})+\mathcal{I}_{\Omega}(\mathcal{X})\\
\textrm{s.t.}\ &\left(
\begin{array}{cccc}
\textbf{I} & 0          & 0      & 0\\
0          & \textbf{I} & \cdots & 0\\
0          & 0          & \vdots & \vdots\\
0          & 0          & \cdots & \textbf{I}\\
\end{array}
\right)
\left(
\begin{array}{cccc}
\textbf{g}_{1}\\
\textbf{g}_{2}\\
\vdots        \\
\textbf{g}_{\lceil \frac{j}{2} \rceil}\\
\end{array}
\right)
-
\left(
\begin{array}{cccc}
\textbf{I}\\
\textbf{I}\\
\vdots    \\
\textbf{I}\\
\end{array}
\right)\textbf{x}=\textbf{0},
\end{split}
\]
where $\{\textbf{g}_{n}\}_{n=1}^{\lceil j/2 \rceil}$ and $\textbf{x}$ denote the vectorization of $\{\mathcal{G}_{n}\}_{n=1}^{\lceil j/2 \rceil}$ and $\mathcal{X}$, respectively. We can see that the proposed nonconvex model fits the framework of (\ref{gengral ADMM}).

To show the convergence of the proposed algorithm, we verify that our model fits the assumptions A1$-$A5. A1 holds because of the coercivity of $\mathcal{E}(\mathcal{G})+\mathcal{F}(\mathcal{X})$. A2 and A3 hold because both the coefficient matrices of $[\textbf{g}_{1}^{\top};\cdots;\textbf{g}_{\lceil j/2 \rceil}^{\top}]^{\top}$ and $\textbf{x}$ are full column rank. A4 holds because the logdet function is Lipschitz differentiable \cite{Fazel2003logDet}. A5 holds because the indicator function $\mathcal{I}_{\Omega}(\mathcal{X})$ is lower semi-continuous. This completes the proof.
\end{proof}

\section{Experiments}
\label{section:experiments}
\vspace{-0.35cm}
\begin{table}[!th]
\centering
\caption{Summary of compared methods.}
\vspace{-0.2cm}
\label{table:method}
\begin{tabular}{cc}
\hline

\hline
  \vspace{-0.35cm}
  \\
  \multirow{2}{*}{Methods}                     & Low-rankness characterization and
  \\
                                               & additional regularization
  \\
  \hline
  \vspace{-0.25cm}
  \\
  \vspace{0.2cm}
  HaLRTC \cite{Liu2013tensor}                  & Tucker rank
  \\
  \multirow{2}{*}{NSNN \cite{Ji2017A}}         & logdet-based Tucker rank
  \vspace{-0.1cm}
  \\
                                               & approximation
  \vspace{0.2cm}
  \\
  \multirow{2}{*}{LRTC-TV \cite{Li2017LRTCTV}} & Tucker rank with
  \vspace{-0.1cm}
  \\
                                               & anisotropic total variation
  \vspace{0.2cm}
  \\
  SiLTRC-TT \cite{Bengua2017Efficient}         &  TT rank
  \vspace{0.2cm}

  \\
  tSVD \cite{Zhang2017tSVD}                    & tubal rank
  \vspace{0.2cm}
  \\
  \vspace{0.08cm}
  \multirow{2}{*}{KBR \cite{Xie2018Kronecker}} & Kronecker-basis-representation
  \vspace{-0.1cm}
  \\
                                               &  based tensor sparsity measure
  \vspace{0.2cm}
  \\
  TRNN \cite{Huang2019TRNN}                    &  TR rank
  \vspace{0.2cm}
  \\
  \multirow{2}{*}{LogTR}                         & logdet-based TR rank
  \vspace{-0.1cm}
  \\
                                               & approximation
  \\
\hline

\hline
\end{tabular}
\end{table}

In this section, we show the effectiveness of the proposed method on various real-world data including color images, multispectral images (MSIs), and color videos. We compare our method, called tensor completion via \textbf{log}det-based \textbf{t}ensor \textbf{r}ing rank minimization (LogTR), with seven state-of-the-art approaches, namely HaLRTC \cite{Liu2013tensor}, NSNN \cite{Ji2017A}, LRTC-TV \cite{Li2017LRTCTV}, SiLTRC-TT \cite{Bengua2017Efficient}, tSVD \cite{Zhang2017tSVD}, KBR \cite{Xie2018Kronecker}, and TRNN \cite{Huang2019TRNN}, which are summarized in Table \ref{table:method}. All test tensors are scaled into the interval [0, 255]. All the methods are implemented by MATLAB; the simulations are performed on a desktop equipped with Windows 10 64-bit, Intel(R) Core(TM) i7-6700 CPU with 3.40 GHz core, and 8 GB RAM.

\begin{figure}[!t]
\scriptsize\setlength{\tabcolsep}{0.5pt}
\begin{center}
\begin{tabular}{ccccc}
\includegraphics[width=0.19\textwidth]{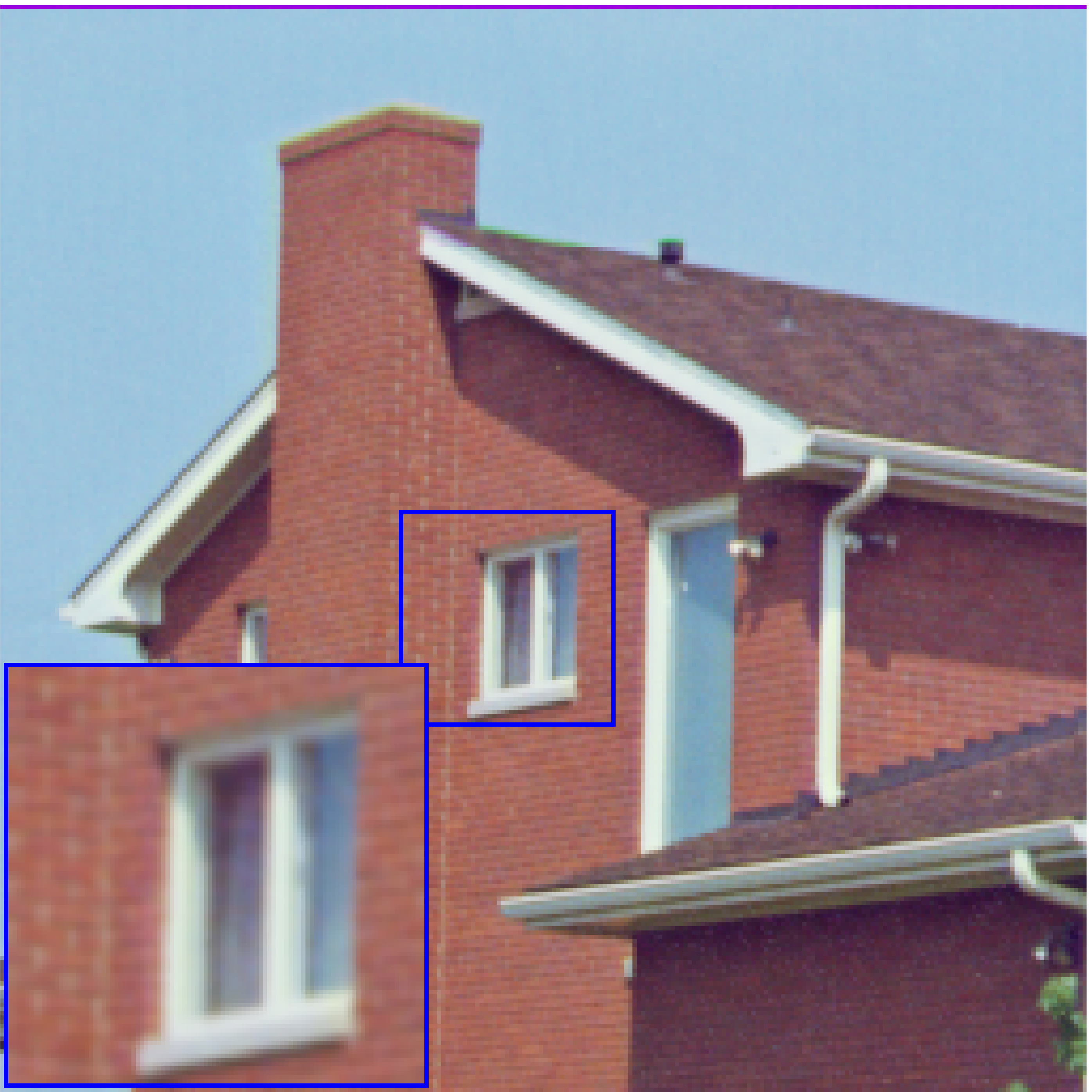}&
\includegraphics[width=0.19\textwidth]{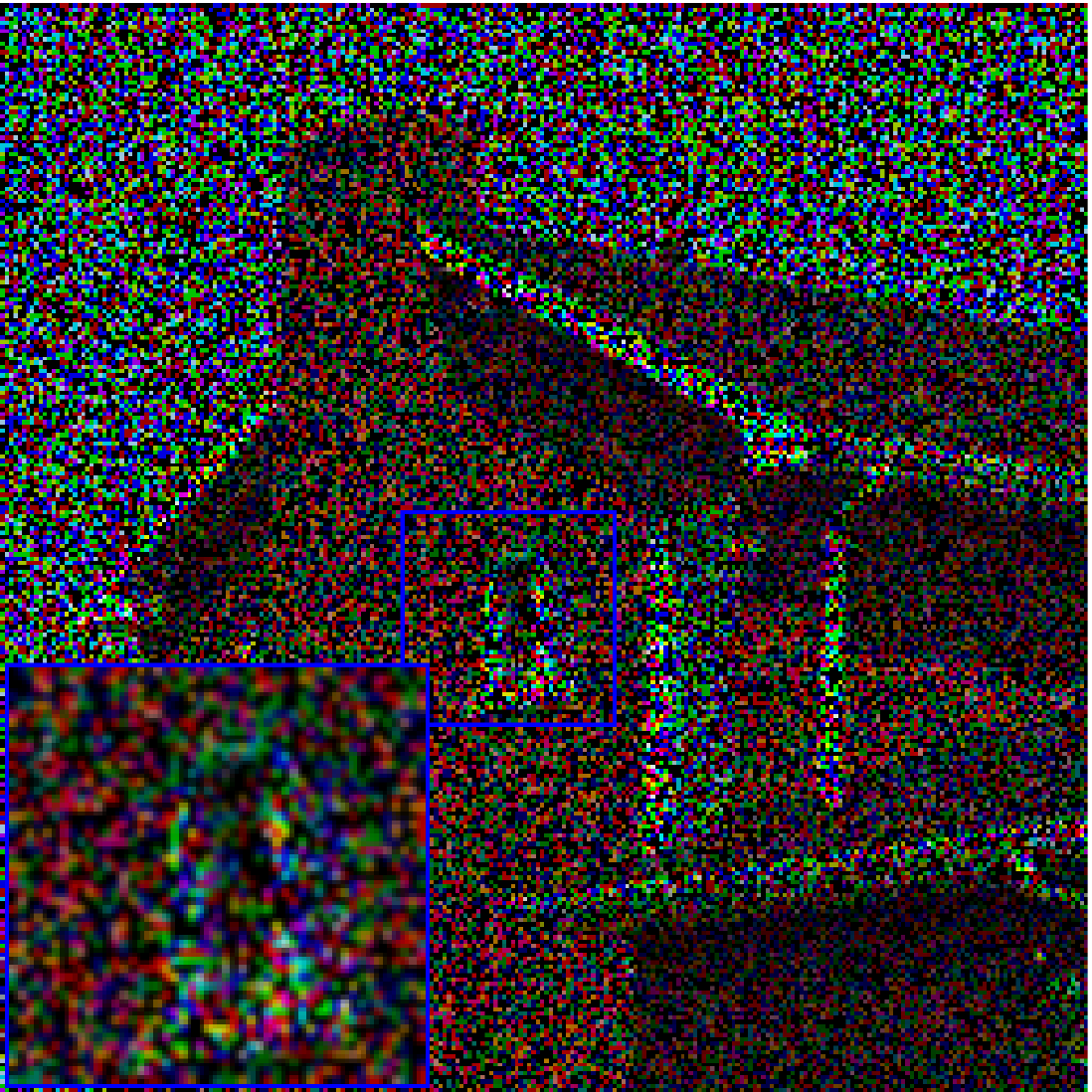}&
\includegraphics[width=0.19\textwidth]{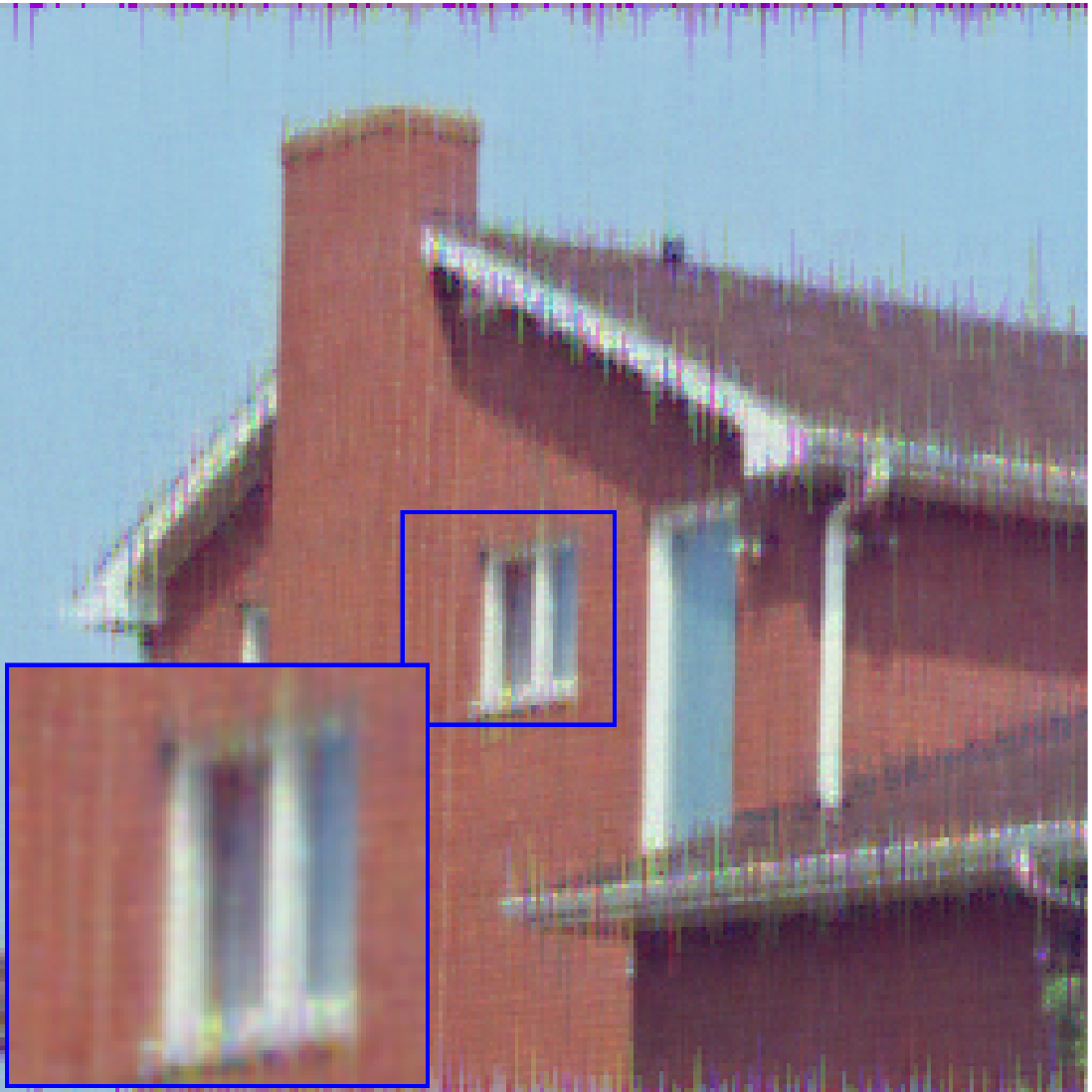}&
\includegraphics[width=0.19\textwidth]{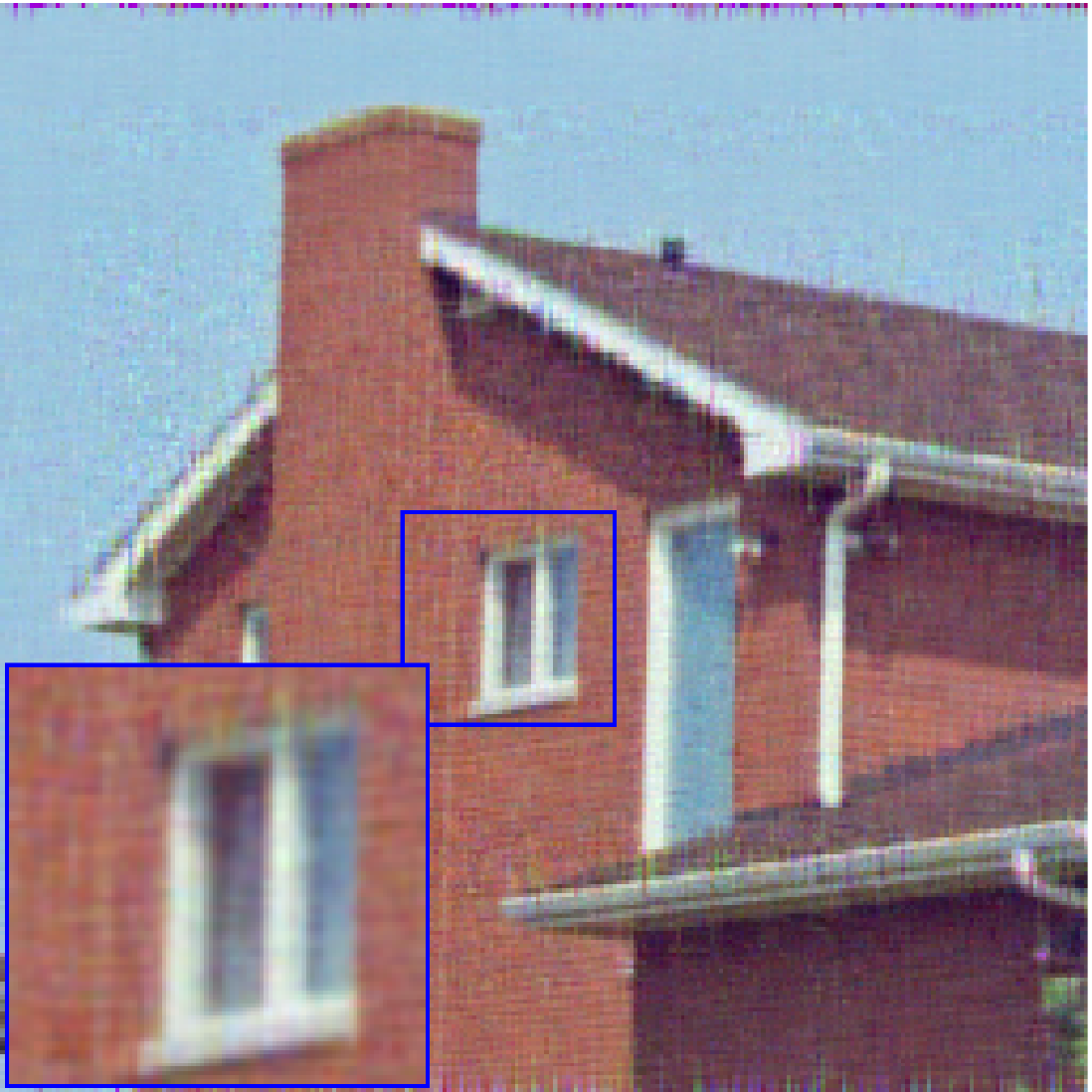}&
\includegraphics[width=0.19\textwidth]{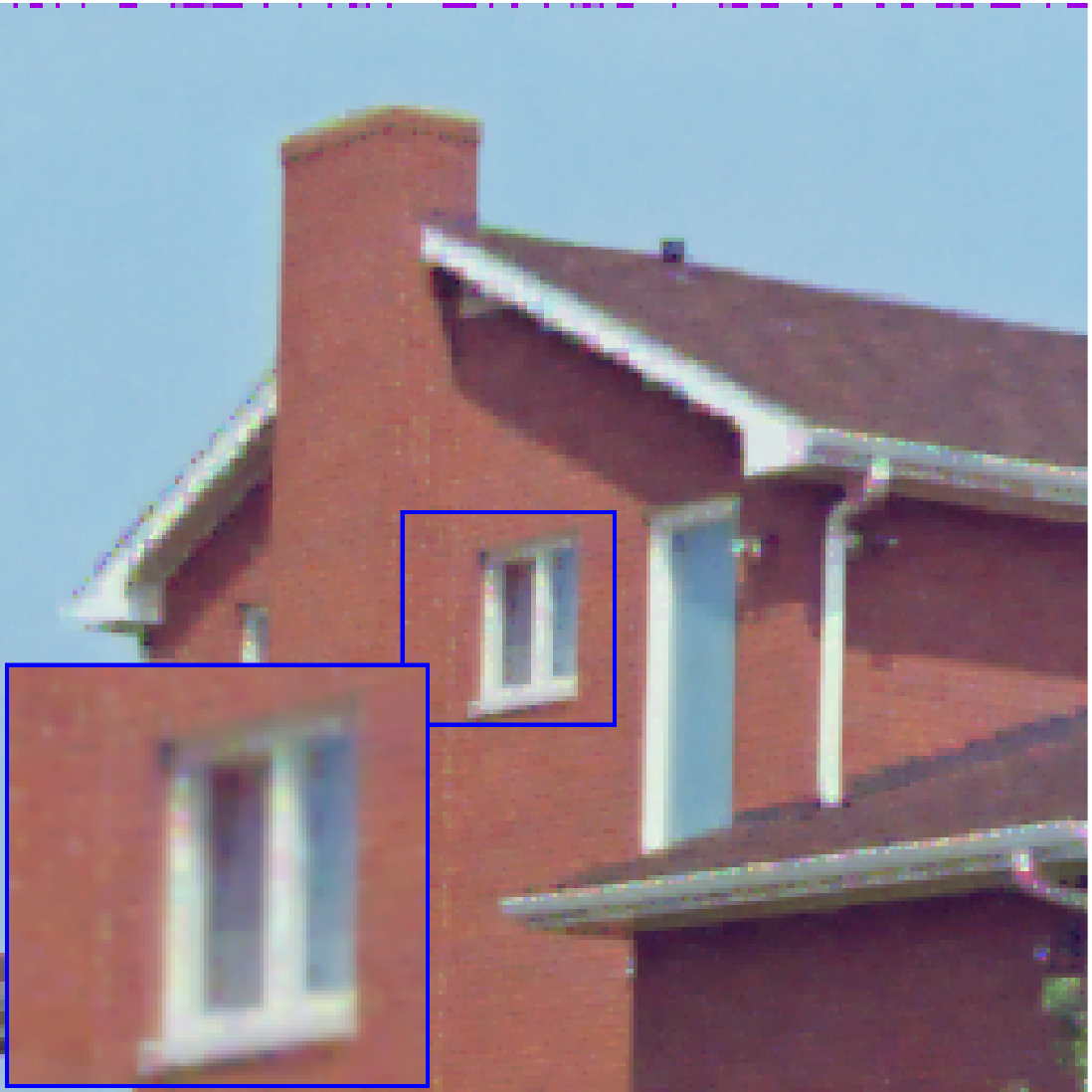}\vspace{0.01cm}\\
(a)Original & (b) Observed & (c) HaLRTC & (d) NSNN & (e) LRTC-TV\\
\includegraphics[width=0.19\textwidth]{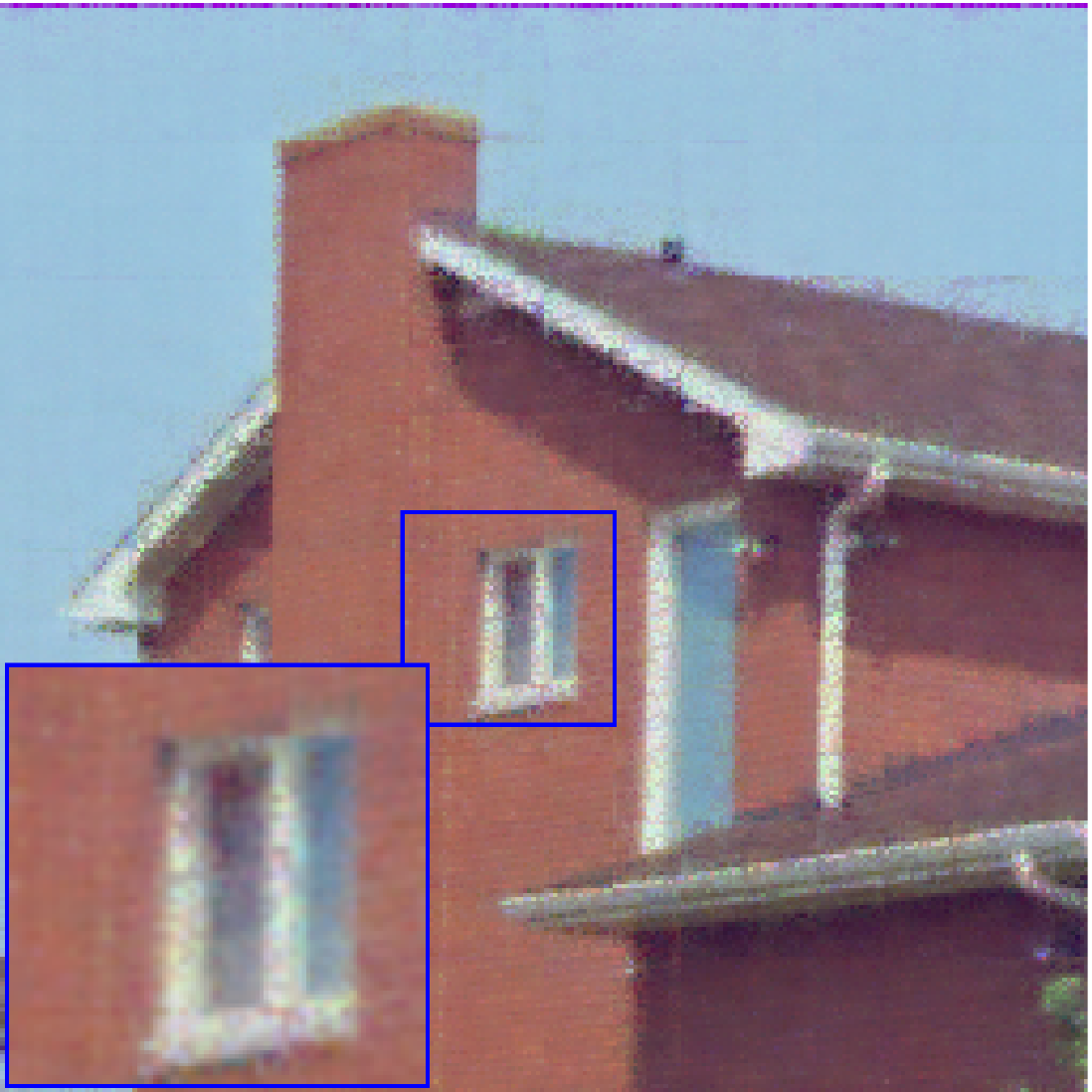}&
\includegraphics[width=0.19\textwidth]{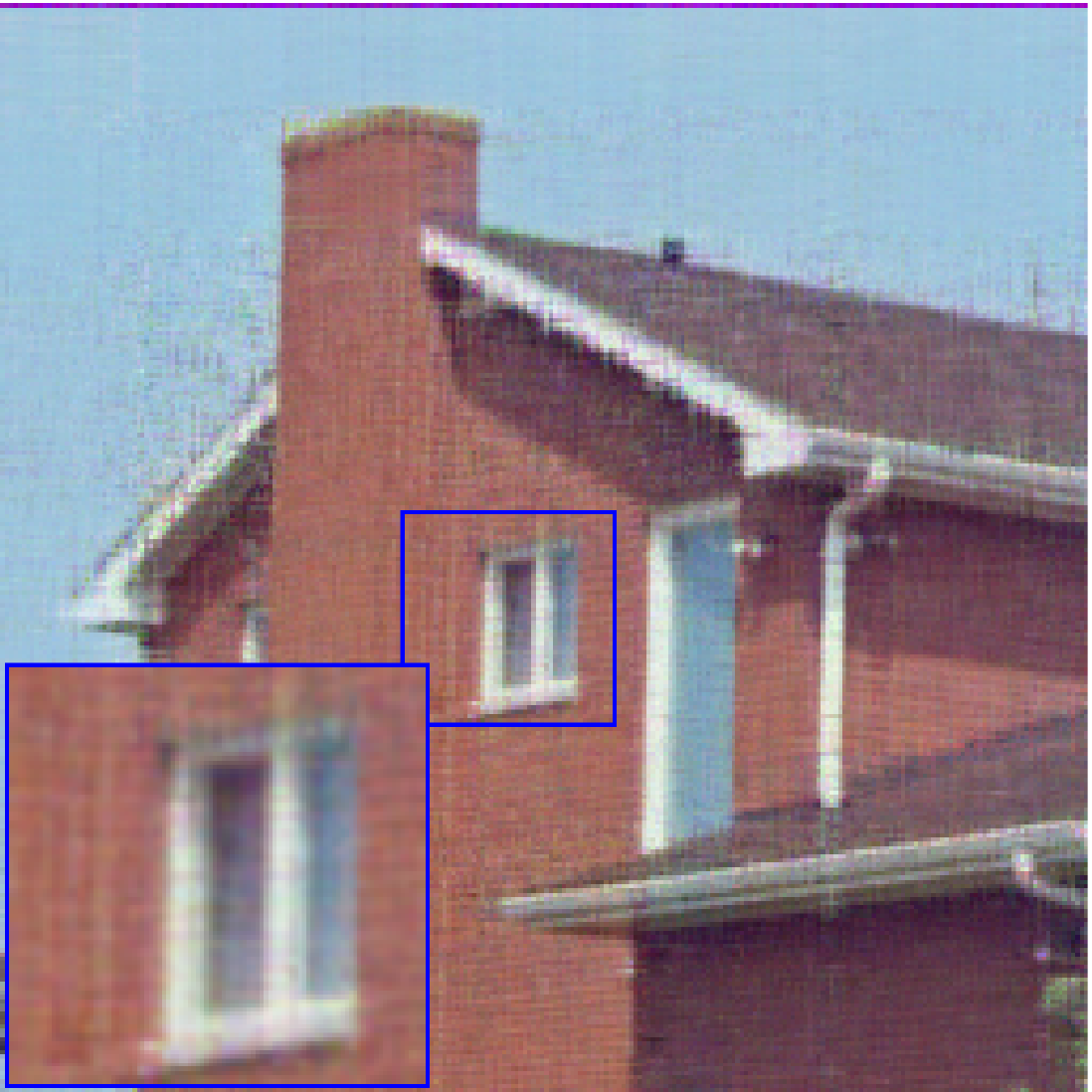}&
\includegraphics[width=0.19\textwidth]{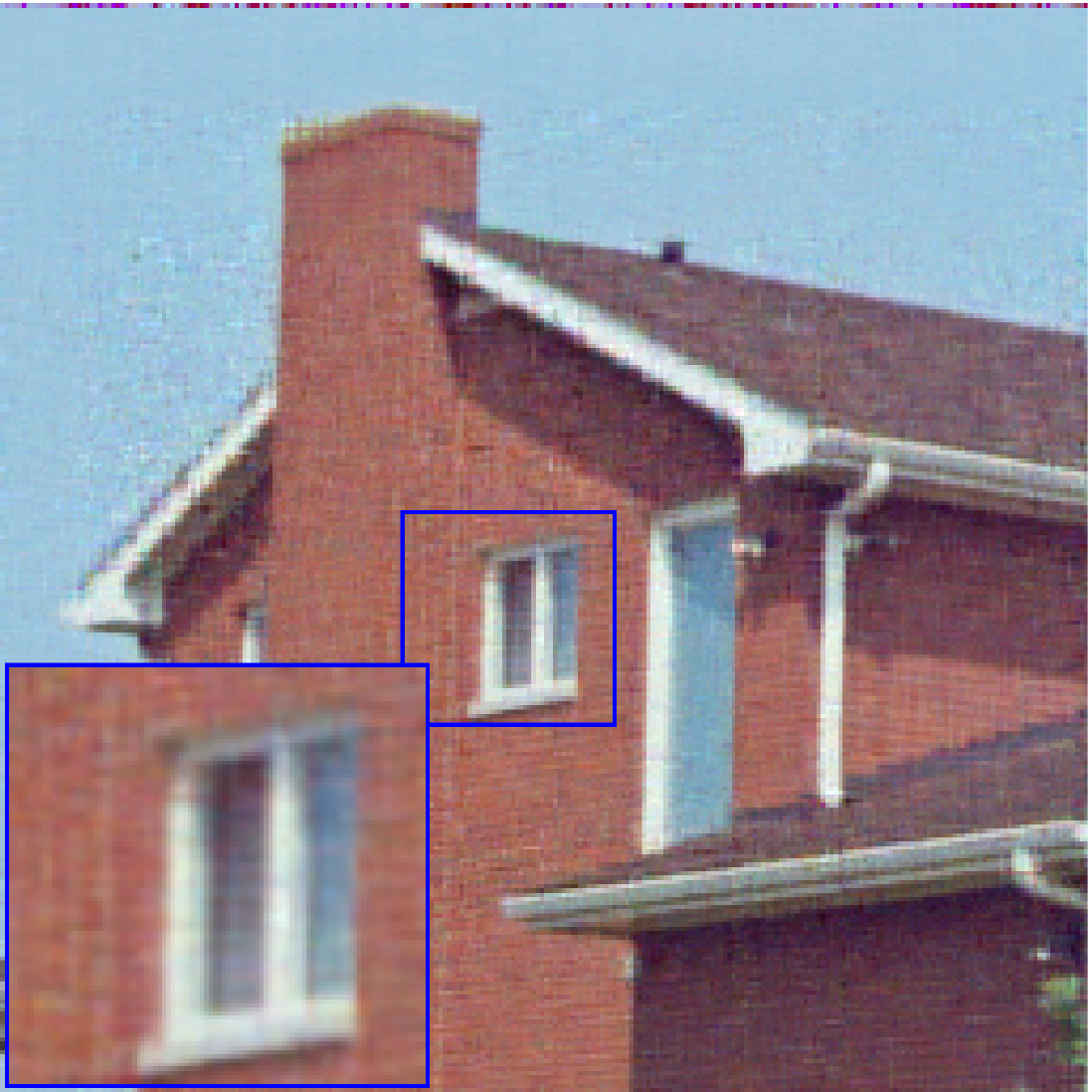}&
\includegraphics[width=0.19\textwidth]{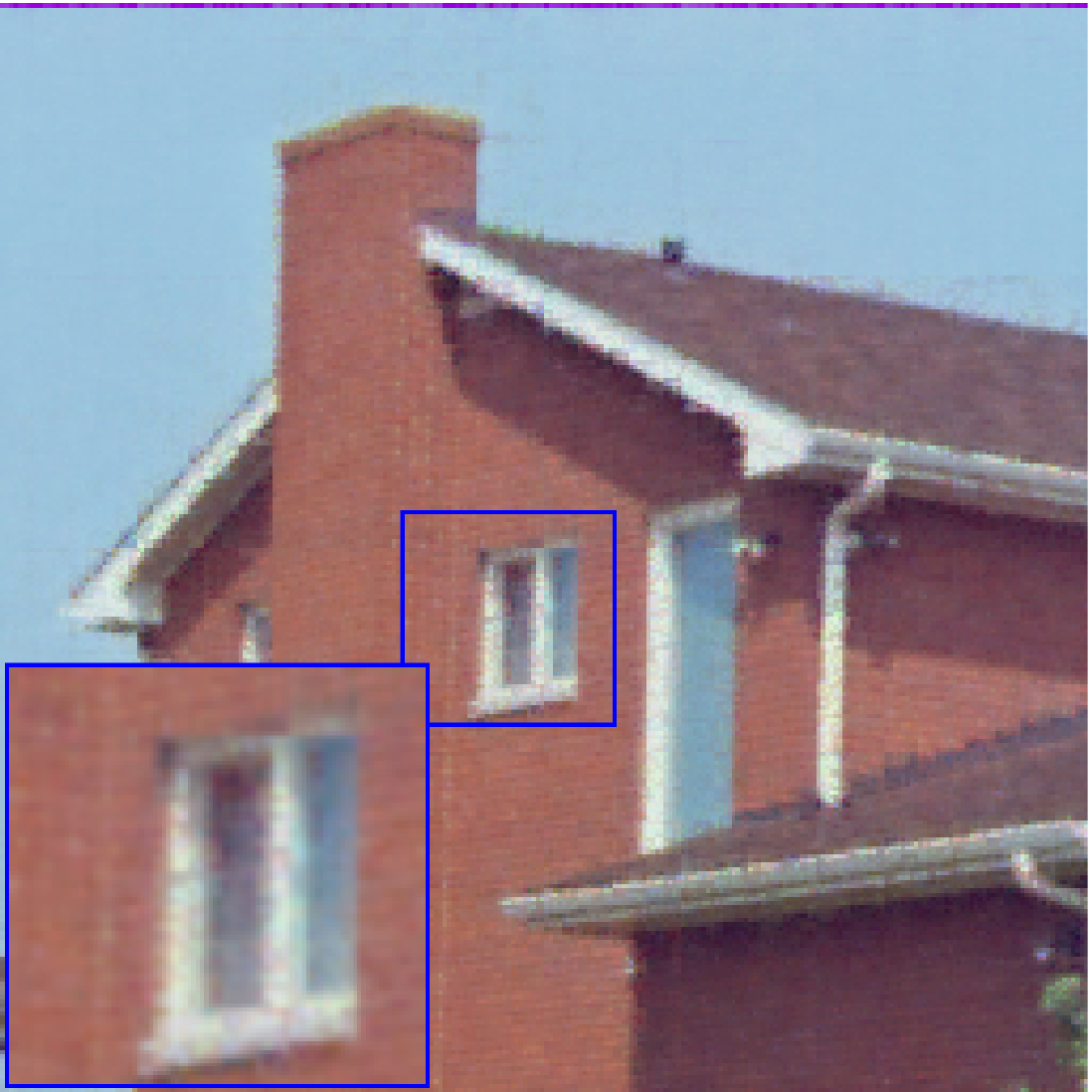}&
\includegraphics[width=0.19\textwidth]{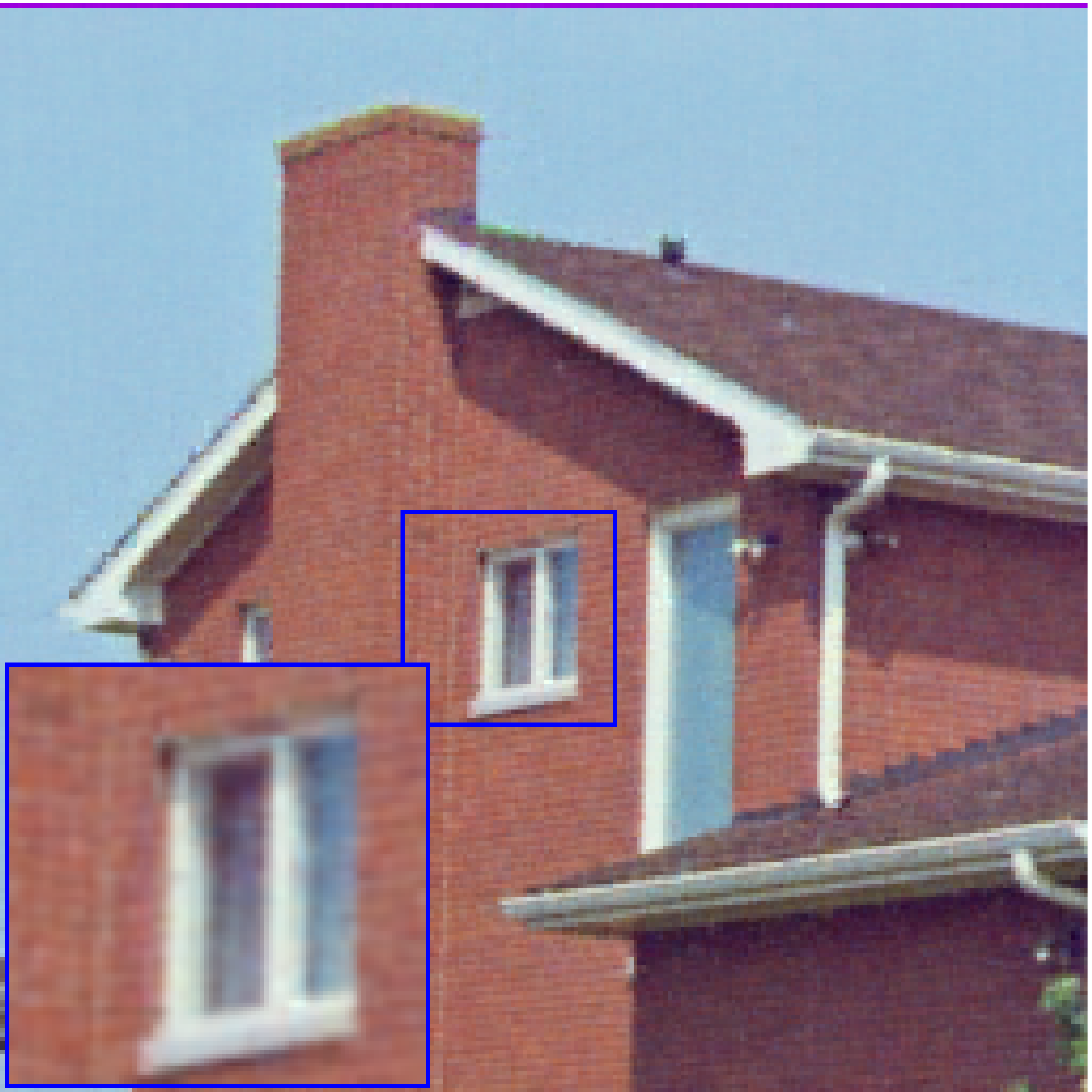}\vspace{0.01cm}\\
(f) SiLRTC-TT & (g) tSVD & (h) KBR & (i) TRNN & (j) LogTR\\
\includegraphics[width=0.19\textwidth]{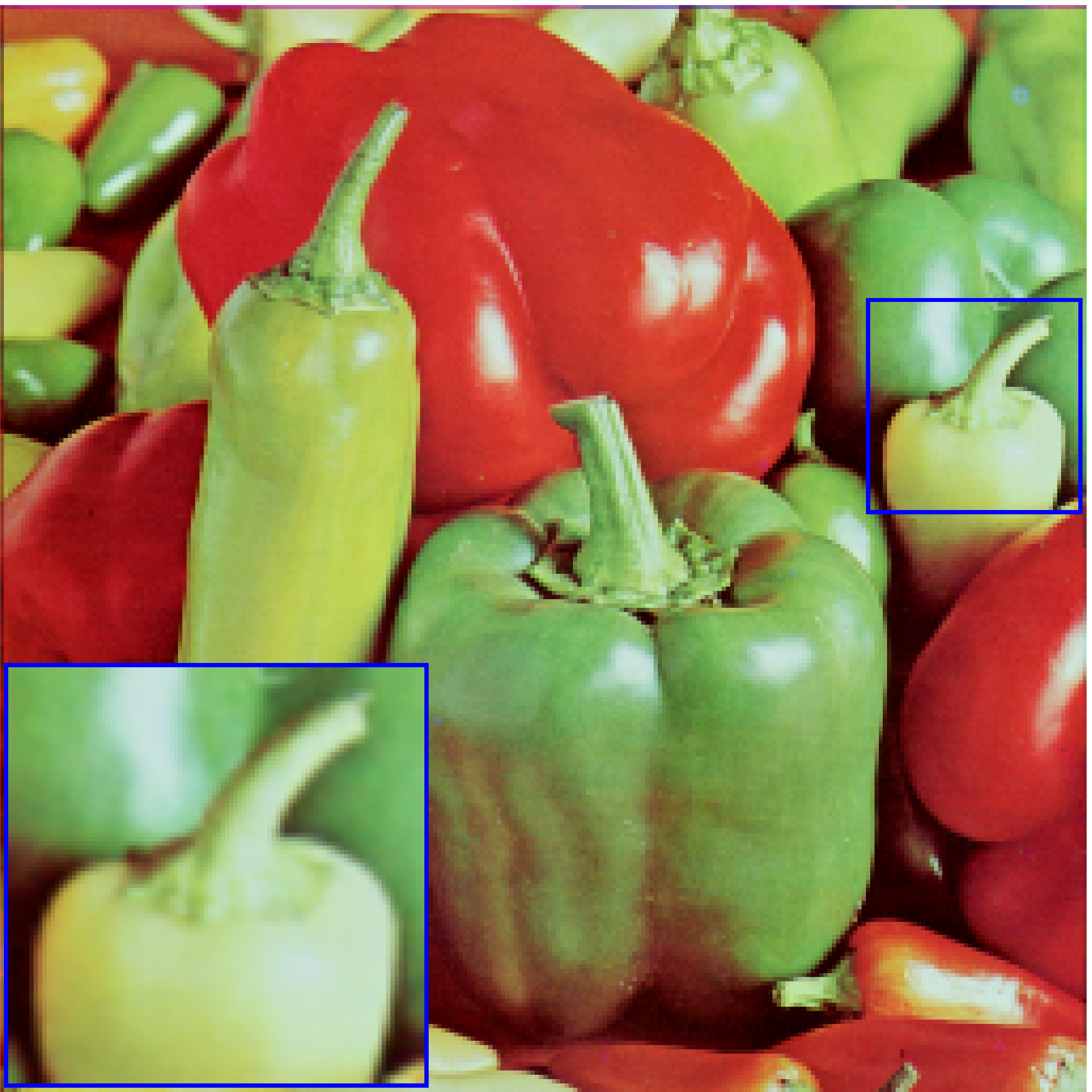}&
\includegraphics[width=0.19\textwidth]{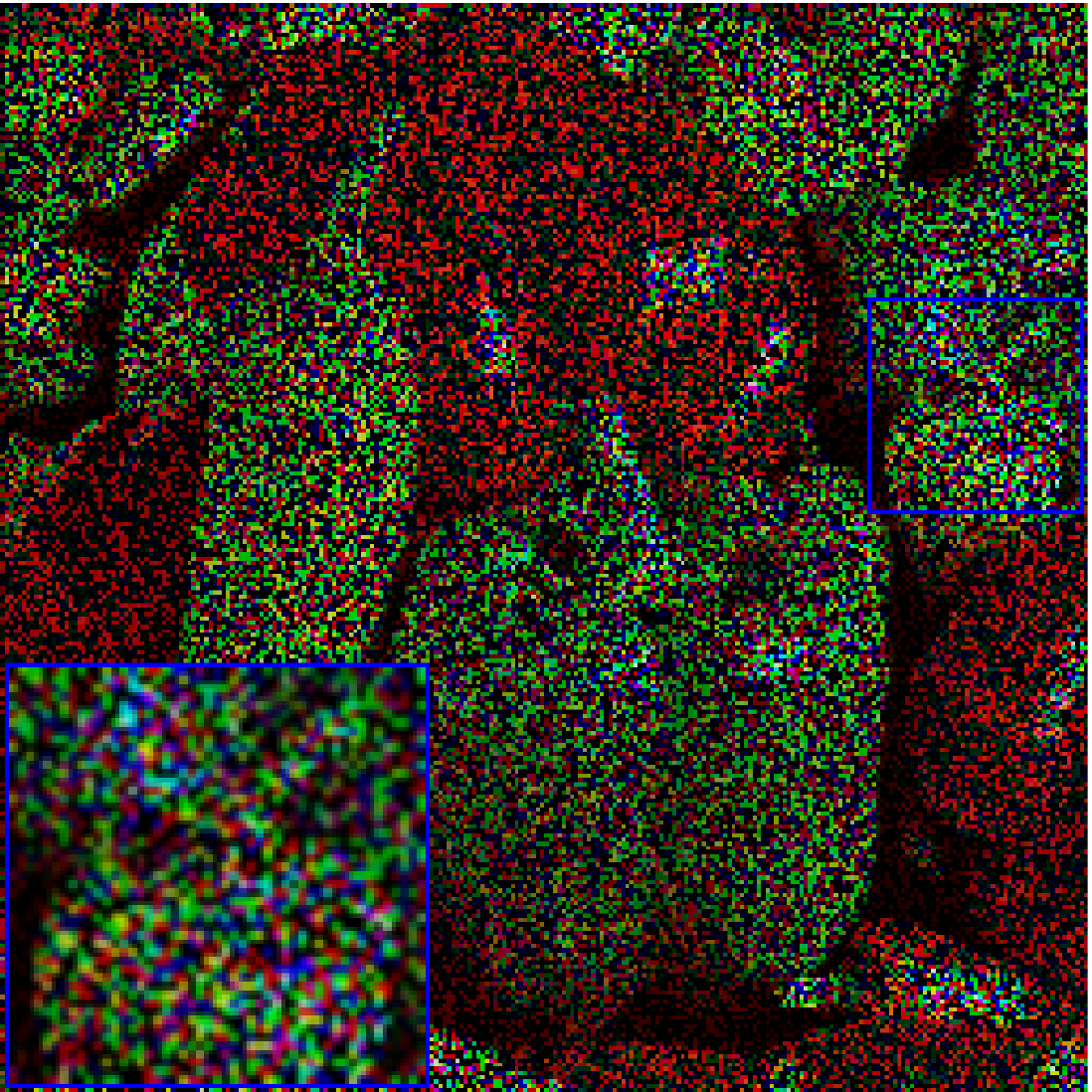}&
\includegraphics[width=0.19\textwidth]{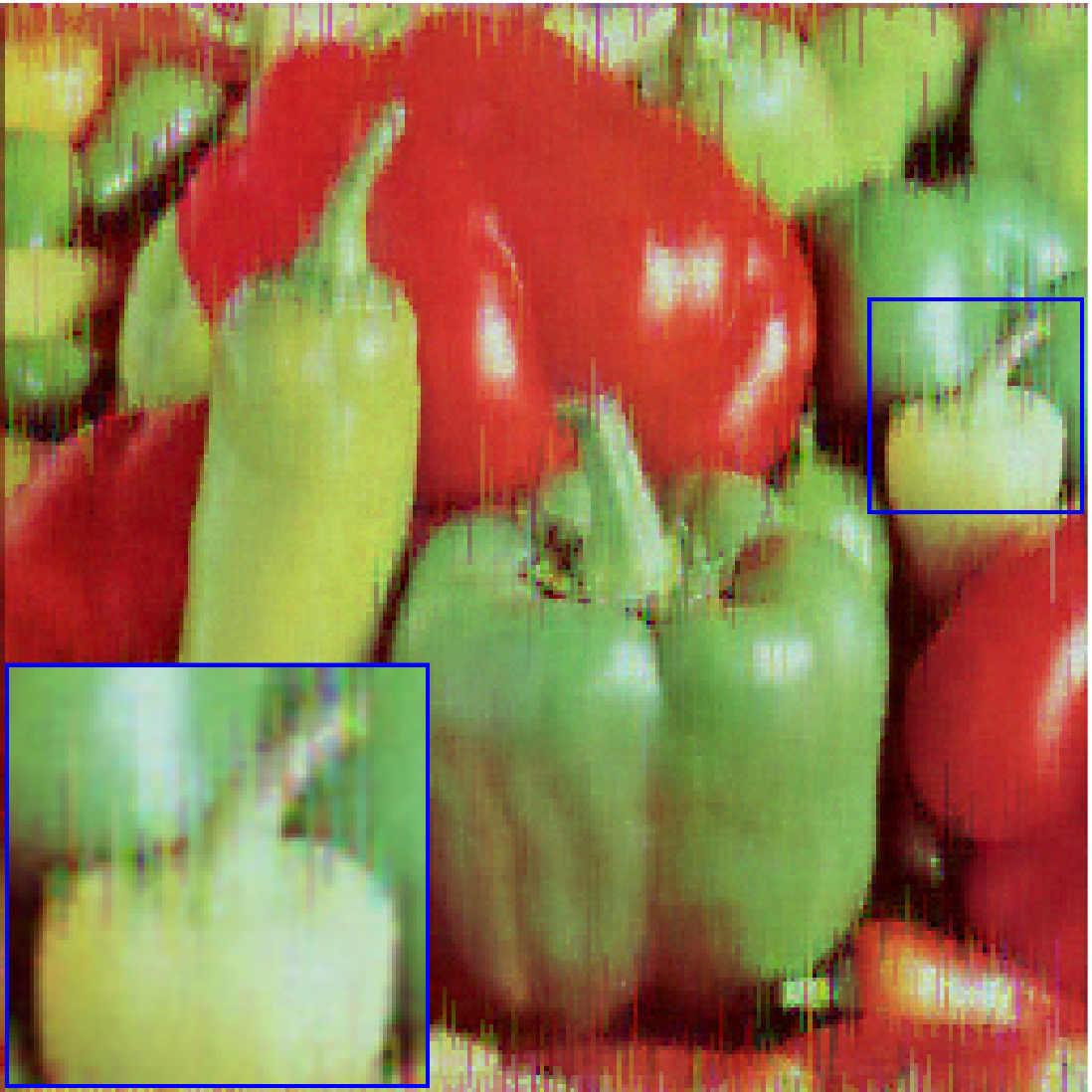}&
\includegraphics[width=0.19\textwidth]{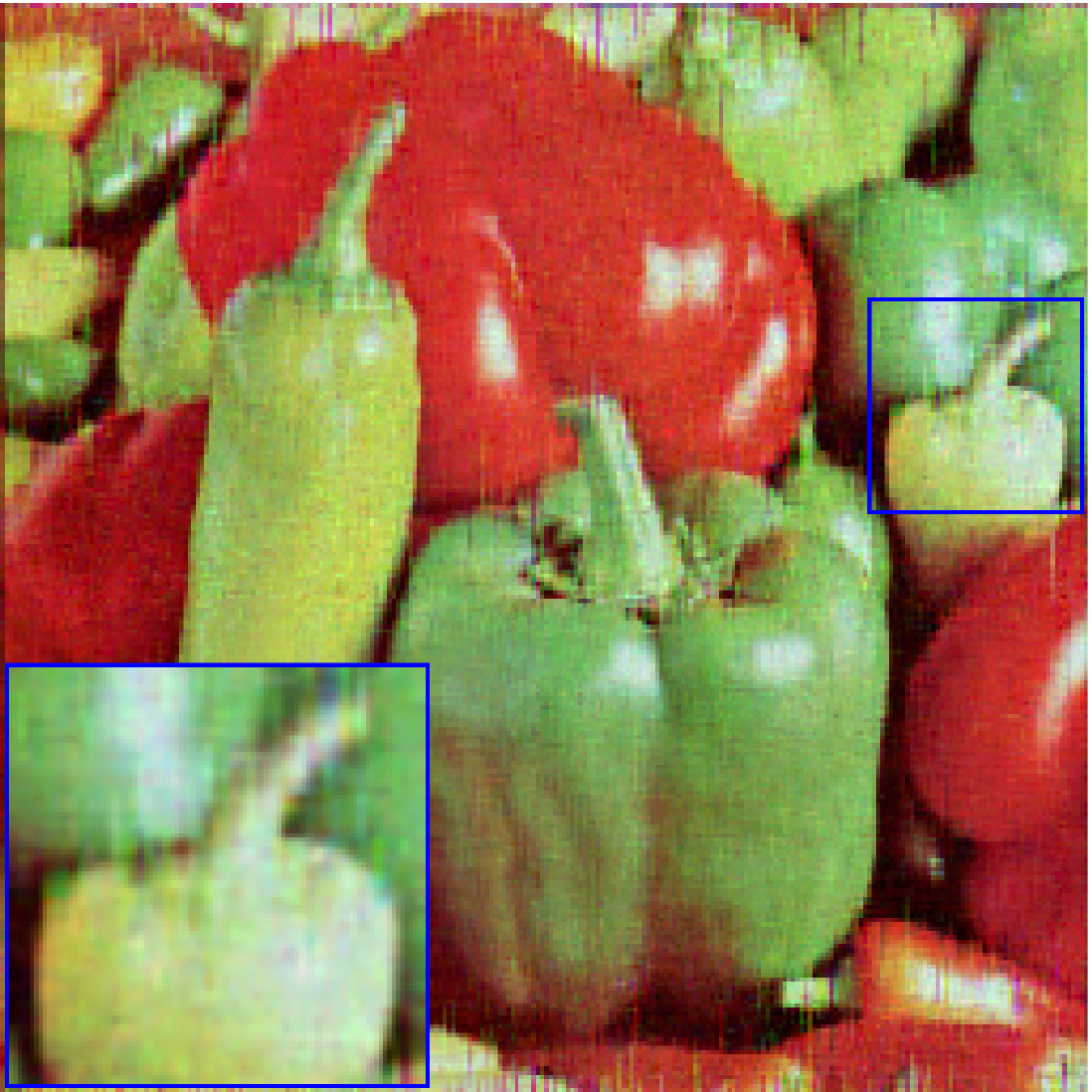}&
\includegraphics[width=0.19\textwidth]{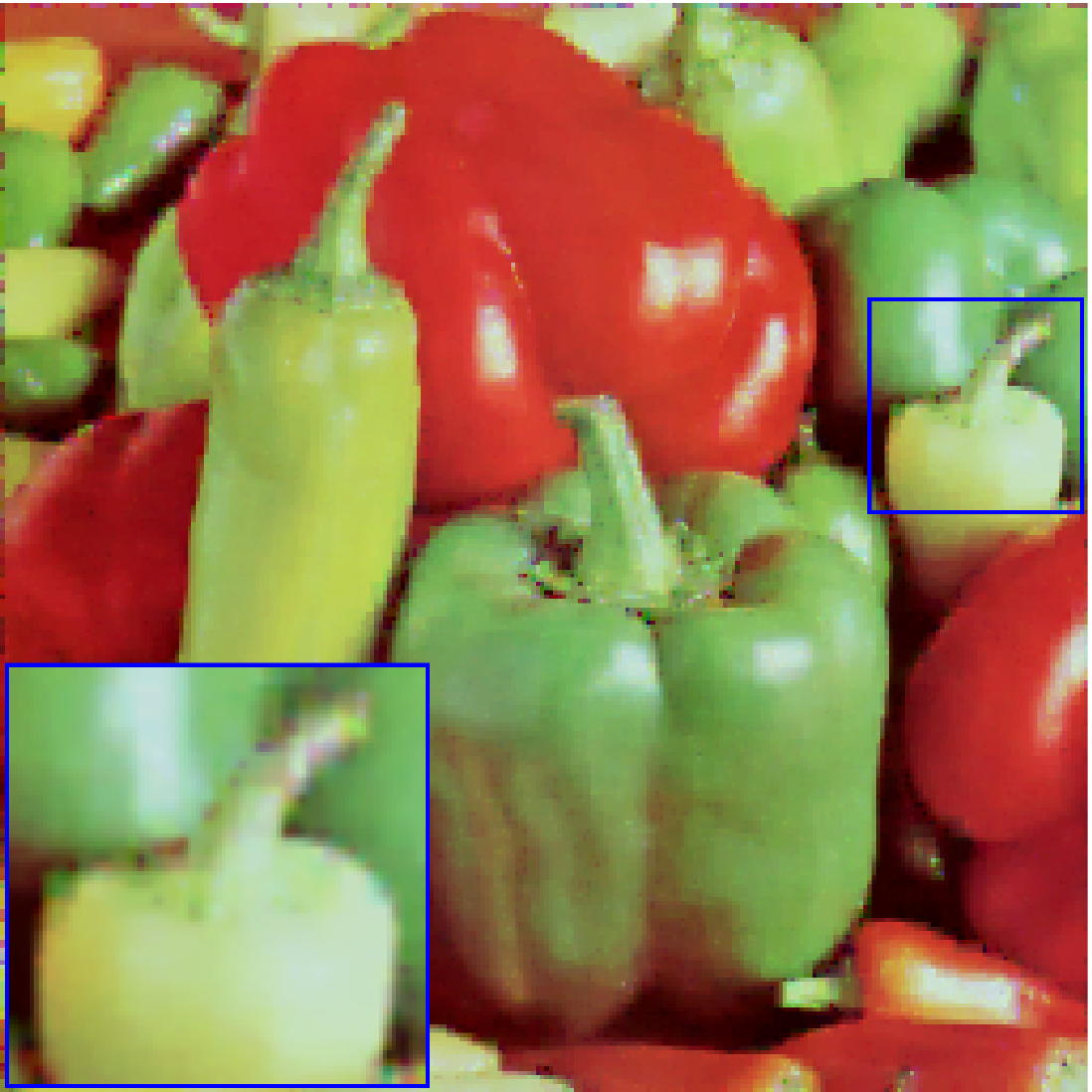}\vspace{0.01cm}\\
(a) Original & (b) Observed & (c) HaLRTC & (d) NSNN & (e) LRTC-TV\\
\includegraphics[width=0.19\textwidth]{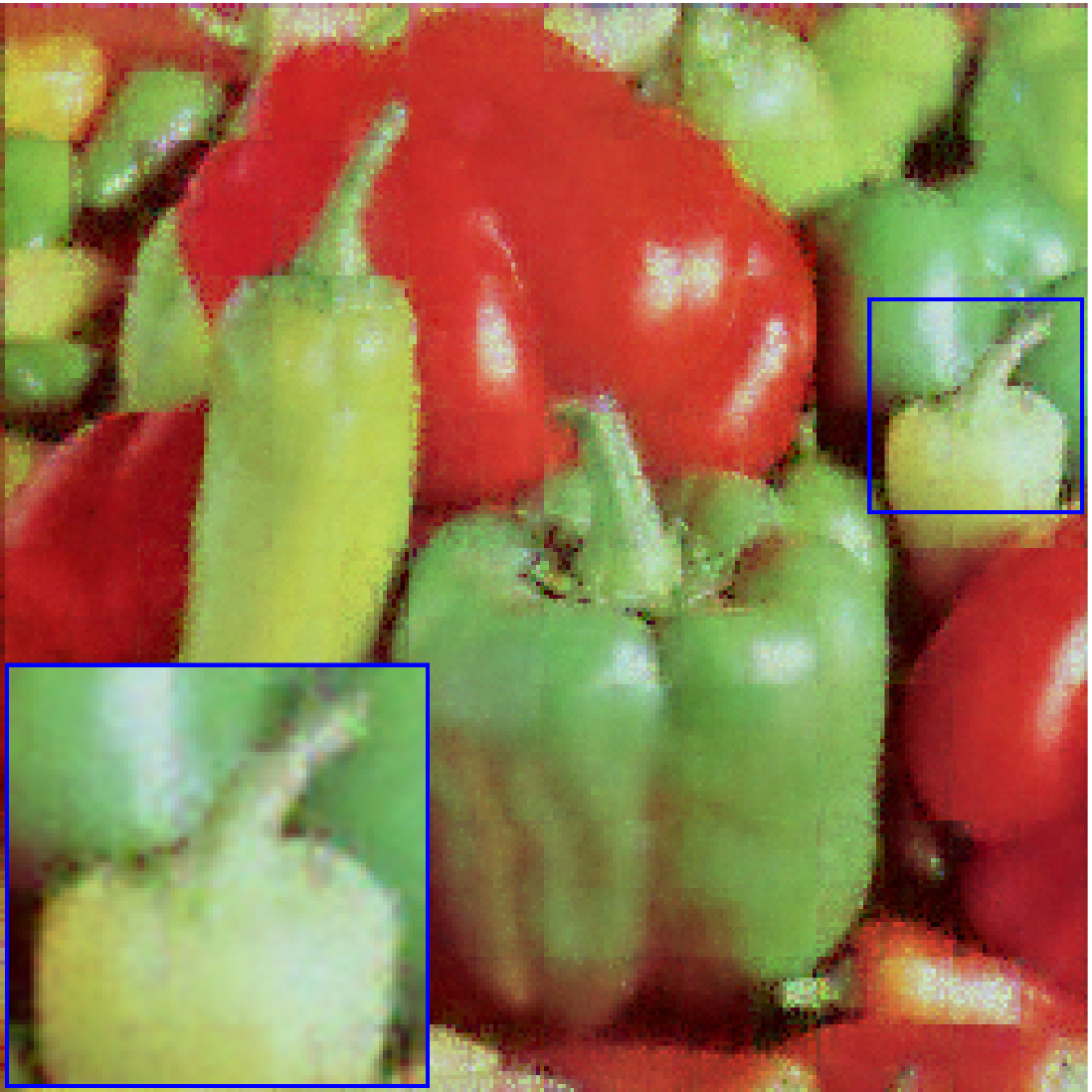}&
\includegraphics[width=0.19\textwidth]{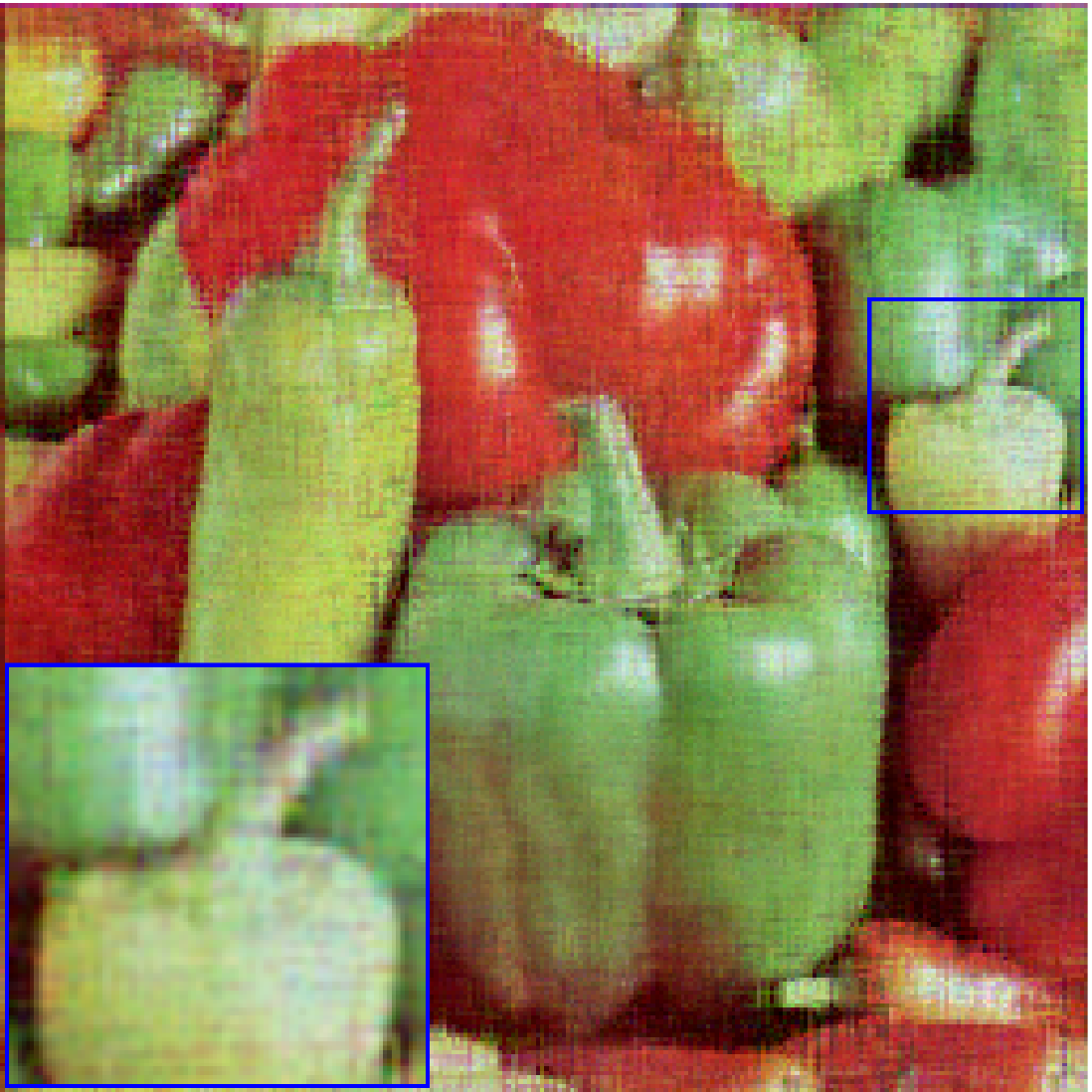}&
\includegraphics[width=0.19\textwidth]{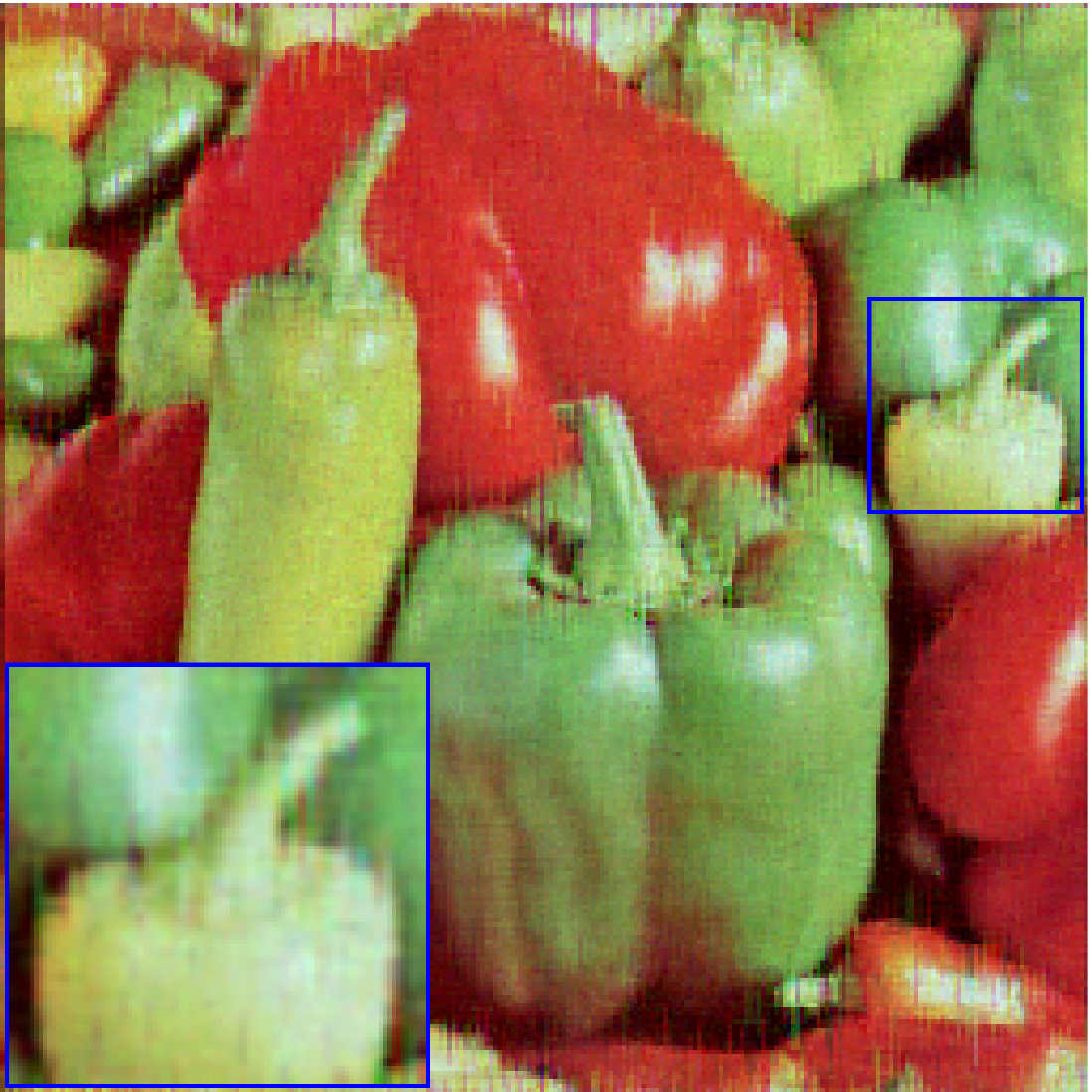}&
\includegraphics[width=0.19\textwidth]{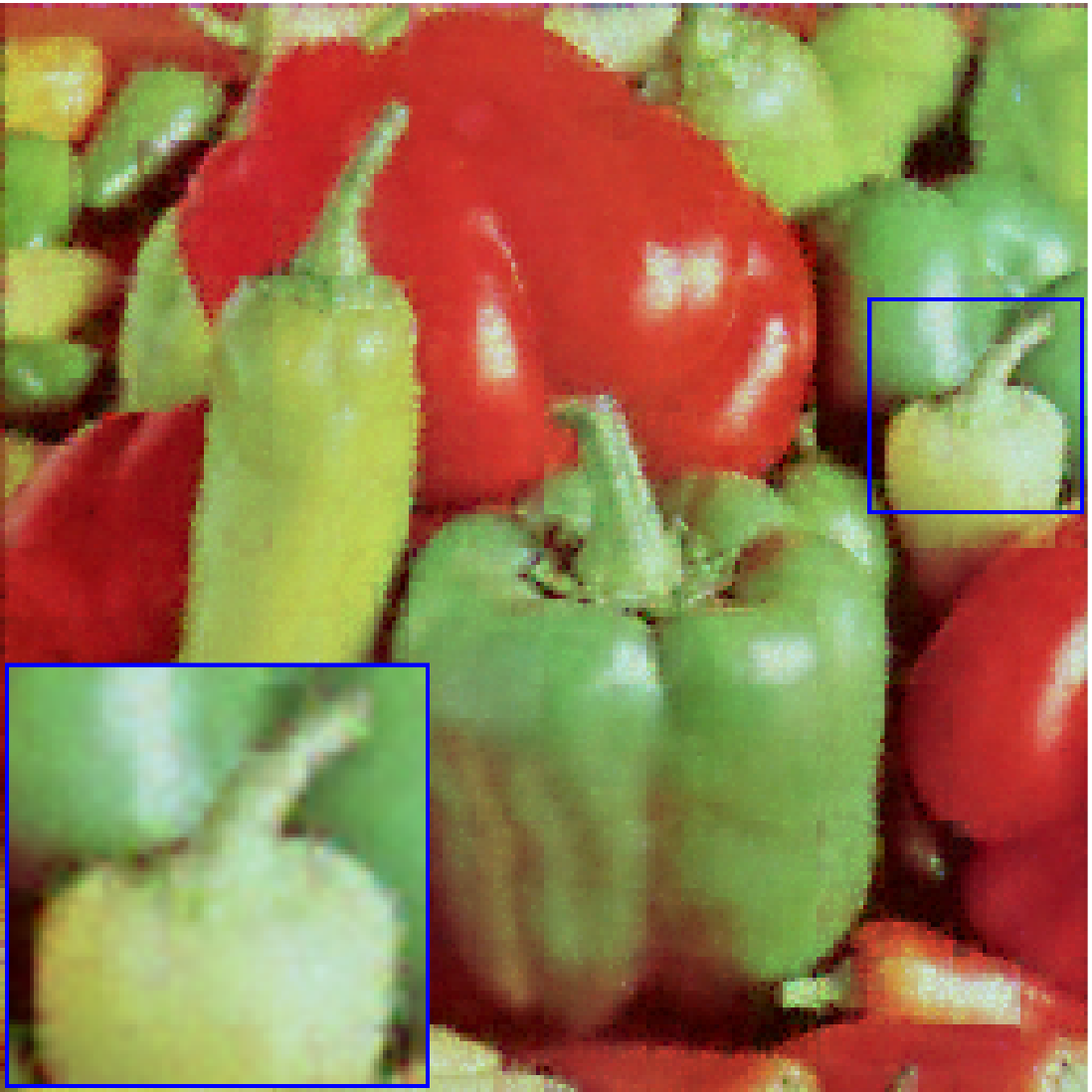}&
\includegraphics[width=0.19\textwidth]{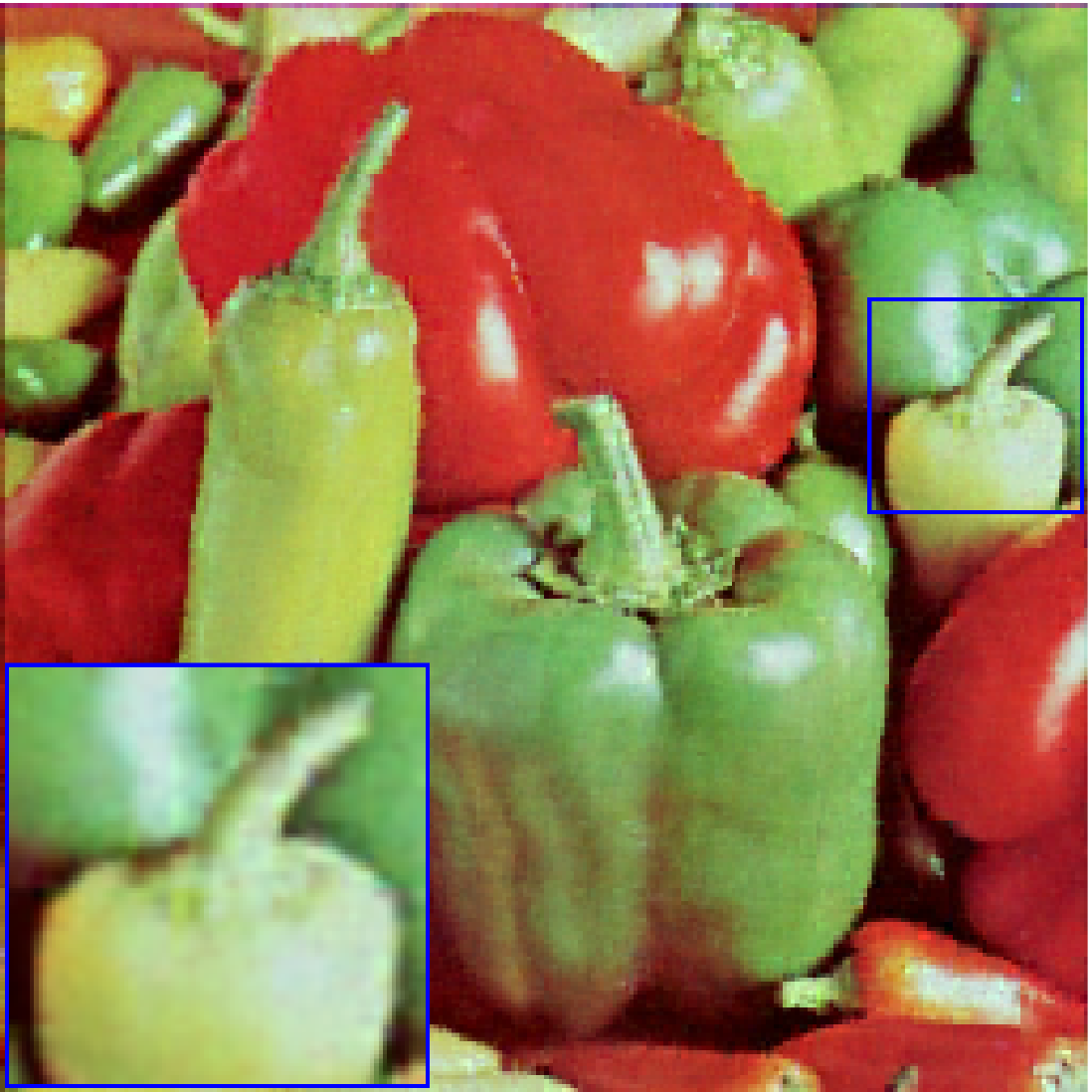}\vspace{0.01cm}\\
(f) SiLRTC-TT & (g) tSVD & (h) KBR & (i) TRNN & (j) LogTR\\
\includegraphics[width=0.19\textwidth]{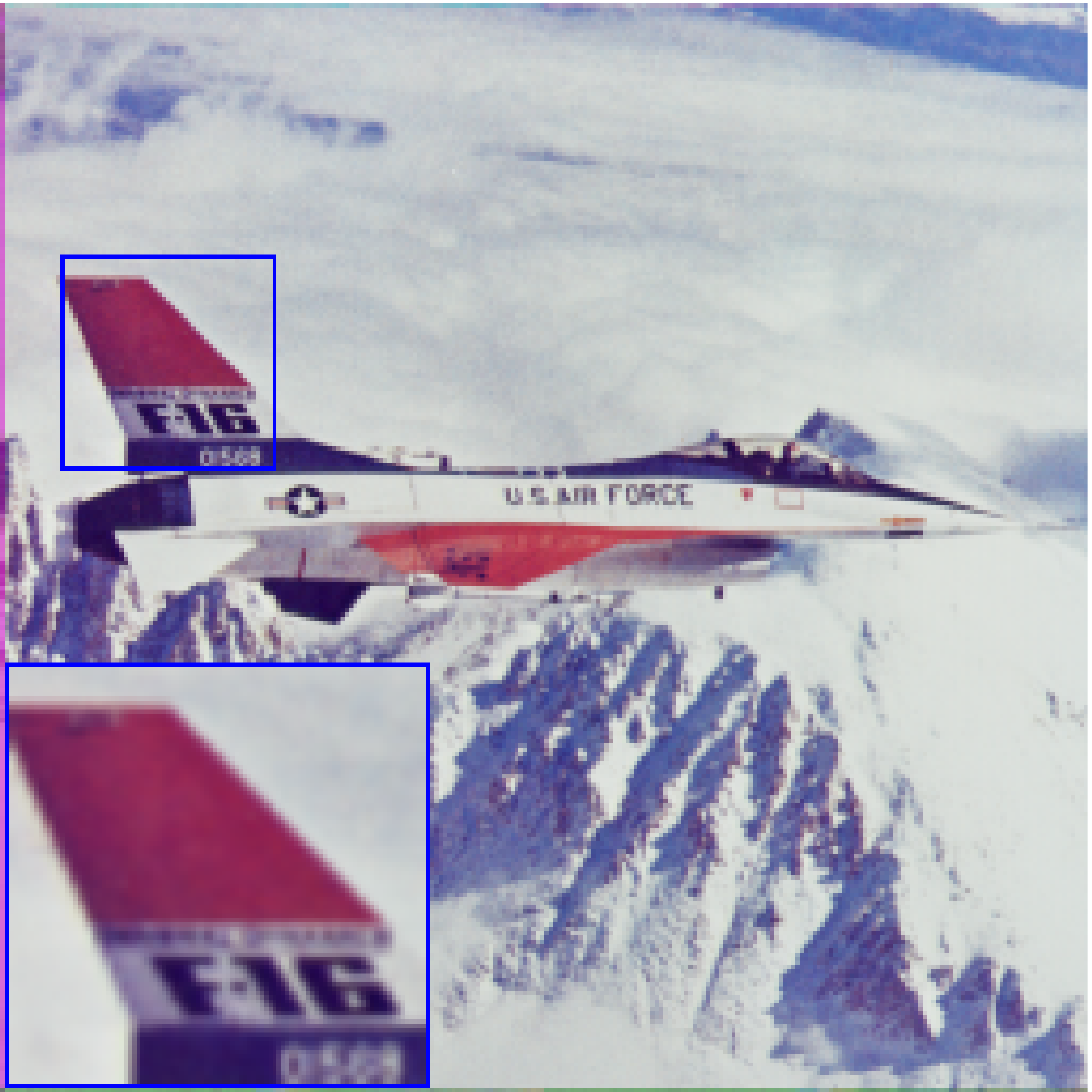}&
\includegraphics[width=0.19\textwidth]{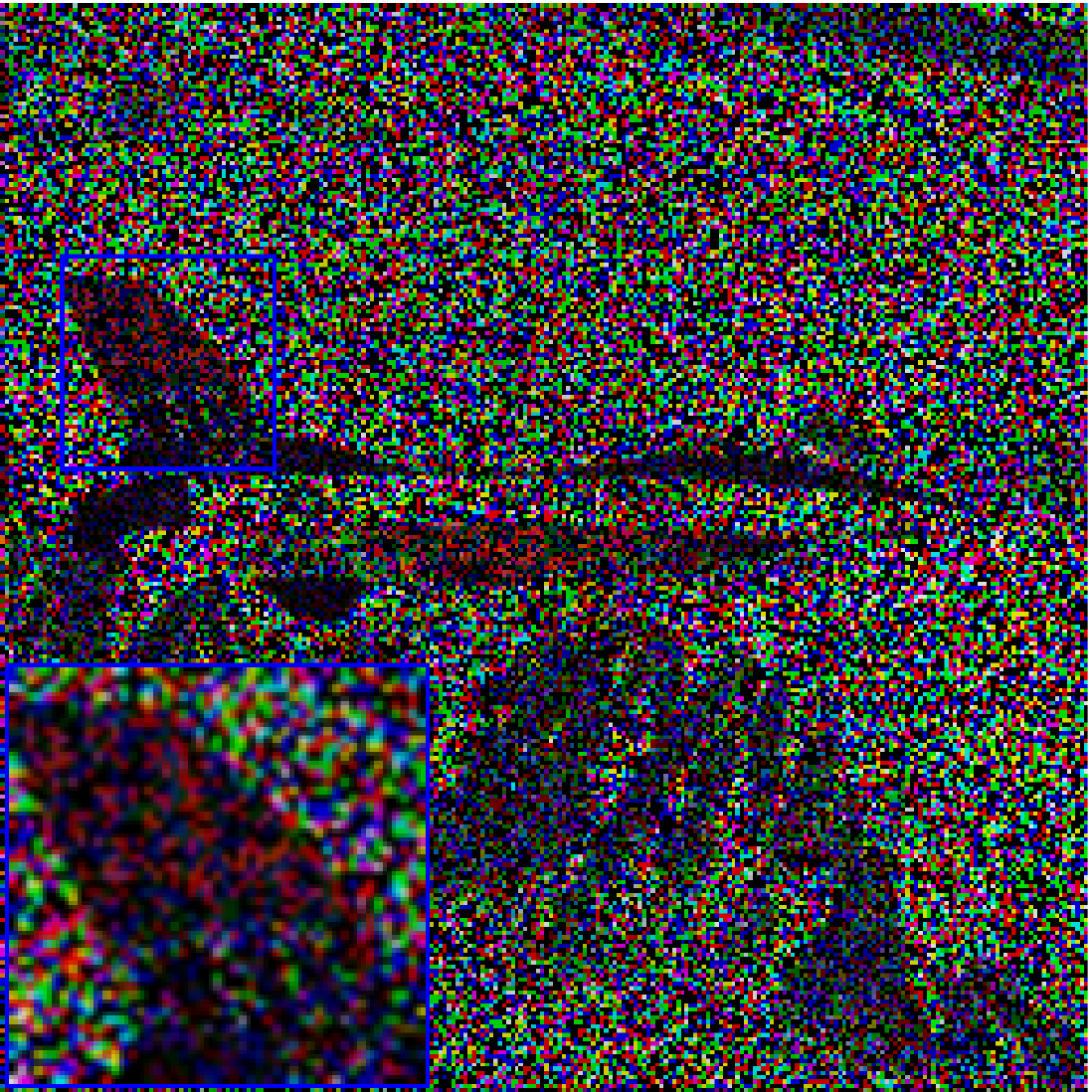}&
\includegraphics[width=0.19\textwidth]{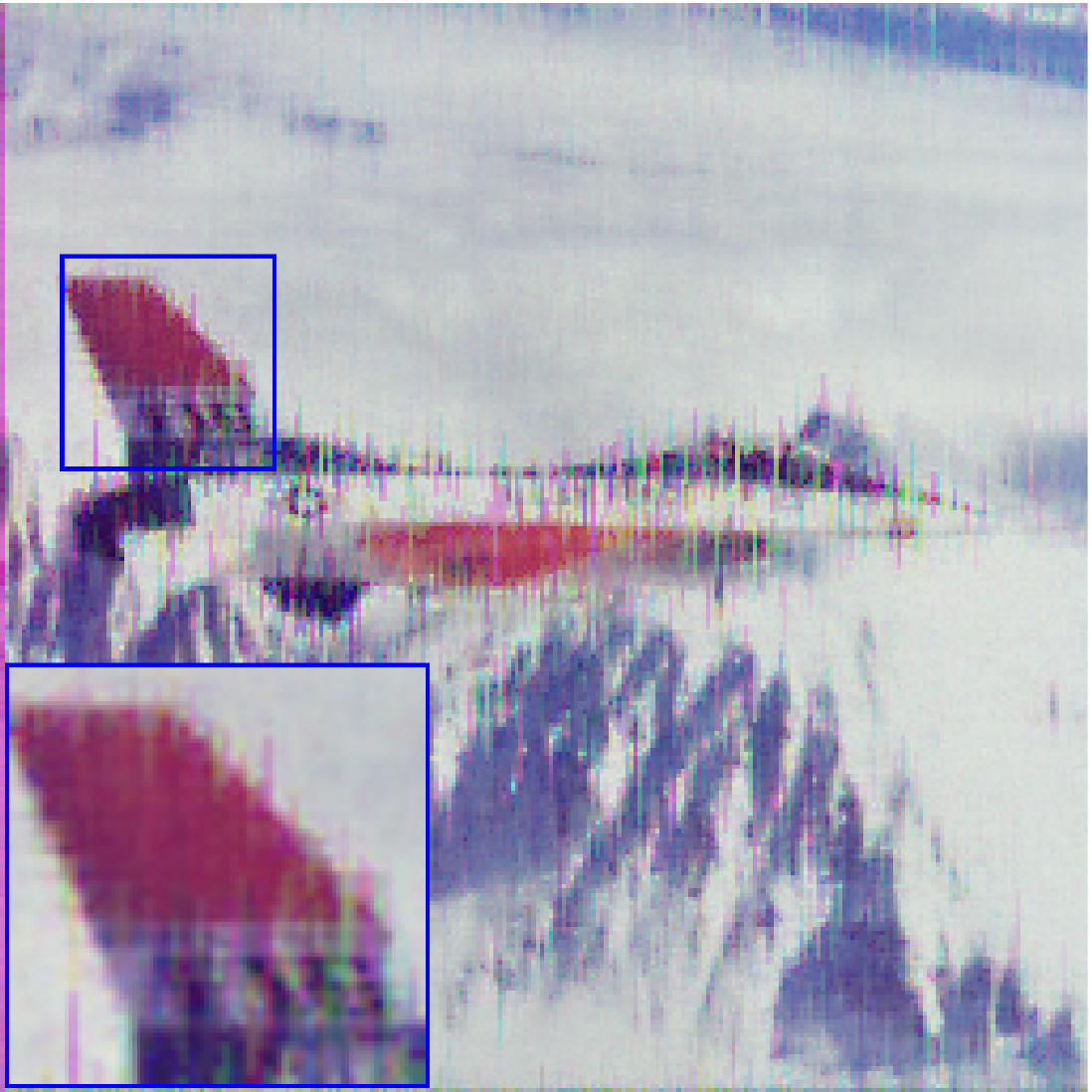}&
\includegraphics[width=0.19\textwidth]{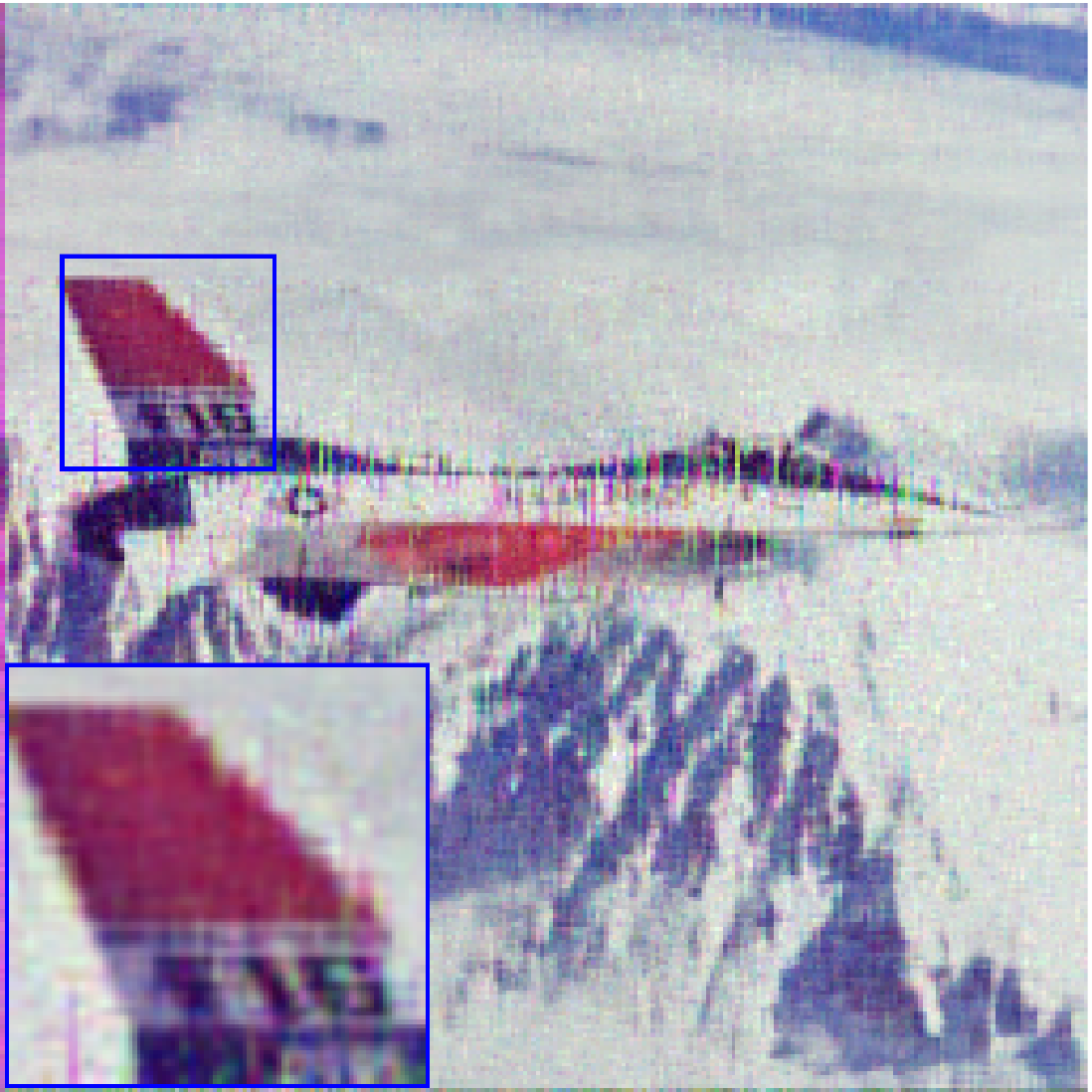}&
\includegraphics[width=0.19\textwidth]{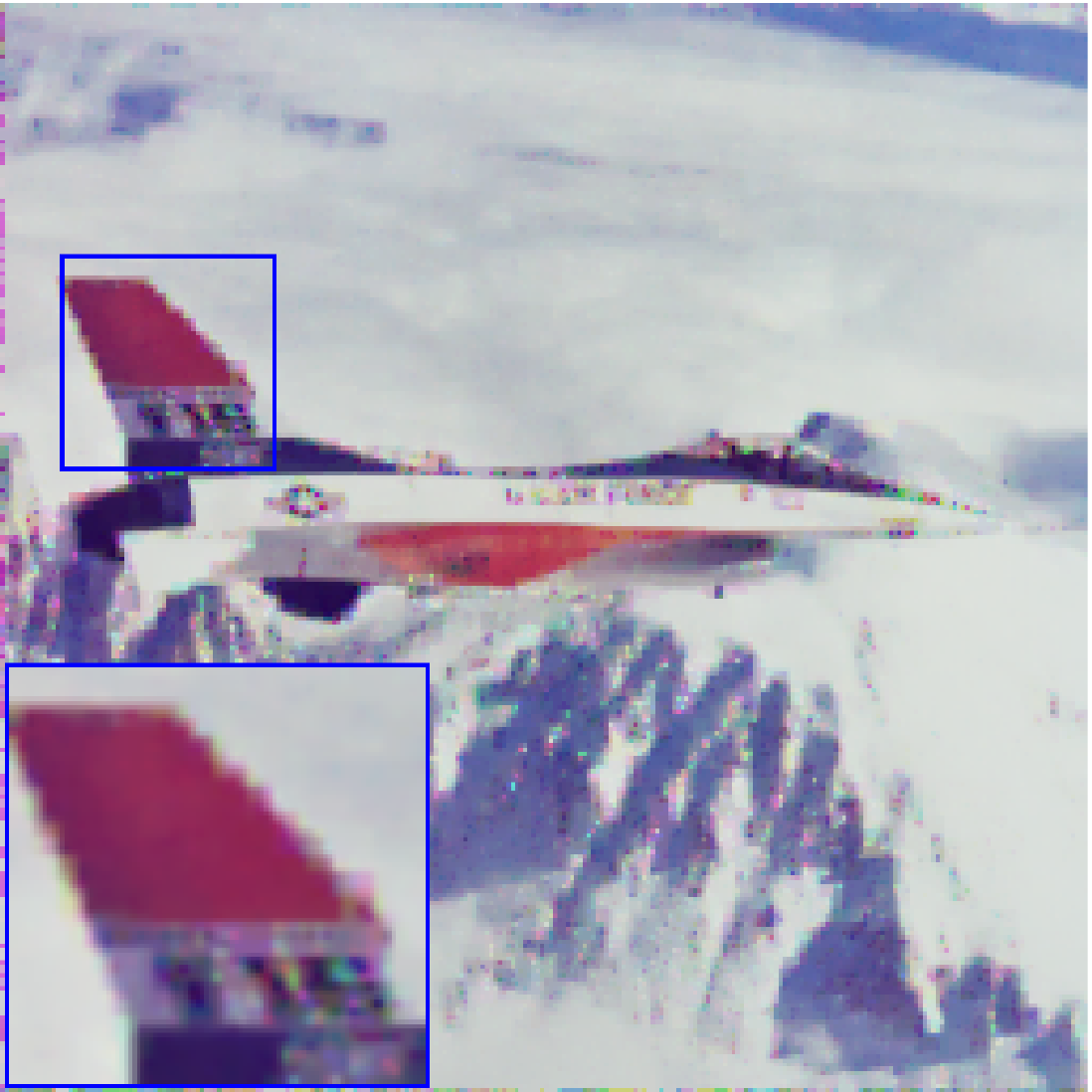}\vspace{0.01cm}\\
(a) Original& (b) Observed & (c) HaLRTC & (d) NSNN & (e) LRTC-TV\\
\includegraphics[width=0.19\textwidth]{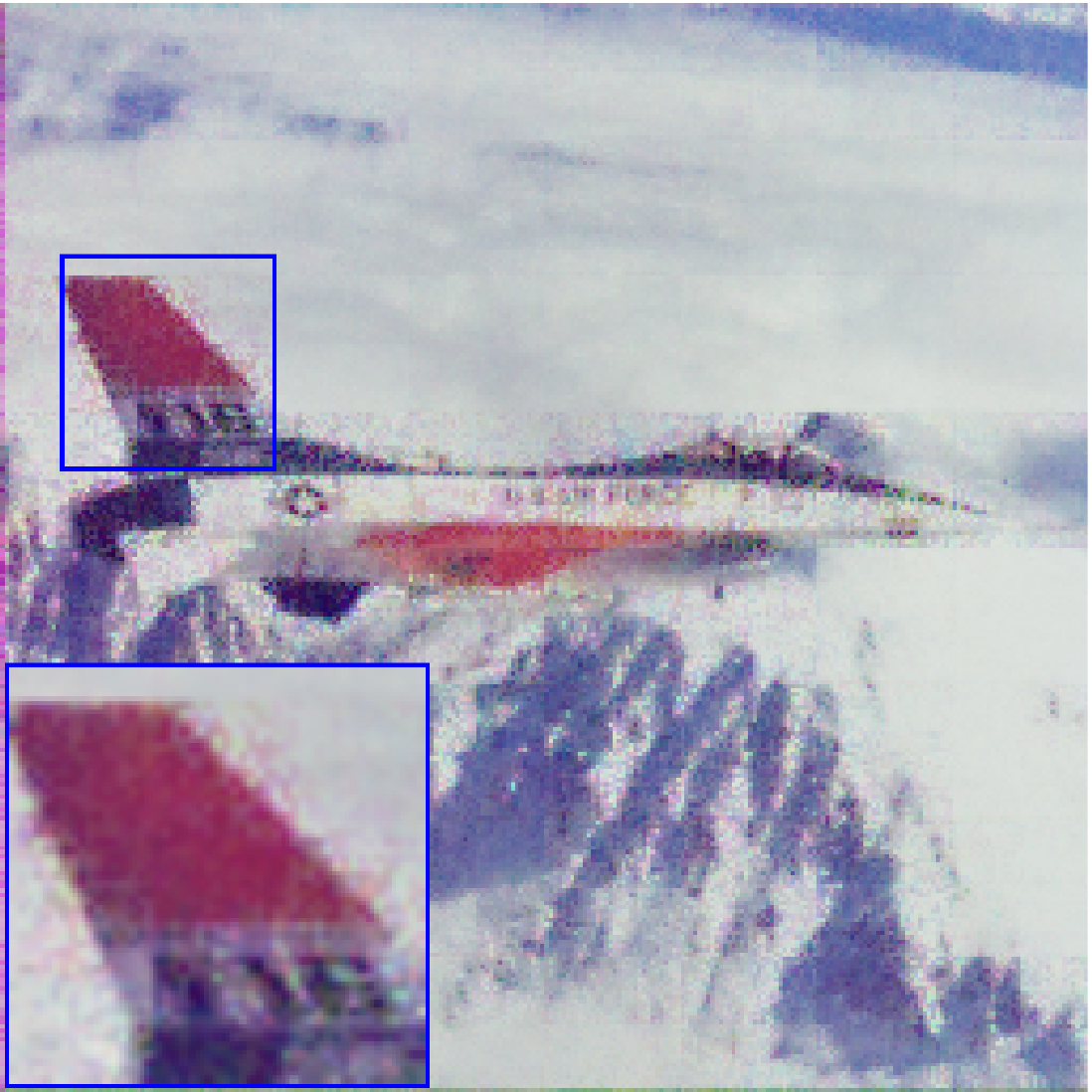}&
\includegraphics[width=0.19\textwidth]{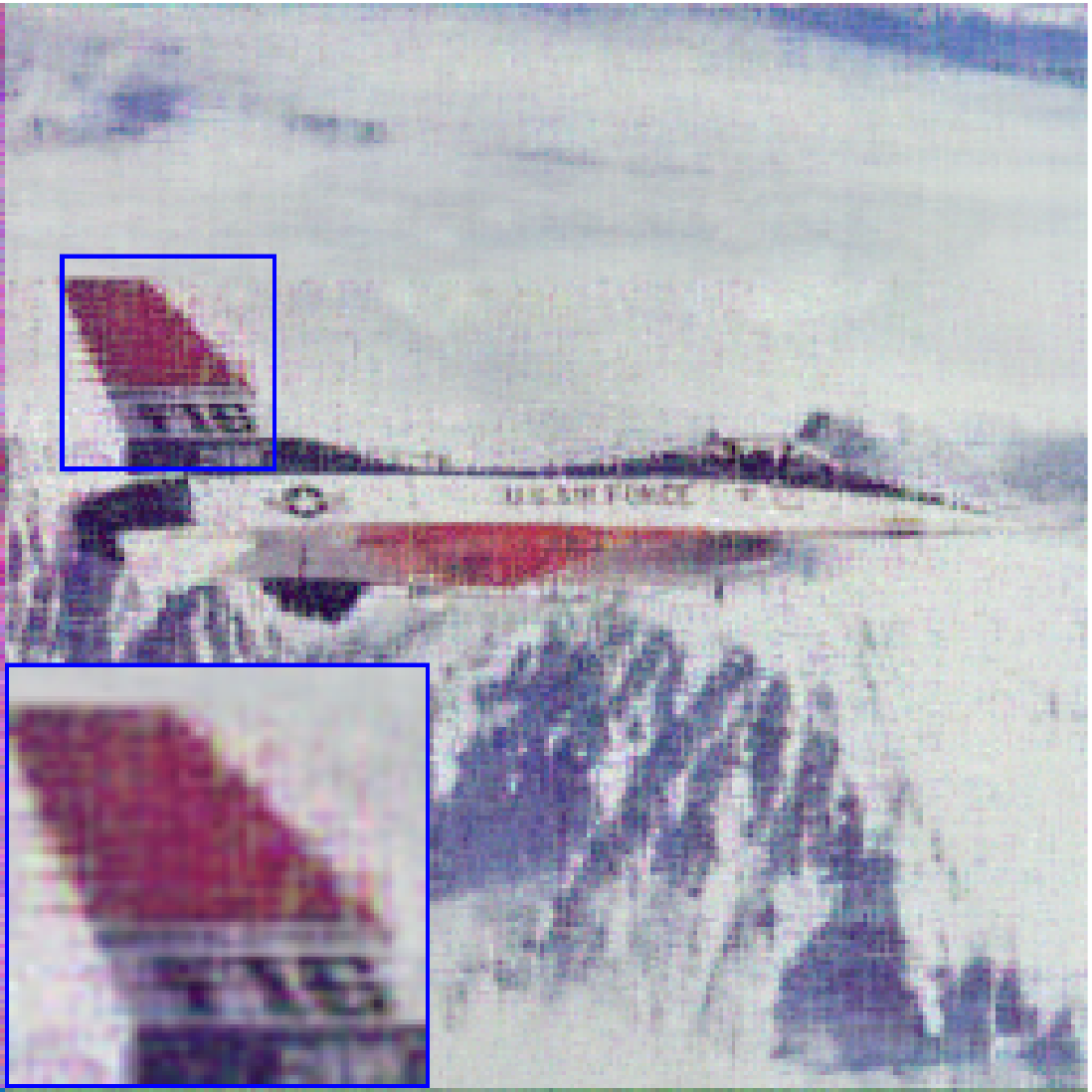}&
\includegraphics[width=0.19\textwidth]{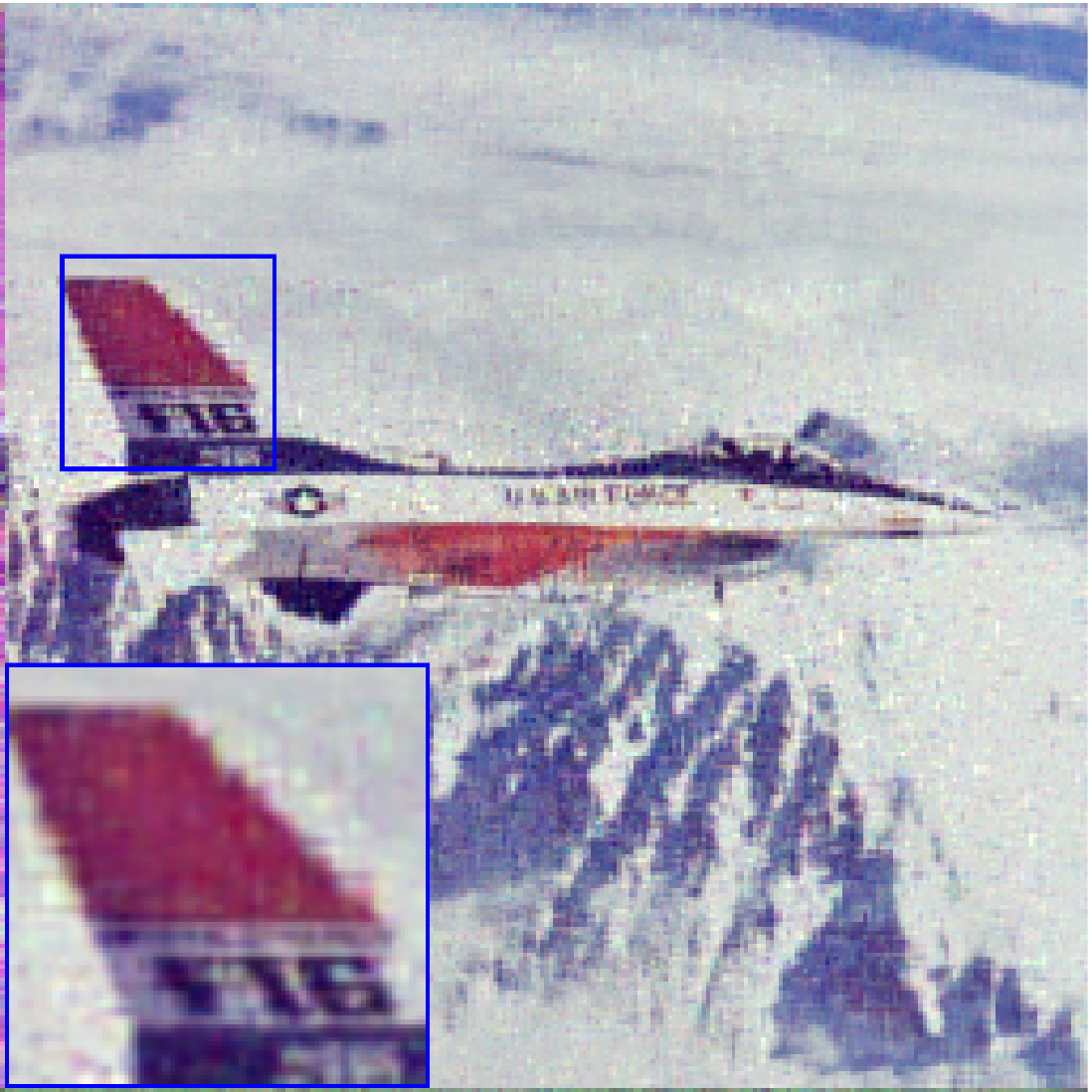}&
\includegraphics[width=0.19\textwidth]{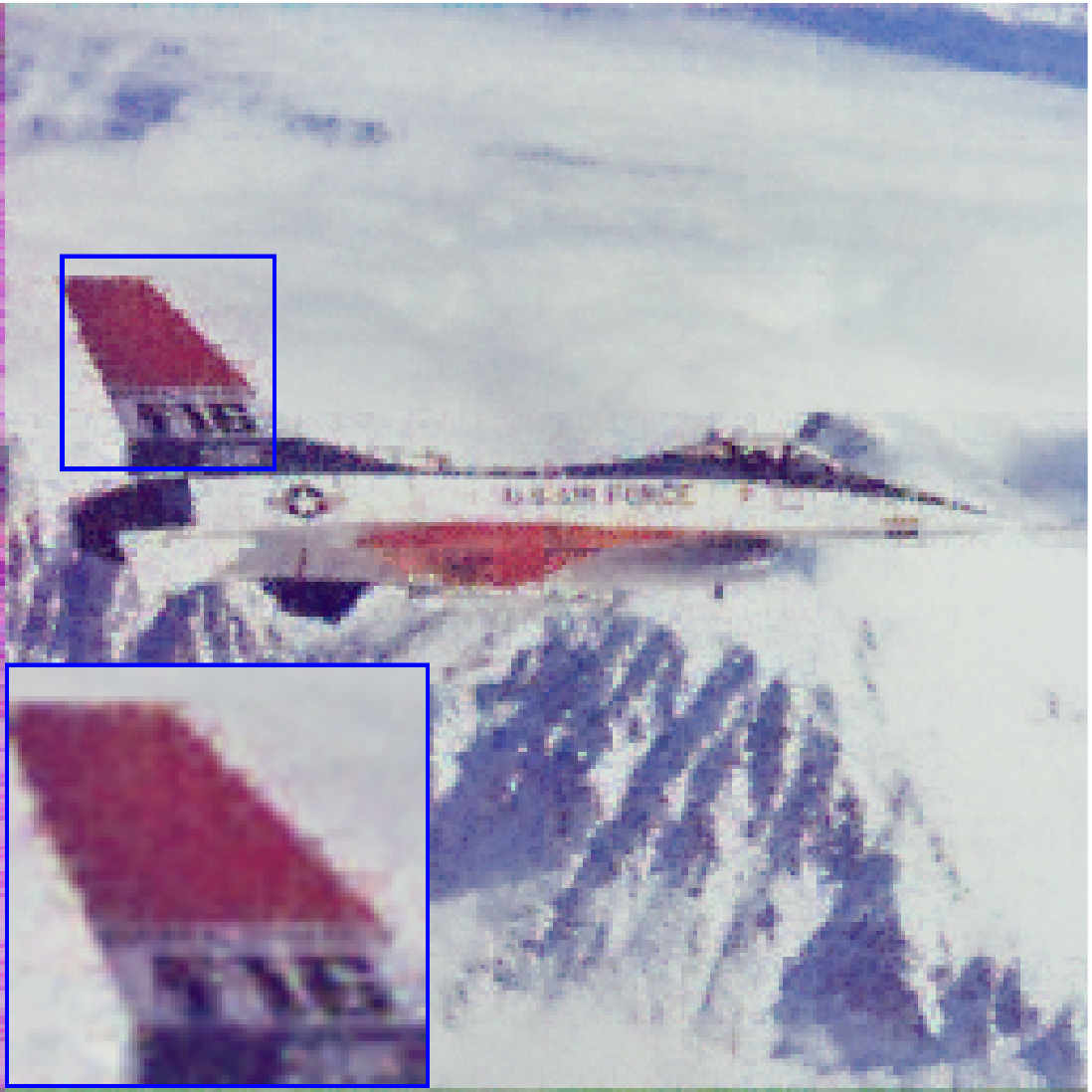}&
\includegraphics[width=0.19\textwidth]{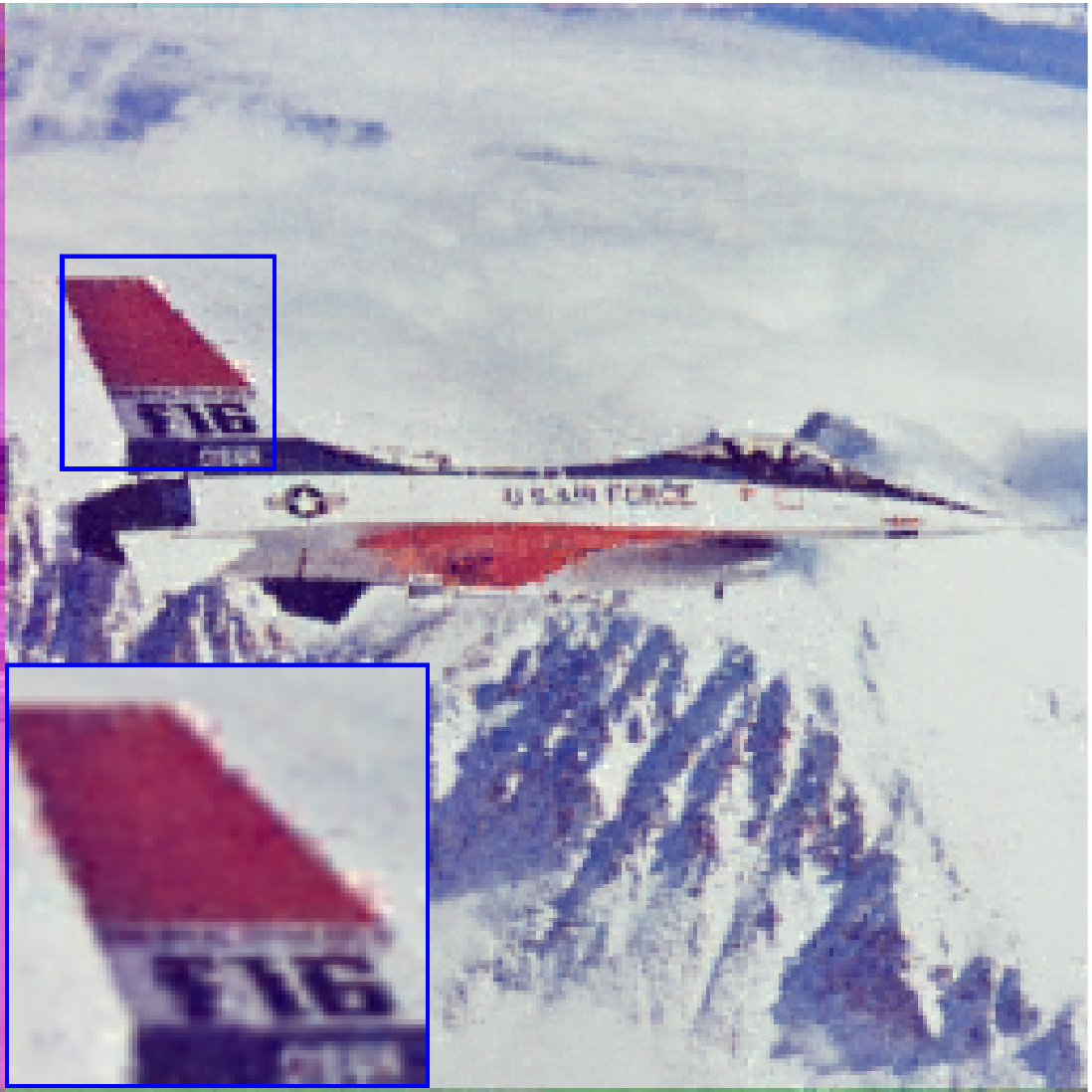}\vspace{0.01cm}\\
(f) SiLRTC-TT & (g) tSVD & (h) KBR & (i) TRNN & (j) LogTR\\
\end{tabular}
\caption{\small{Recovered color images \emph{House}, \emph{Peppers}, and \emph{Airplane} for random missing entries with $SR = 0.3$.}}
  \label{fig:image_random}
  \end{center}\vspace{-0.3cm}
\end{figure}

To evaluate the results, we adopt the peak signal-to-noise ratio (PSNR) (dB) and the structural similarity index (SSIM), which are defined as
\[
\textrm{PSNR}=10\log_{10}\Bigg(\frac{Max_{\textbf{F},\widetilde{\textbf{F}}}^{2}}{\frac{1}{m_{1}m_{2}}\|\textbf{F}-\widetilde{\textbf{F}}\|_{F}^{2}}\Bigg),
\]
\[
\textrm{SSIM}=\frac{(2\mu_{\textbf{F}}\mu_{\widetilde{\textbf{F}}}+c_{1})(2\sigma+c_{2})}
{(\mu_{\textbf{F}}^{2}+\mu_{\widetilde{\textbf{F}}}^{2}+c_{1})(\sigma_{\textbf{F}}^{2}+\sigma_{\widetilde{\textbf{F}}}^{2}+c_{2})},
\]
where $\textbf{F}\in \mathbb{R}^{m_{1}\times m_{2}}$ and $\widetilde{\textbf{F}}\in \mathbb{R}^{m_{1}\times m_{2}}$ are the original grayscale image and the restored grayscale image, respectively, $Max_{\textbf{F},\widetilde{\textbf{F}}}$ is the maximum possible pixel value of $\textbf{F}$ and $\widetilde{\textbf{F}}$, $\mu_{\textbf{F}}$ and $\mu_{\tilde{\textbf{F}}}$ are the mean values of $\textbf{F}$ and $\widetilde{\textbf{F}}$, $\sigma_{\textbf{F}}^{2}$ and $\sigma_{\widetilde{\textbf{F}}}^{2}$ are the standard variances of $\textbf{F}$ and $\widetilde{\textbf{F}}$, $\sigma$ is the covariance of $\widetilde{\textbf{F}}$ and $\textbf{F}$, and $c_{1}$, $c_{2}>0$ are constants. For color images and MSIs, we use the average PSNR and SSIM corresponding to channels or bands as the quality index of the entire result. For color videos, we calculate the PSNR and SSIM values by averaging the two values of all color frames. High PSNR and SSIM values indicate good performance.

\textbf{Parameter setting.} The proposed method involves the following parameters: weights $\{\beta_{n}\}_{n=1}^{\lceil j/2\rceil}$ and the penalty parameter $\eta$. In our model \eqref{our model}, we assign larger weight to $\textbf{X}_{\{n\}}$ with more balanced size, i.e.,
\[
\beta_{n}=\frac{\delta_{n}}{\sum_{n=1}^{\lceil j/2\rceil}\delta_{n}} \ \  \text{with} \ \ \delta_{n}=\min(\Pi_{d=\lceil j/2\rceil}^{\lceil j/2\rceil+n-1}m_{d}, \Pi_{d=\lceil j/2\rceil+n}^{\lceil j/2\rceil}m_{d}).
\]
Besides, we set $\eta^{k+1}=1.1 \eta^{k}$ and select the initial value $\eta^{0}$ from one of values in $\{10^{-9}, 10^{-8}, 10^{-7}, 10^{-6}\}$, to obtain the highest PSNR value.

\begin{table}[!t]\scriptsize
\renewcommand\arraystretch{1.2}
\caption{PSNR and SSIM values of different methods on color image completion with different SRs.}
\vspace{-0.5cm}
\begin{center}
\begin{tabular}{c|c|ccccccccc}
\hline \hline
\multirow{1}{*}{Images}&\multicolumn{1}{c|}{SR} &Method &HaLRTC &NSNN &LRTC-TV & SiLRTC-TT & tSVD &KBR & TRNN & LogTR \\  \hline
\multirow{6}{*}{\emph{House}}
          &\multirow{2}{*}{0.1} & PSNR    &20.05	  &20.16     &23.23   &21.88   &20.61	    &22.44     &23.32       &\textbf{26.94}  \\
          &                     & SSIM    &0.4745	  &0.3770    &0.7226  &0.6132  &0.3806	    &0.4571    &0.6903      &\textbf{0.7334} \\ \cline{2-11}

          &\multirow{2}{*}{0.3} & PSNR    &25.69	  &27.61     &29.04   &27.55   &27.92       &28.40     &29.86       &\textbf{32.96}  \\
          &                     & SSIM    &0.7681	  &0.7204    &0.8649  &0.8138  &0.7435      &0.7536    &0.8684      &\textbf{0.8834} \\ \cline{2-11}

          &\multirow{2}{*}{0.5} & PSNR    &30.11      &32.09     &32.64   &32.34   &32.98      	&33.06     &34.33       &\textbf{37.03}  \\
          &                     & SSIM    &0.8945     &0.8646    &0.9253  &0.9144  &0.8850      &0.8764    &0.9369      &\textbf{0.9415} \\
          \hline
%
%
\multirow{6}{*}{\emph{Peppers}}
          &\multirow{2}{*}{0.1} & PSNR    &18.80	  &18.98     &22.06  &20.40     &17.06	        &18.98     &21.42       &\textbf{24.07}  \\
          &                     & SSIM    &0.4493	  &0.3842    &\textbf{0.7118}  &0.5622   &0.2469	    &0.4198    &0.6363      &0.7045 \\ \cline{2-11}

          &\multirow{2}{*}{0.3} & PSNR    &24.93	  &25.91     &28.04   &23.52   &23.79       	&26.00     &27.19       &30.36  \\
          &                     & SSIM    &0.7748	  &0.6864    &\textbf{0.8916}  &0.7140   &0.6080        &0.7323    &0.8513      &0.8823 \\ \cline{2-11}

          &\multirow{2}{*}{0.5} & PSNR    &29.41      &30.63     &31.76  &30.25    &28.96      	&30.70     &31.39       &\textbf{34.00 } \\
          &                     & SSIM    &0.8994     &0.8395    &\textbf{0.9452}  &0.9123   &0.8207      	&0.8665    &0.9313      &0.9395 \\
          \hline
\multirow{6}{*}{\emph{Airplane}}
          &\multirow{2}{*}{0.1} & PSNR    &19.52	  &19.50     &22.35   &20.82   &19.91	    &20.69     &22.47       &\textbf{25.23}  \\
          &                     & SSIM    &0.5012	  &0.4082    &0.7149  &0.6018  &0.4301	    &0.4079    &0.7020      &\textbf{0.7897} \\
          \cline{2-11}
          &\multirow{2}{*}{0.3} & PSNR    &24.41	  &25.40     &26.94   &25.92   &25.89       &27.46     &28.36       &\textbf{31.42}  \\
          &                     & SSIM    &0.7786	  &0.7061    &0.8806  &0.8257  &0.7383      &0.7522    &0.8952      &\textbf{0.9277} \\
          \cline{2-11}
          &\multirow{2}{*}{0.5} & PSNR    &28.75      &31.68     &30.55   &30.60   &30.54      	&32.64     &33.18       &\textbf{36.78}  \\
          &                     & SSIM    &0.9028     &0.8786    &0.9436  &0.9301  &0.8857      &0.8951    &0.9601      &\textbf{0.9721} \\
          \hline
\multirow{6}{*}{\emph{Barbara}}
          &\multirow{2}{*}{0.1} & PSNR    &19.59	  &19.66     &22.32   &20.90   &18.98	    &20.52     &22.10       &\textbf{24.74}  \\
          &                     & SSIM    &0.4435	  &0.3999    &0.6381  &0.5318  &0.3572	    &0.4110    &0.6225      &\textbf{0.7224} \\
          \cline{2-11}
          &\multirow{2}{*}{0.3} & PSNR    &25.24	  &26.14     &27.23   &26.30   &25.65       &26.80     &28.04       &\textbf{31.09}  \\
          &                     & SSIM    &0.7643	  &0.7242    &0.8384  &0.8035  &0.7244      &0.7424    &0.8638      &\textbf{0.9063} \\
          \cline{2-11}
          &\multirow{2}{*}{0.5} & PSNR    &29.35      &31.30     &30.36   &30.78   &31.15      	&31.99     &33.08       &\textbf{36.60}  \\
          &                     & SSIM    &0.8924     &0.8834    &0.9134  &0.9224  &0.8980      &0.8979    &0.9530      &\textbf{0.9696} \\
          \hline
\multirow{6}{*}{\emph{Monarch}}
          &\multirow{2}{*}{0.1} & PSNR    &17.25	  &17.02     &18.59   &18.33   &17.18	    &17.37     &19.11       &\textbf{21.58}  \\
          &                     & SSIM    &0.4996	  &0.4144    &0.6806  &0.5902  &0.3458	    &0.3437    &0.6739      &\textbf{0.7729} \\
          \cline{2-11}
          &\multirow{2}{*}{0.3} & PSNR    &21.11	  &21.75     &23.40   &22.68   &22.47       &24.03     &24.87       &\textbf{28.90}  \\
          &                     & SSIM    &0.7666	  &0.7161    &0.8846  &0.8223  &0.6923      &0.7668    &0.8908      &\textbf{0.9369} \\
          \cline{2-11}
          &\multirow{2}{*}{0.5} & PSNR    &25.22      &27.90     &27.87   &27.81   &27.95      	&29.76     &30.42       &\textbf{34.81}  \\
          &                     & SSIM    &0.9001     &0.8890    &0.9548  &0.9338  &0.8784      &0.9072    &0.9634      &\textbf{0.9804} \\
          \hline \hline
\end{tabular}\label{table:image}
\end{center}
\end{table}

We terminate our algorithm when the following stopping condition holds:
\[
\frac{\|\mathcal{X}^{k+1}-\mathcal{X}^{k}\|_{F}}{\|\mathcal{X}^{k}\|_{F}}\leq 10^{-4}.
\]
Also, we set the maximum iterations as 500. All parameters corresponding to compared methods are carefully tuned according to the reference papers' suggestion. In the experiments, SiLRTC-TT, TRNN, and LogTR use VDT to transform a low-order tensor to a high-order one, while HaLRTC, NSNN, LRTC-TV, tSVD, and KBR are performed directly on the original data.

\begin{figure}[!t]
\scriptsize\setlength{\tabcolsep}{0.5pt}
\begin{center}
\begin{tabular}{ccc}
\includegraphics[width=0.95\textwidth]{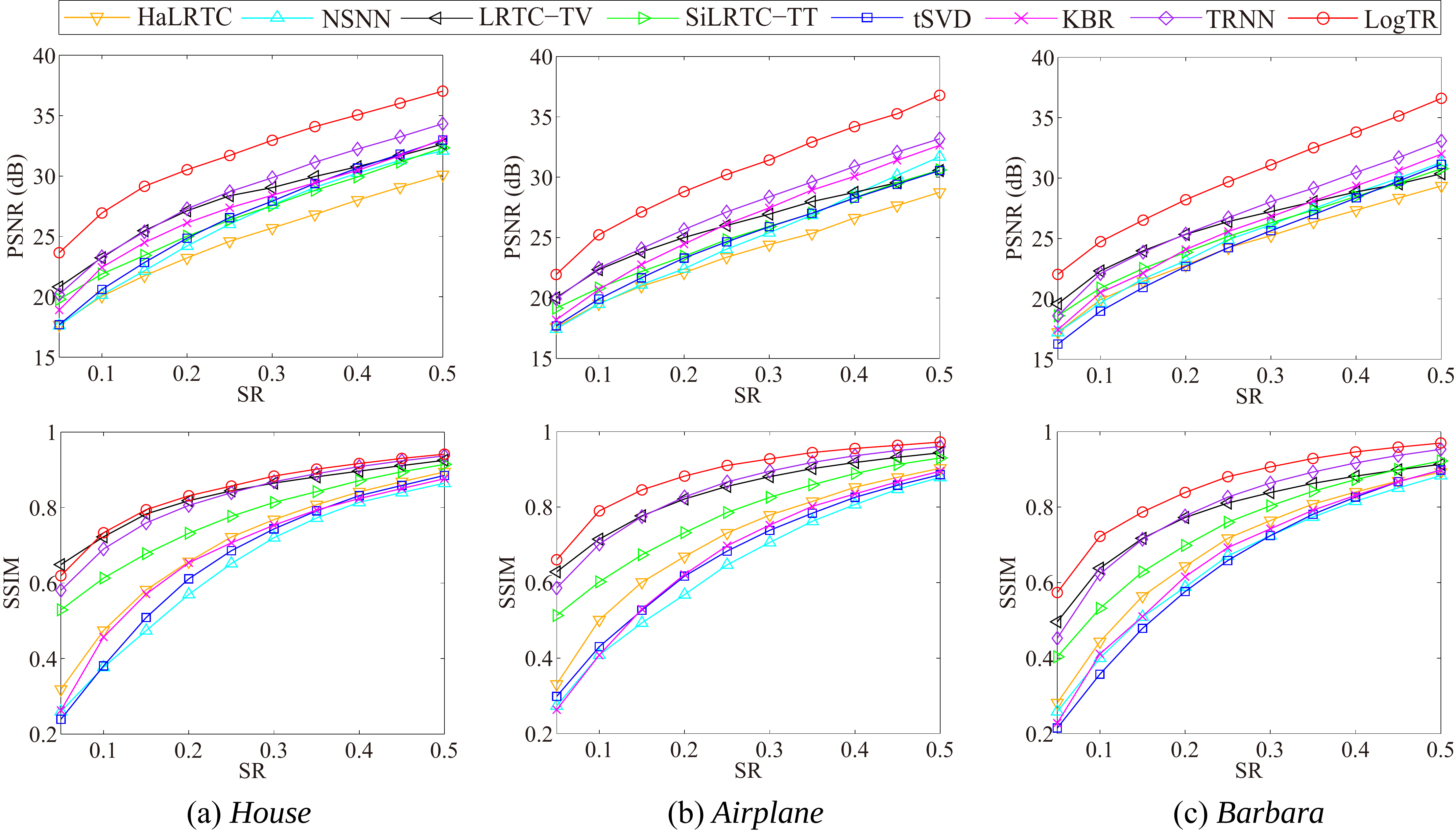}
\end{tabular}
\caption{PSNR and SSIM values of different methods on color images completion with different SRs.}
  \label{fig:image_ps_ss}
  \end{center}\vspace{-0.3cm}
\end{figure}

%
\subsection{Color images}
In this subsection, we test the proposed method using five color images of size $256\times 256 \times 3$. We test two sampling cases, including random missing entries and structural missing entries. In general, the later is more challenging than the former. We first use VDT to transform a third-order color image into a ninth-order tensor, whose size is $4 \times 4 \times 4 \times 4 \times 4 \times 4  \times 4 \times 4 \times 3$.

\textbf{Random missing.} We randomly sample the incomplete images using sampling rates (SRs) from 0.05 to 0.5. In Figure \ref{fig:image_random}, we show the performance of all methods under random missing case with $SR=0.3$. The results obtained by both HaLRTC and NSNN have undesirable artifacts. Although achieving better results than HaLRTC and NSNN, the results by LRTC-TV are over-smooth and lose many details. SiLRTC-TT and TRNN create some block-artifacts, and tSVD and KBR exhibit many artifacts, such as the eaves area of \emph{House} and the tail of \emph{Airplane}. While LogTR performs better than the compared methods in keeping the smoothness of backgrounds and clear structures; please see the zoom-in regions of recovered images.

Table \ref{table:image} lists the recovered PSNR and SSIM values by different methods. We label the best values for each quality index in bold. Figure \ref{fig:image_ps_ss} shows the PSNR and SSIM curves under different SRs. We can see that the proposed method achieves higher PSNR and SSIM values than all compared methods in most cases.

\begin{figure}[!th]
\scriptsize\setlength{\tabcolsep}{0.5pt}
\begin{center}
\begin{tabular}{ccccc}
\includegraphics[width=0.19\textwidth]{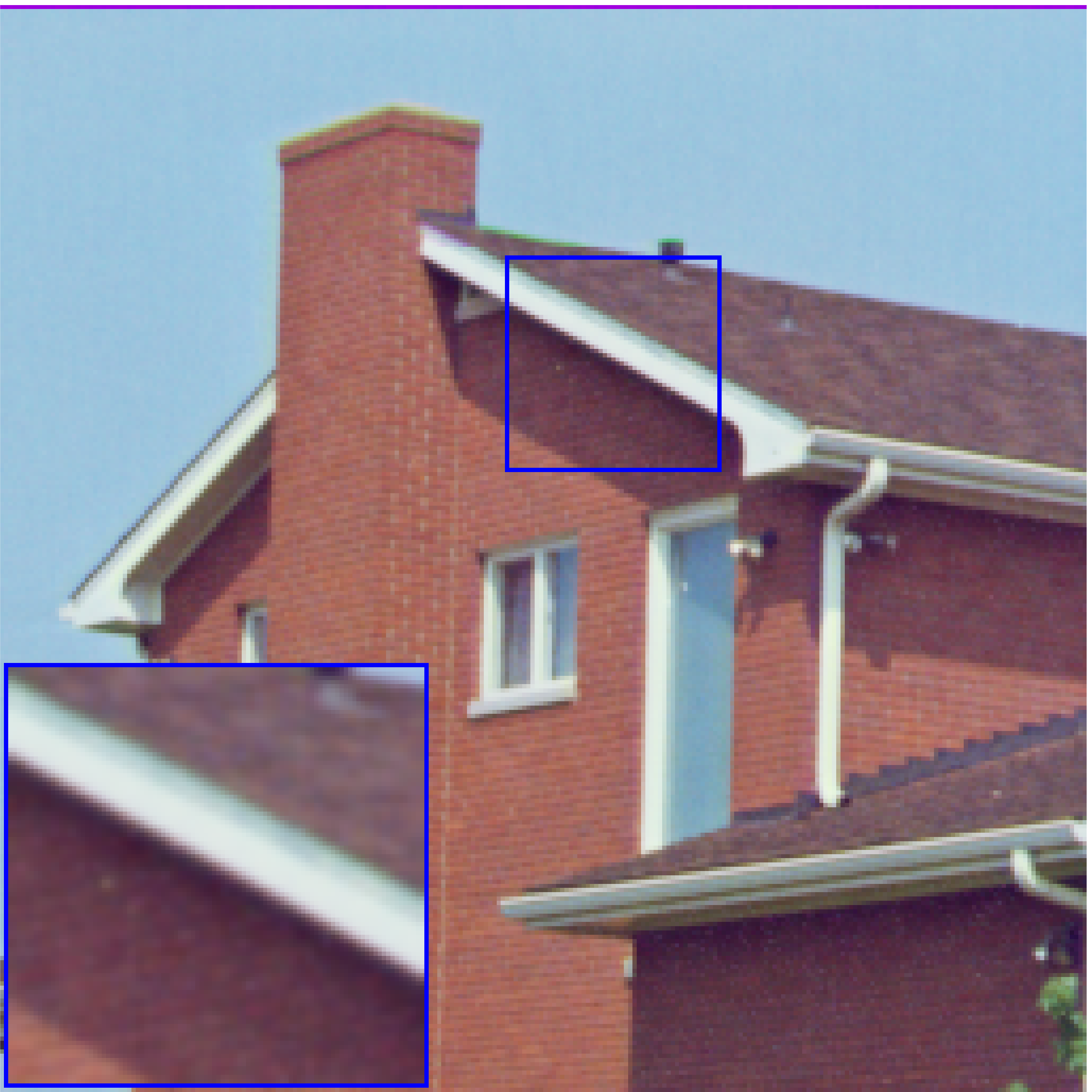}&
\includegraphics[width=0.19\textwidth]{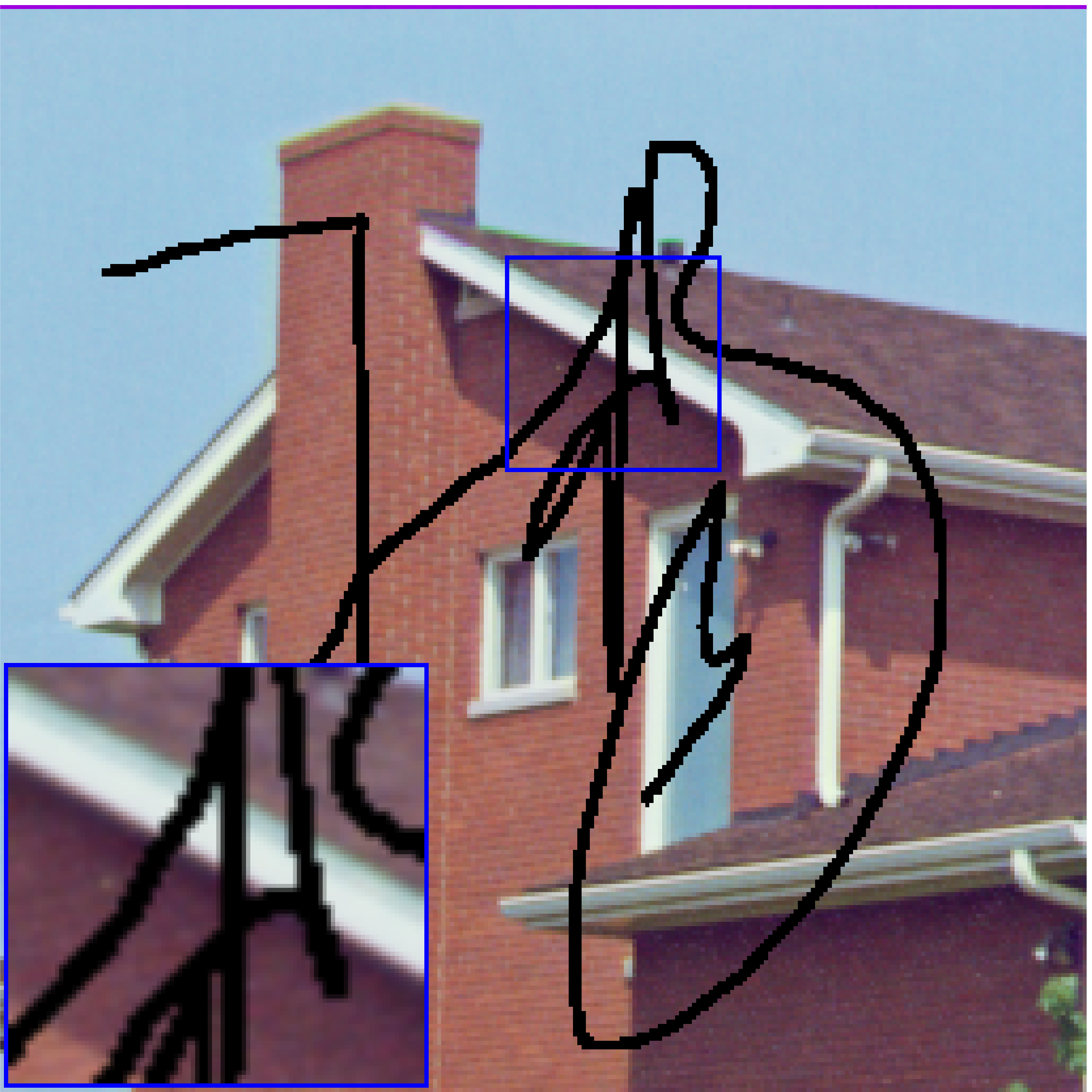}&
\includegraphics[width=0.19\textwidth]{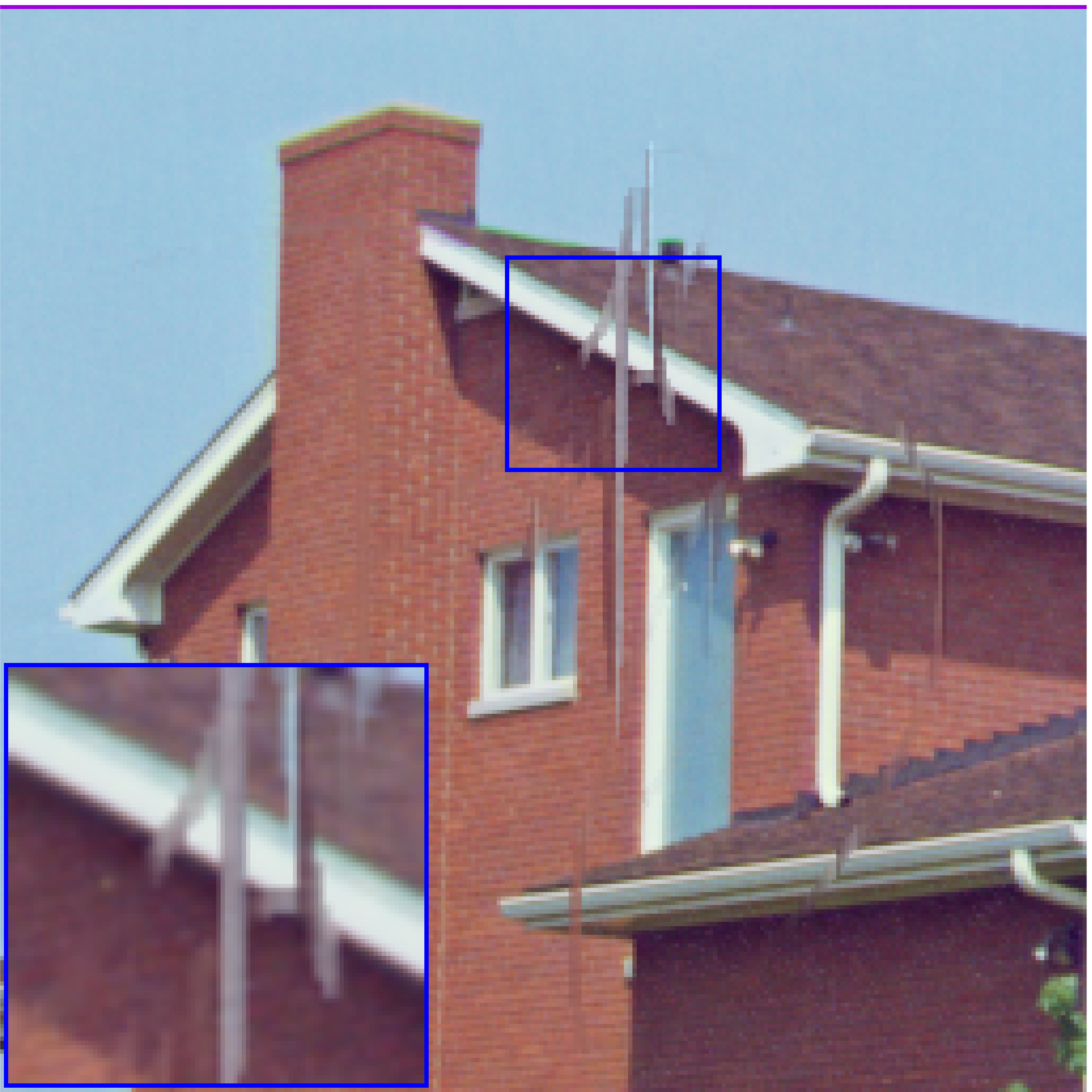}&
\includegraphics[width=0.19\textwidth]{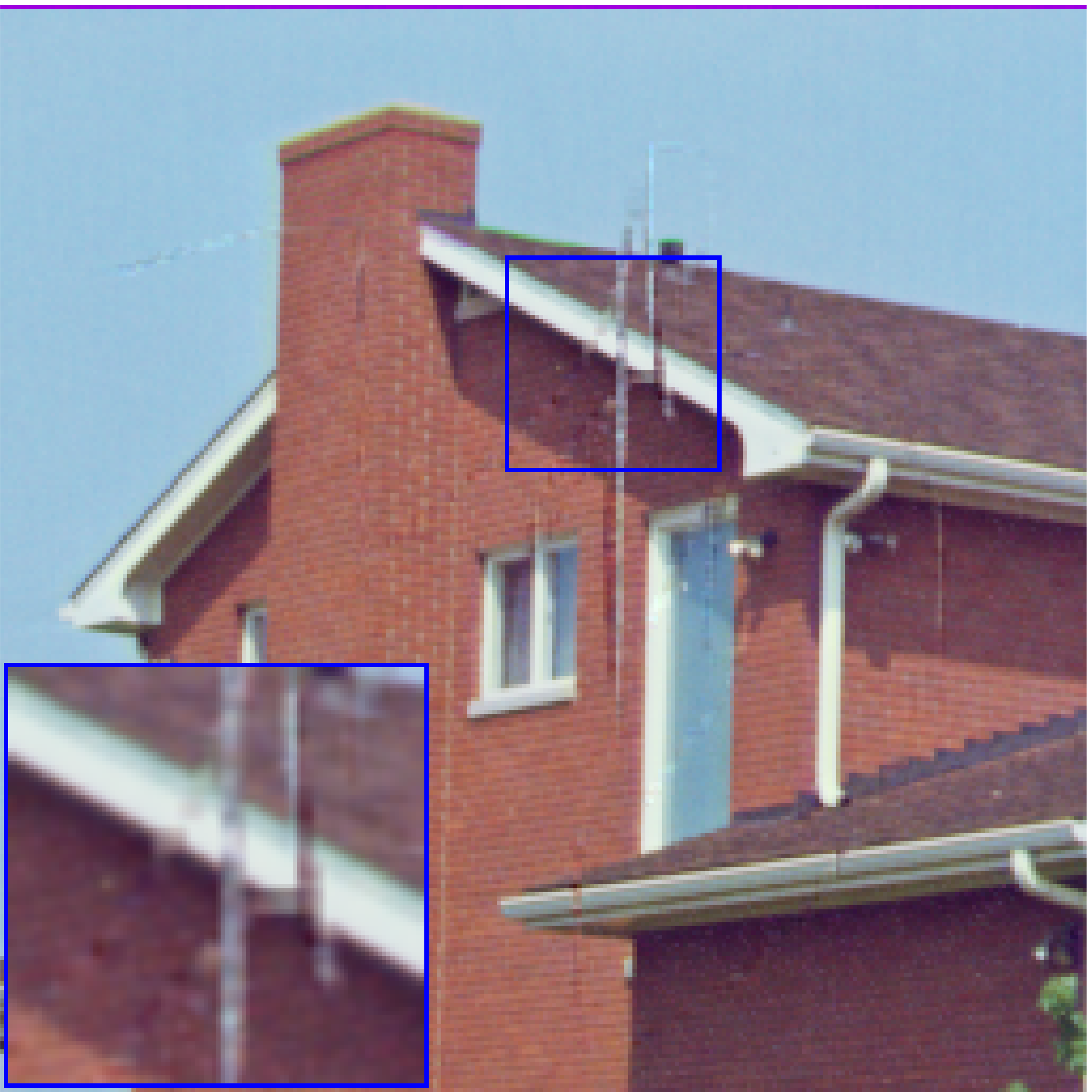}&
\includegraphics[width=0.19\textwidth]{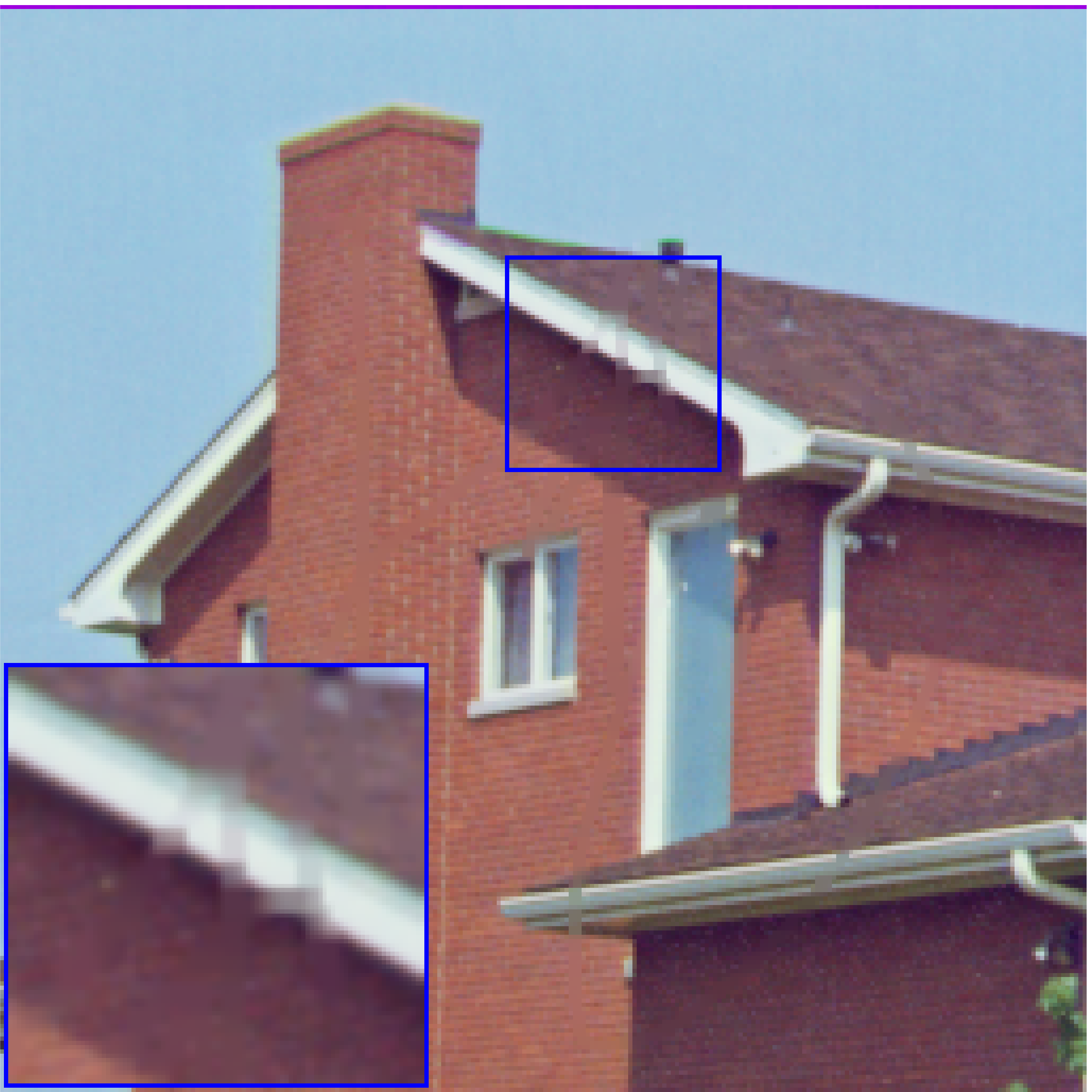}\vspace{0.01cm}\\
(a) (PSNR, SSIM) & (b) Observed & (c) (33.20, 0.9643) & (d) (34.85, 0.9644) & (e) (40.94, 0.9857)\\
\includegraphics[width=0.19\textwidth]{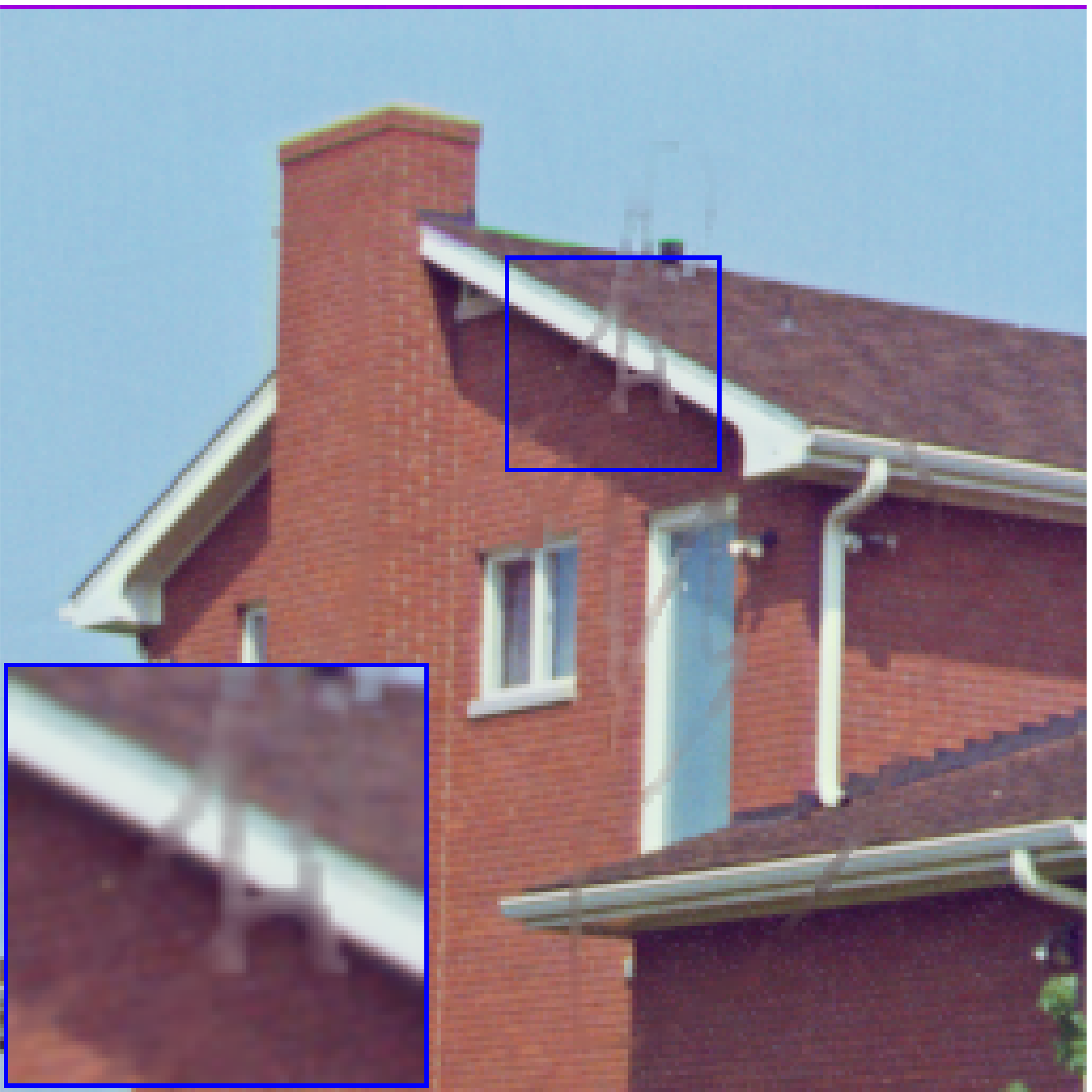}&
\includegraphics[width=0.19\textwidth]{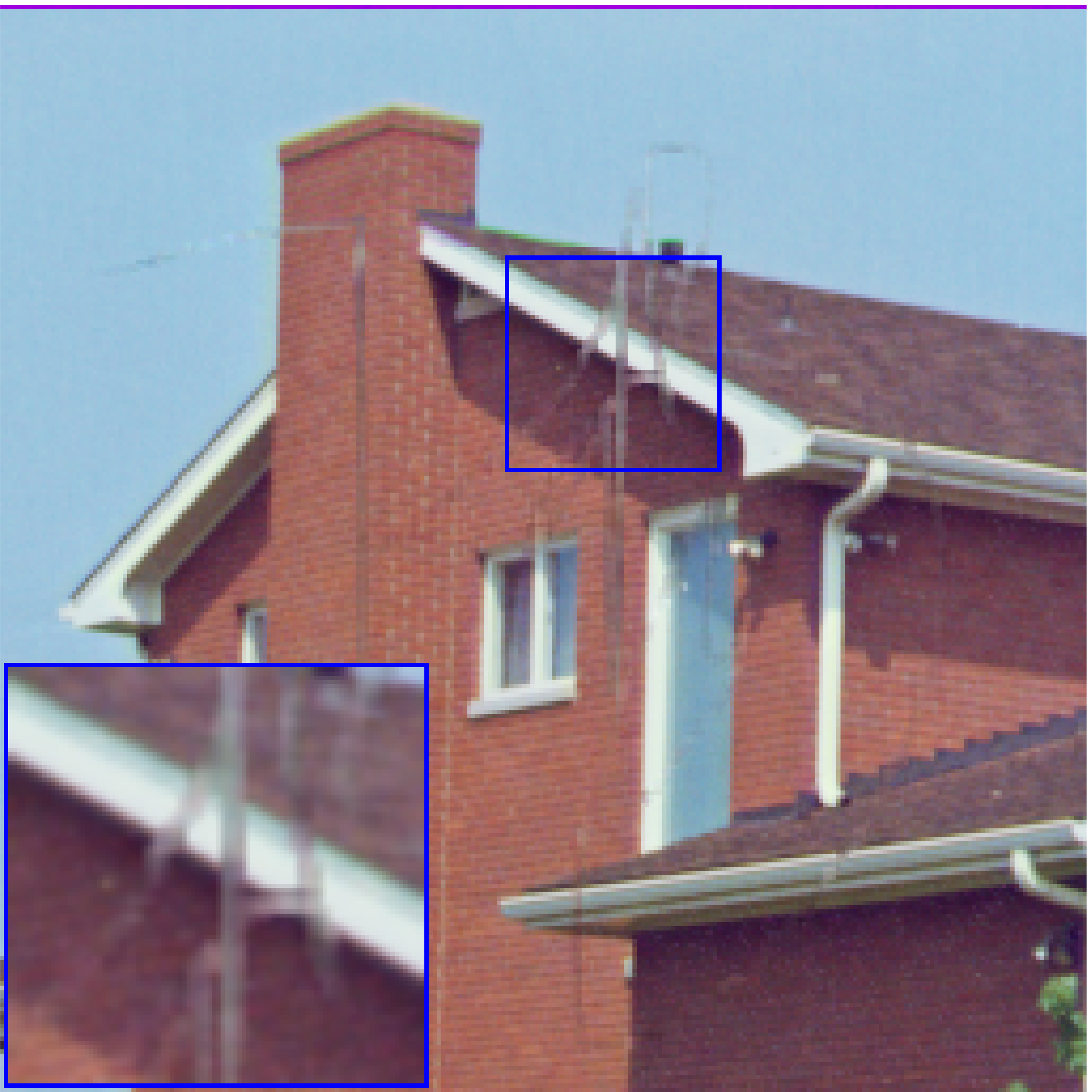}&
\includegraphics[width=0.19\textwidth]{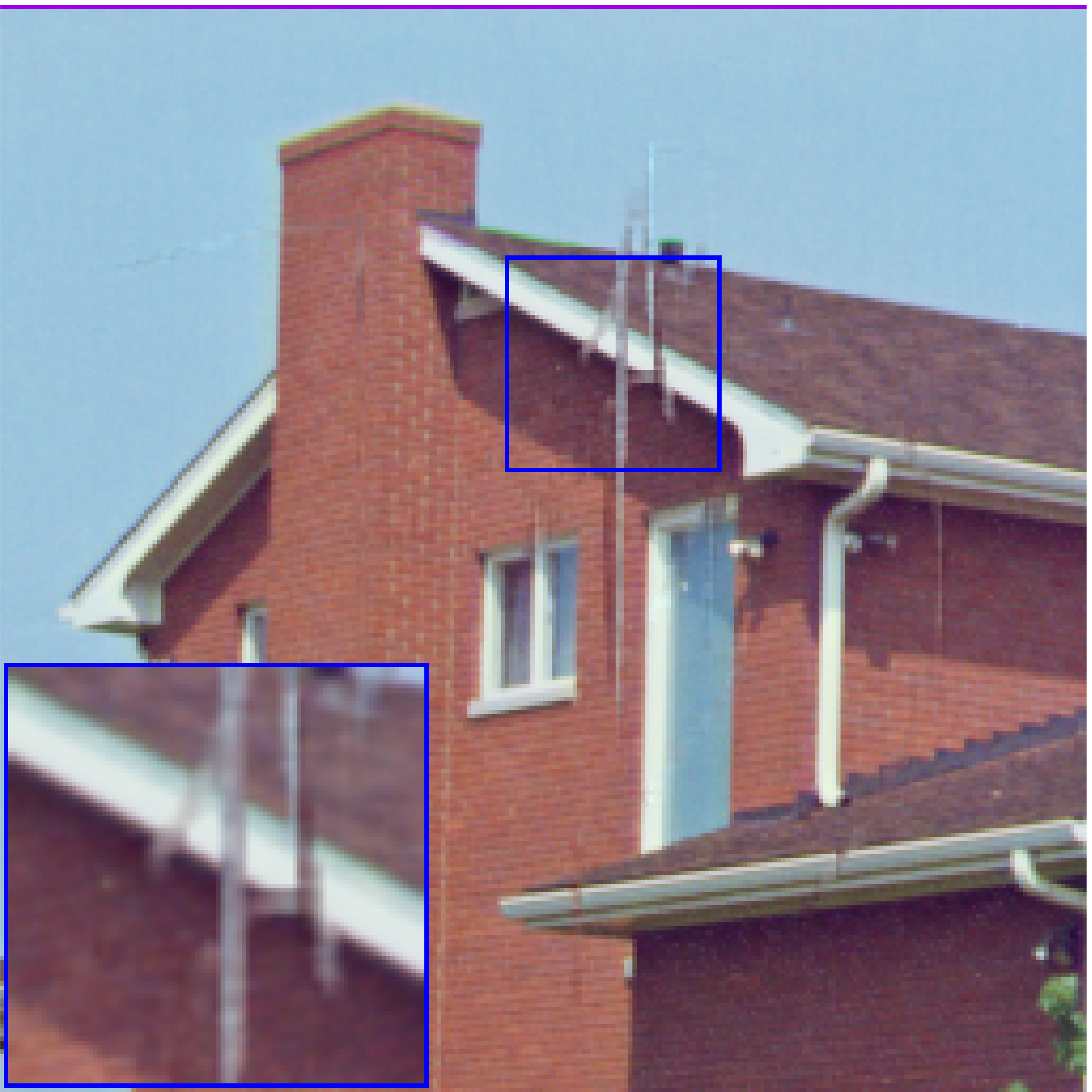}&
\includegraphics[width=0.19\textwidth]{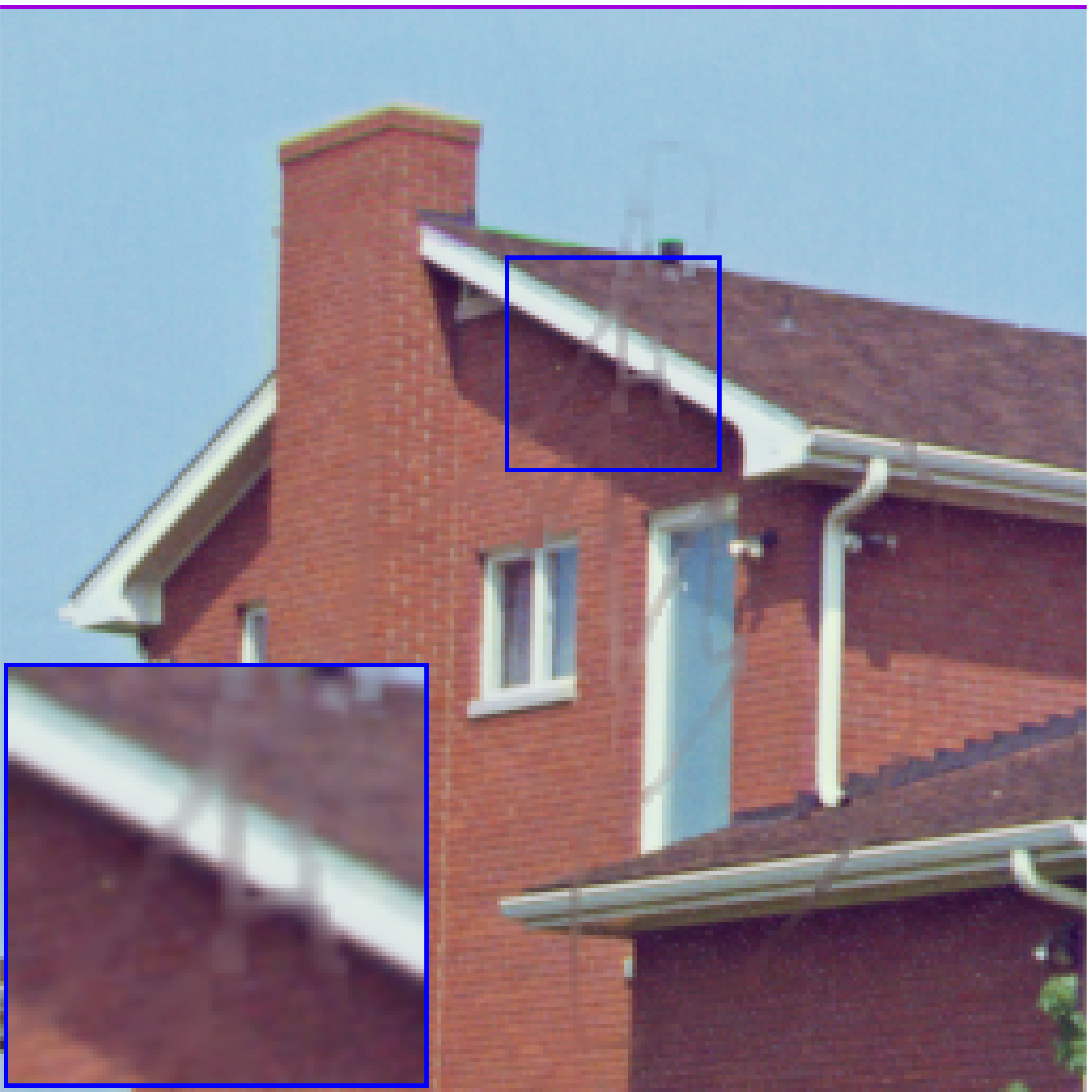}&
\includegraphics[width=0.19\textwidth]{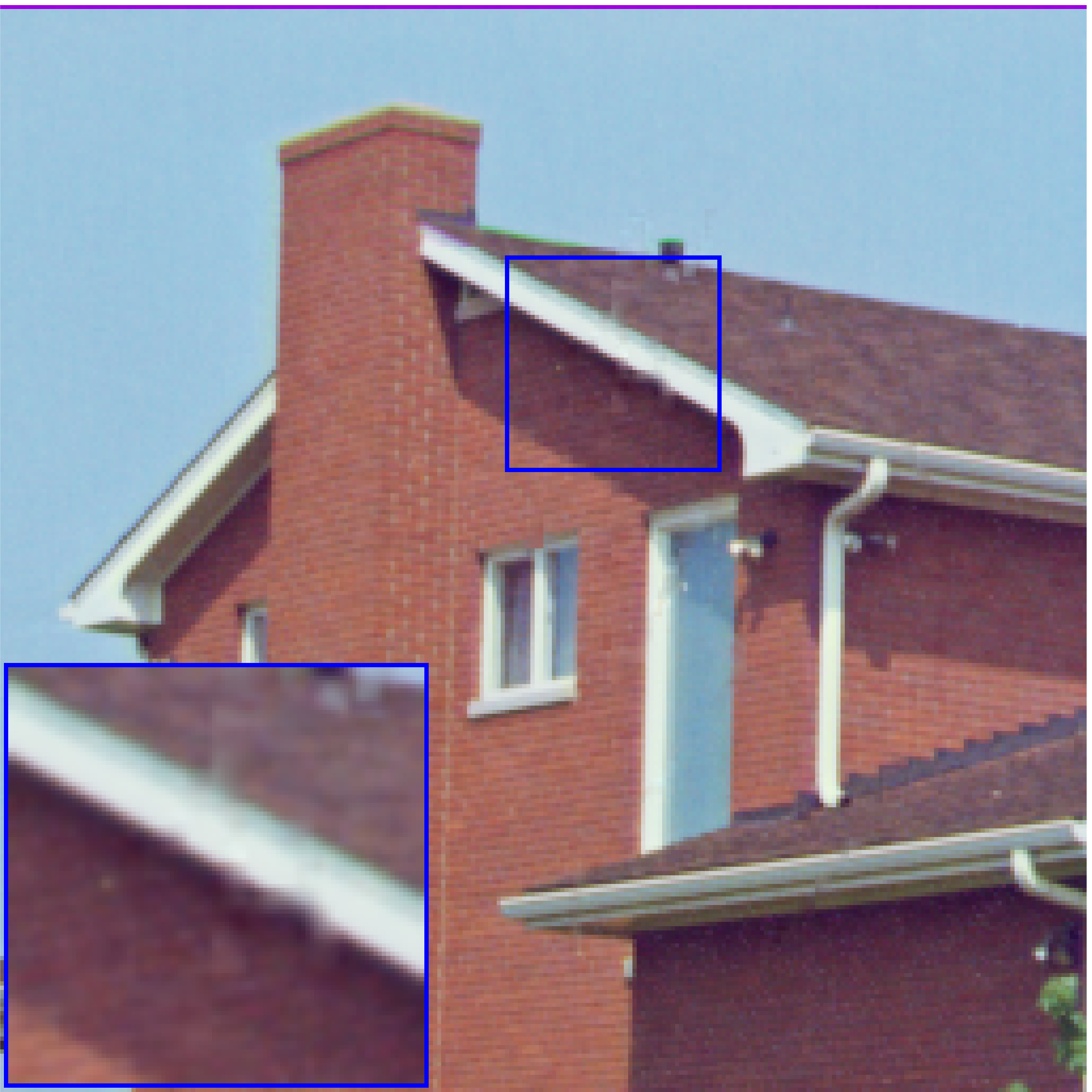}\vspace{0.01cm}\\
(f) (37.60, 0.9786) & (g) (36.14, 0.9681) & (h) (34.61, 0.9677) & (i) (38.71, 0.9813) & (j) (\textbf{42.33}, \textbf{0.9863})\\
\includegraphics[width=0.19\textwidth]{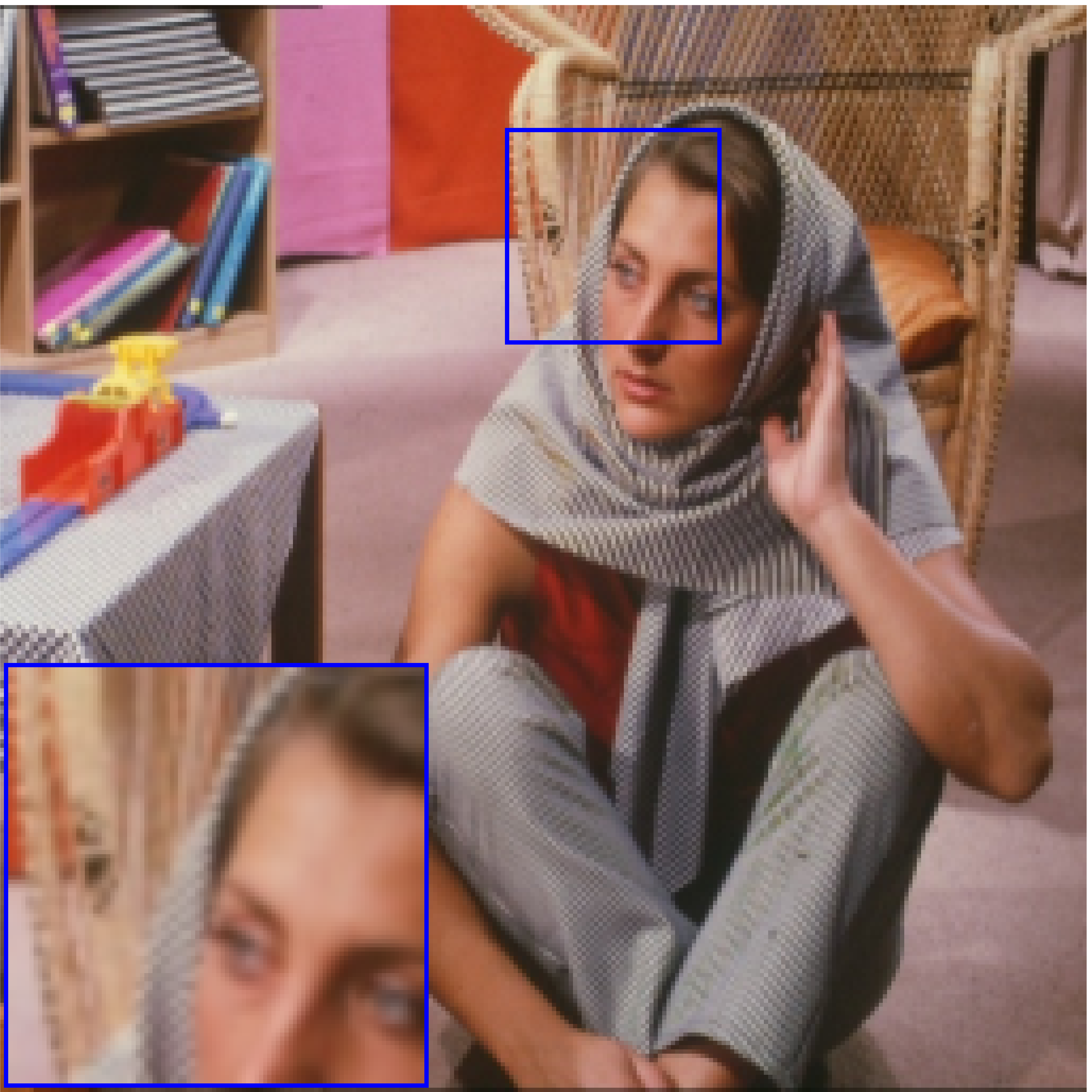}&
\includegraphics[width=0.19\textwidth]{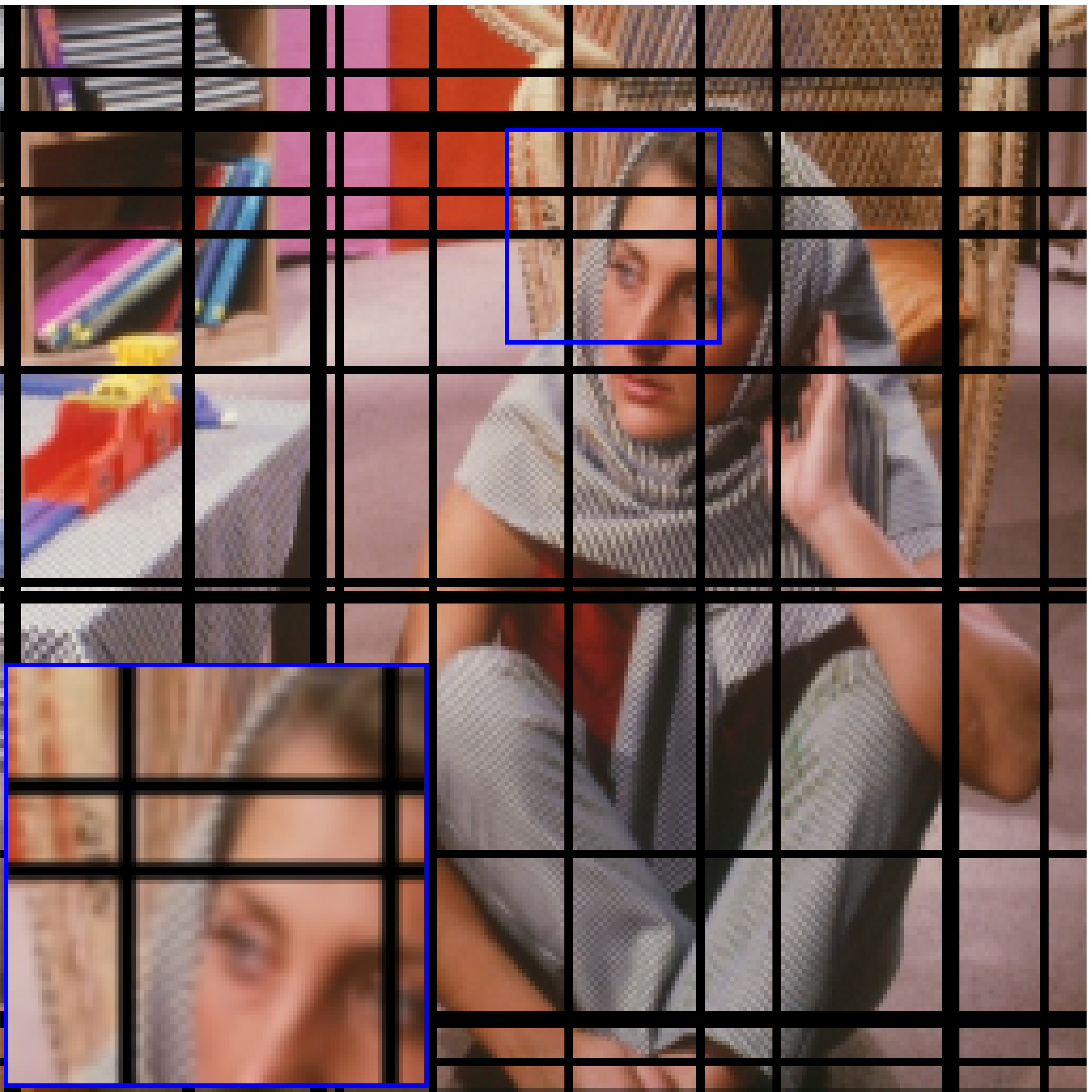}&
\includegraphics[width=0.19\textwidth]{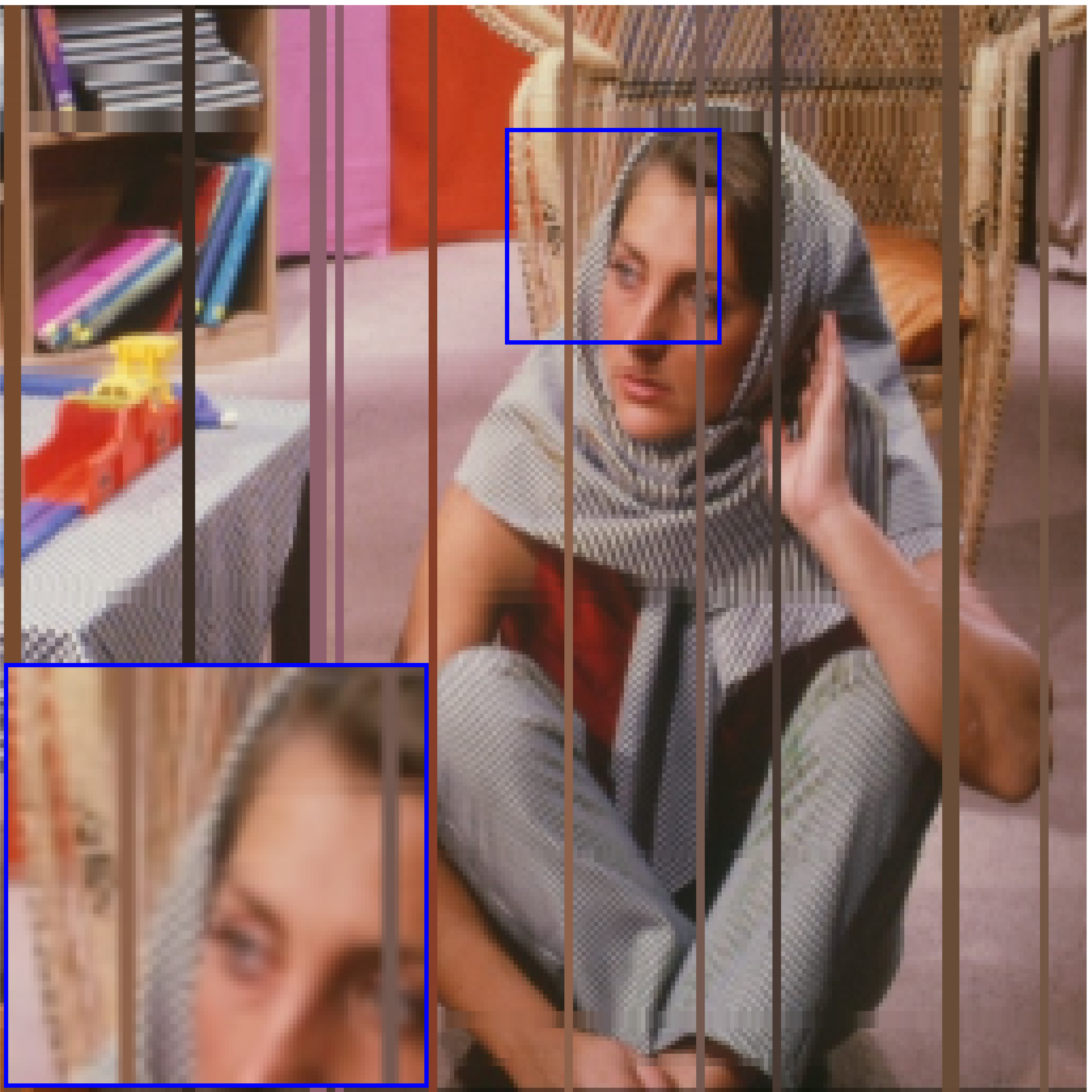}&
\includegraphics[width=0.19\textwidth]{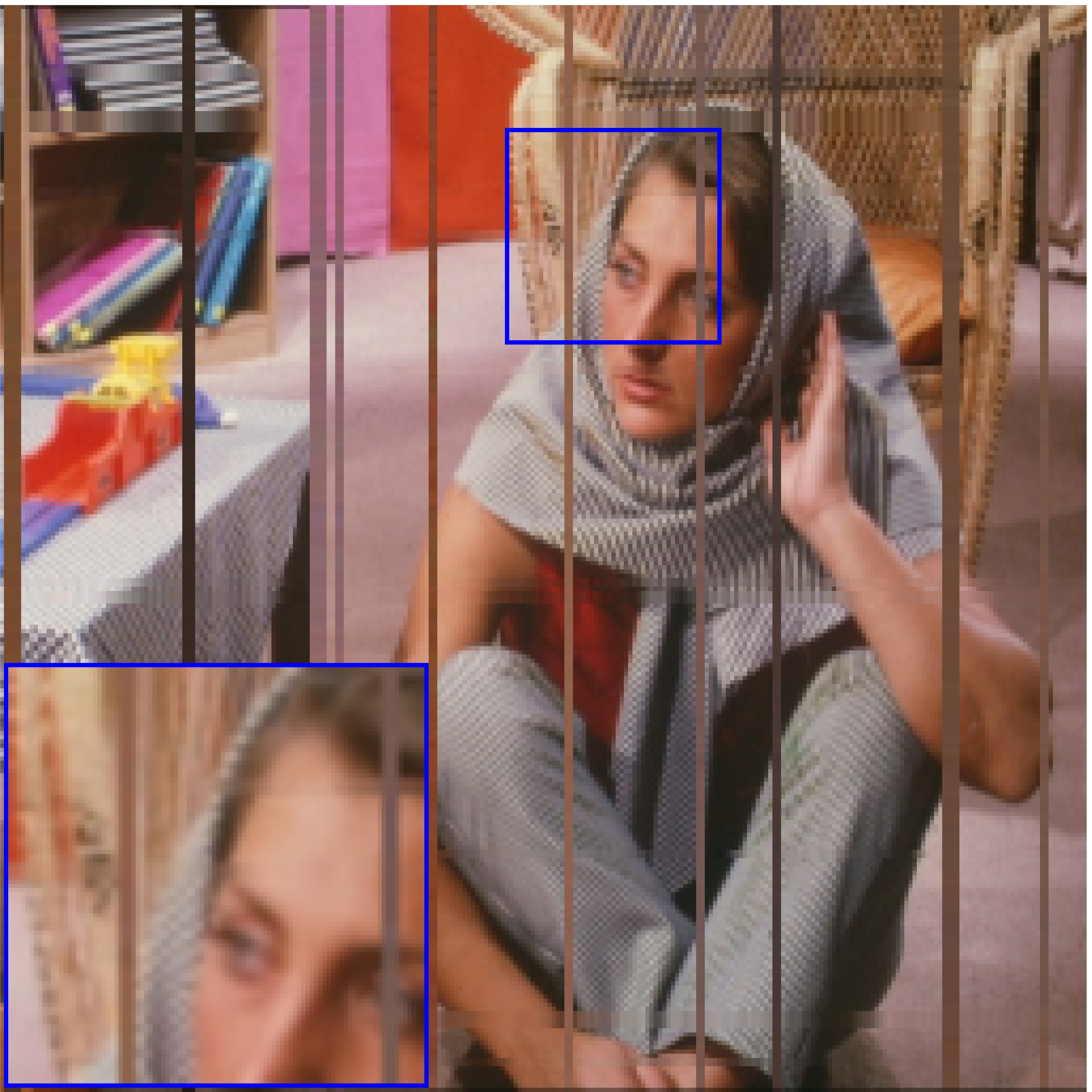}&
\includegraphics[width=0.19\textwidth]{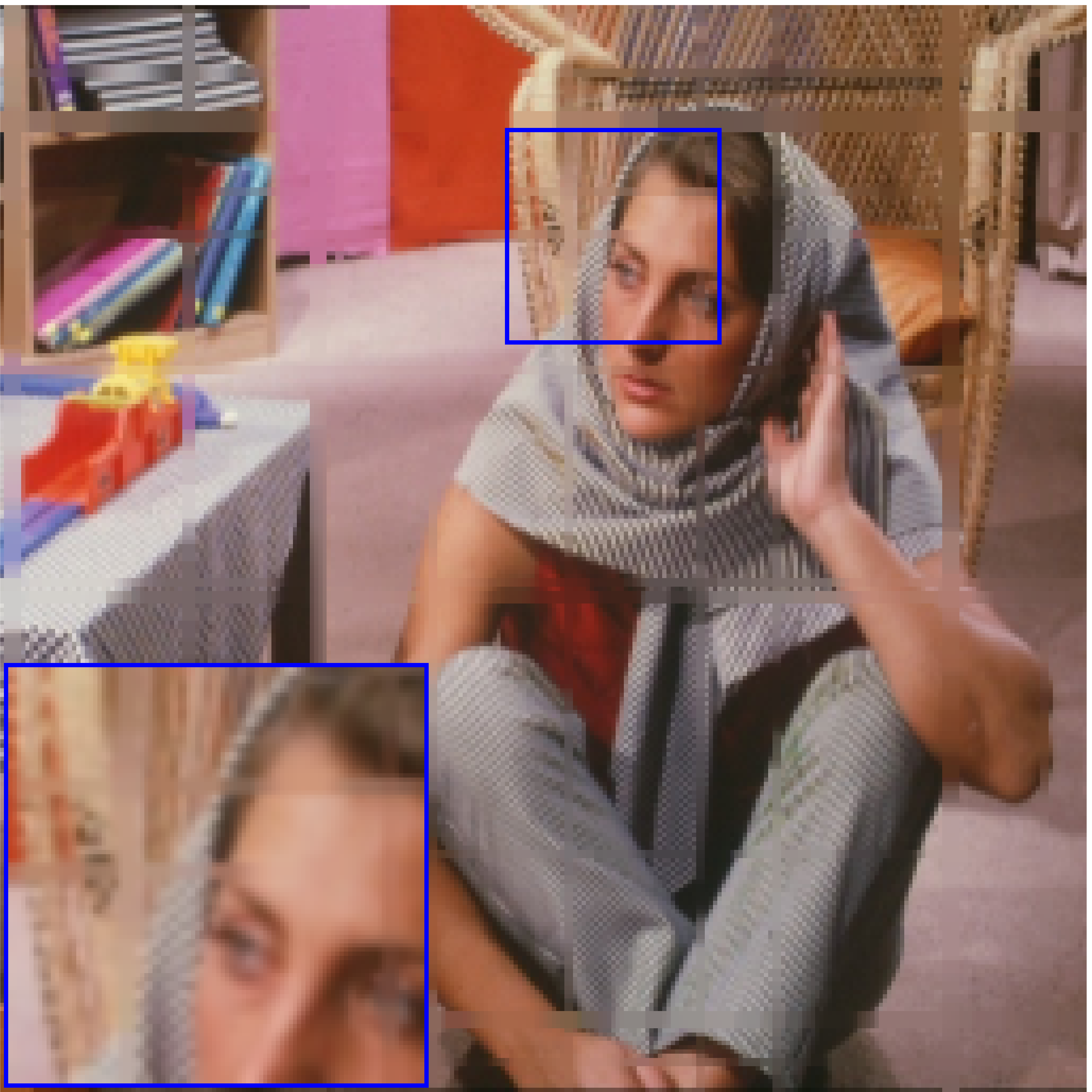}\vspace{0.01cm}\\
(a) (PSNR, SSIM) & (b) Observed & (c) (22.47, 0.8193) & (d) (22.49, 0.8160) & (e) (28.91, 0.9342)\\
\includegraphics[width=0.19\textwidth]{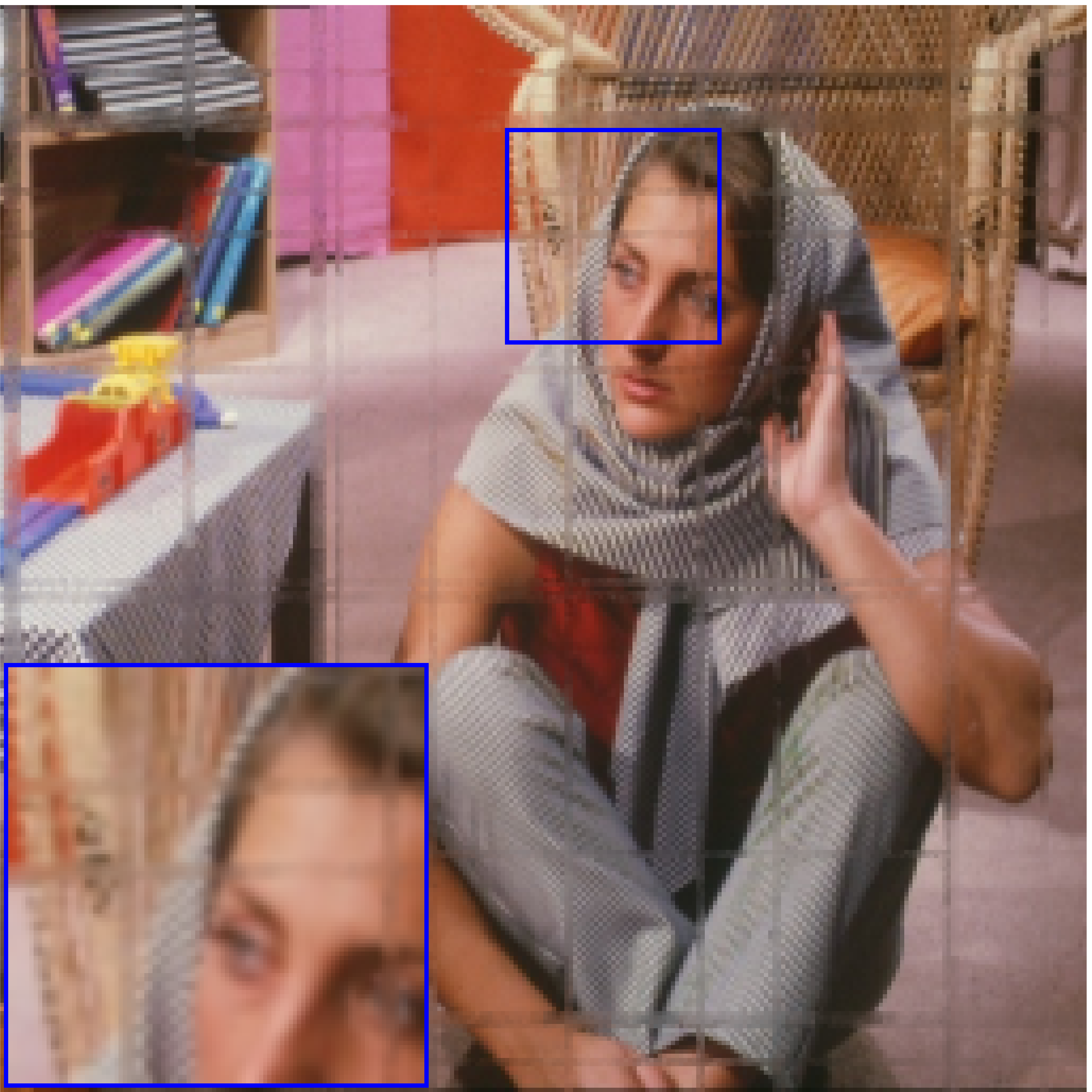}&
\includegraphics[width=0.19\textwidth]{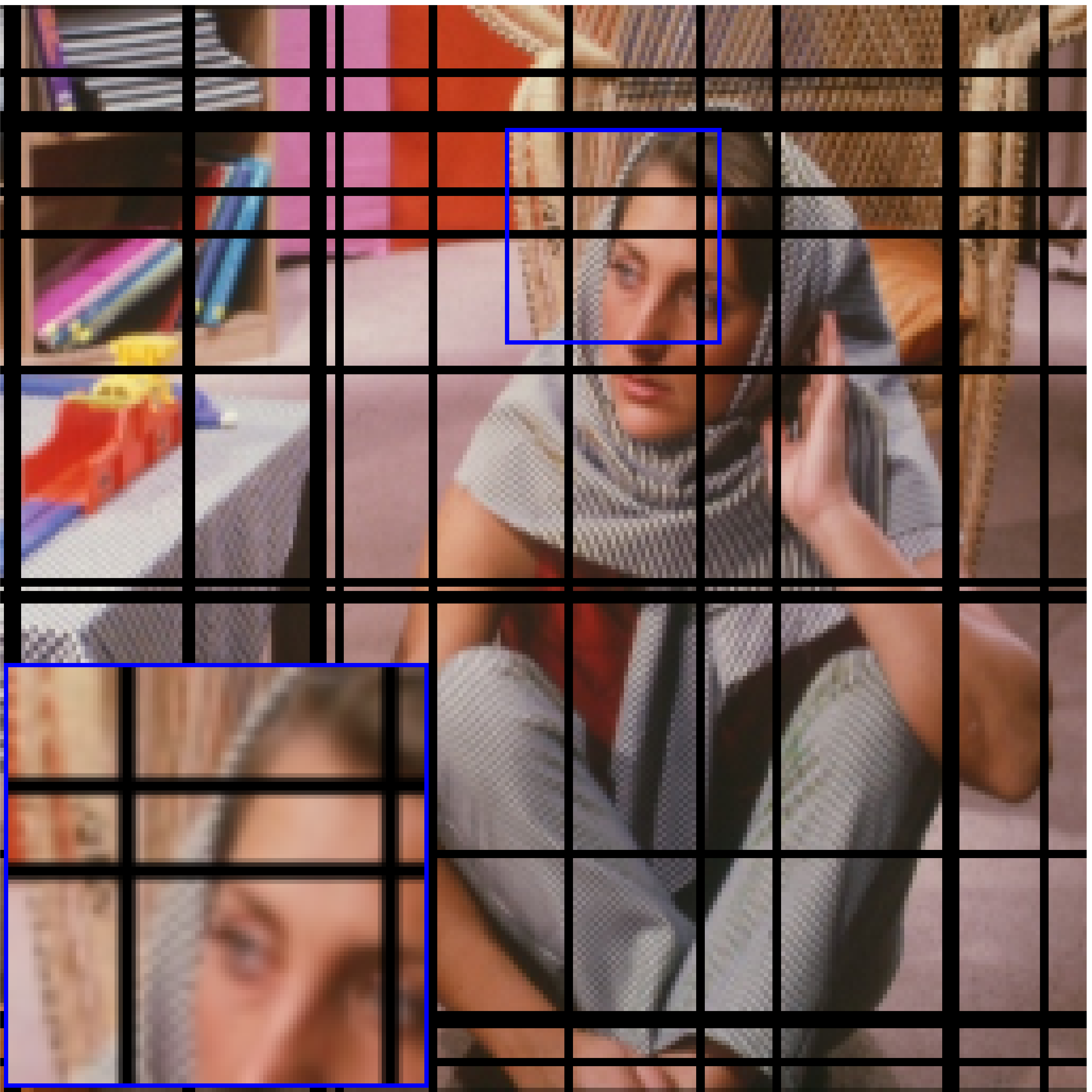}&
\includegraphics[width=0.19\textwidth]{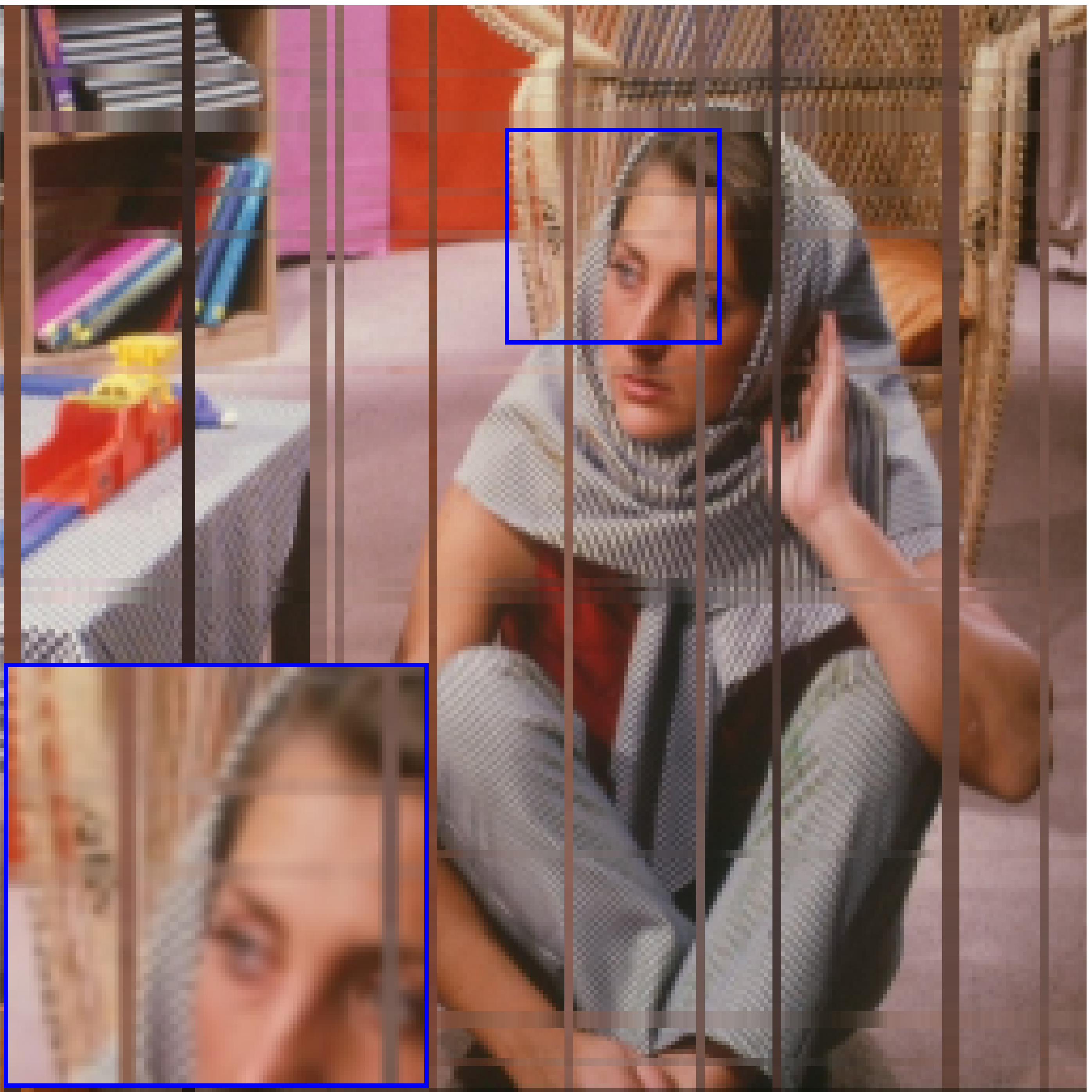}&
\includegraphics[width=0.19\textwidth]{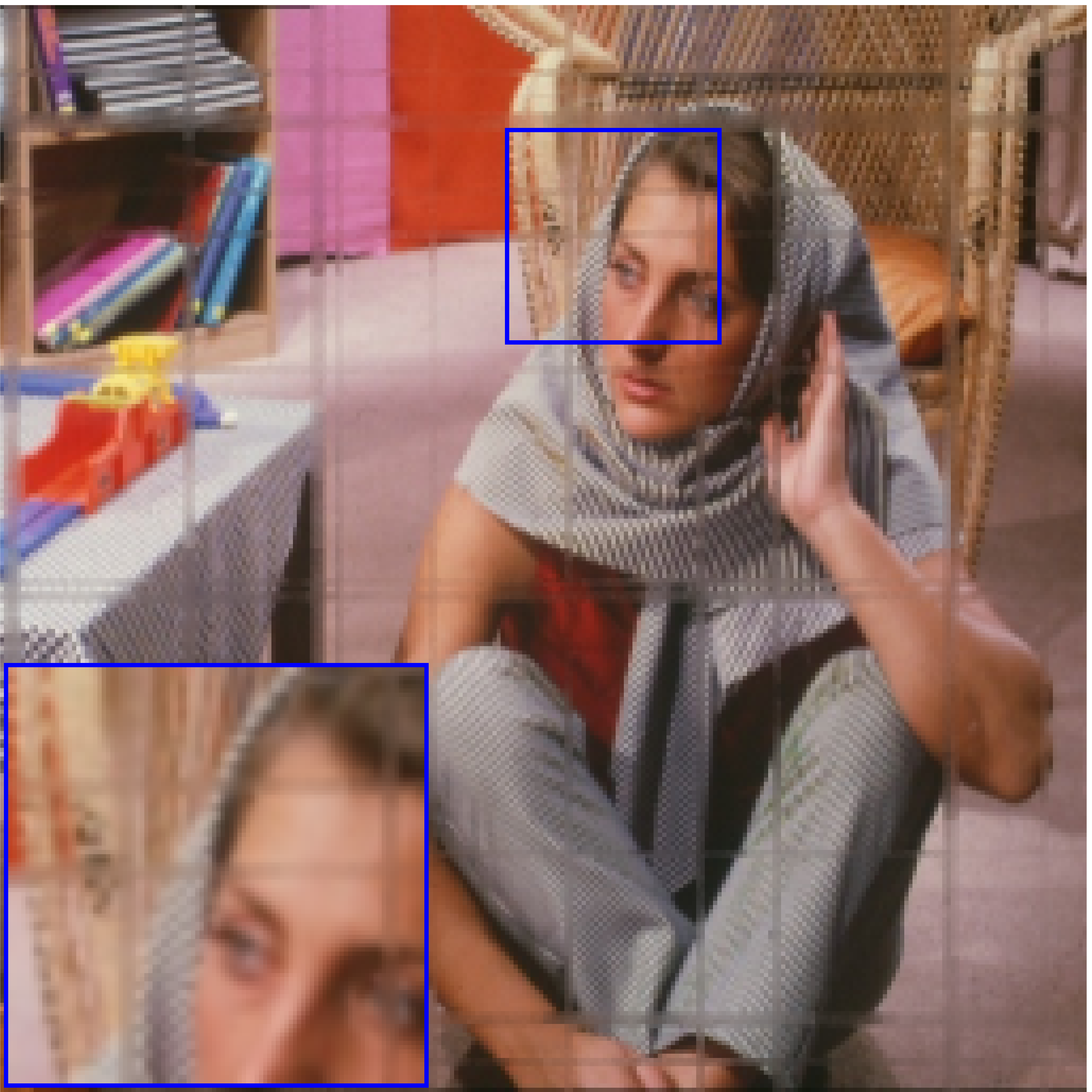}&
\includegraphics[width=0.19\textwidth]{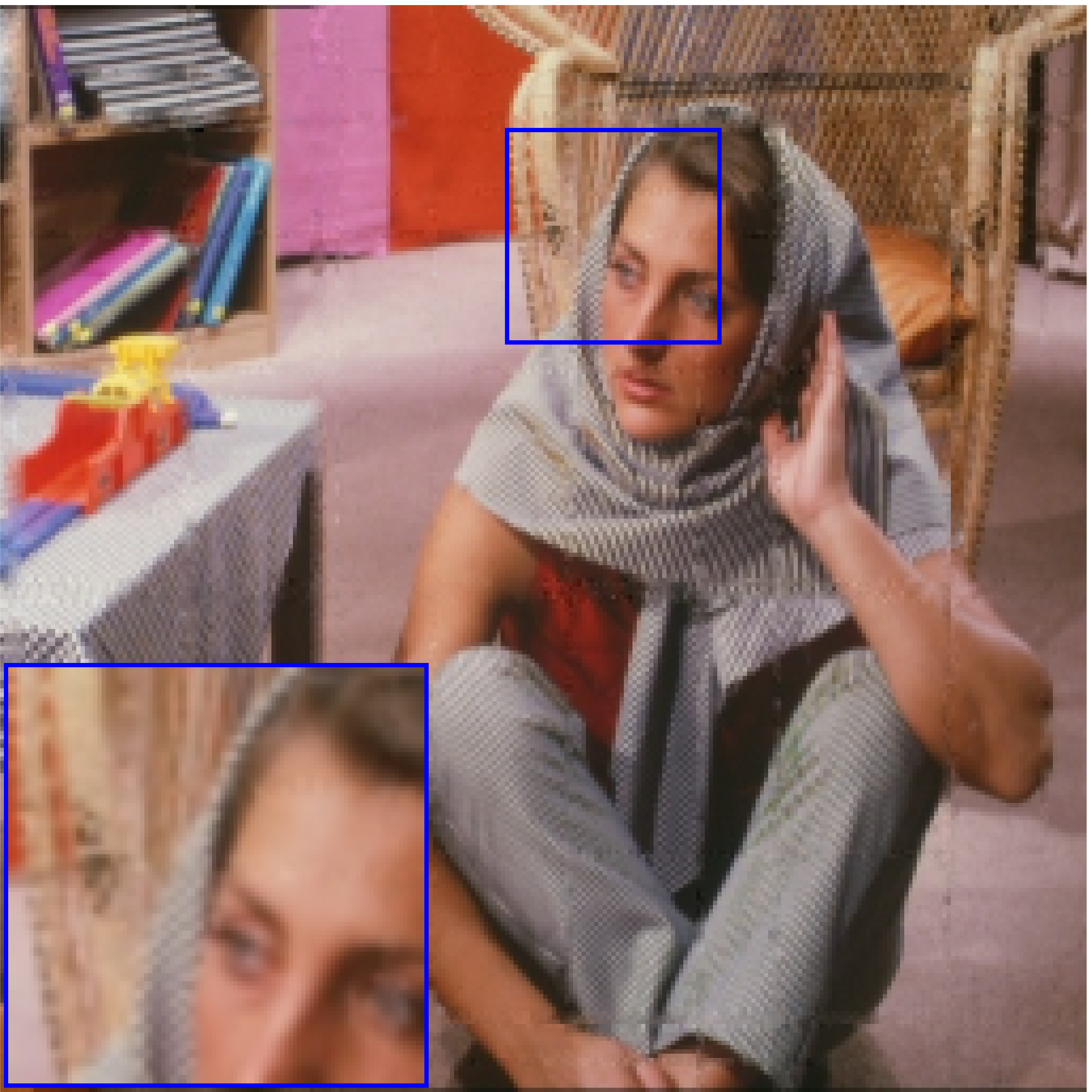}\vspace{0.01cm}\\
(f) (28.89, 0.9141) & (g) (13.64, 0.5462) & (h) (22.46, 0.8080) & (i) (29.52, 0.9228) & (j) (\textbf{30.98}, \textbf{0.9432})\\
\includegraphics[width=0.19\textwidth]{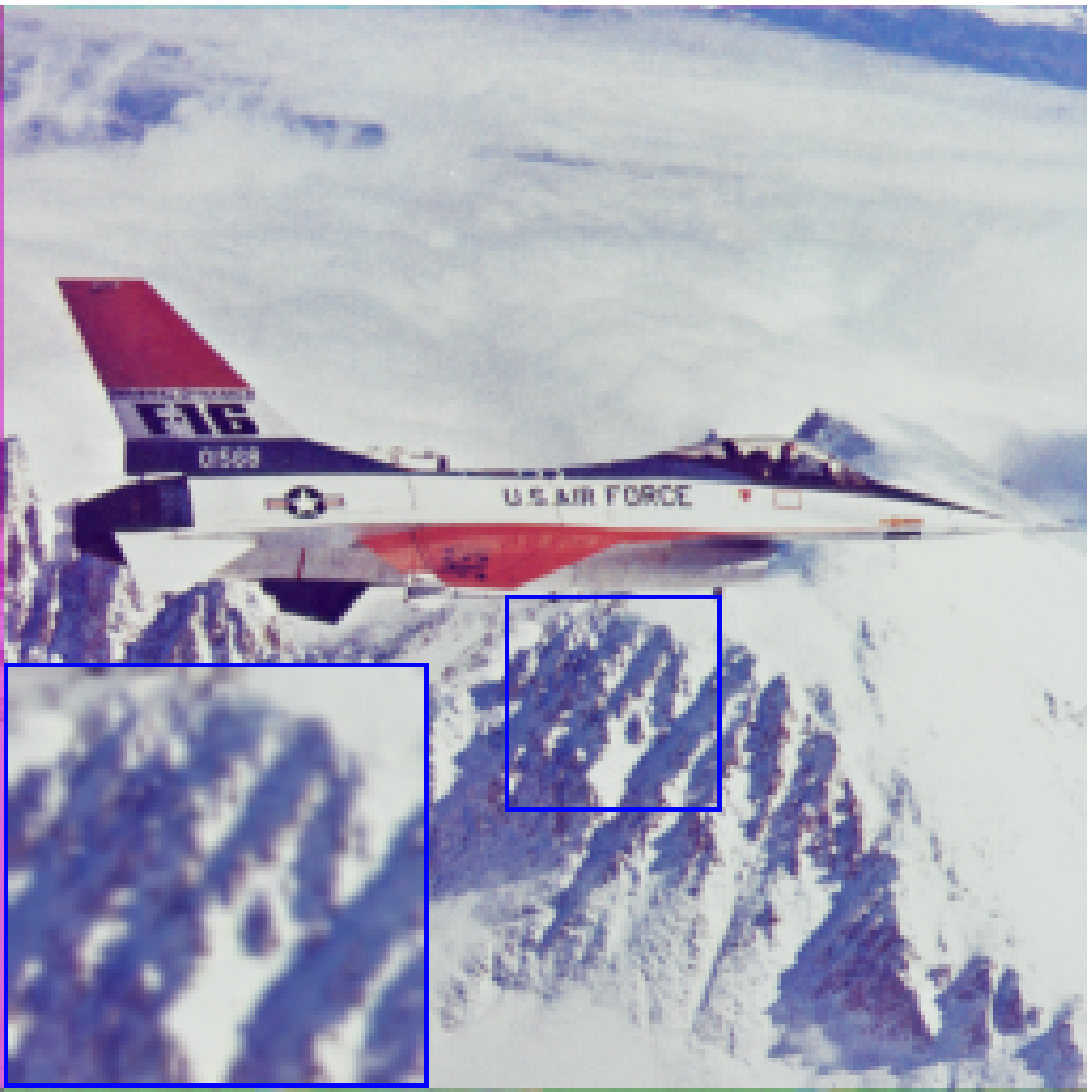}&
\includegraphics[width=0.19\textwidth]{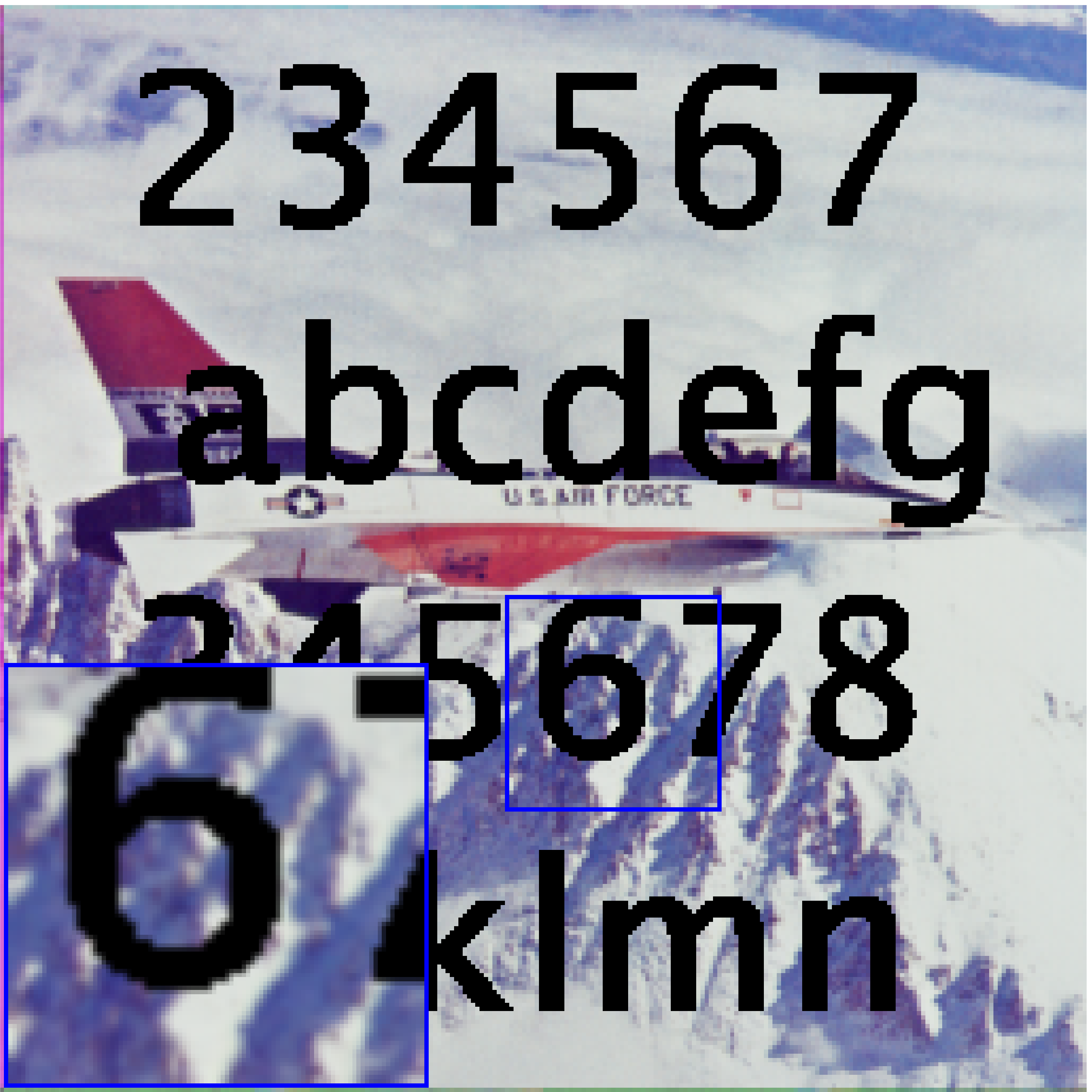}&
\includegraphics[width=0.19\textwidth]{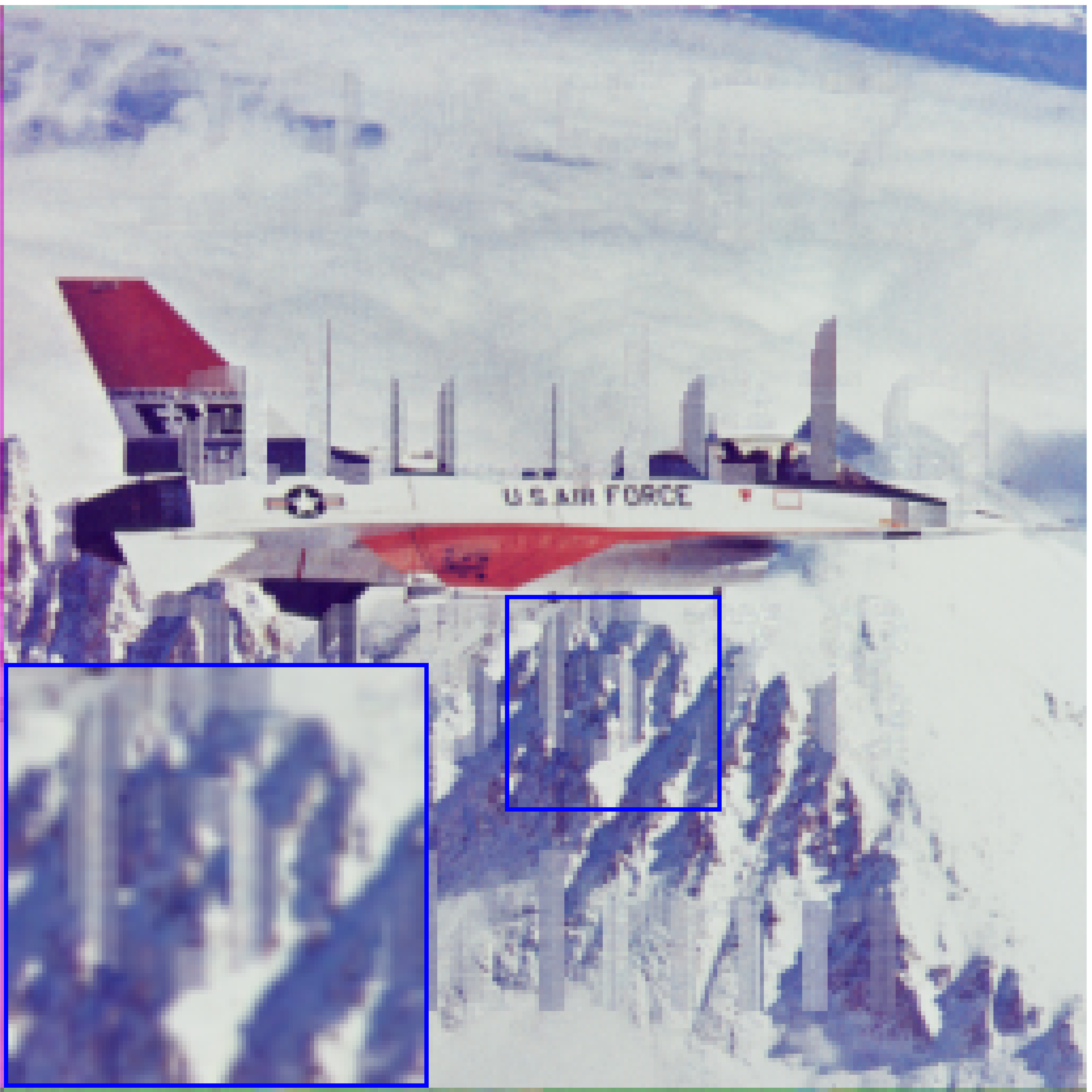}&
\includegraphics[width=0.19\textwidth]{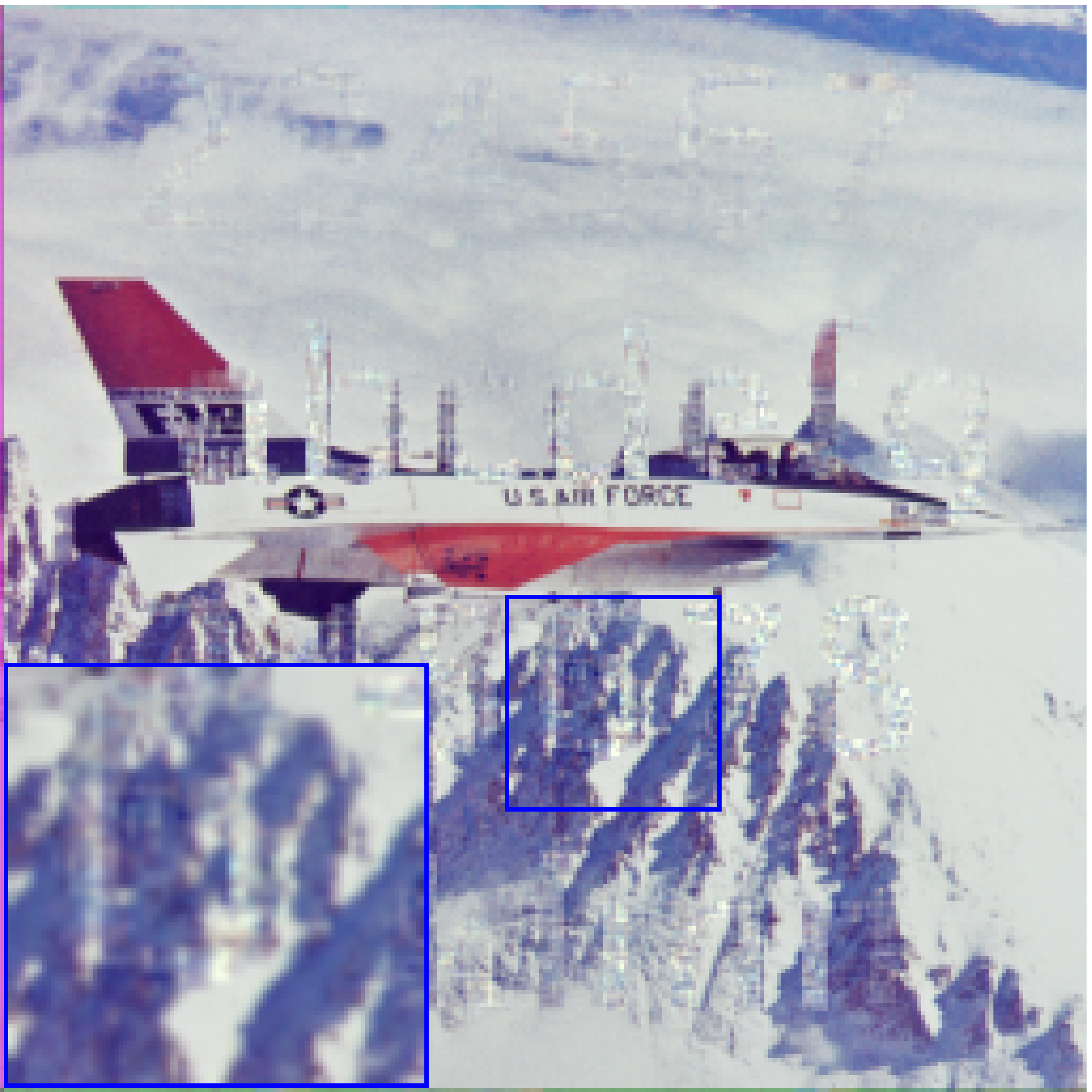}&
\includegraphics[width=0.19\textwidth]{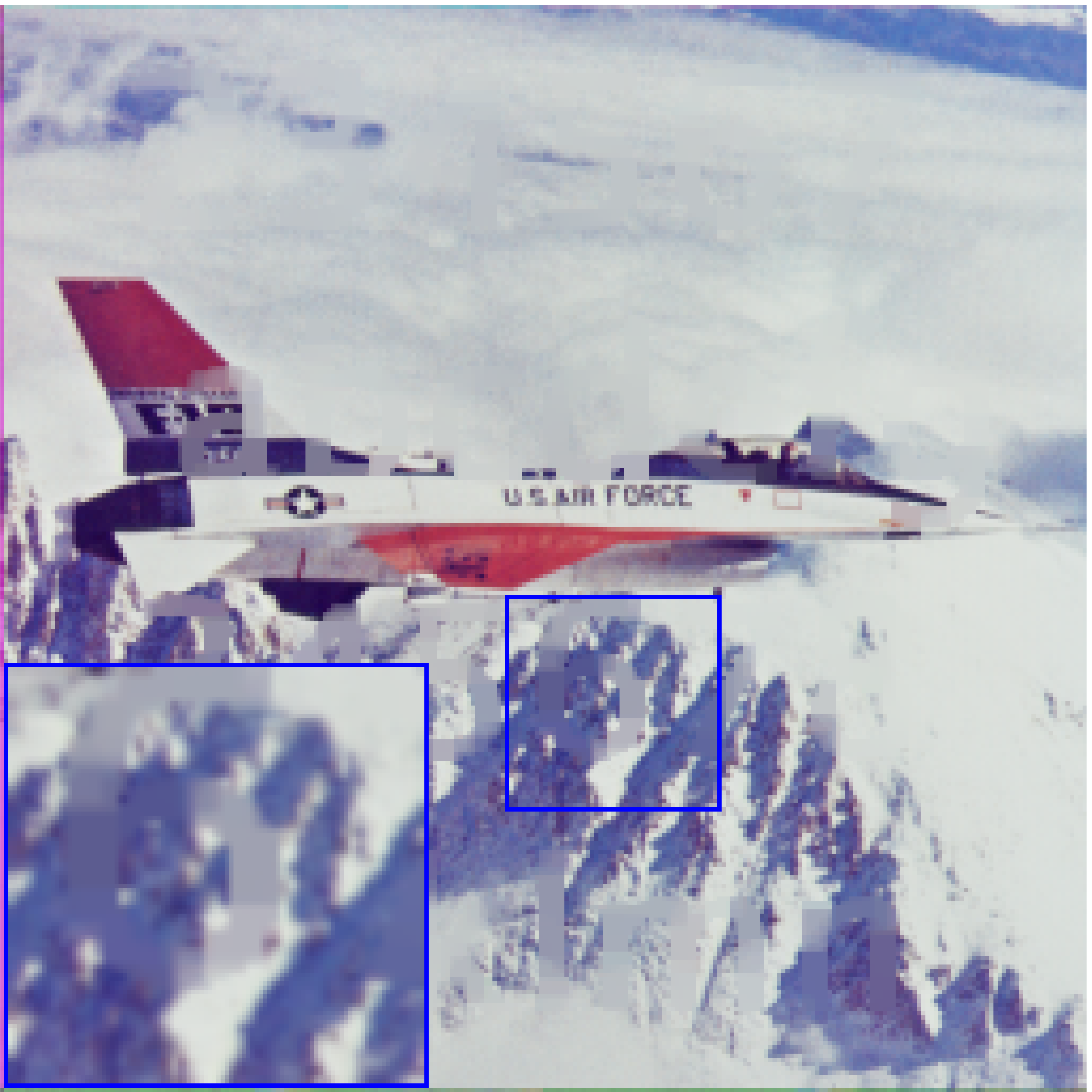}\vspace{0.01cm}\\
(a) (PSNR, SSIM) & (b) Observed & (c) (26.14, 0.9011) & (d) (25.93, 0.8763) & (e) (28.82, 0.9426)\\
\includegraphics[width=0.19\textwidth]{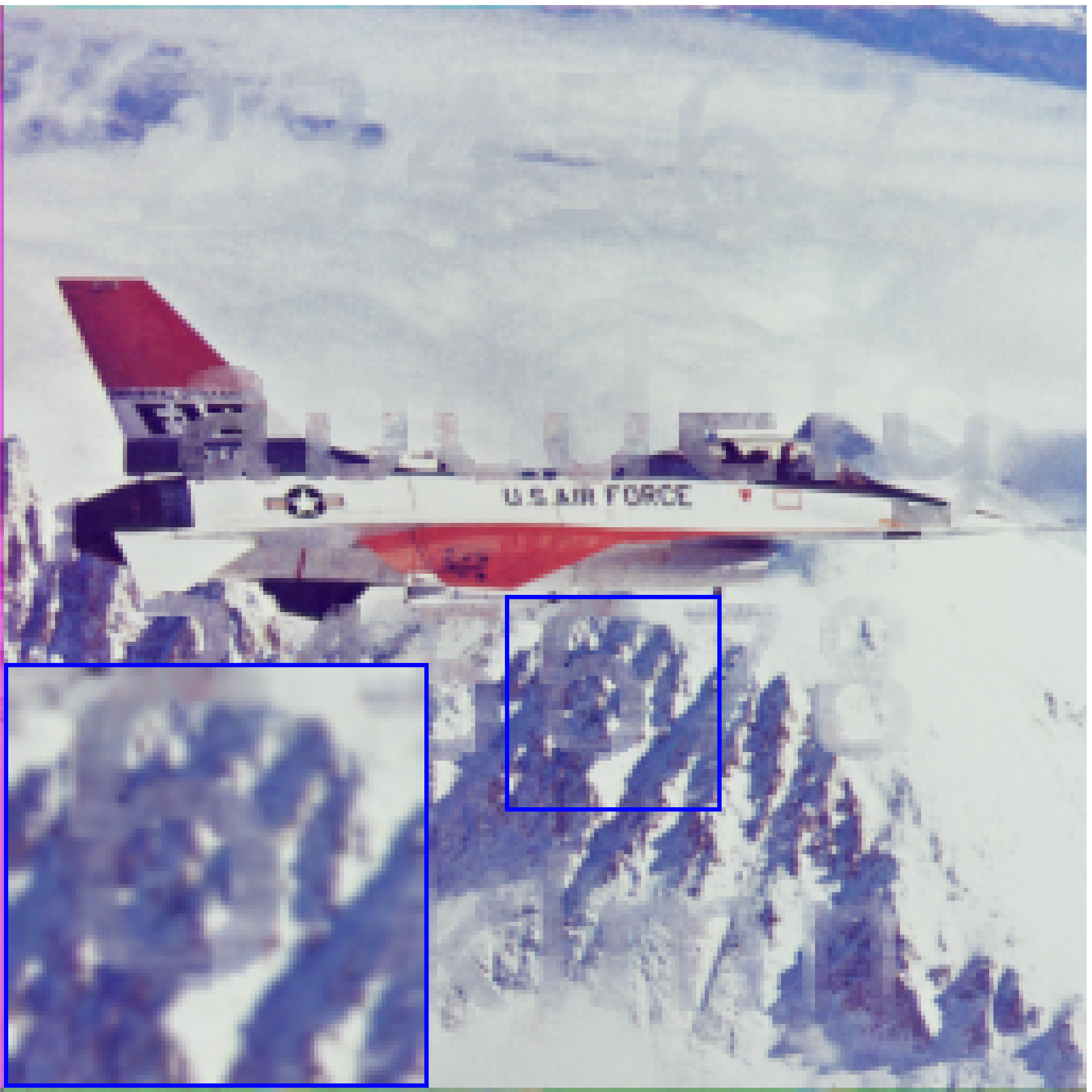}&
\includegraphics[width=0.19\textwidth]{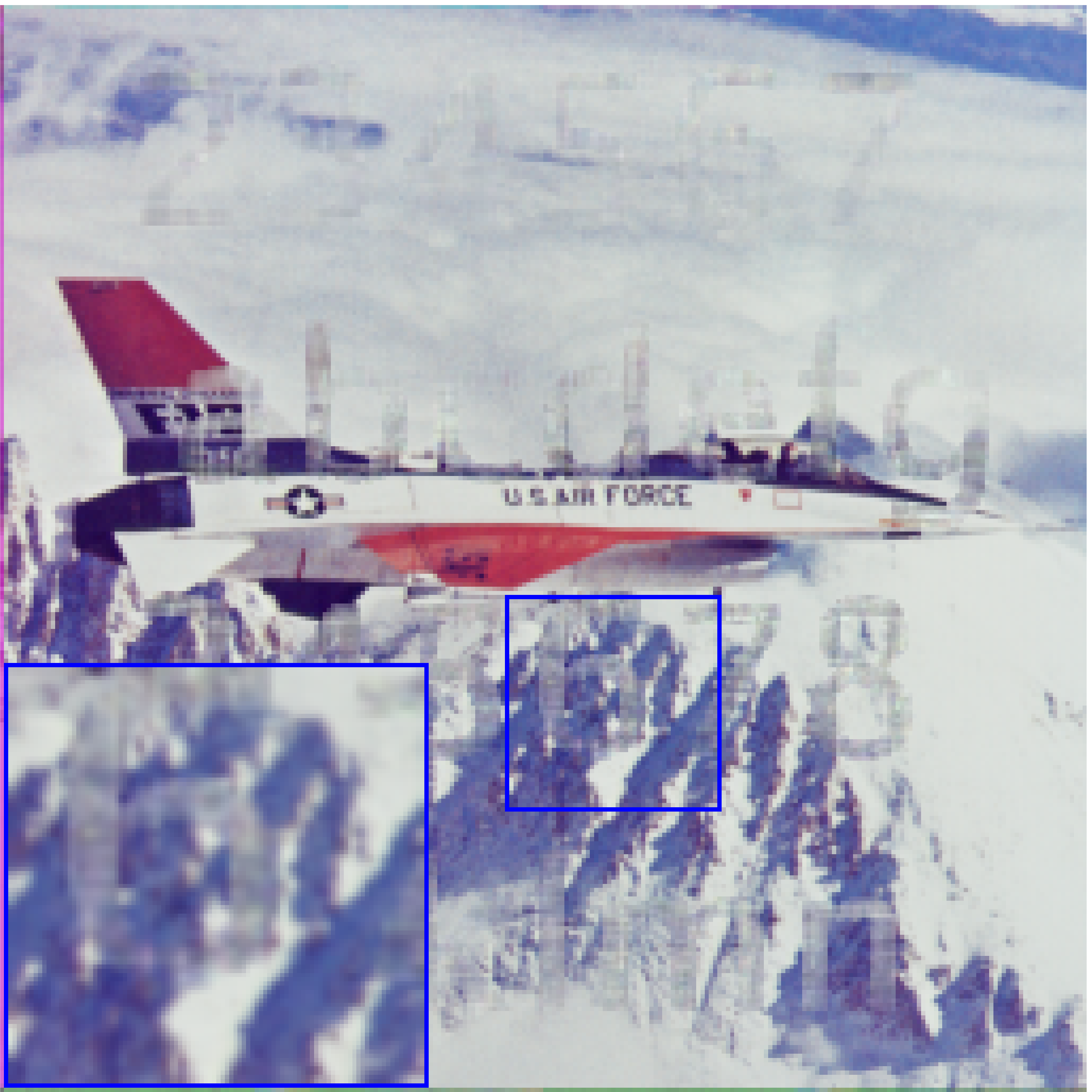}&
\includegraphics[width=0.19\textwidth]{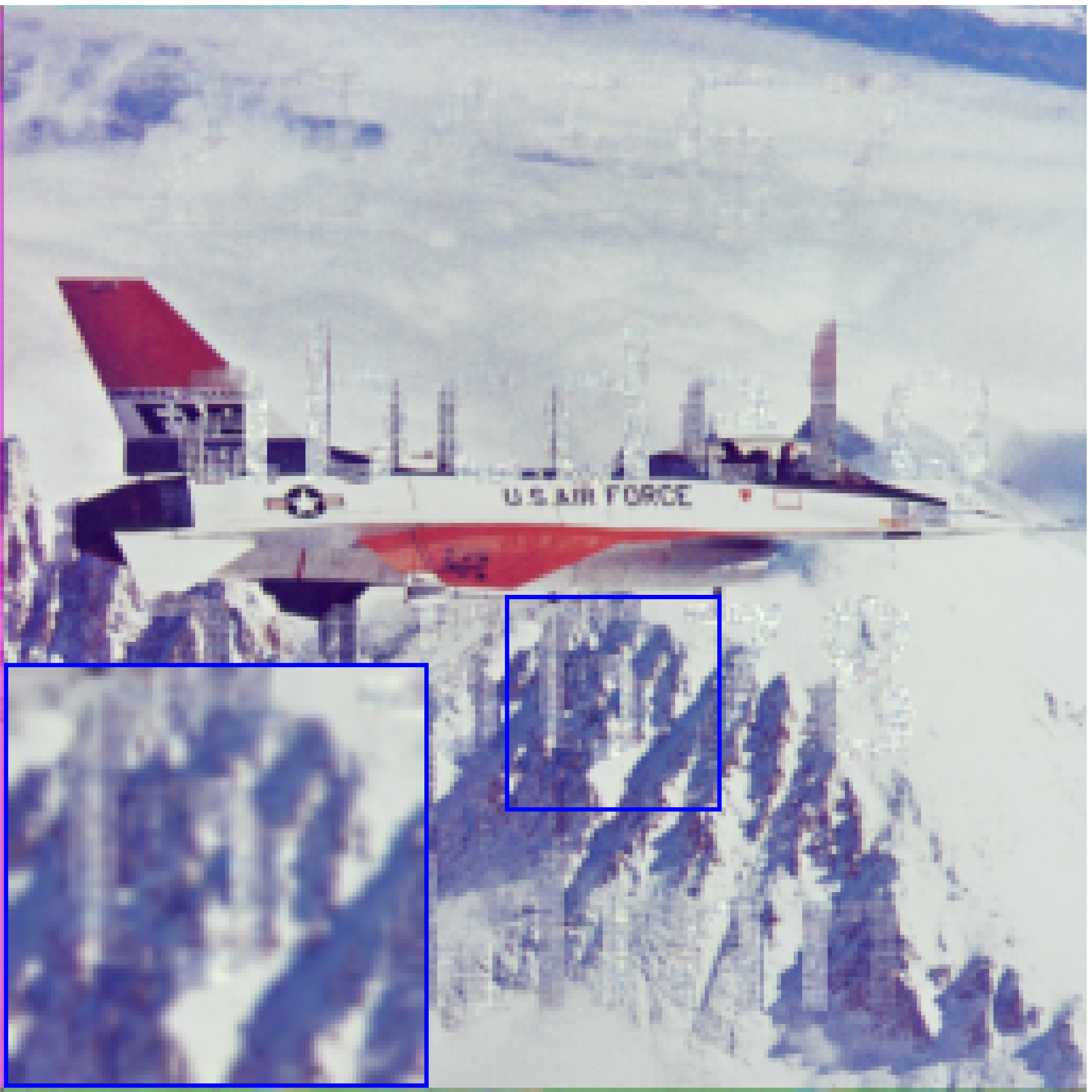}&
\includegraphics[width=0.19\textwidth]{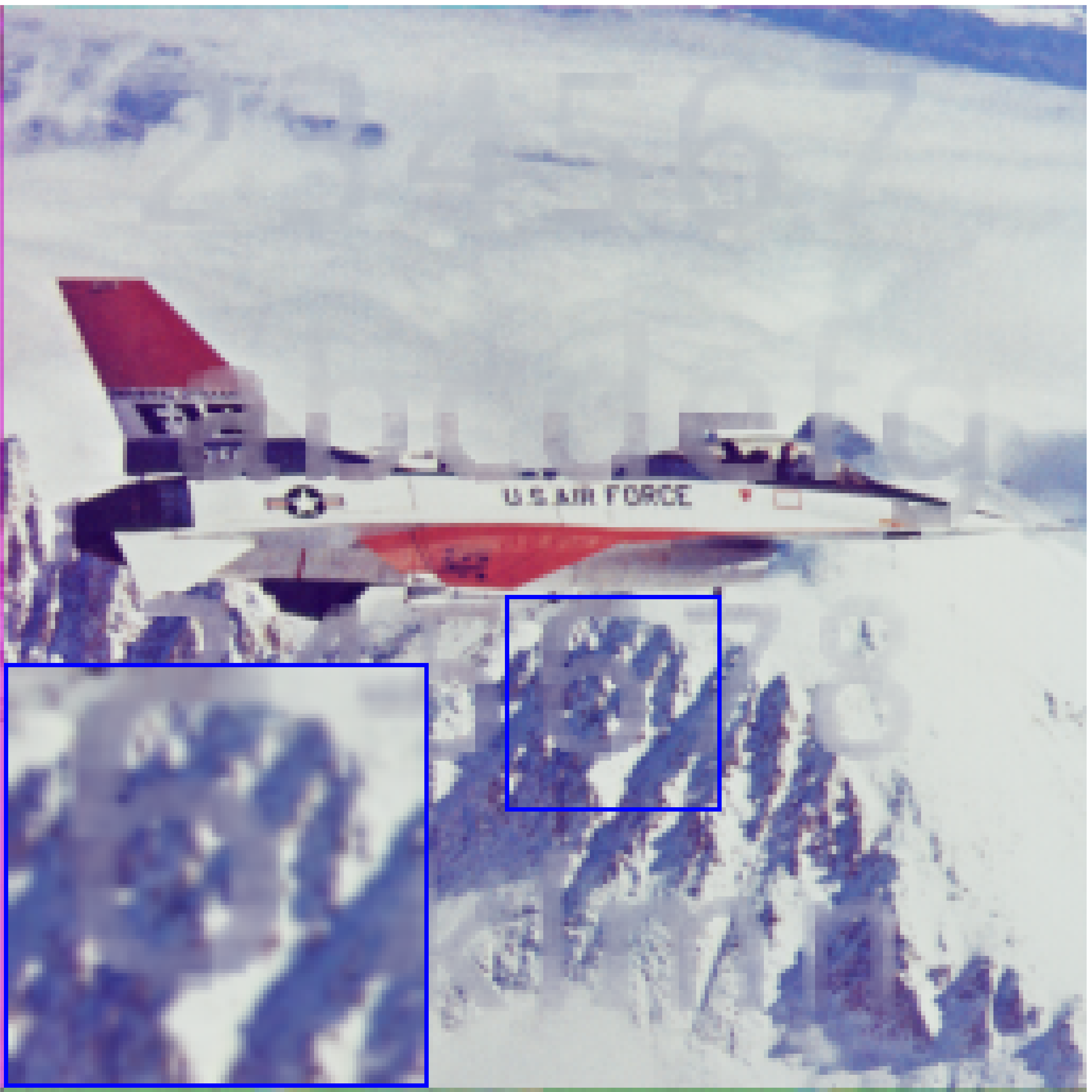}&
\includegraphics[width=0.19\textwidth]{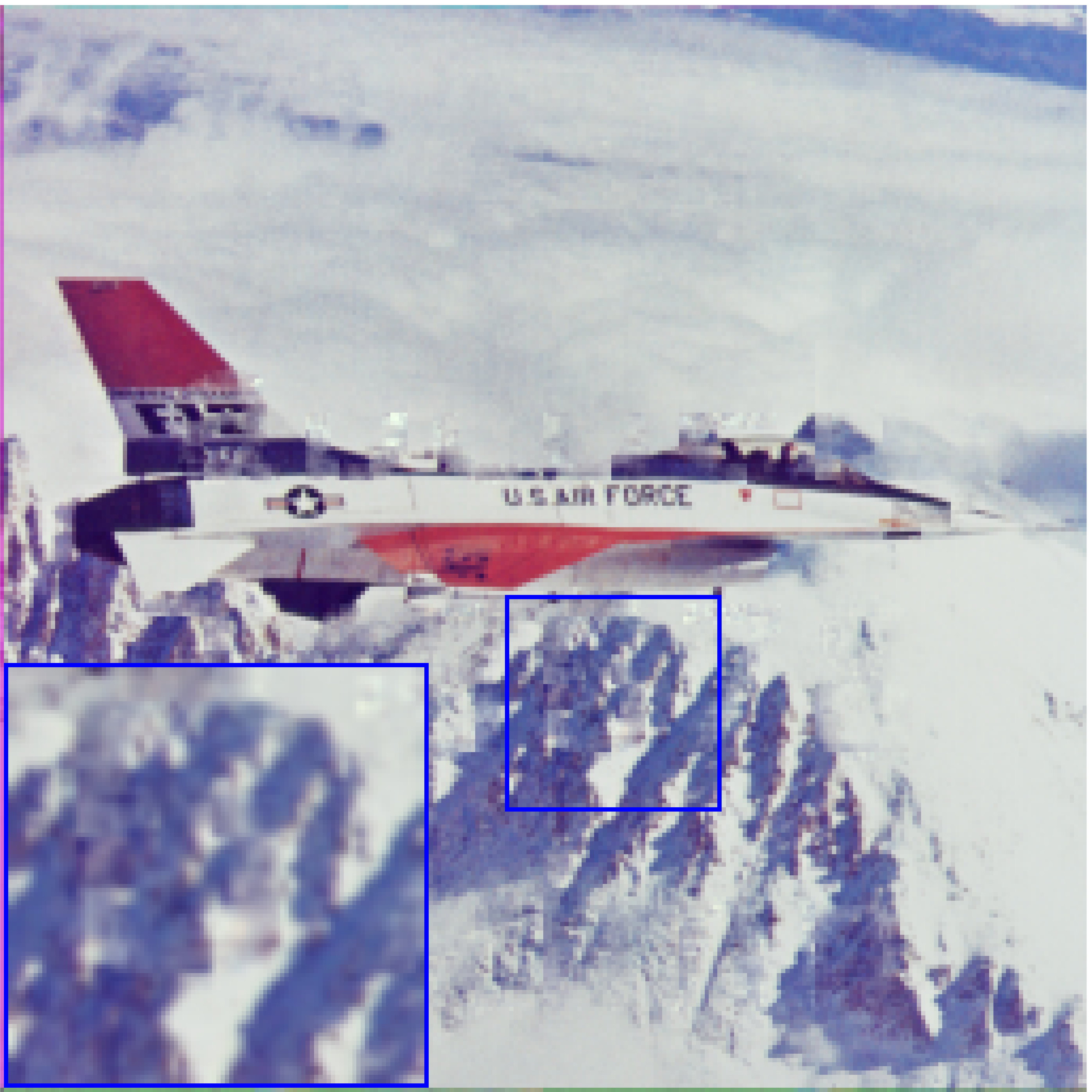}\vspace{0.01cm}\\
(f) (28.41, 0.9201) & (g) (27.40, 0.8868) & (h) (26.41, 0.8930) & (i) (28.56, 0.9215) & (j) (\textbf{29.58}, \textbf{0.9398})
\end{tabular}
\caption{\small{Recovered color images \emph{House}, \emph{Barbara}, and \emph{Airplane} for structural missing entries. (a) original data, (b) observed data, results by (c) HaLRTC, (d) NSNN, (e) LRTC-TV, (f) SiLRTC-TT, (g) tSVD, (h) KBR, (i) TRNN, and (j) LogTR.}}
  \label{fig:image_structural}
  \end{center}\vspace{-0.3cm}
\end{figure}

\begin{figure}[!t]
\scriptsize\setlength{\tabcolsep}{0.5pt}
\begin{center}
\begin{tabular}{ccccc}
\includegraphics[width=0.19\textwidth]{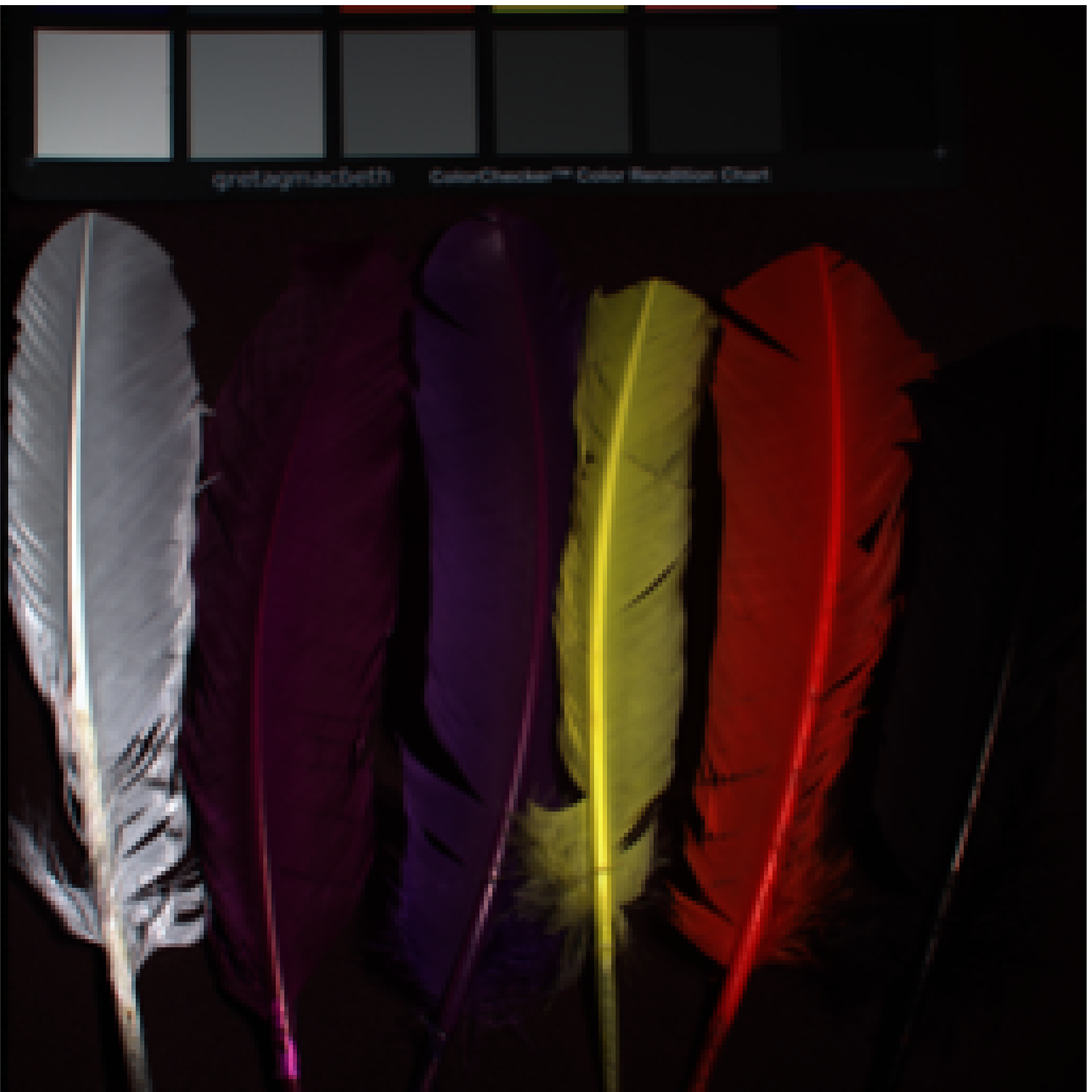}&
\includegraphics[width=0.19\textwidth]{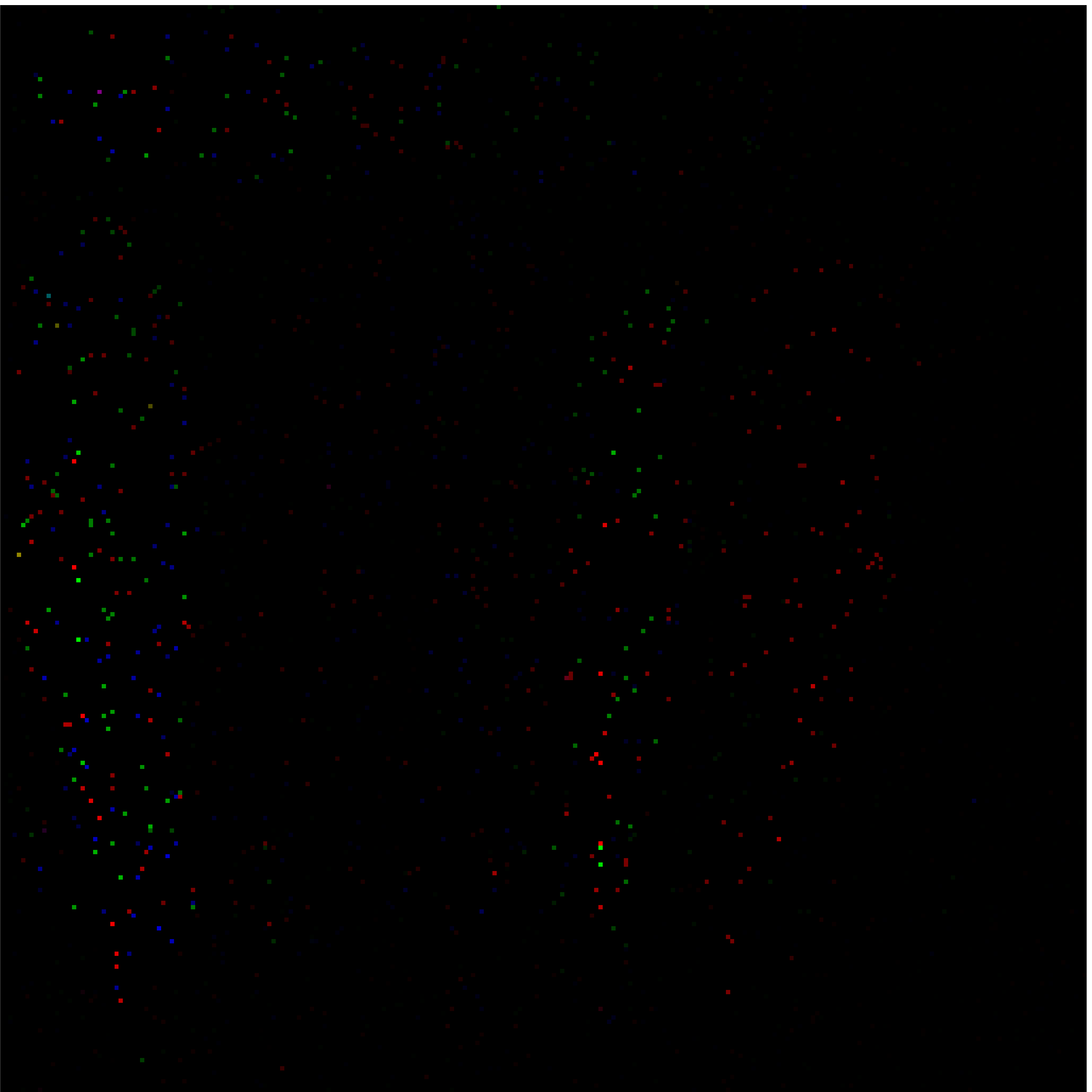}&
\includegraphics[width=0.19\textwidth]{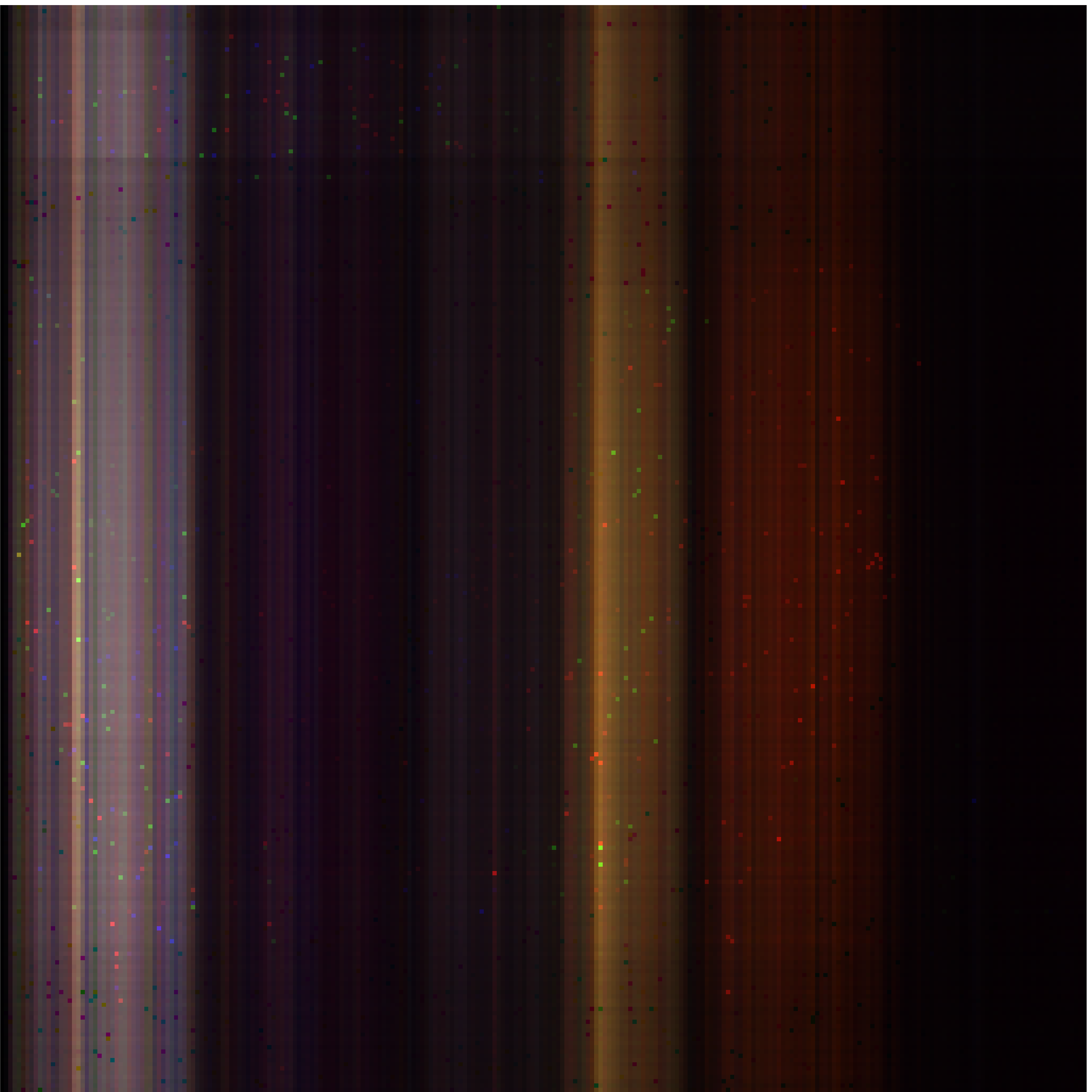}&
\includegraphics[width=0.19\textwidth]{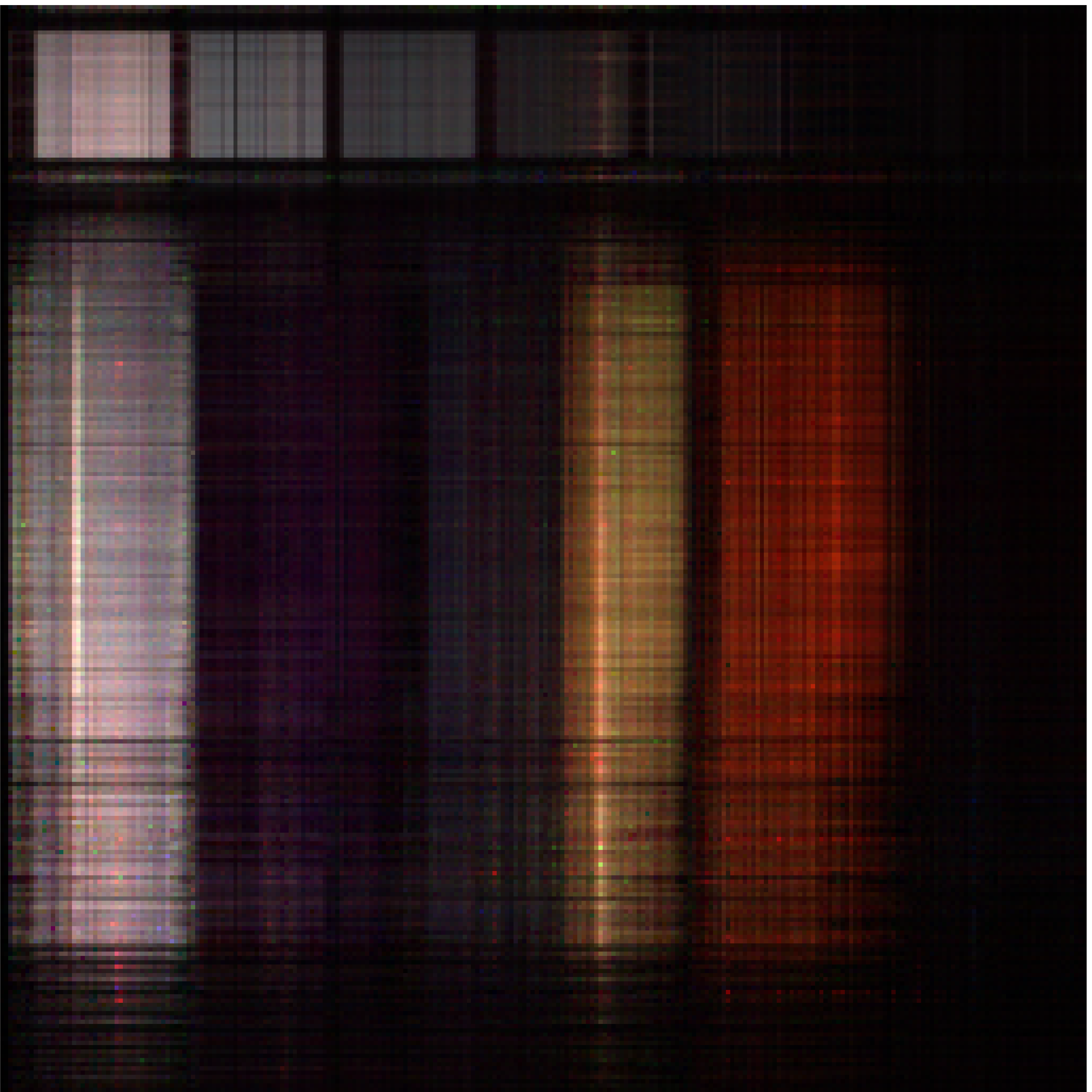}&
\includegraphics[width=0.19\textwidth]{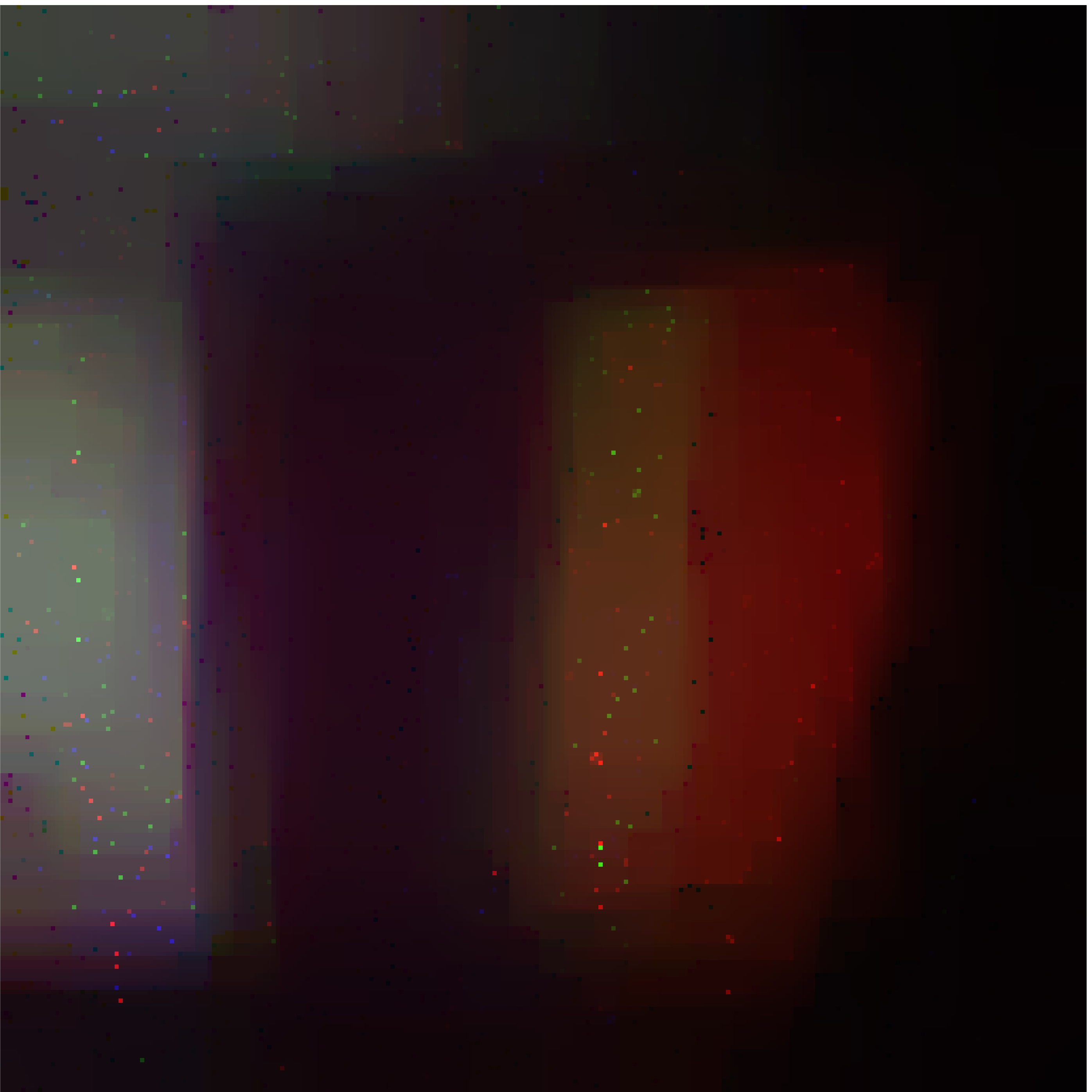}\vspace{0.01cm}\\
(a) Original & (b) Observed & (c) HaLRTC & (d) NSNN & (e) LRTC-TV\\
\includegraphics[width=0.19\textwidth]{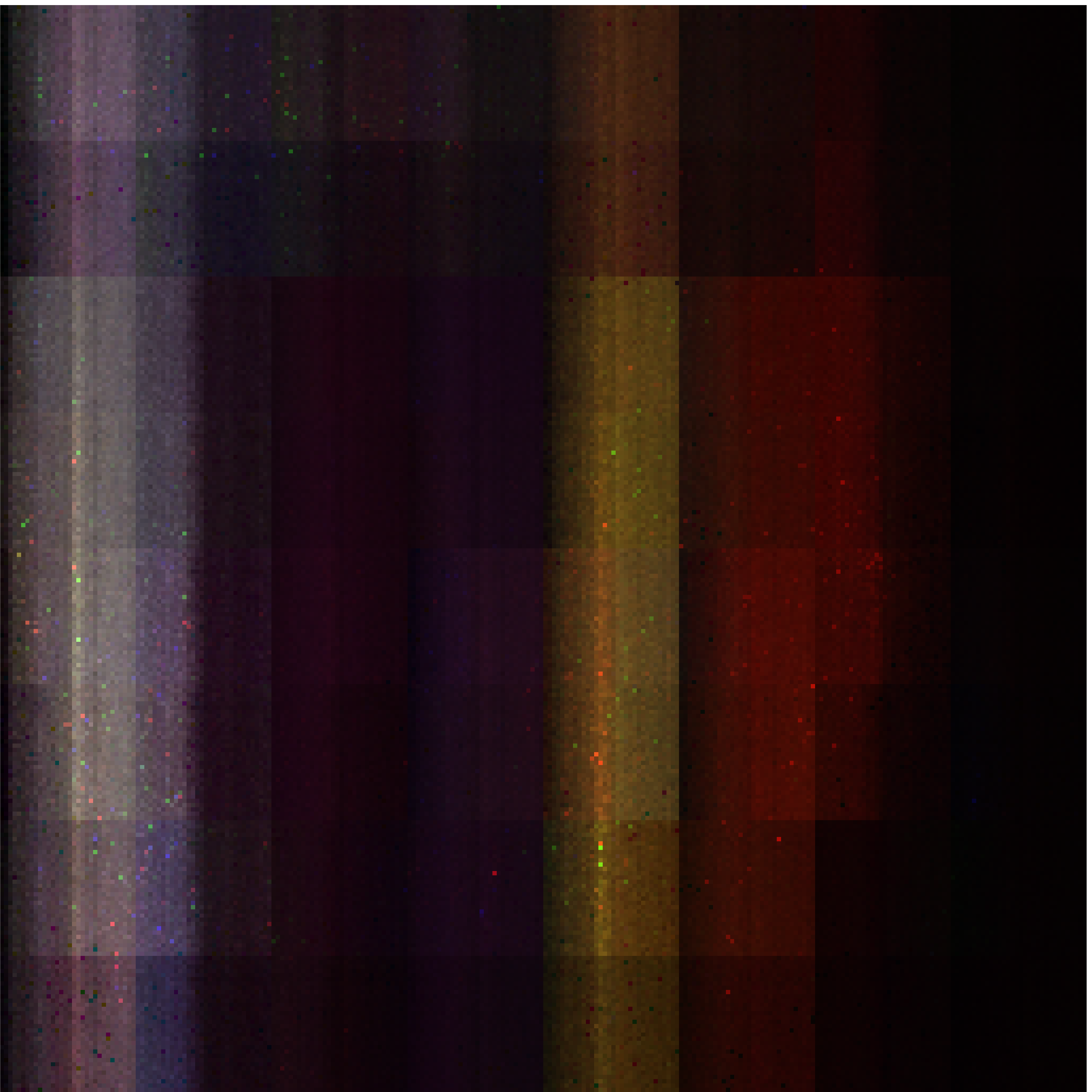}&
\includegraphics[width=0.19\textwidth]{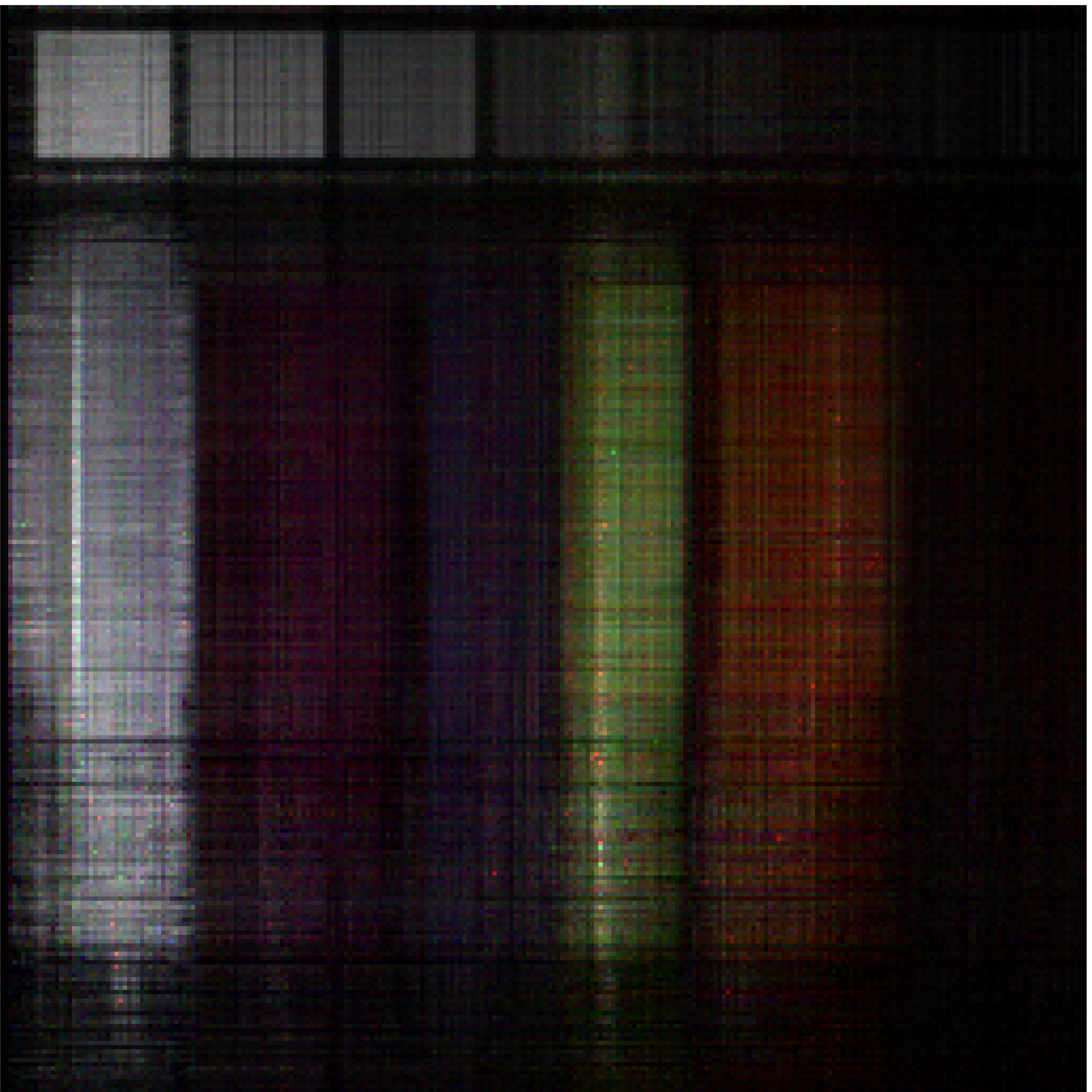}&
\includegraphics[width=0.19\textwidth]{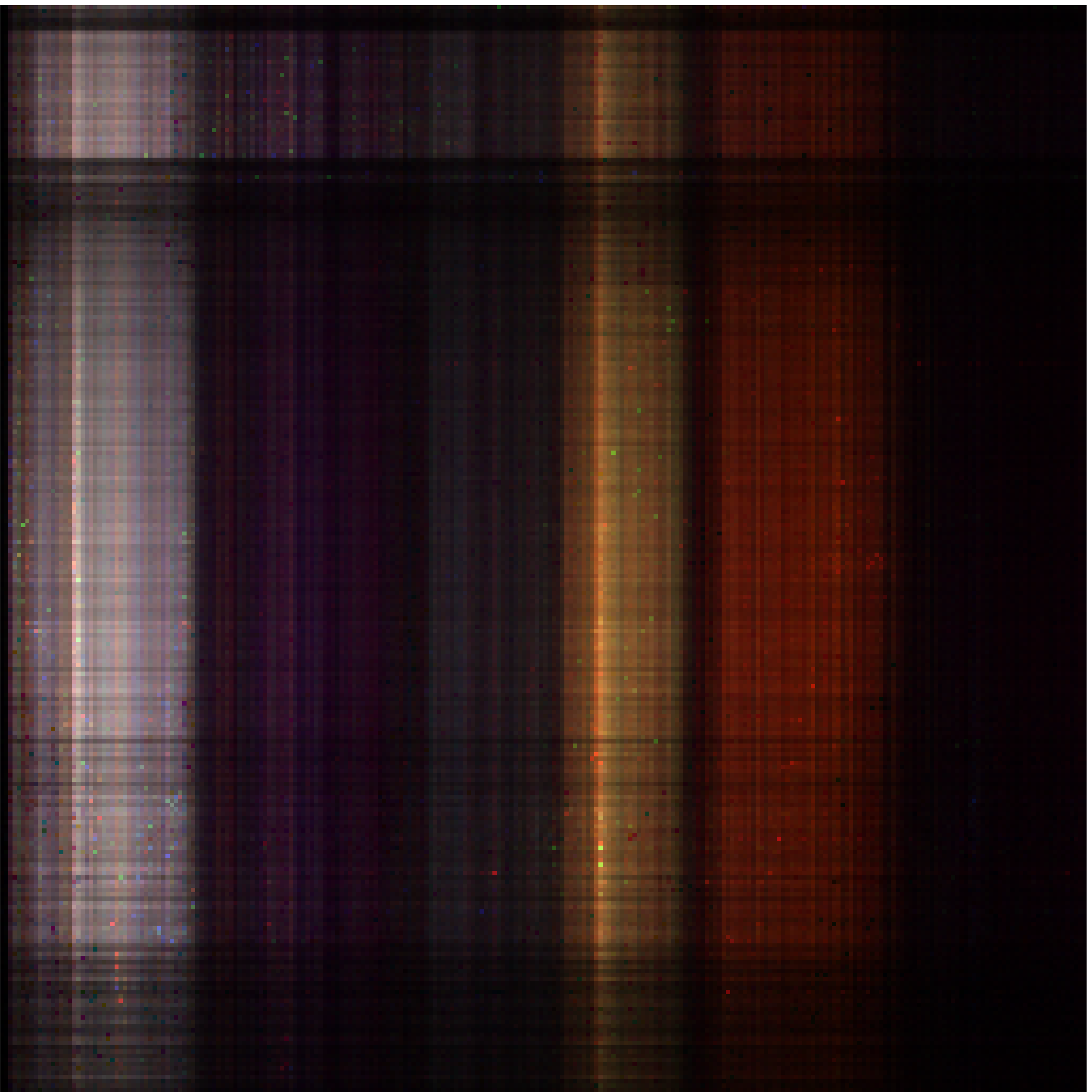}&
\includegraphics[width=0.19\textwidth]{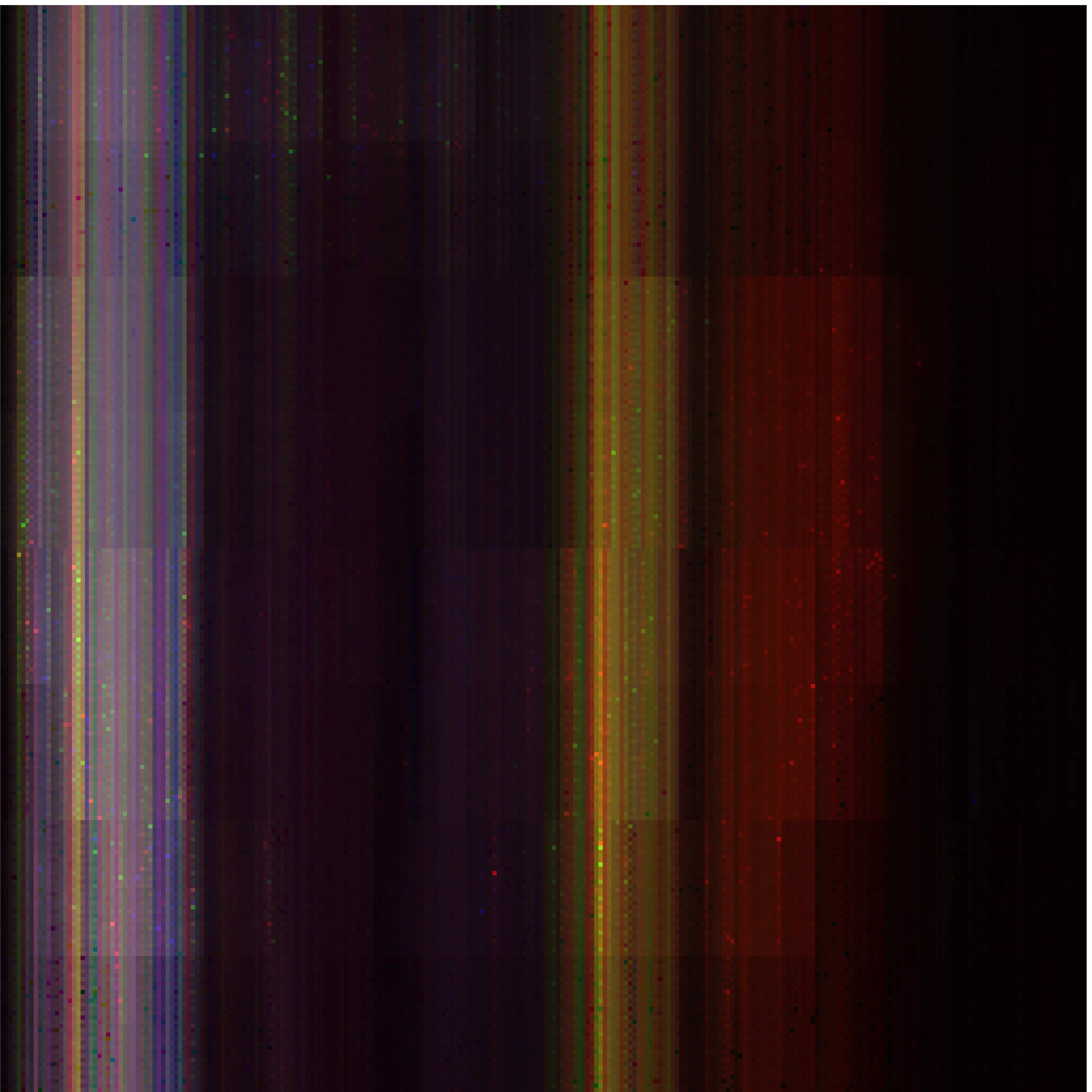}&
\includegraphics[width=0.19\textwidth]{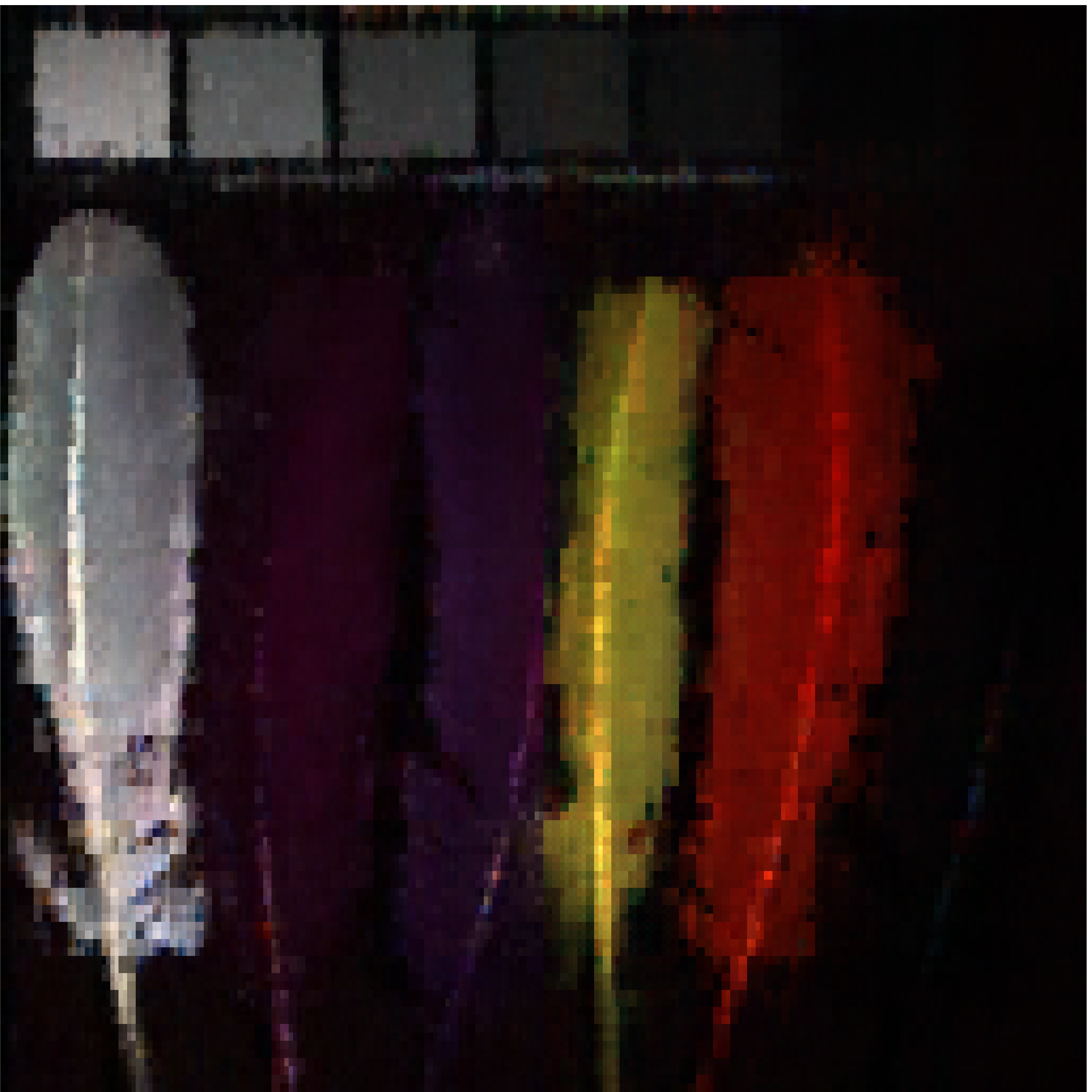}\vspace{0.01cm}\\
(f) SiLRTC-TT & (g) tSVD & (h) KBR & (i) TRNN & (j) LogTR\\
\includegraphics[width=0.19\textwidth]{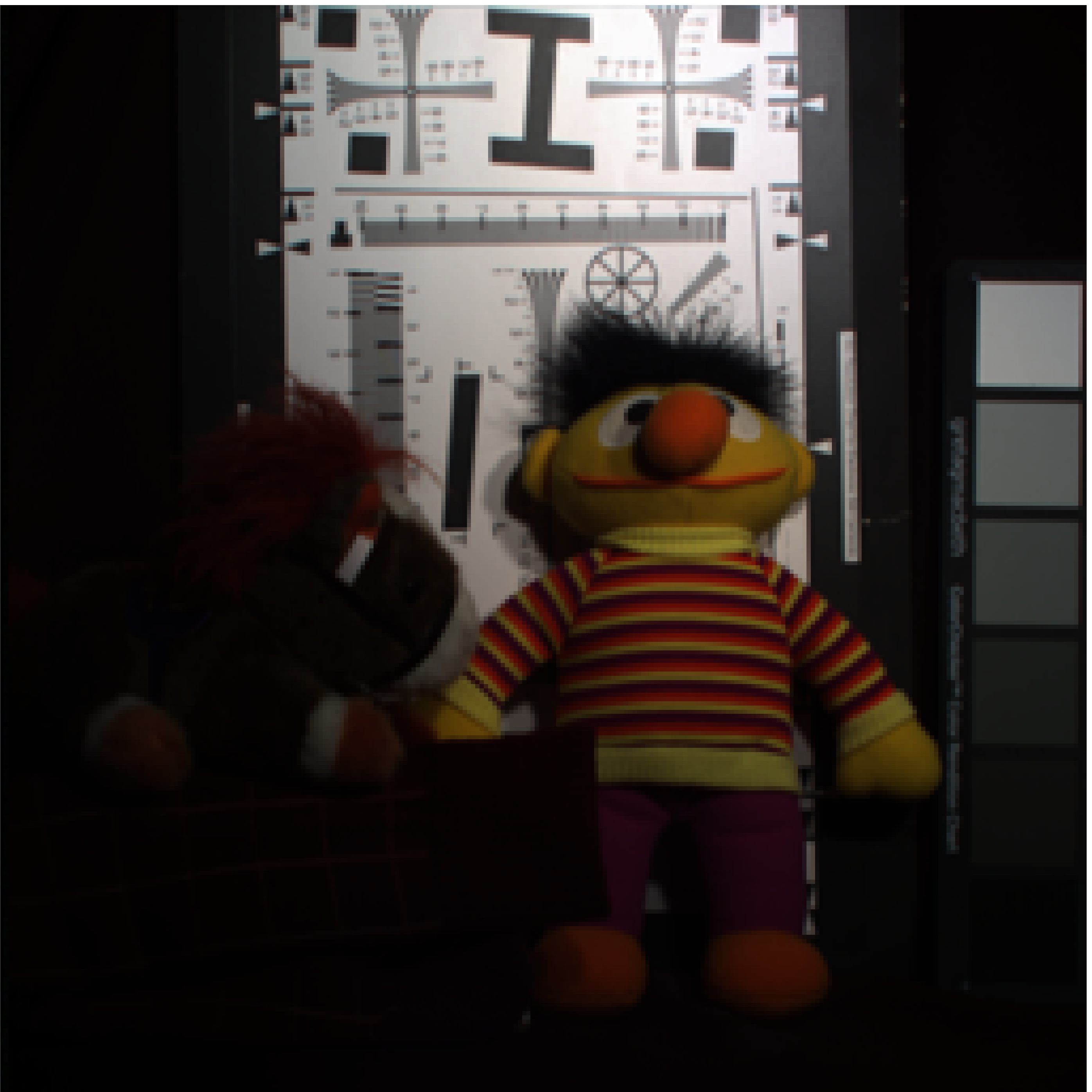}&
\includegraphics[width=0.19\textwidth]{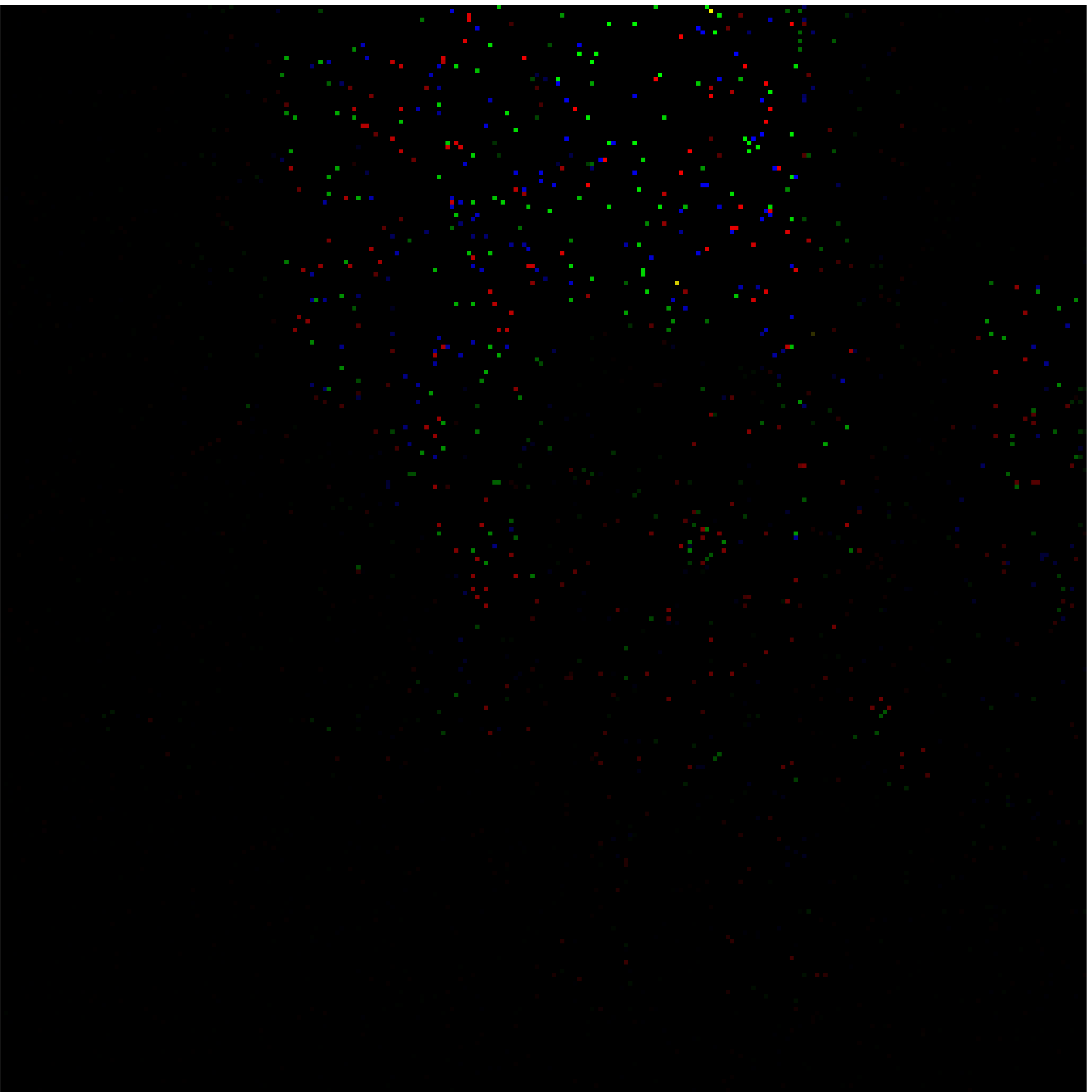}&
\includegraphics[width=0.19\textwidth]{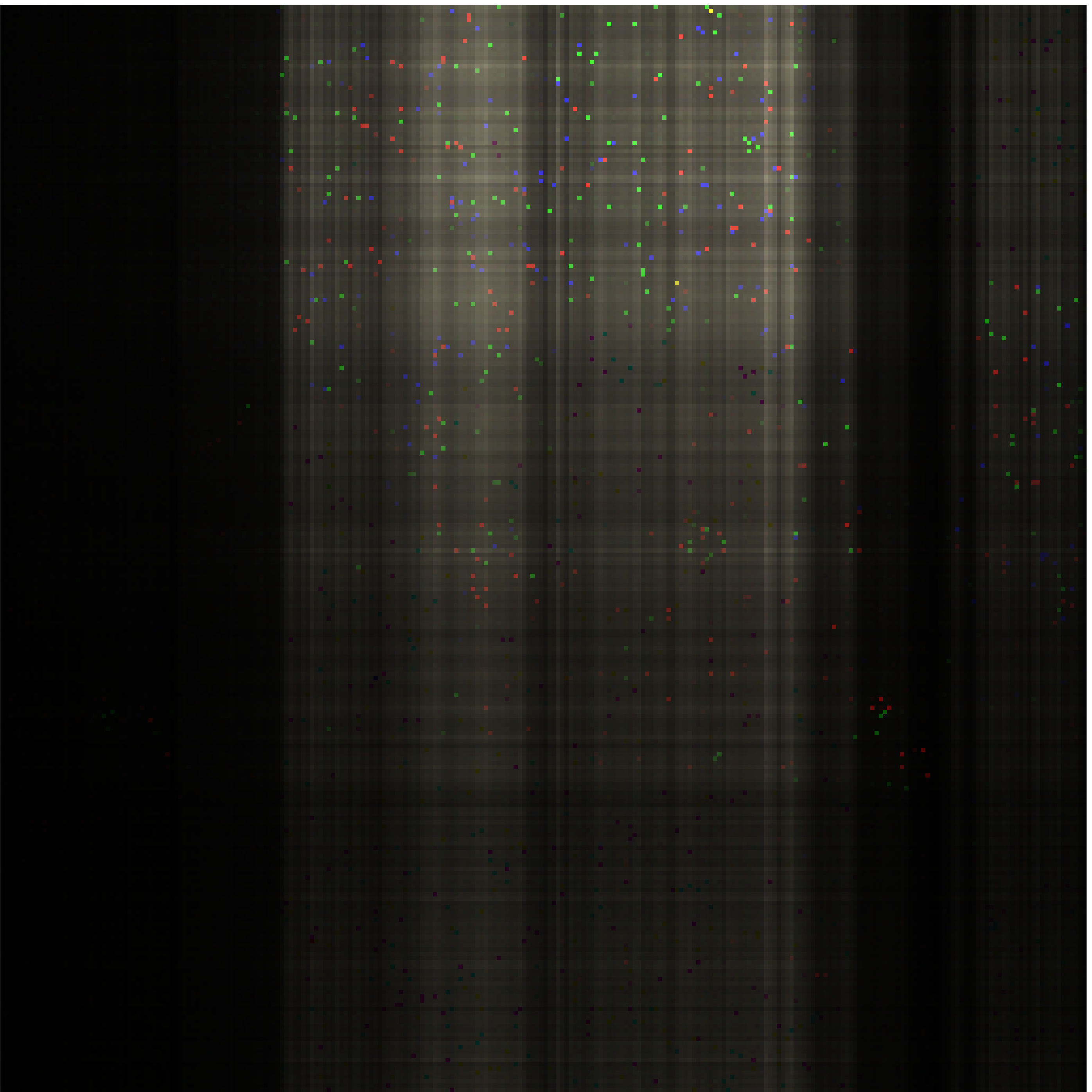}&
\includegraphics[width=0.19\textwidth]{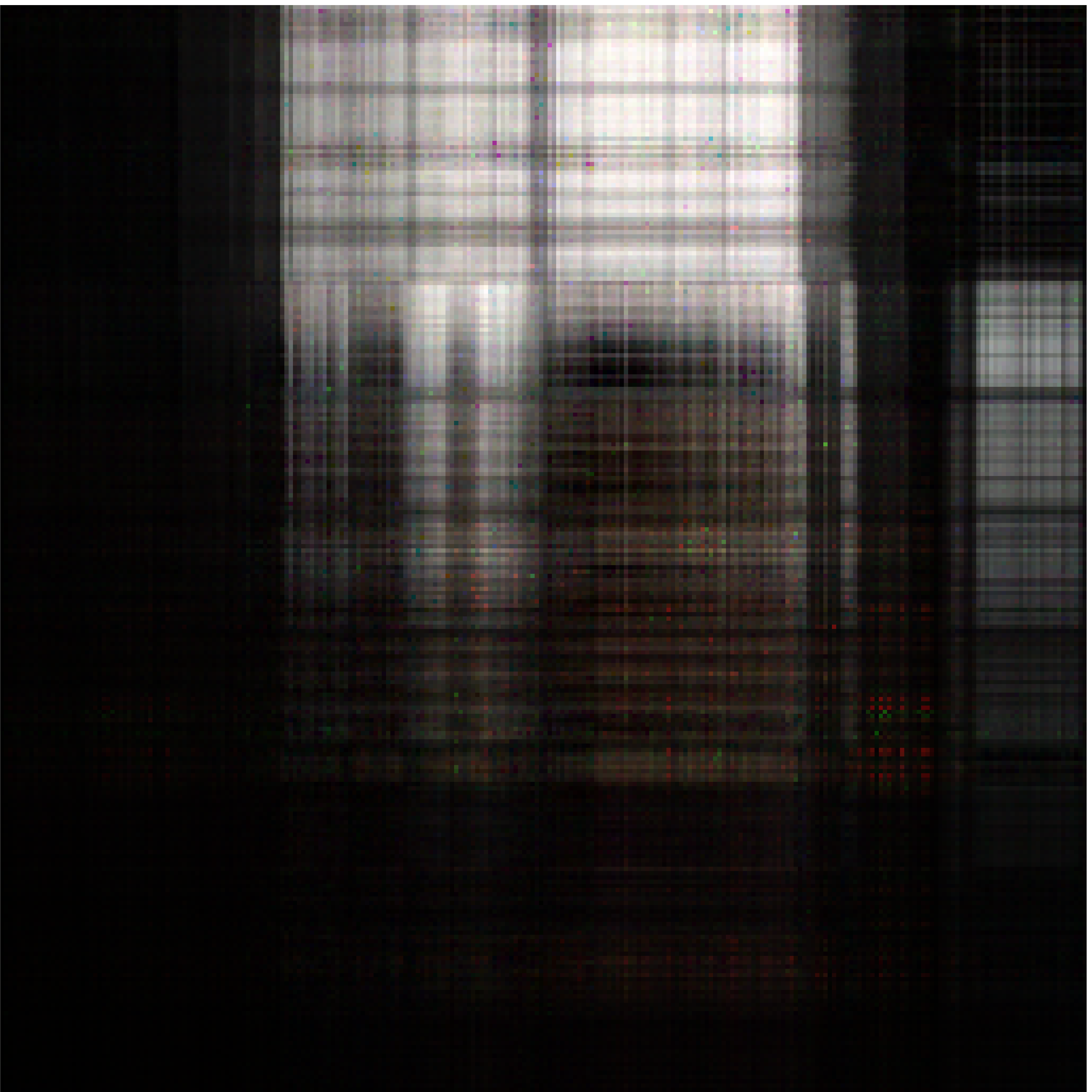}&
\includegraphics[width=0.19\textwidth]{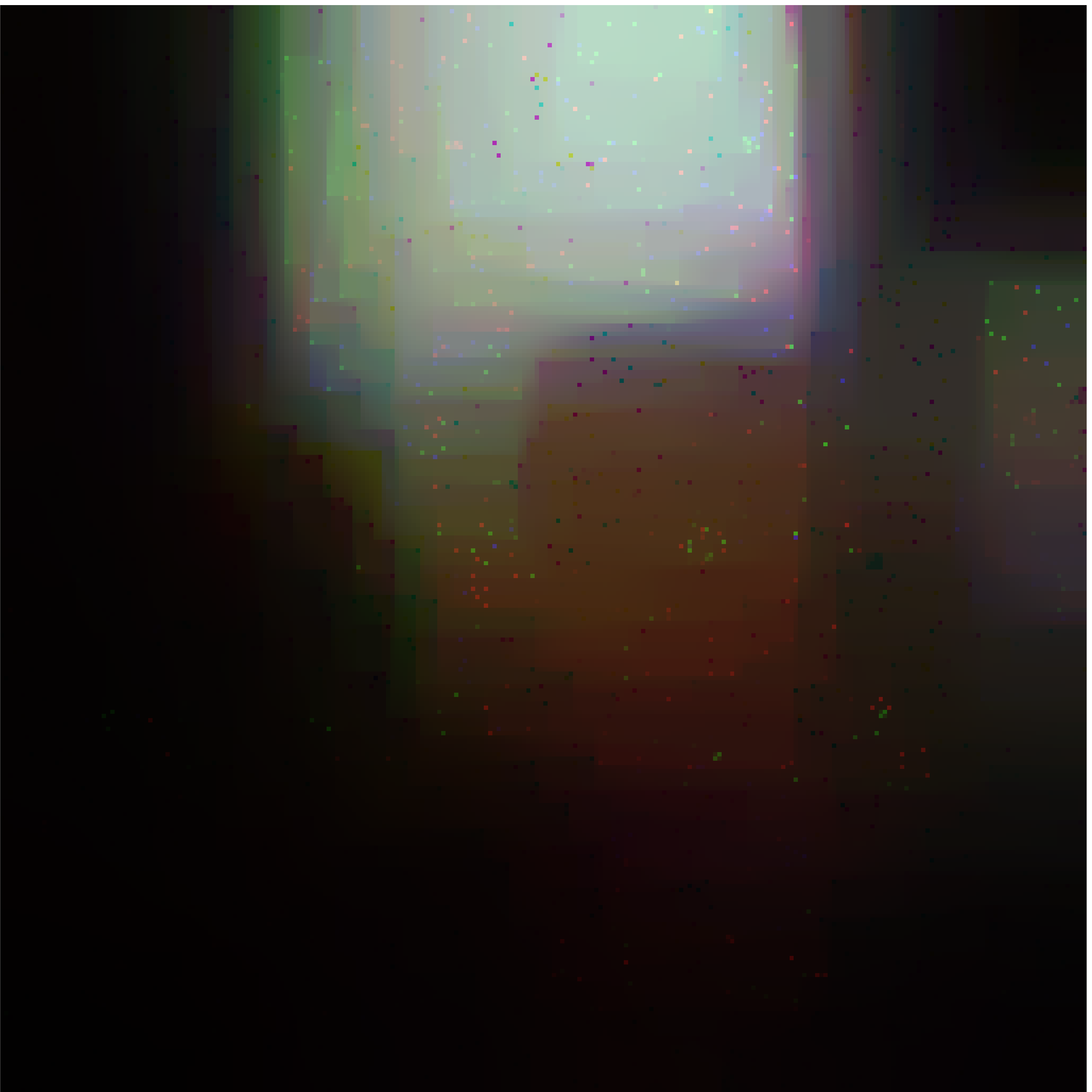}\vspace{0.01cm}\\
(a) Original& (b) Observed & (c) HaLRTC & (d) NSNN & (e) LRTC-TV\\
\includegraphics[width=0.19\textwidth]{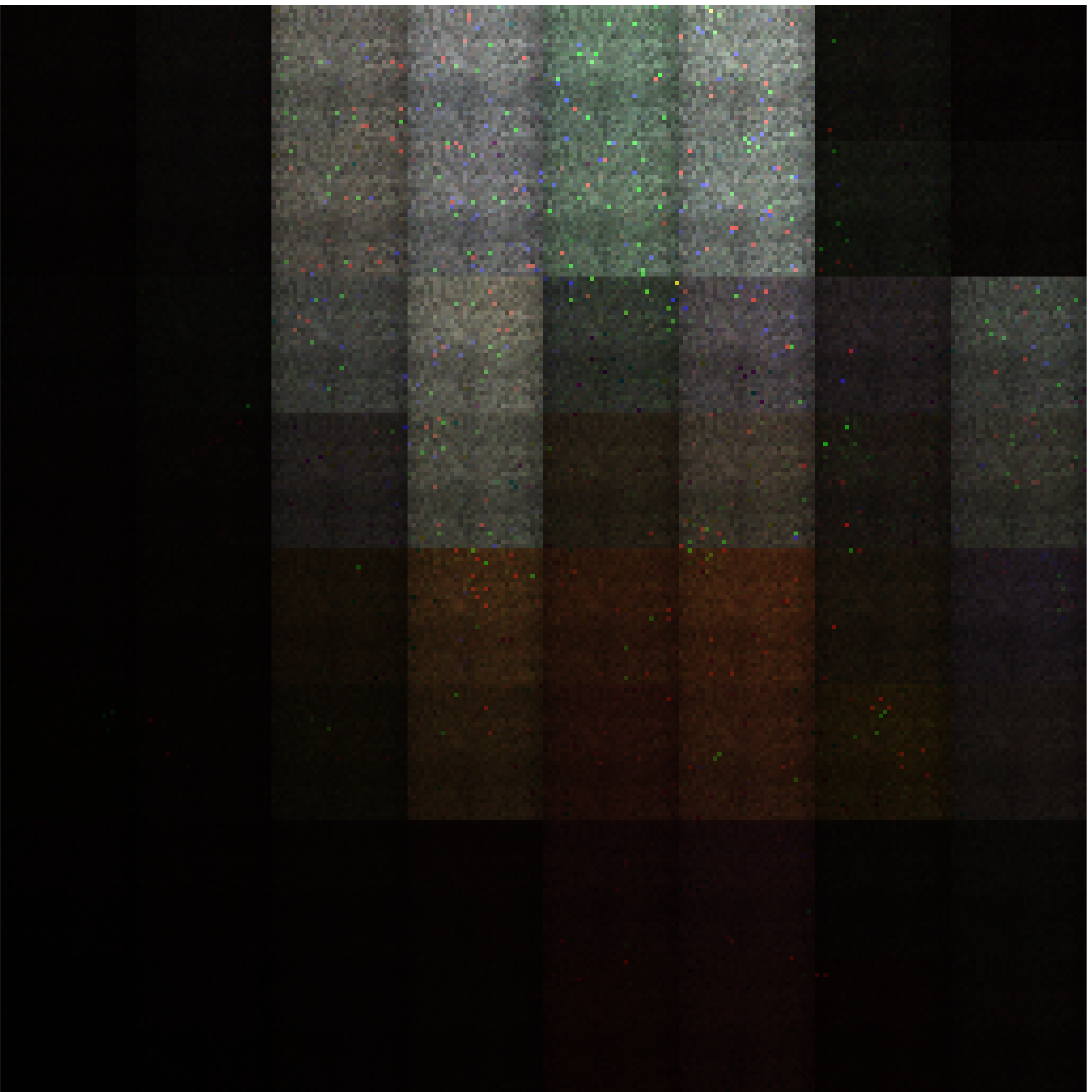}&
\includegraphics[width=0.19\textwidth]{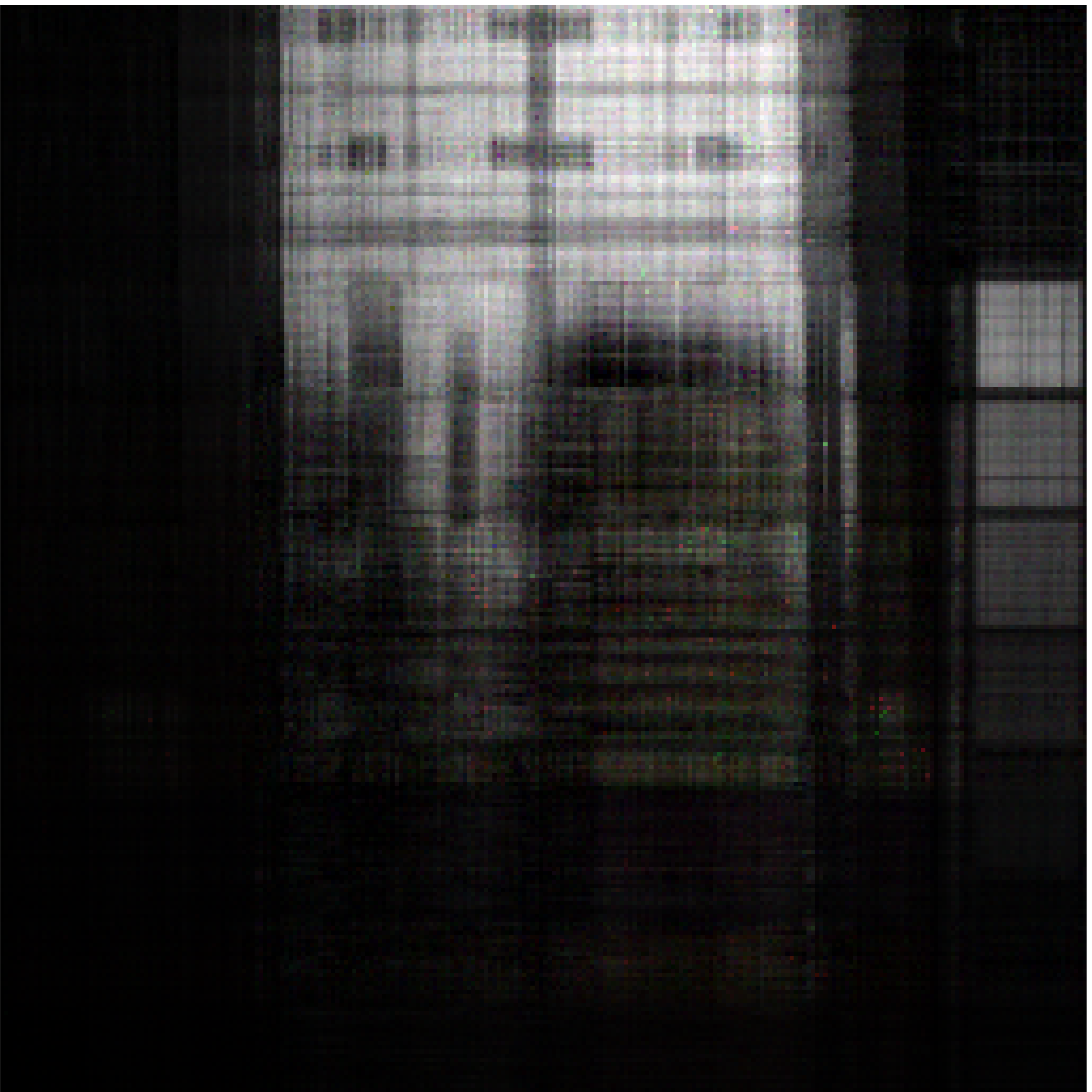}&
\includegraphics[width=0.19\textwidth]{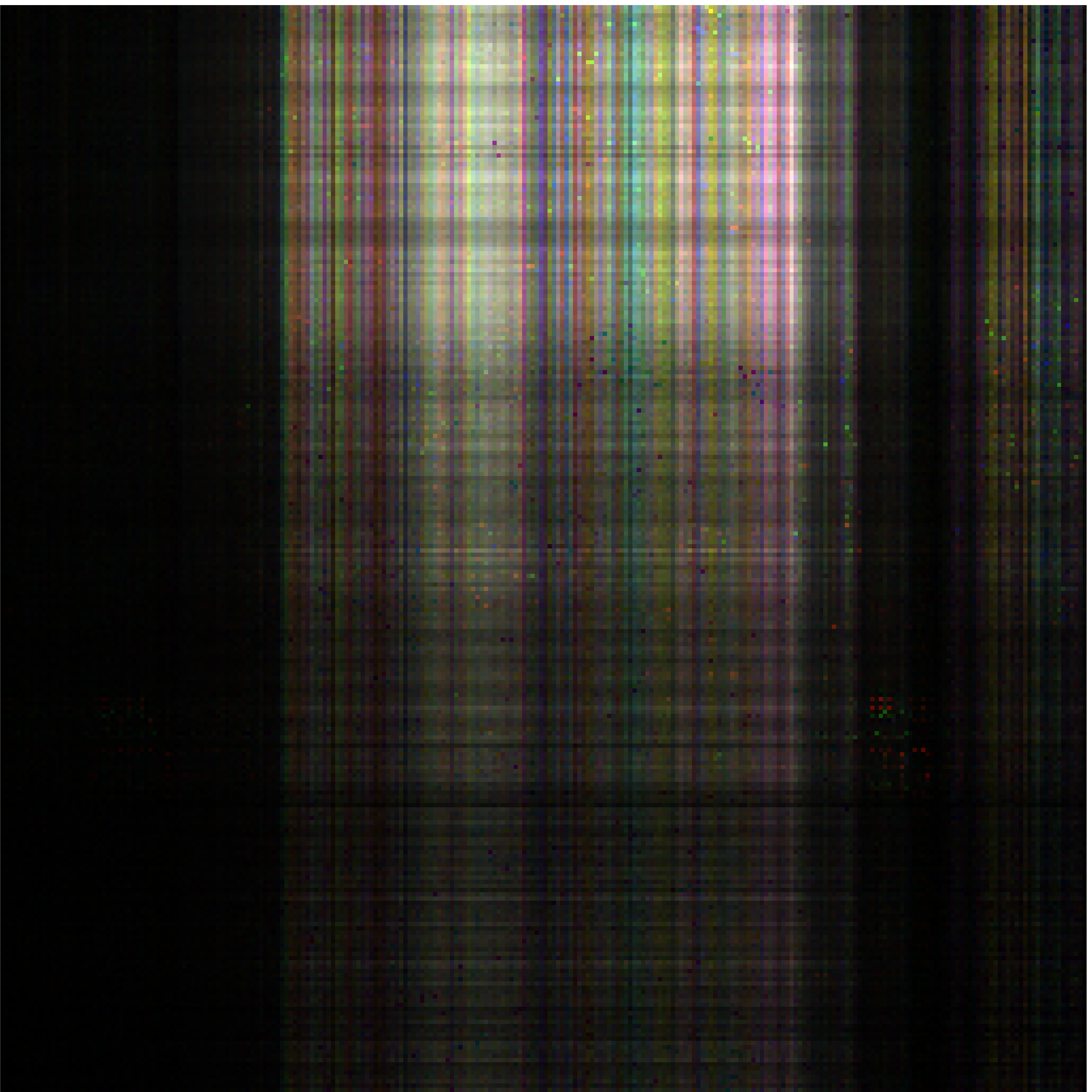}&
\includegraphics[width=0.19\textwidth]{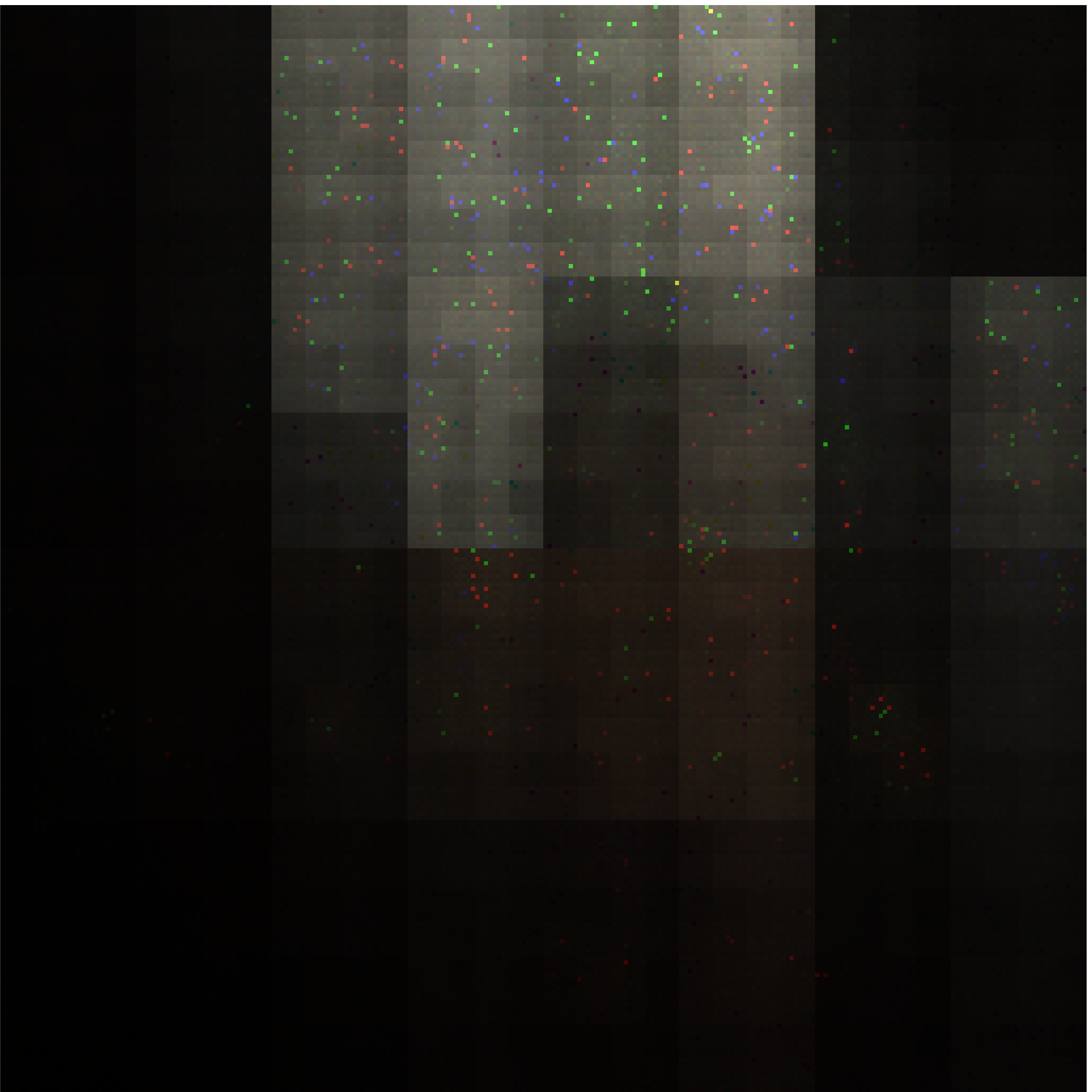}&
\includegraphics[width=0.19\textwidth]{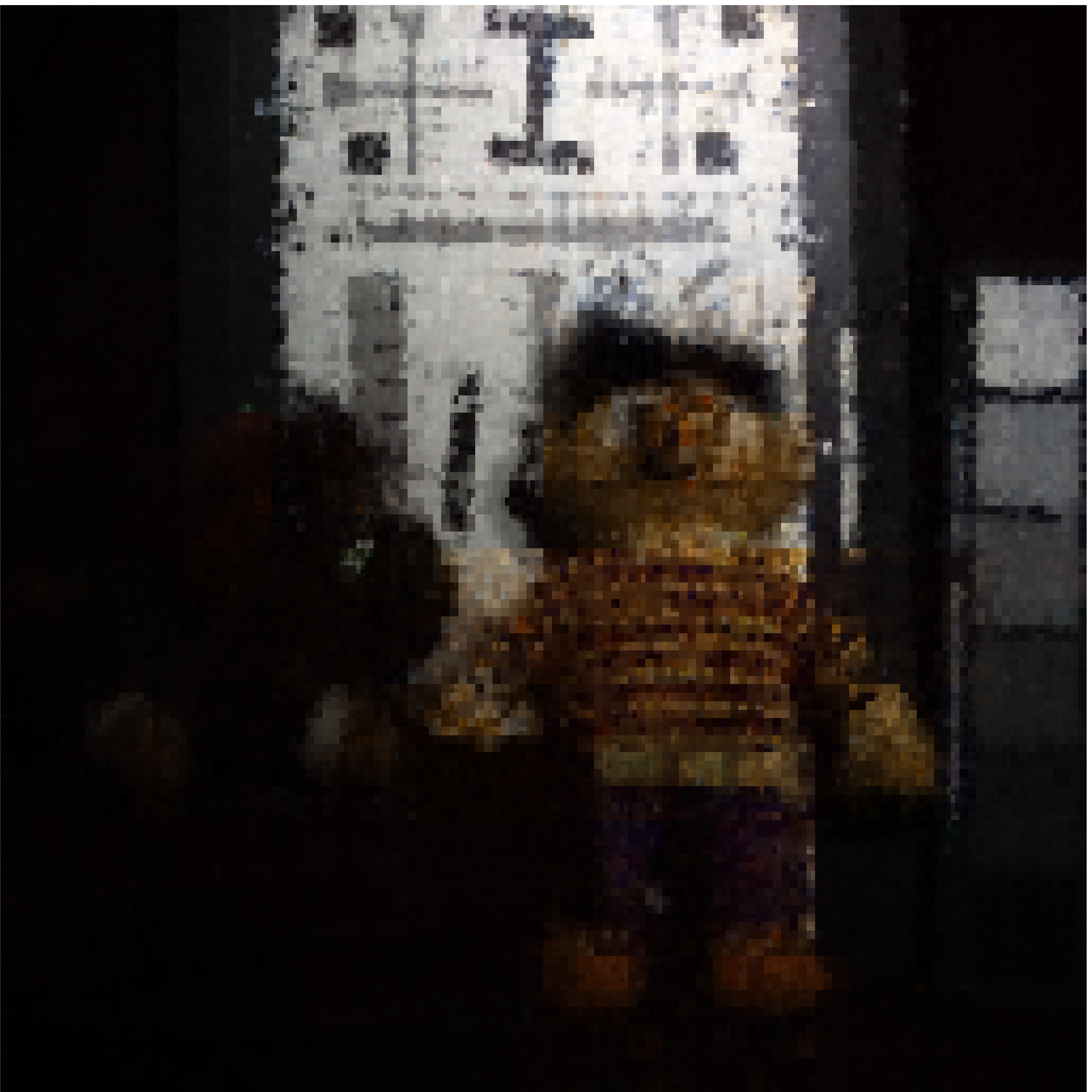}\vspace{0.01cm}\\
(f) SiLRTC-TT & (g) tSVD & (h) KBR & (i) TRNN & (j) LogTR\\
\includegraphics[width=0.19\textwidth]{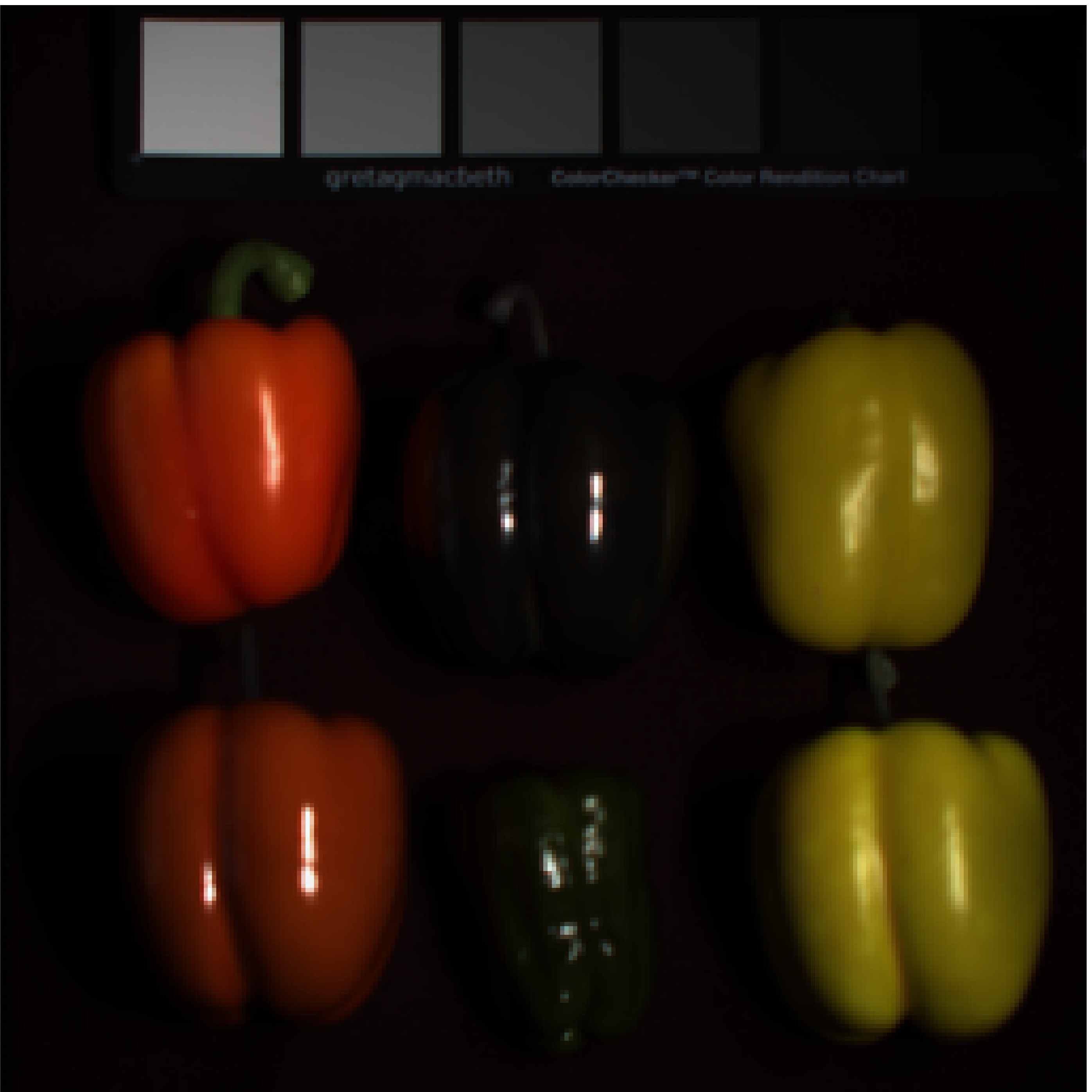}&
\includegraphics[width=0.19\textwidth]{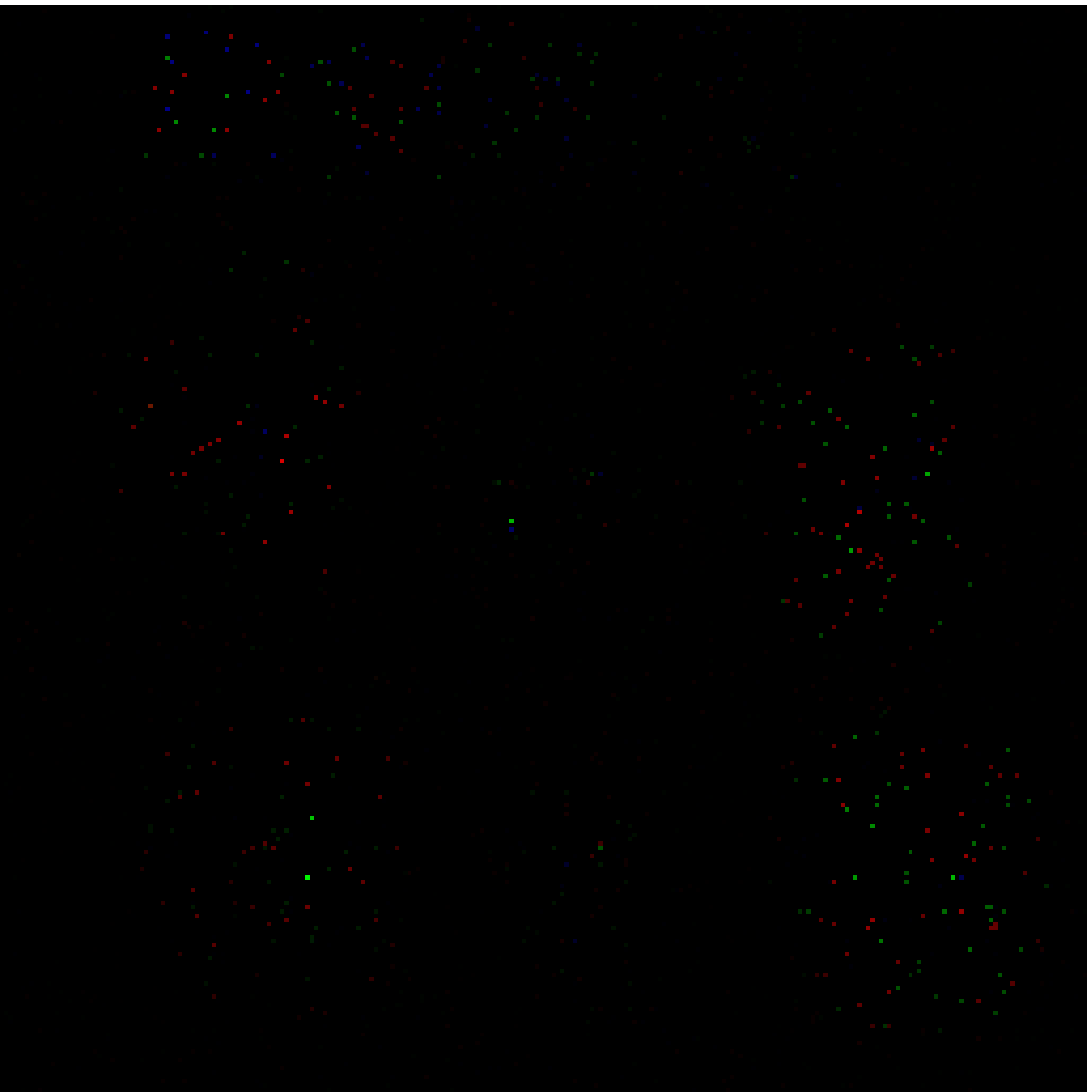}&
\includegraphics[width=0.19\textwidth]{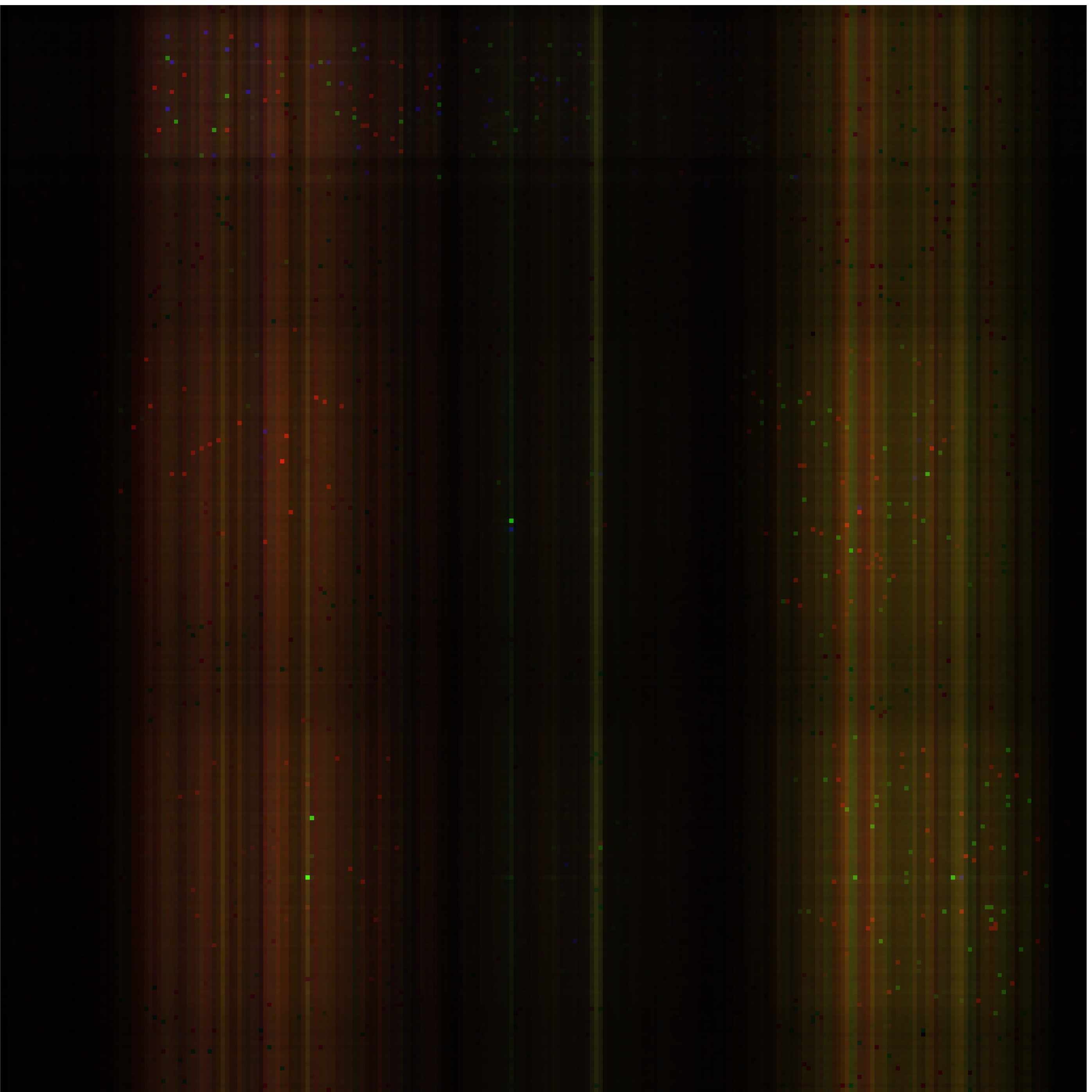}&
\includegraphics[width=0.19\textwidth]{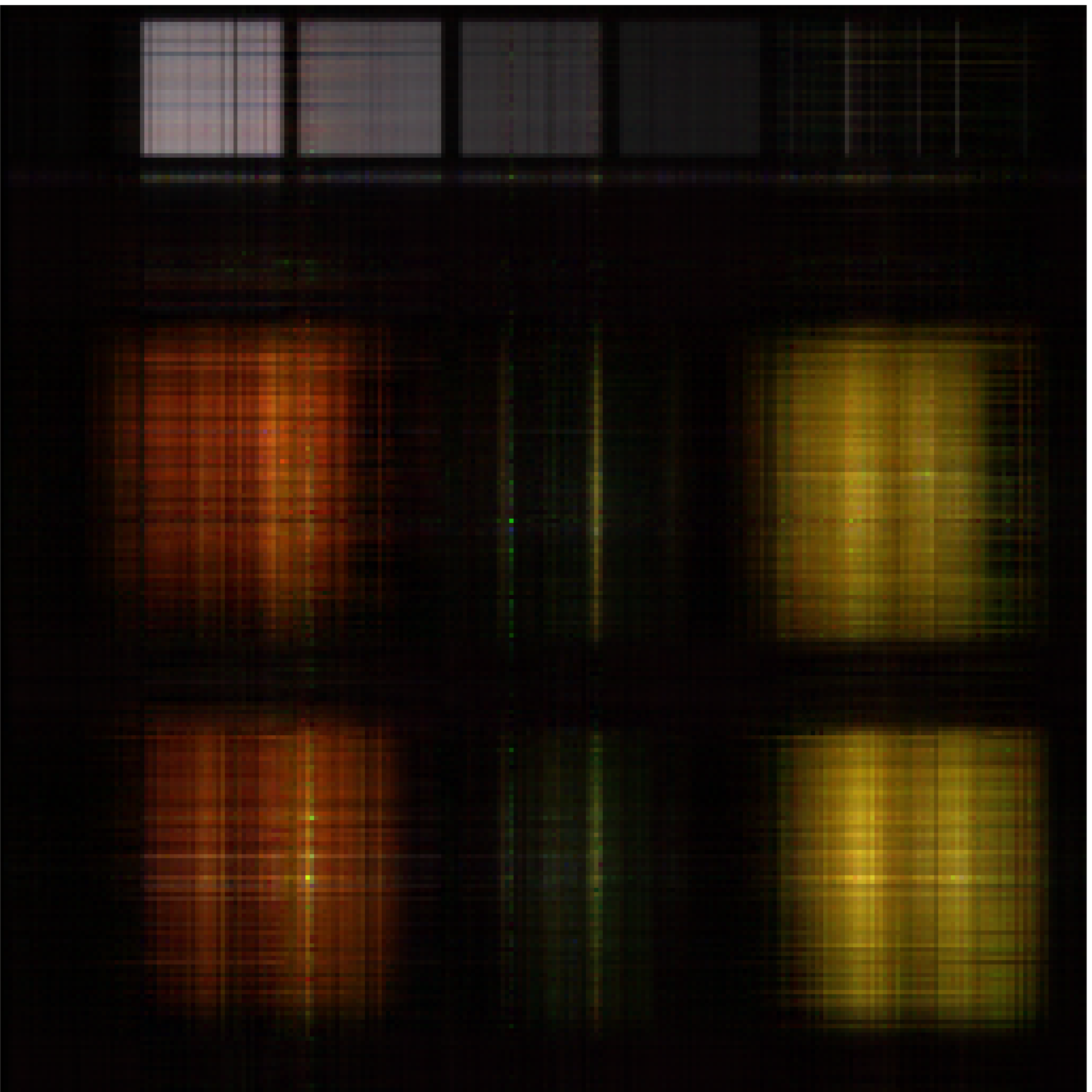}&
\includegraphics[width=0.19\textwidth]{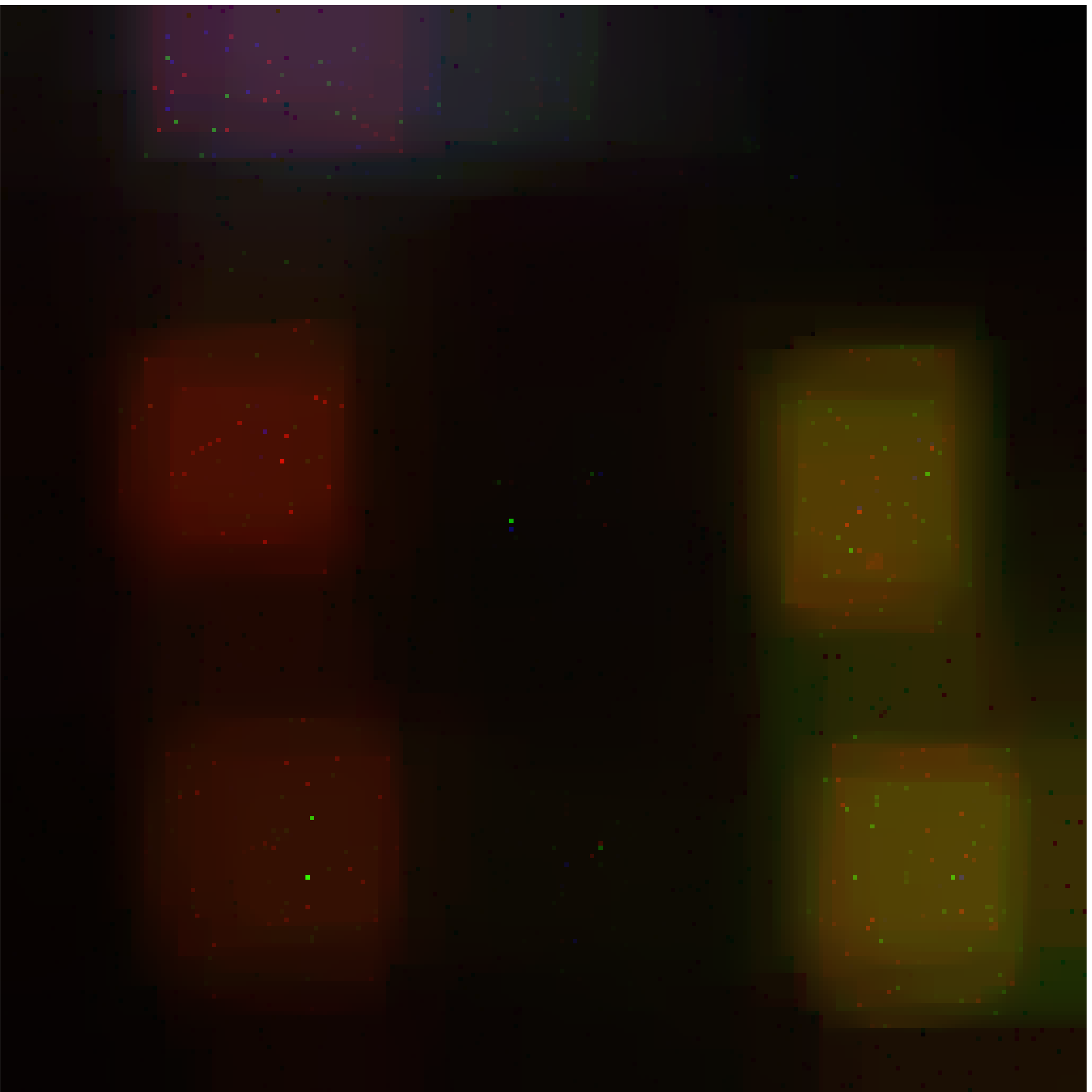}\vspace{0.01cm}\\
(a) Original & (b) Observed & (c) HaLRTC & (d) NSNN & (e) LRTC-TV\\
\includegraphics[width=0.19\textwidth]{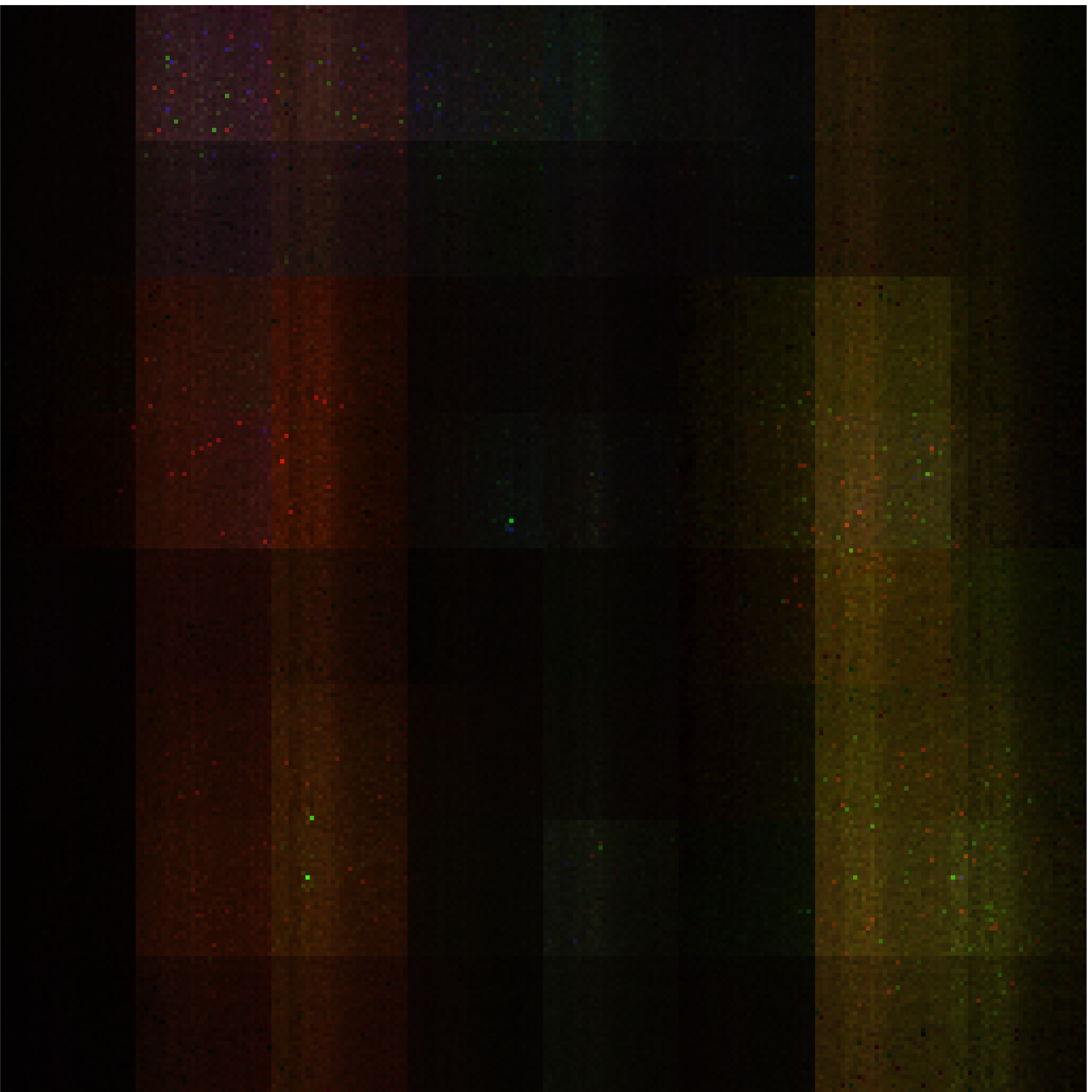}&
\includegraphics[width=0.19\textwidth]{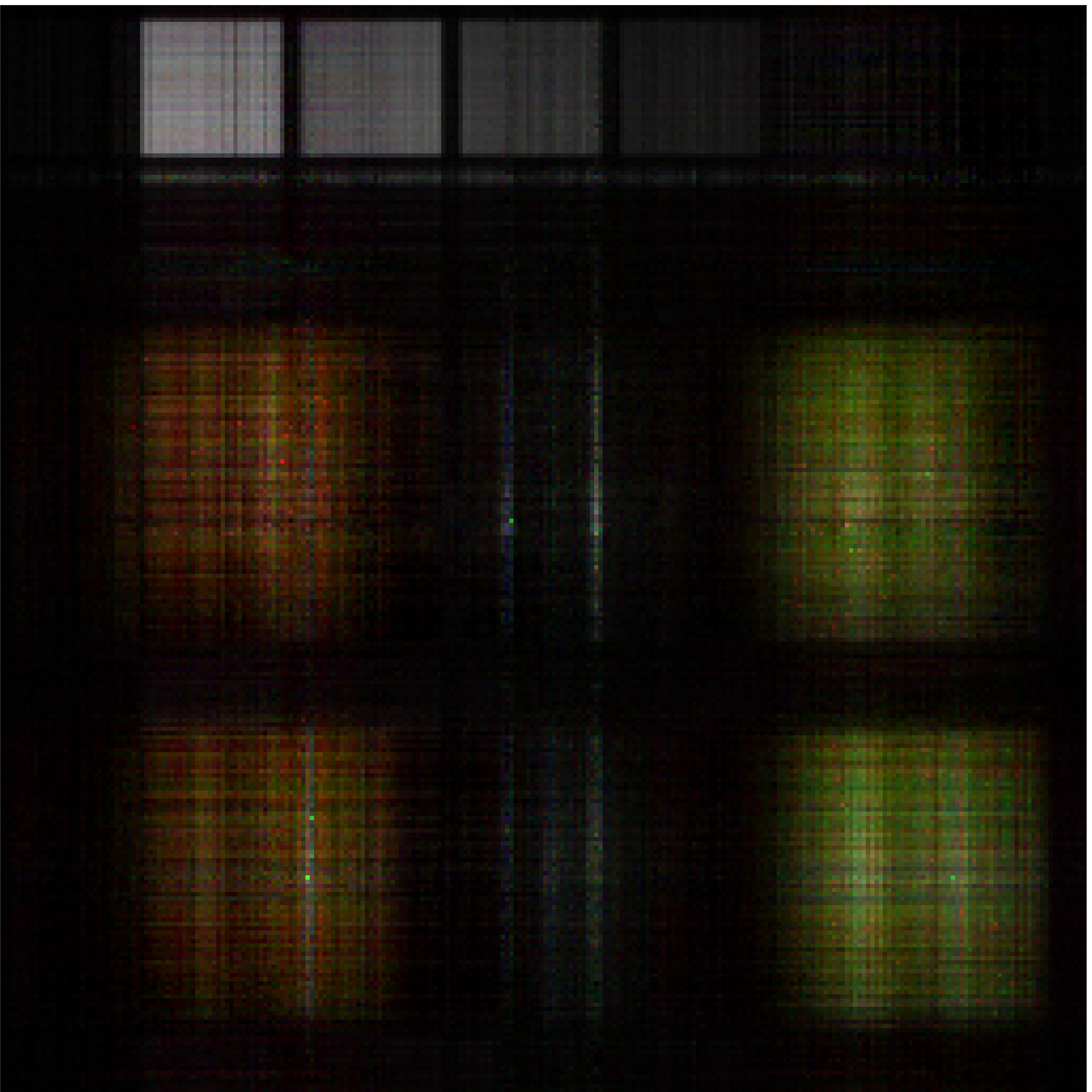}&
\includegraphics[width=0.19\textwidth]{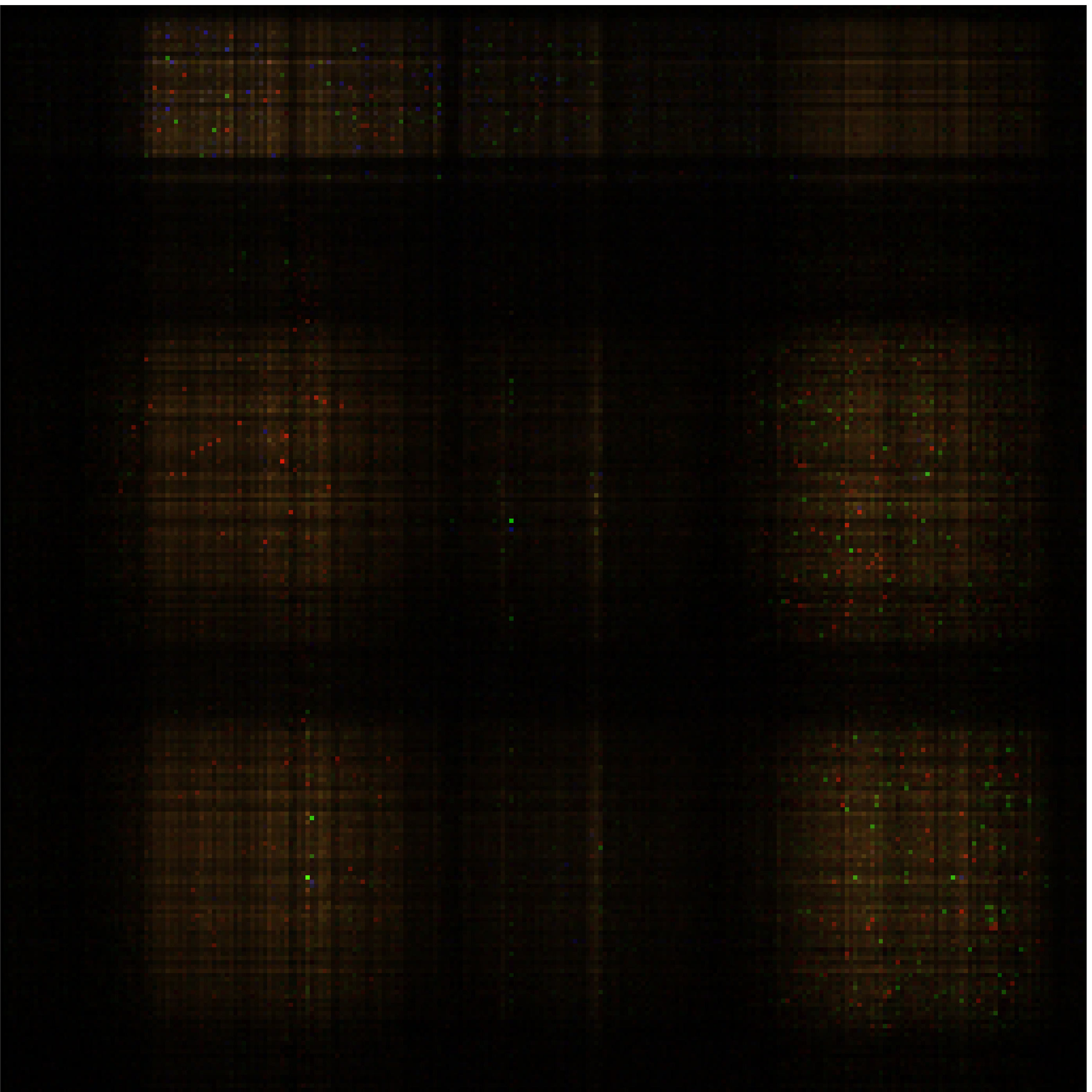}&
\includegraphics[width=0.19\textwidth]{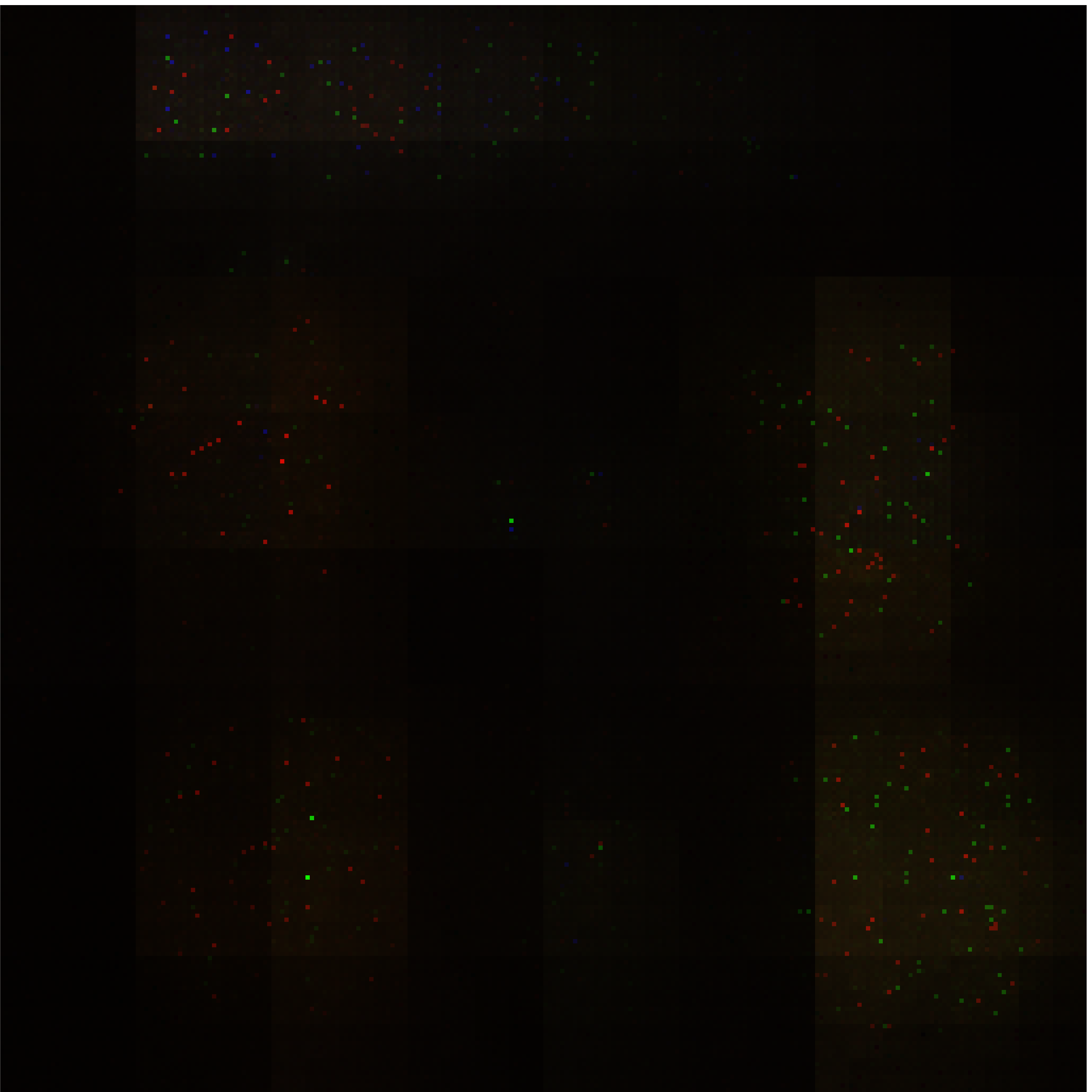}&
\includegraphics[width=0.19\textwidth]{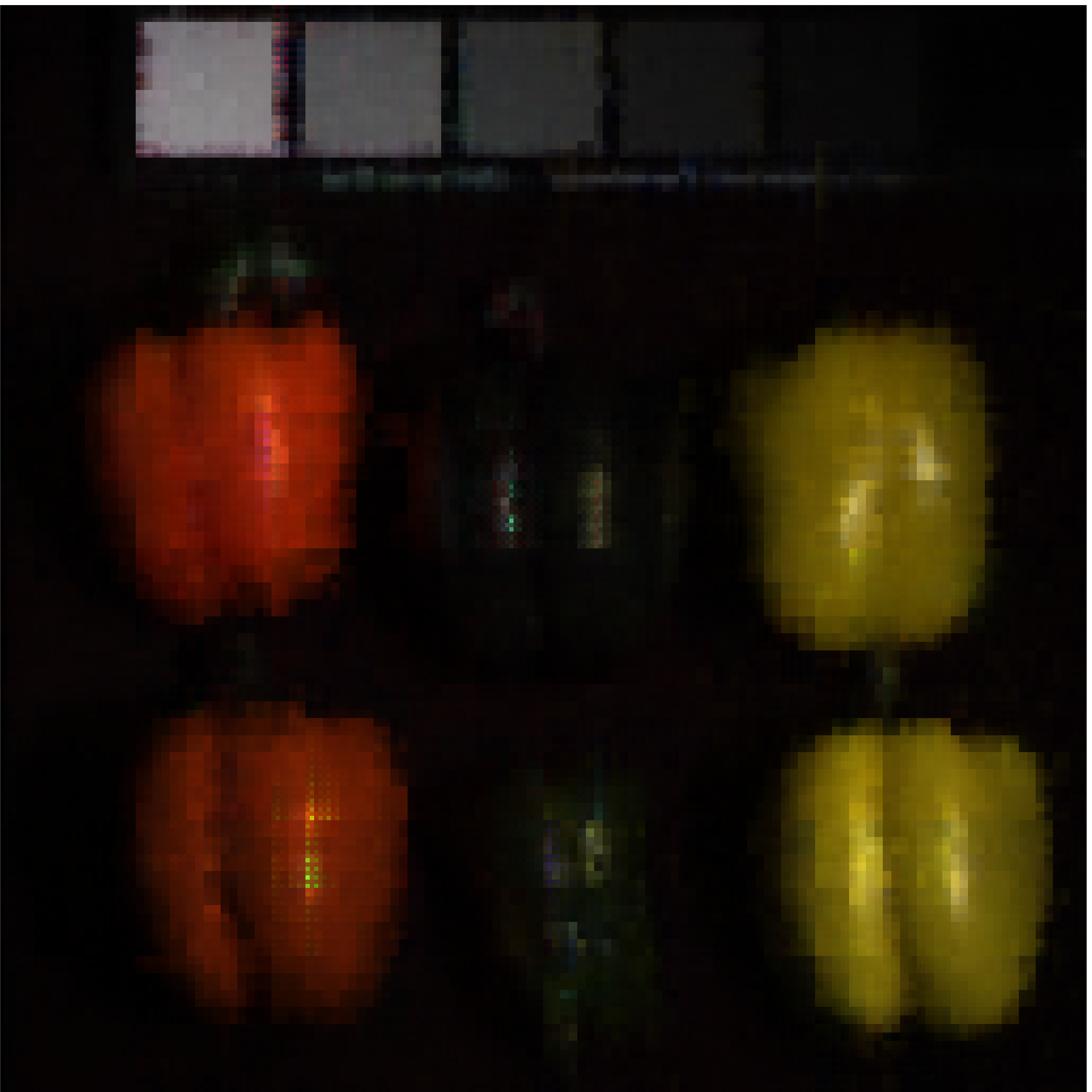}\vspace{0.01cm}\\
(f) SiLRTC-TT & (g) tSVD & (h) KBR & (i) TRNN & (j) LogTR
\end{tabular}
\caption{\small{Recovered MSIs \emph{Feathers}, \emph{Toy}, and \emph{Peppers} for random missing entries with $SR=0.01$. The color image is composed of bands 30, 20, and 10.}}
  \label{fig:msi_01}
  \end{center}\vspace{-0.3cm}
\end{figure}

\begin{figure}[!t]
\scriptsize\setlength{\tabcolsep}{0.5pt}
\begin{center}
\begin{tabular}{ccccc}
\includegraphics[width=0.19\textwidth]{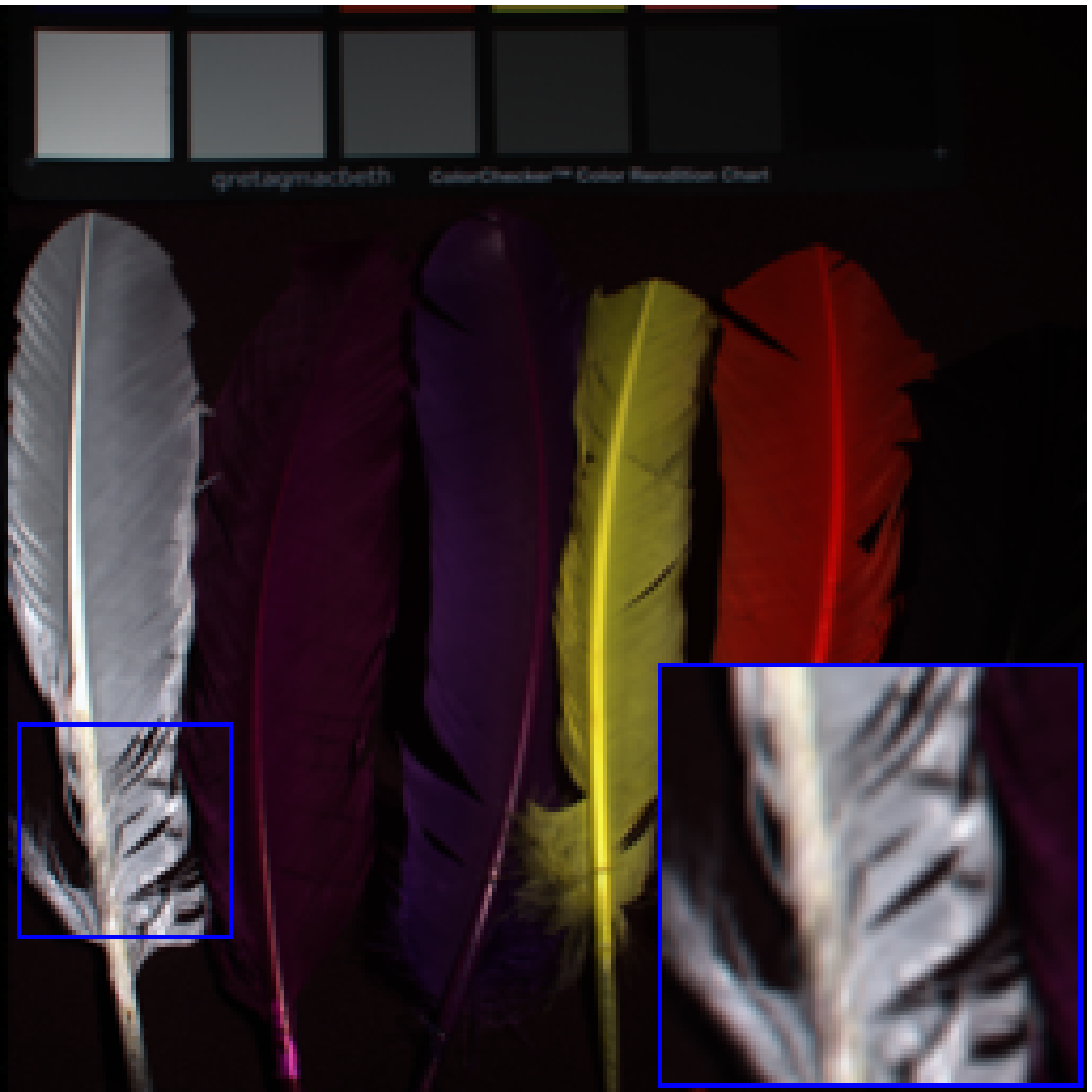}&
\includegraphics[width=0.19\textwidth]{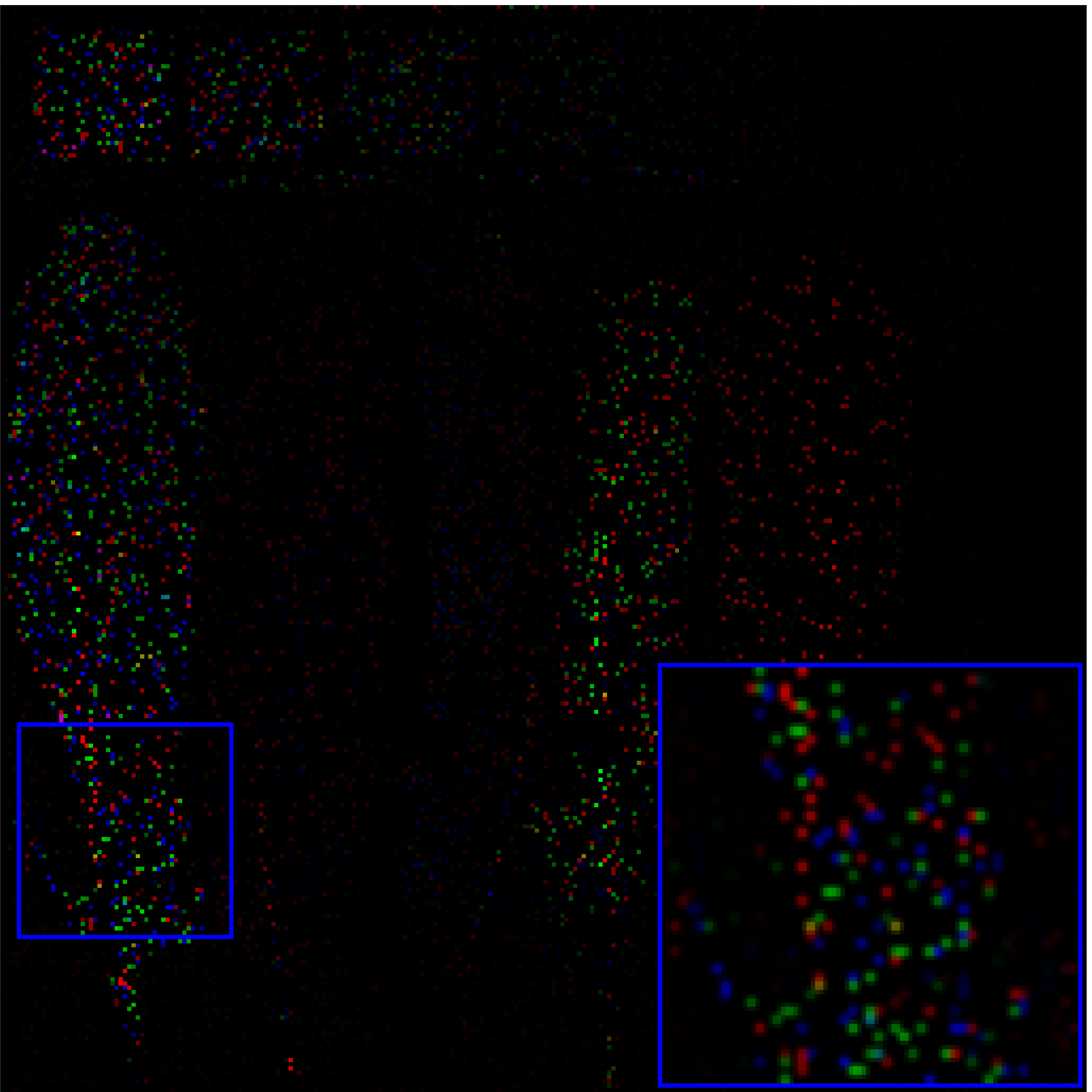}&
\includegraphics[width=0.19\textwidth]{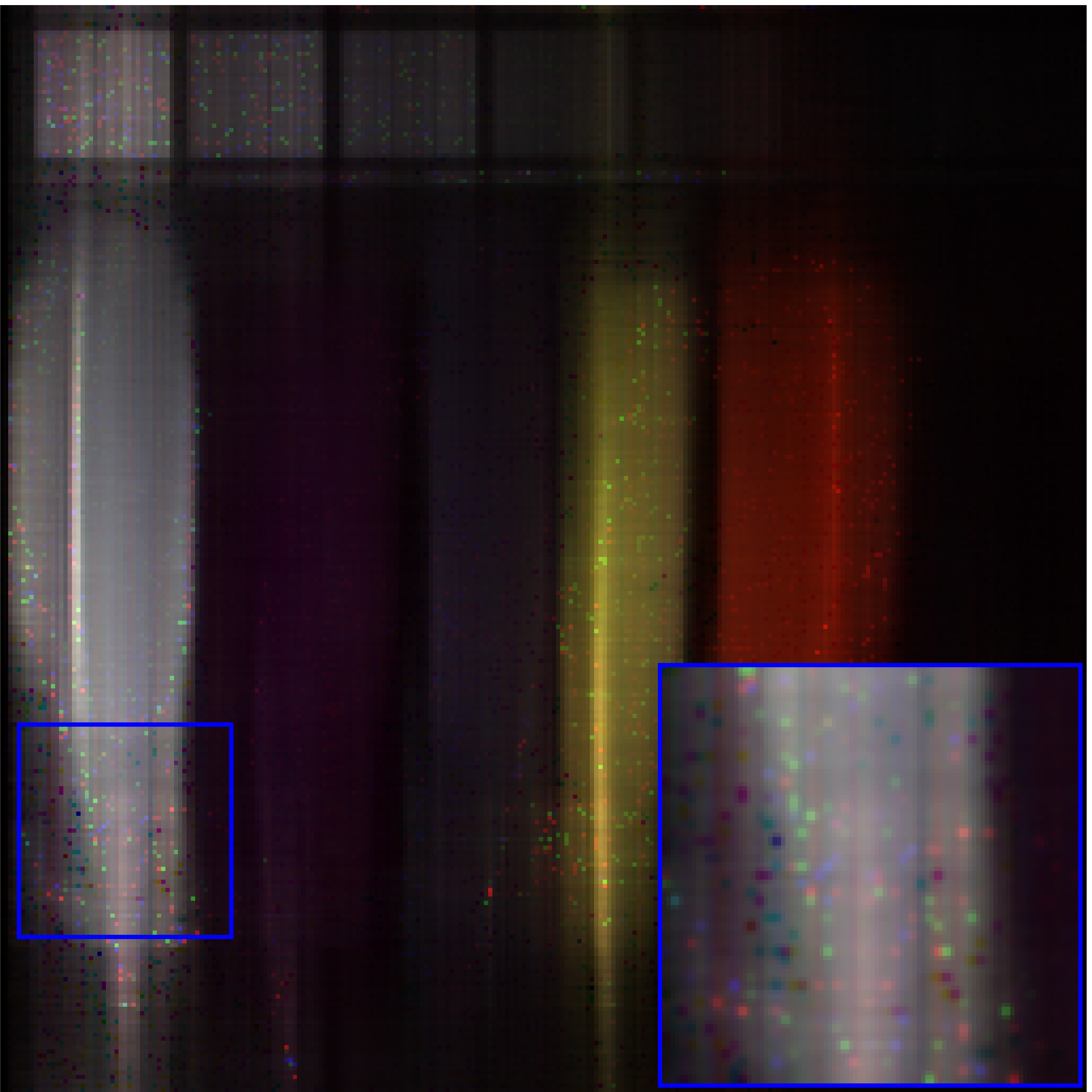}&
\includegraphics[width=0.19\textwidth]{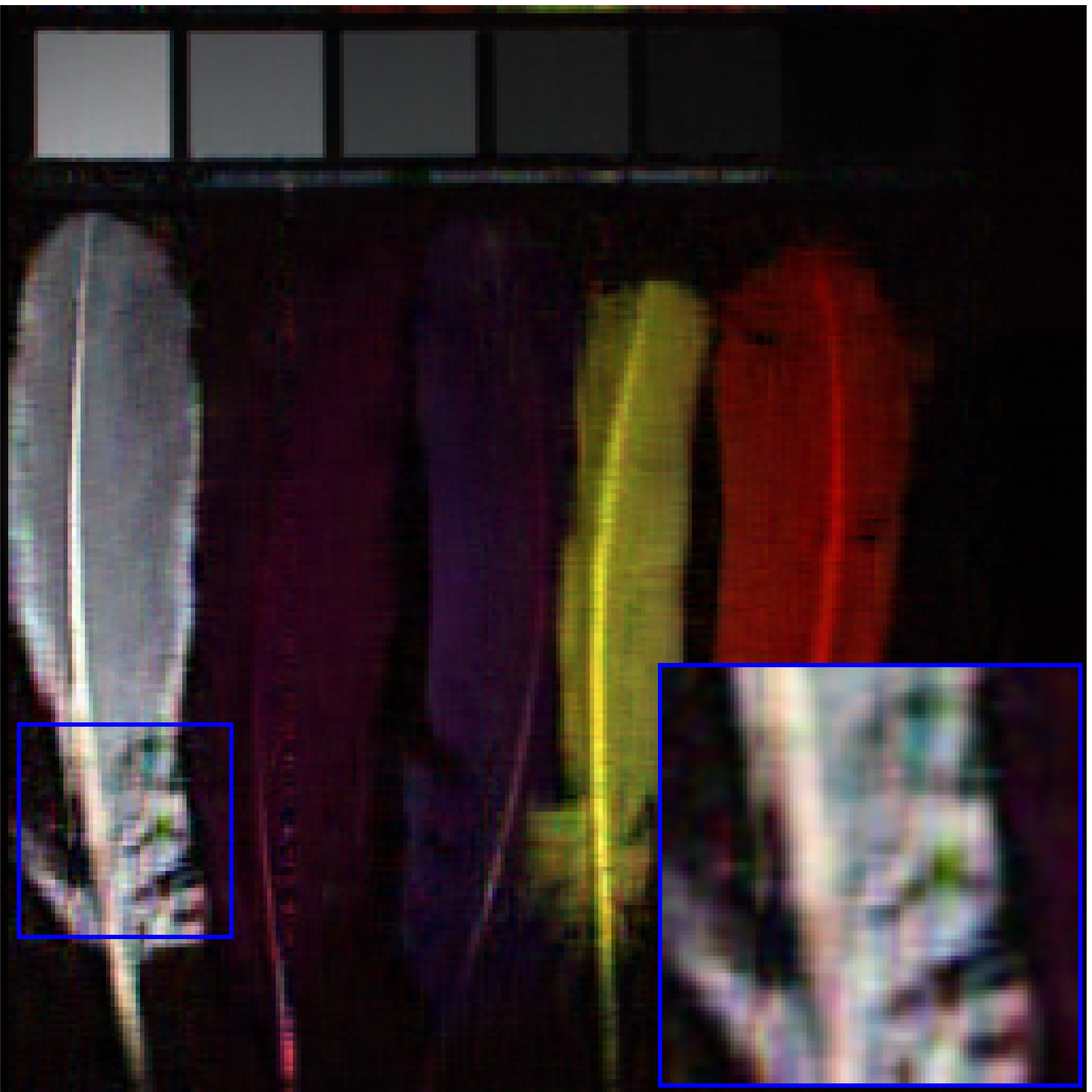}&
\includegraphics[width=0.19\textwidth]{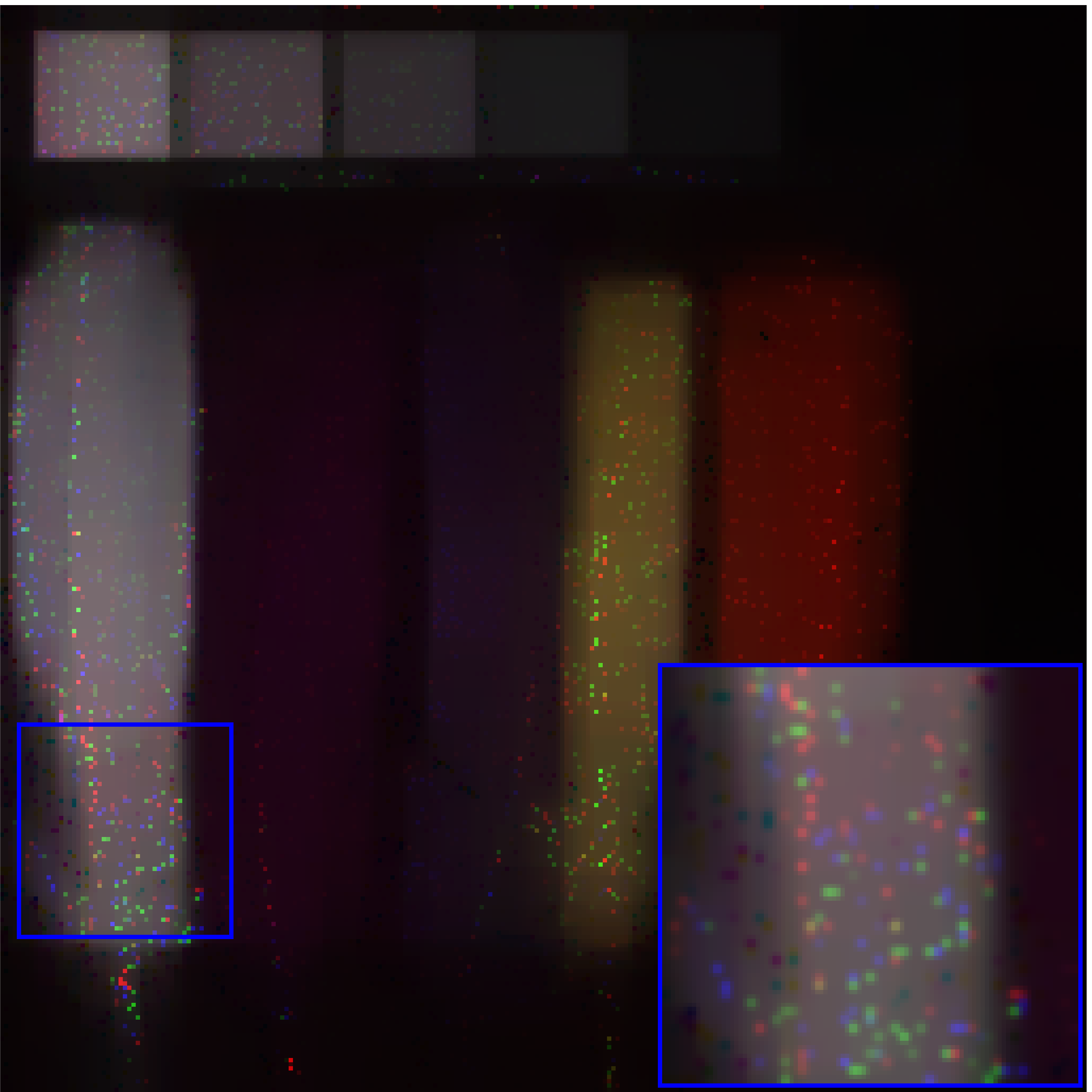}\vspace{0.01cm}\\
(a) Original & (b) Observed & (c) HaLRTC & (d) NSNN & (e) LRTC-TV\\
\includegraphics[width=0.19\textwidth]{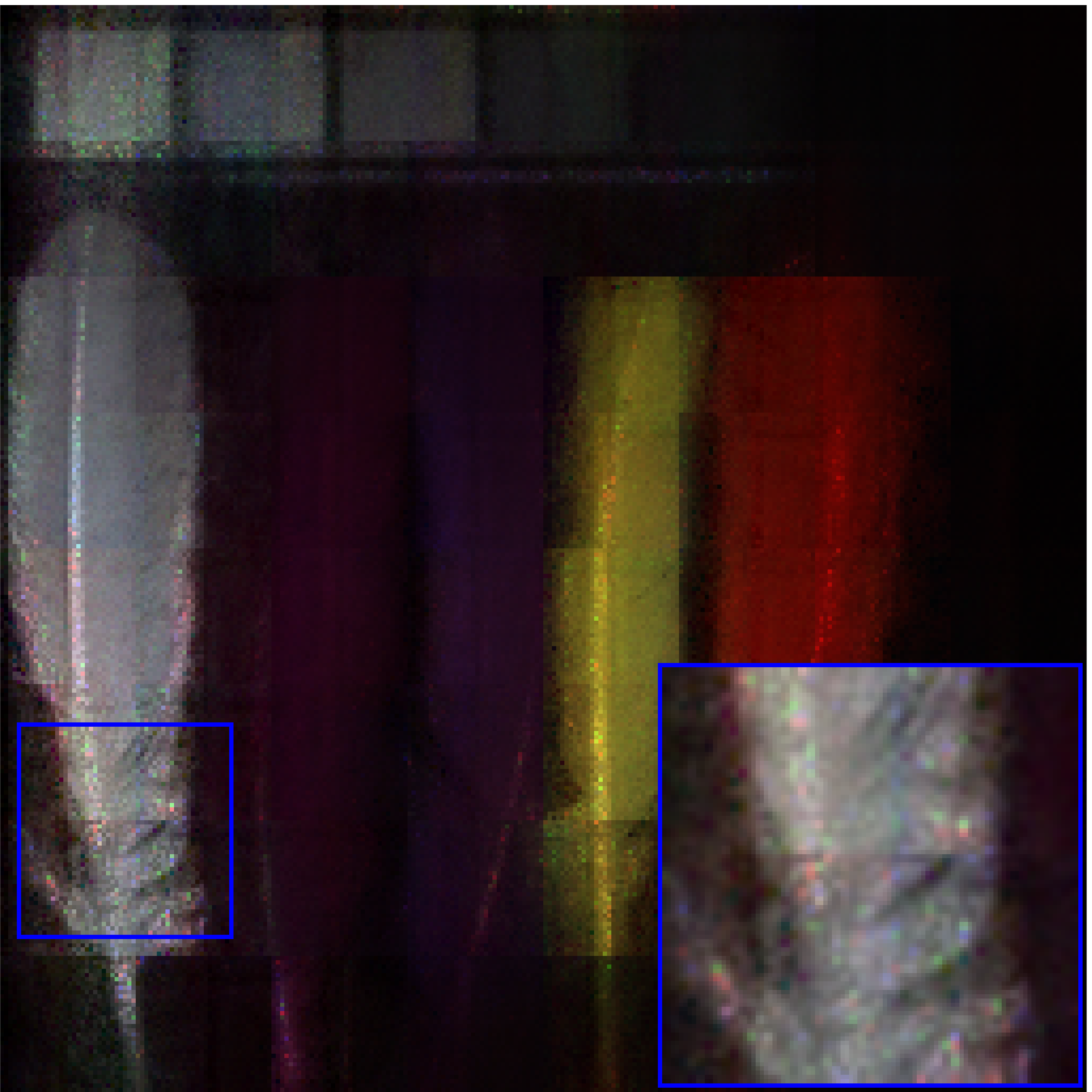}&
\includegraphics[width=0.19\textwidth]{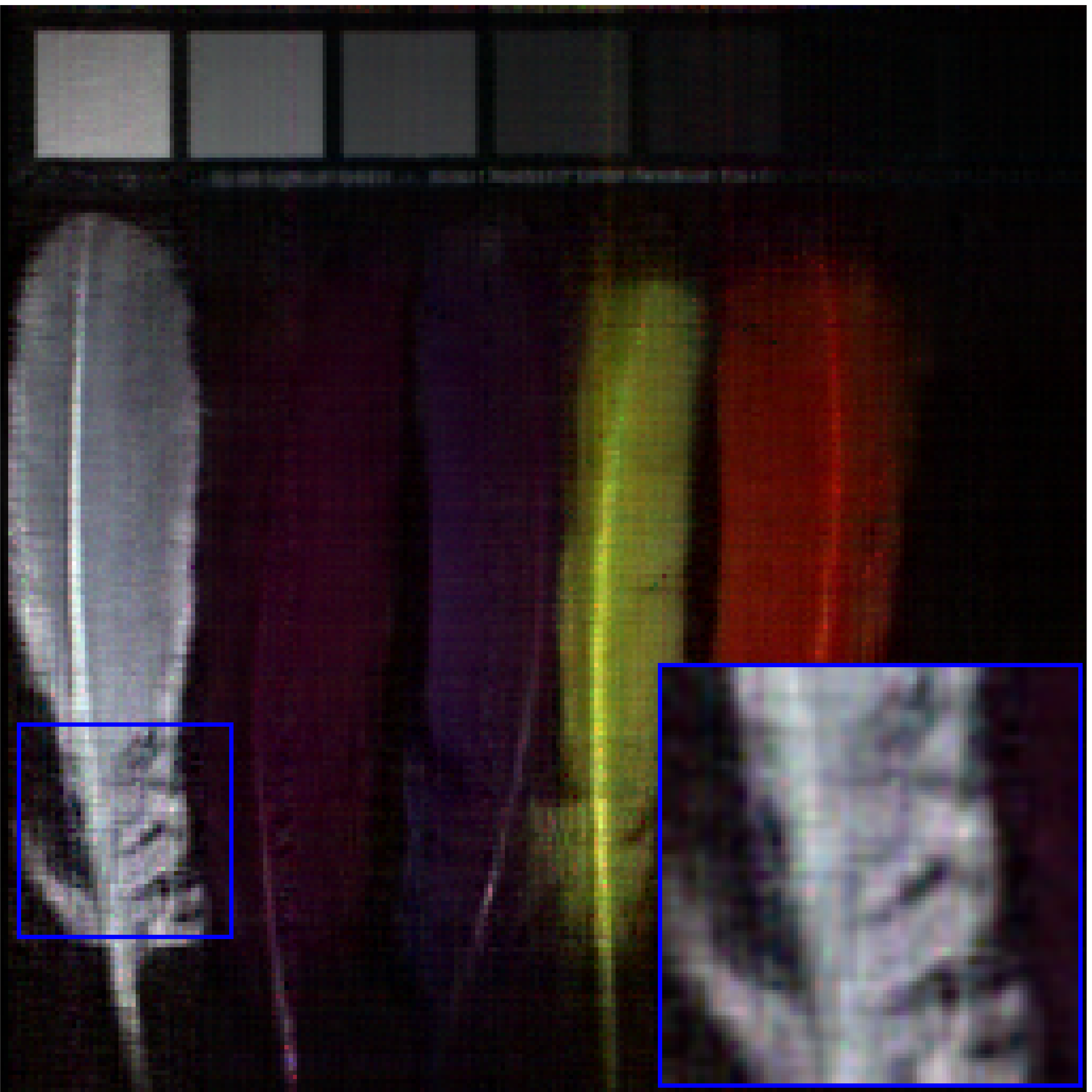}&
\includegraphics[width=0.19\textwidth]{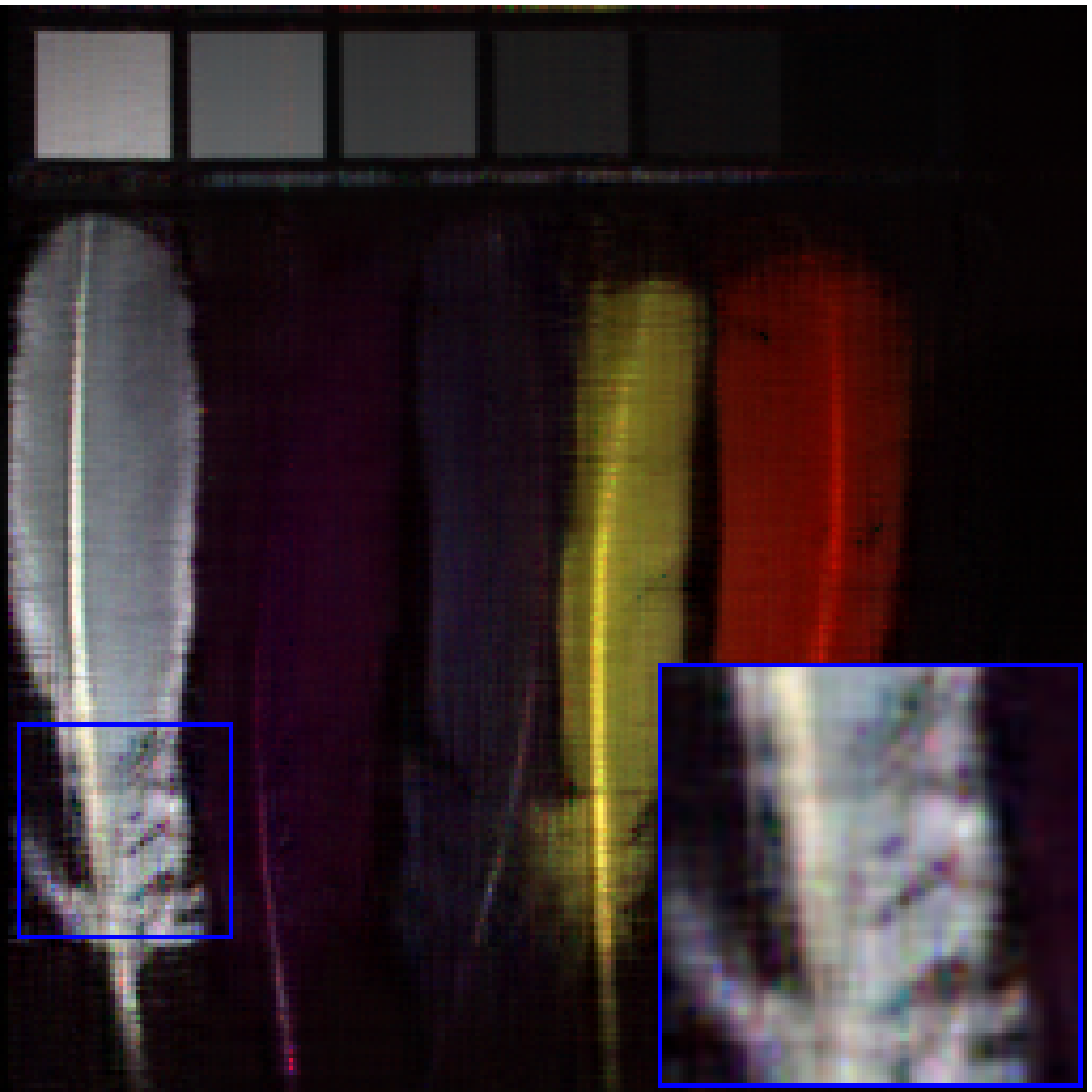}&
\includegraphics[width=0.19\textwidth]{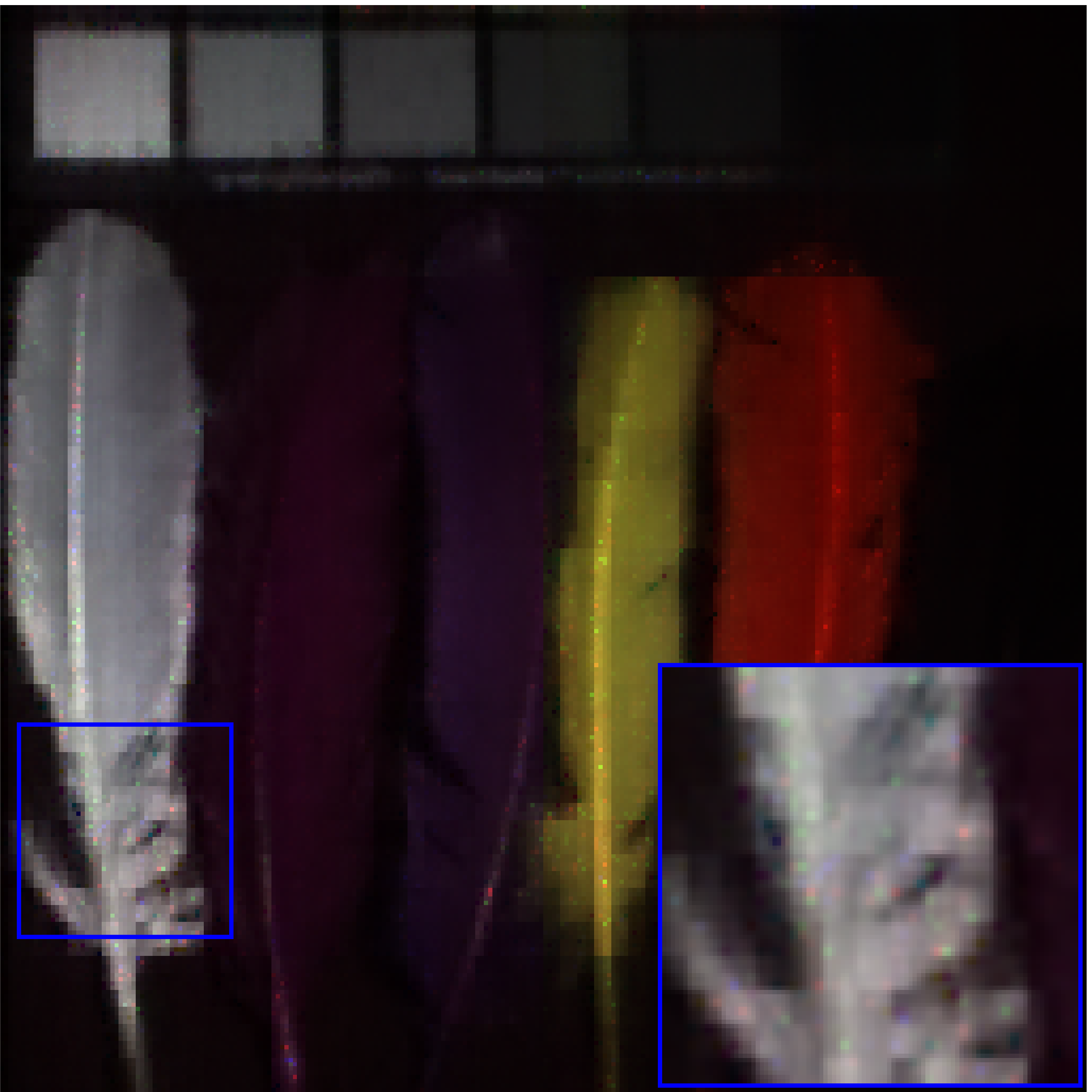}&
\includegraphics[width=0.19\textwidth]{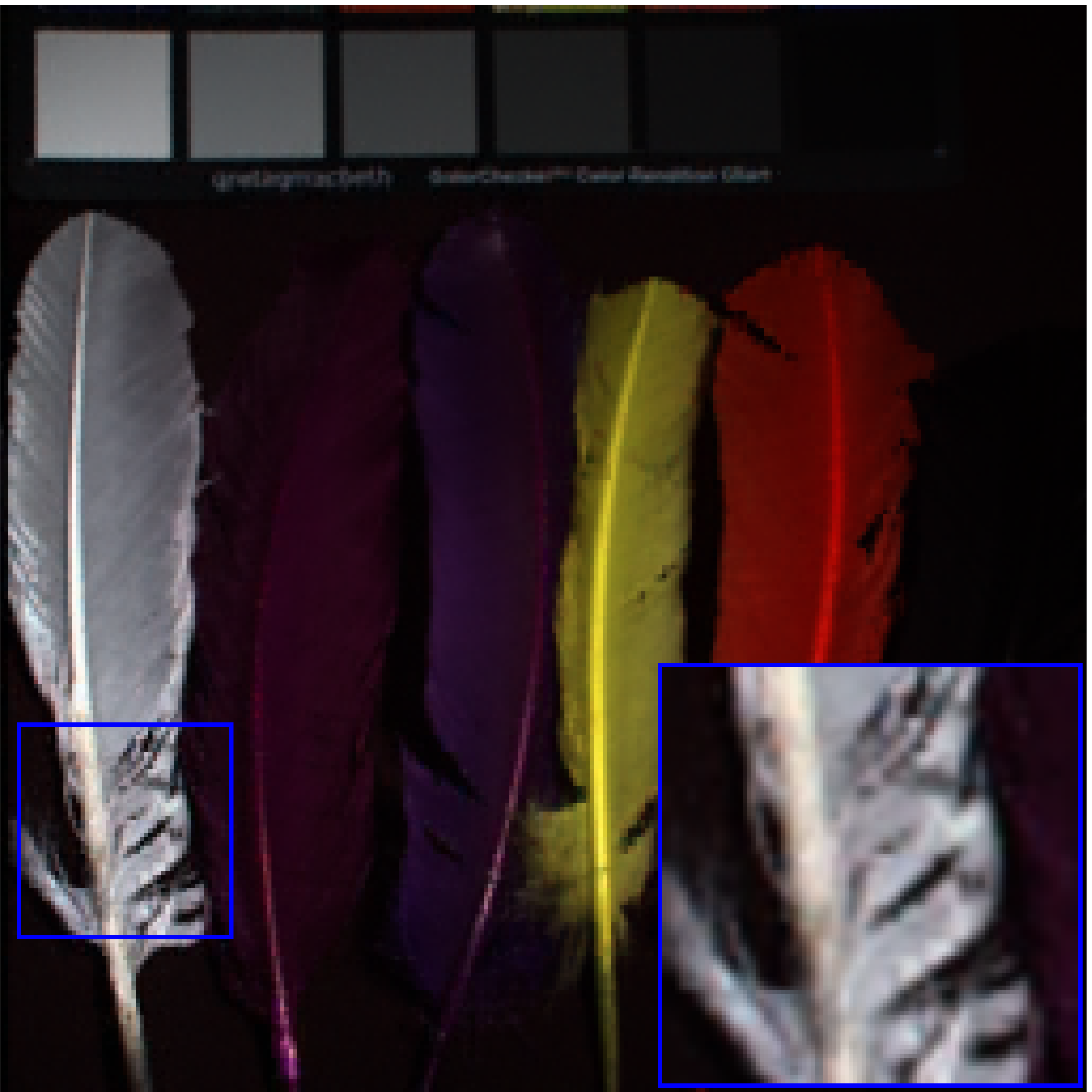}\vspace{0.01cm}\\
(f) SiLRTC-TT & (g) tSVD & (h) KBR & (i) TRNN & (j) LogTR\\
\includegraphics[width=0.19\textwidth]{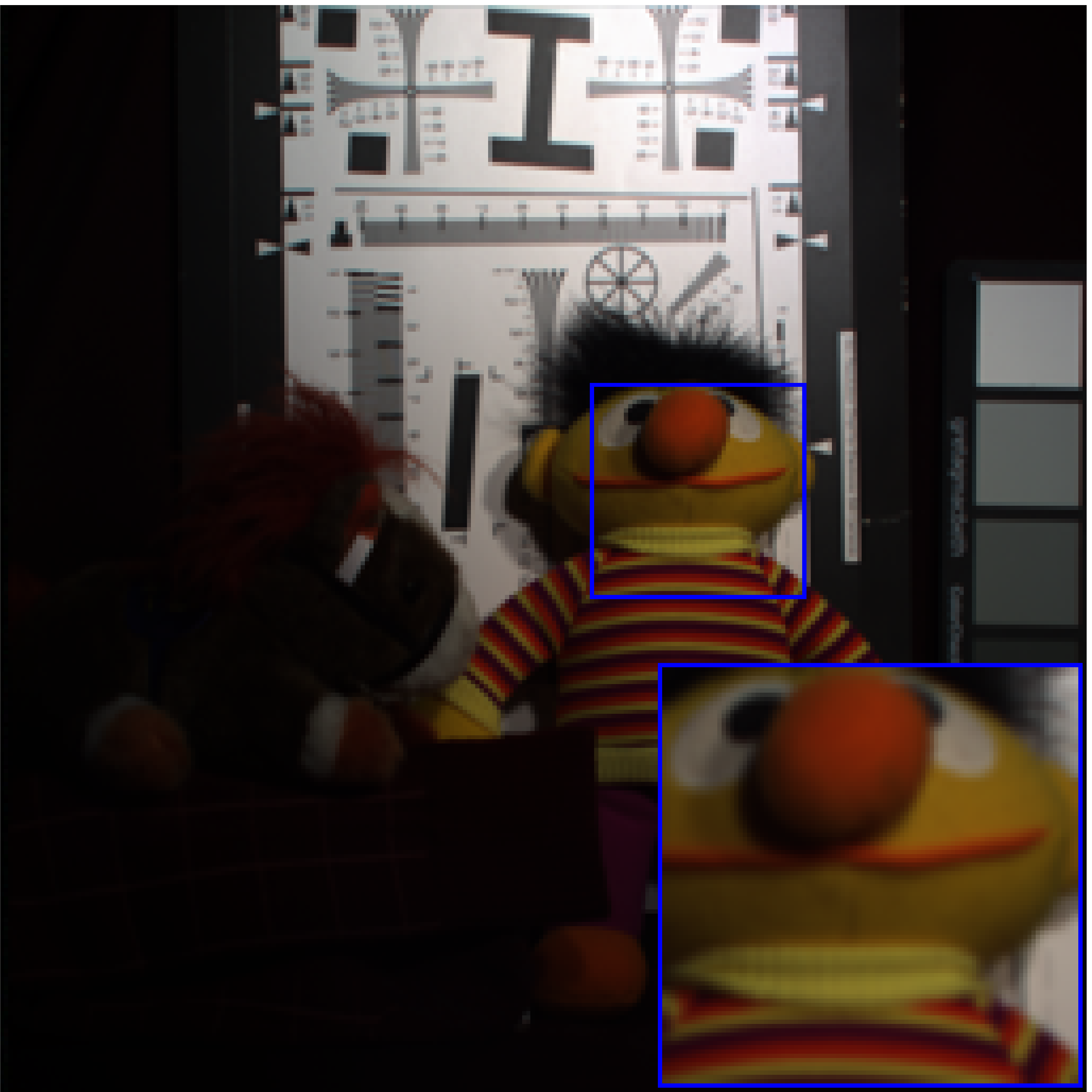}&
\includegraphics[width=0.19\textwidth]{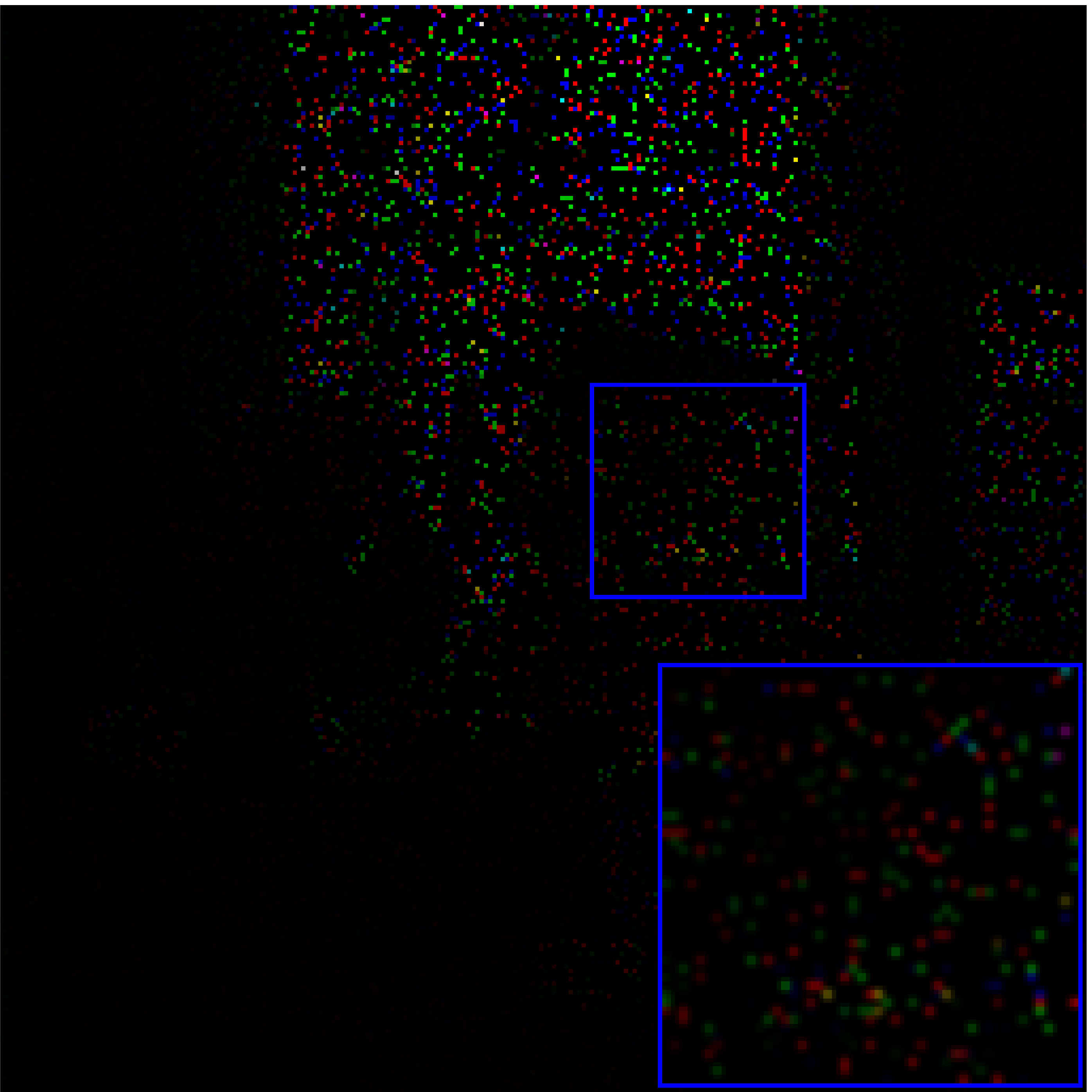}&
\includegraphics[width=0.19\textwidth]{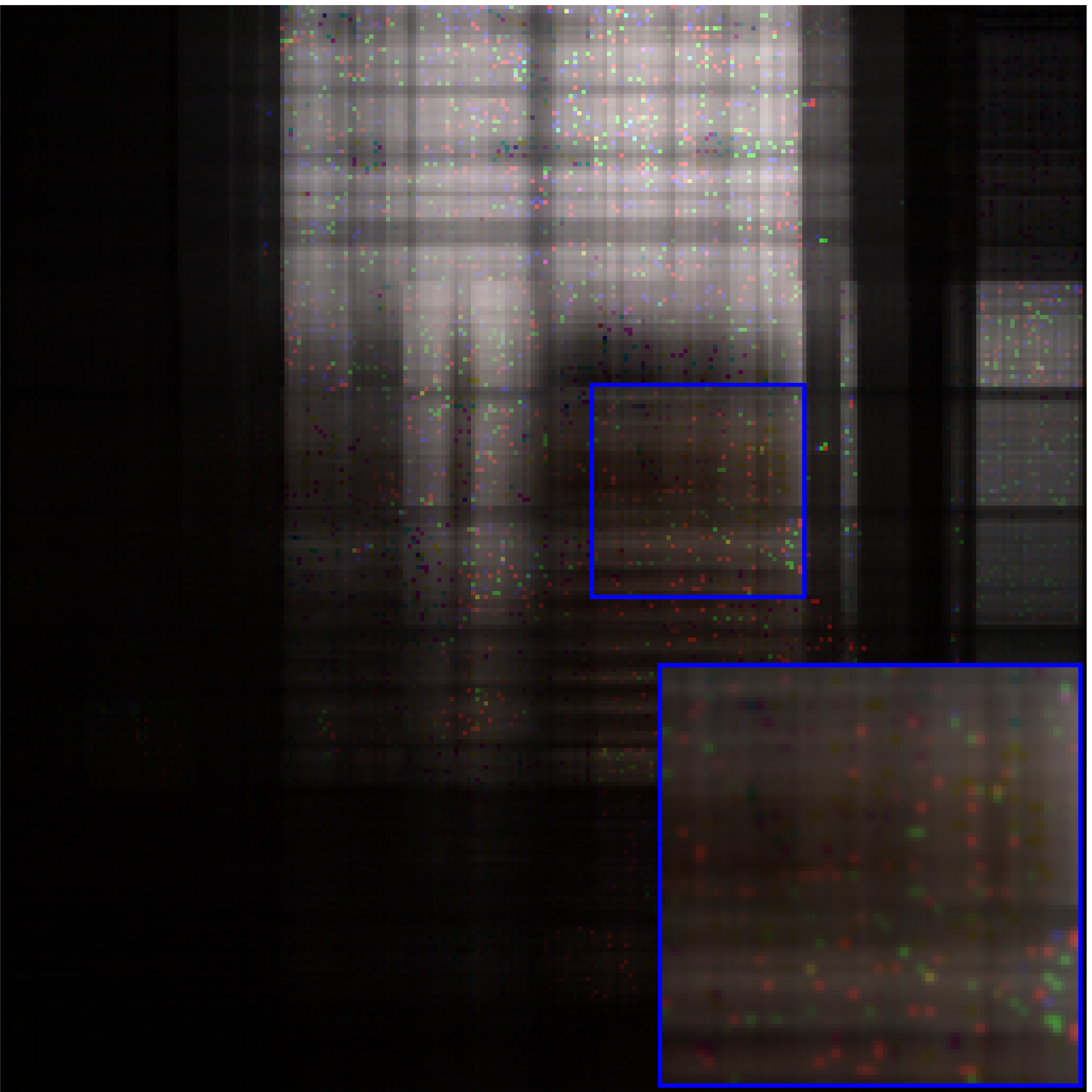}&
\includegraphics[width=0.19\textwidth]{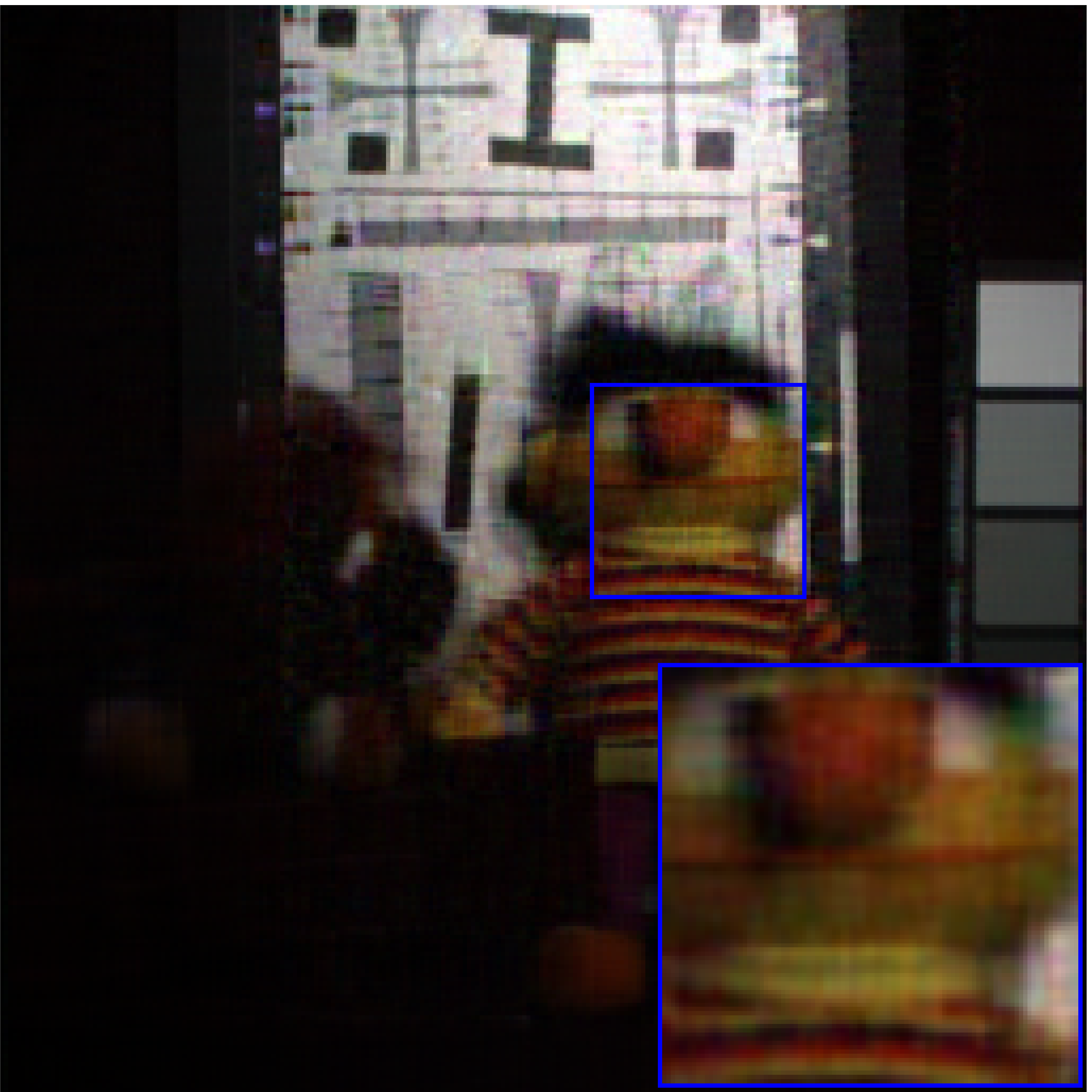}&
\includegraphics[width=0.19\textwidth]{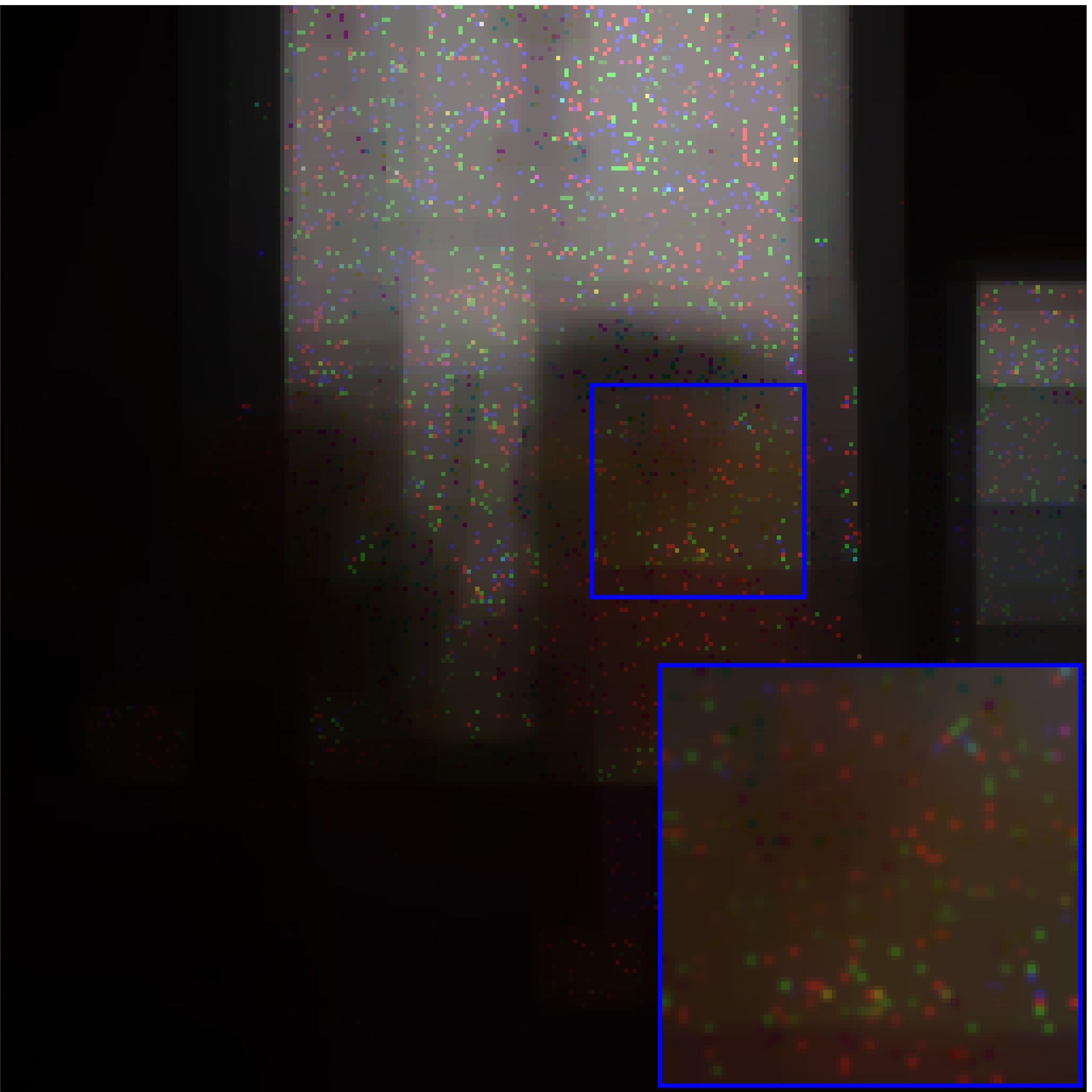}\vspace{0.01cm}\\
(a) Original& (b) Observed & (c) HaLRTC & (d) NSNN & (e) LRTC-TV\\
\includegraphics[width=0.19\textwidth]{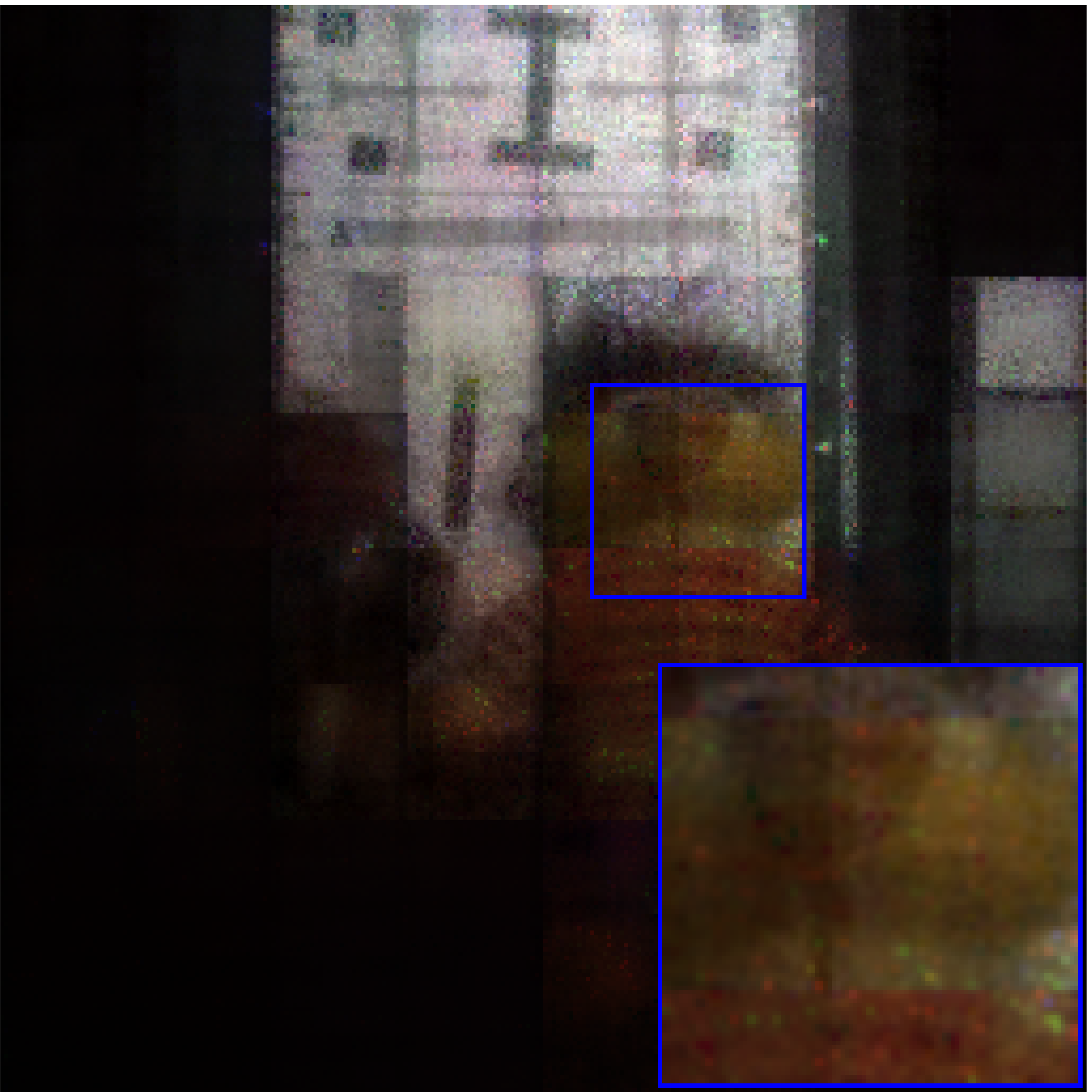}&
\includegraphics[width=0.19\textwidth]{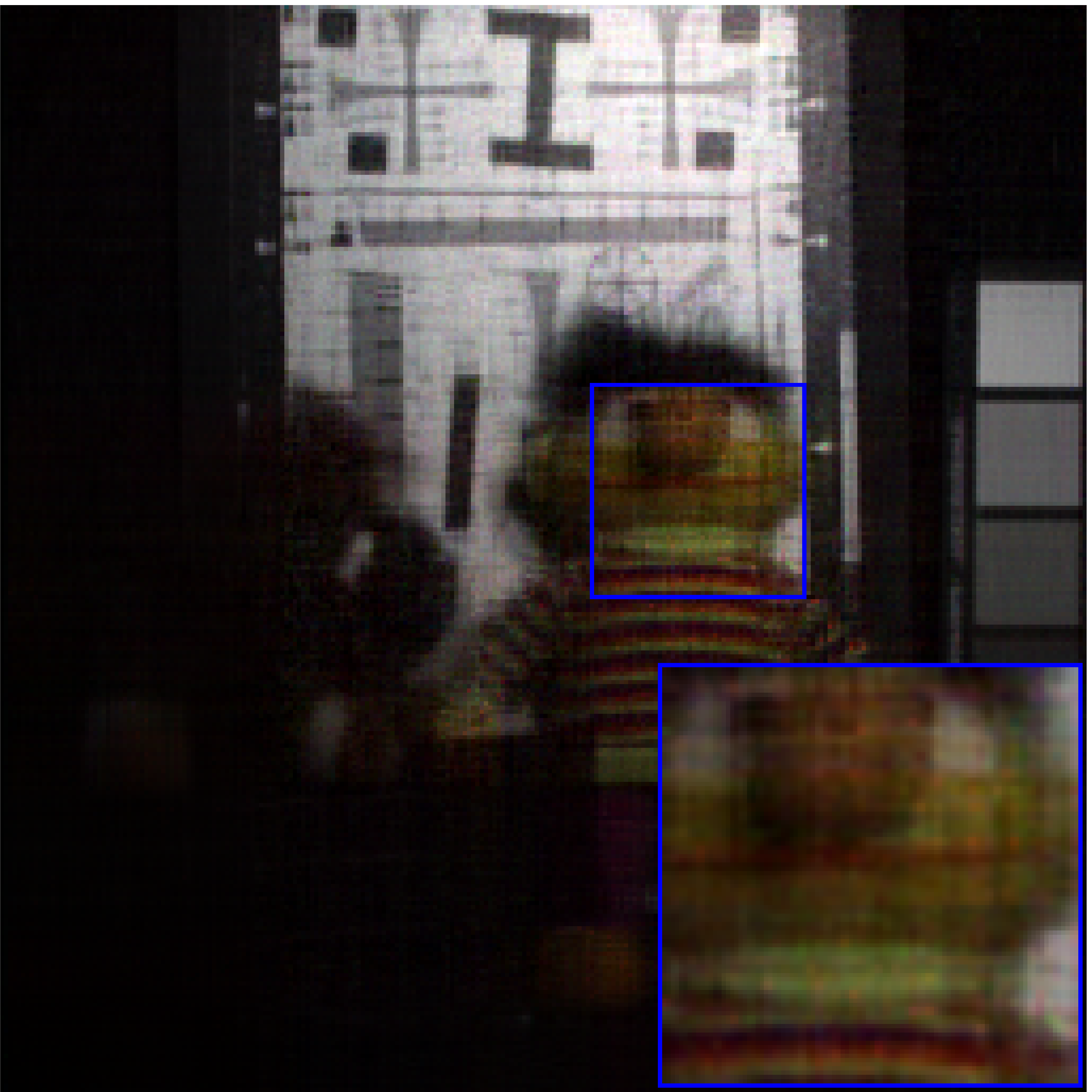}&
\includegraphics[width=0.19\textwidth]{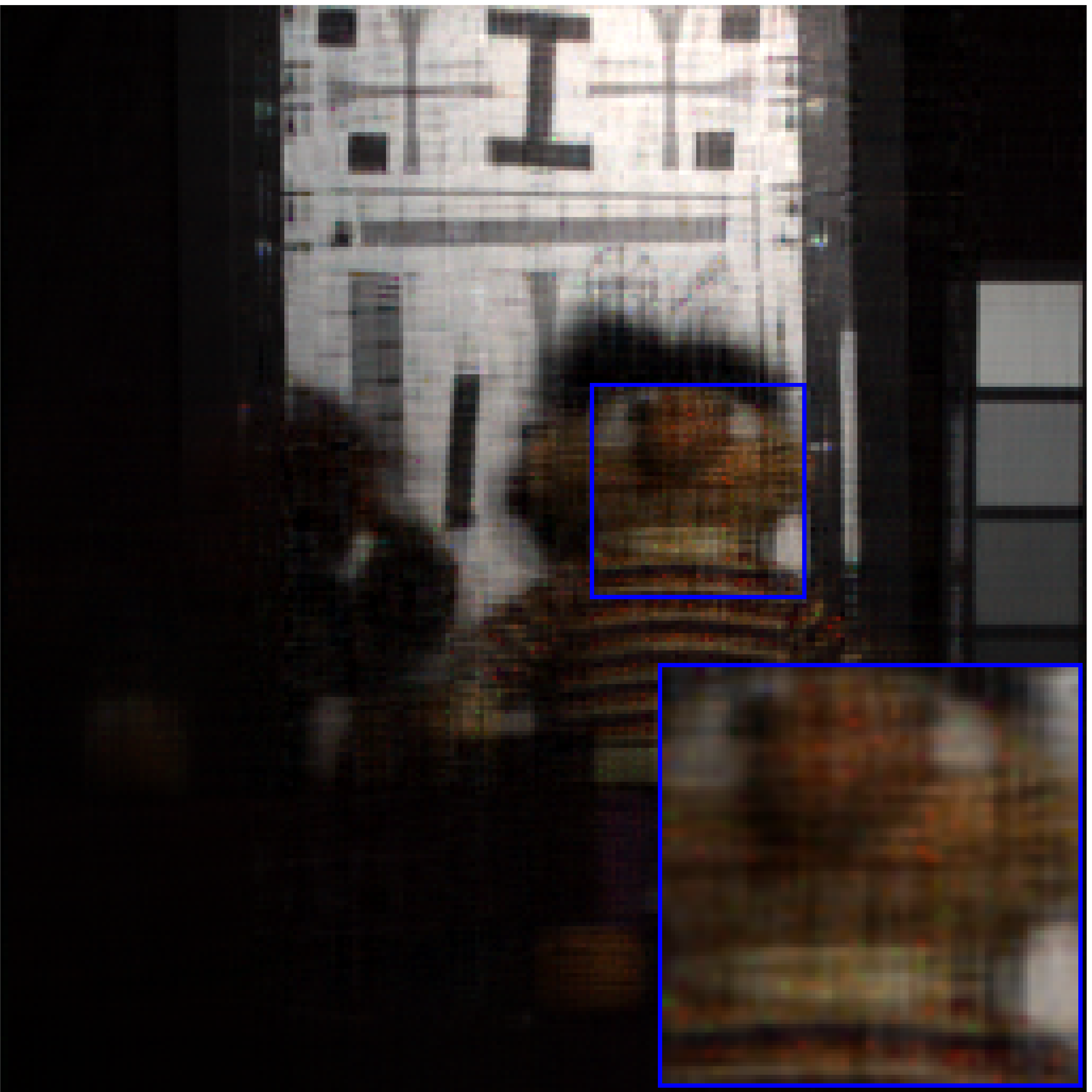}&
\includegraphics[width=0.19\textwidth]{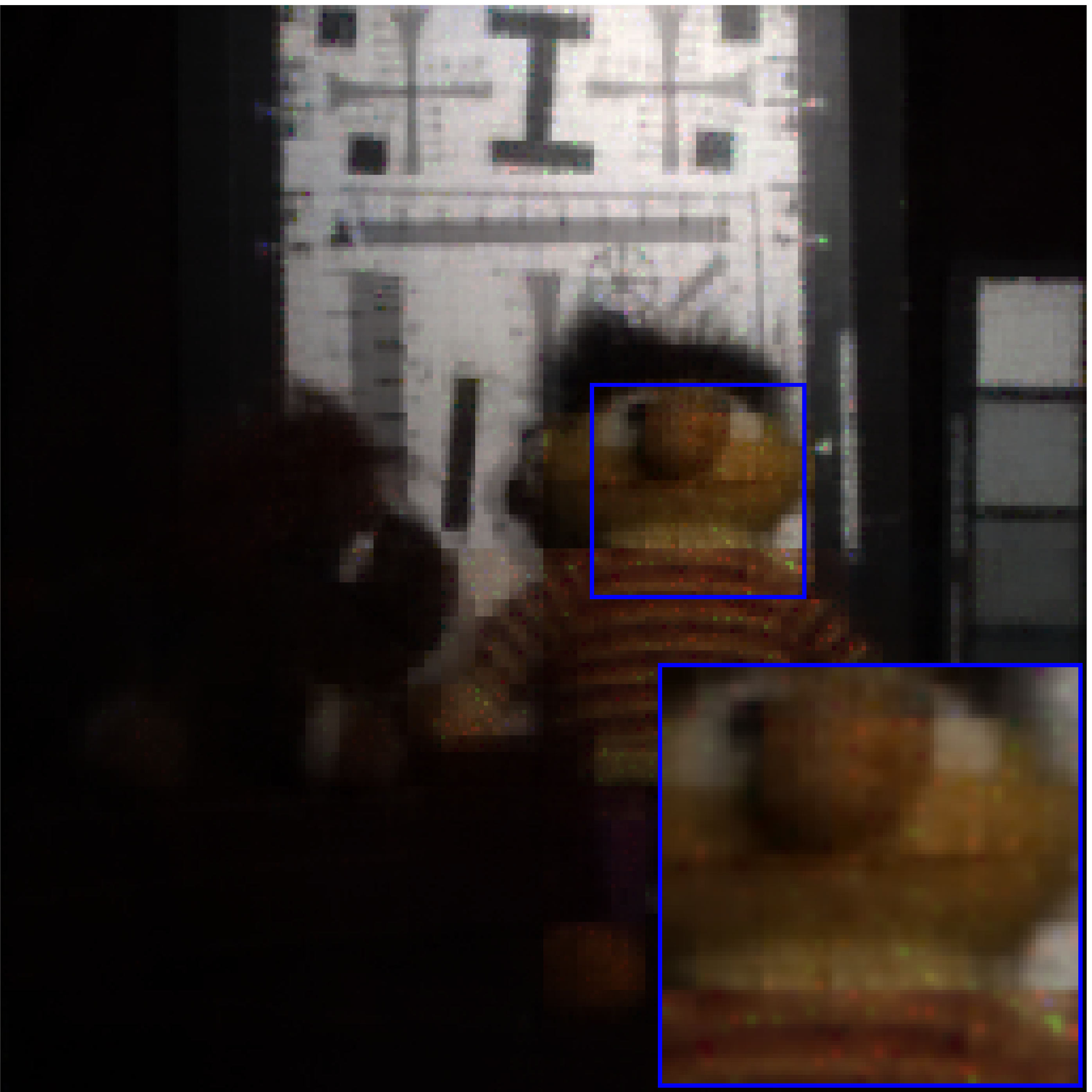}&
\includegraphics[width=0.19\textwidth]{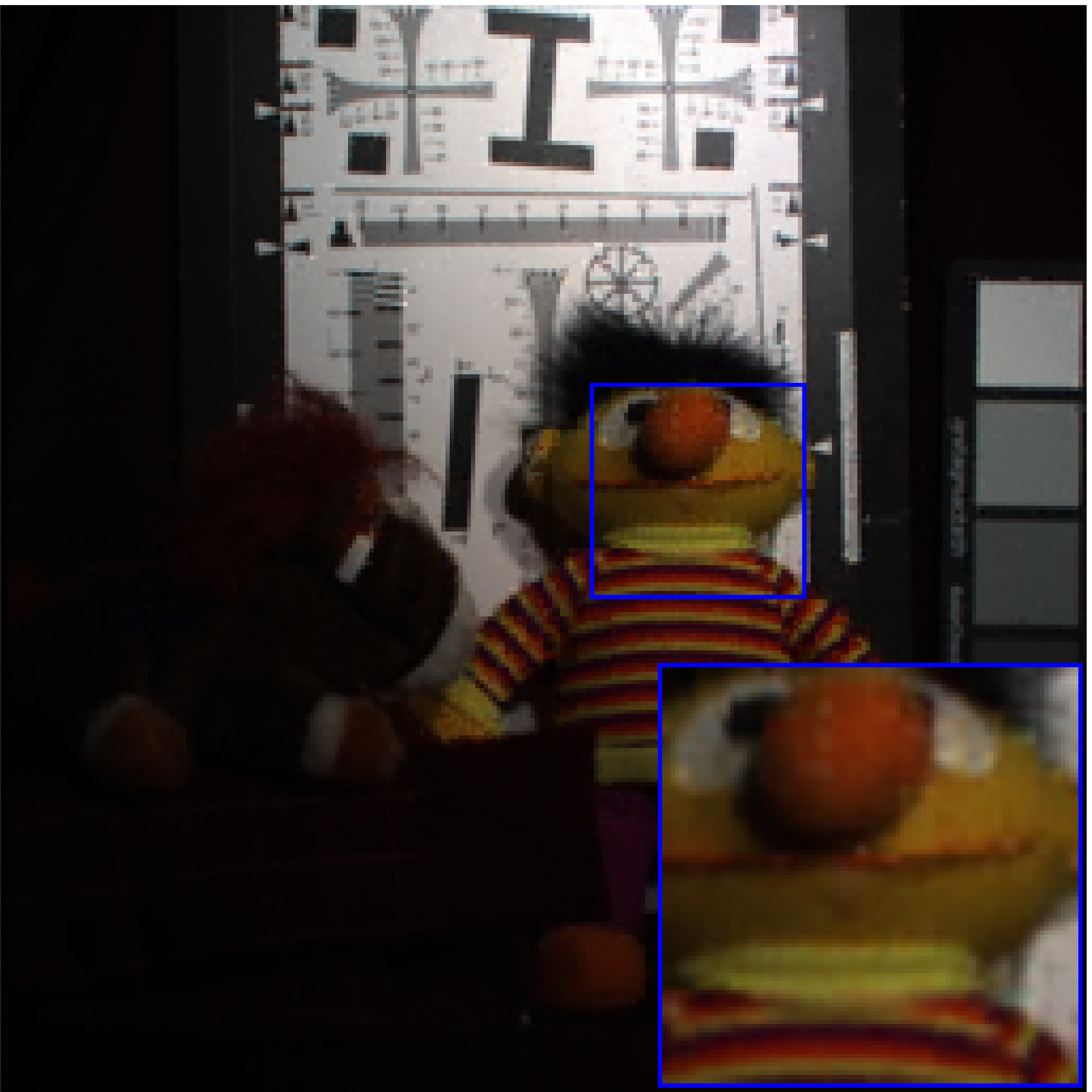}\vspace{0.01cm}\\
(f) SiLRTC-TT & (g) tSVD & (h) KBR & (i) TRNN & (j) LogTR\\
\includegraphics[width=0.19\textwidth]{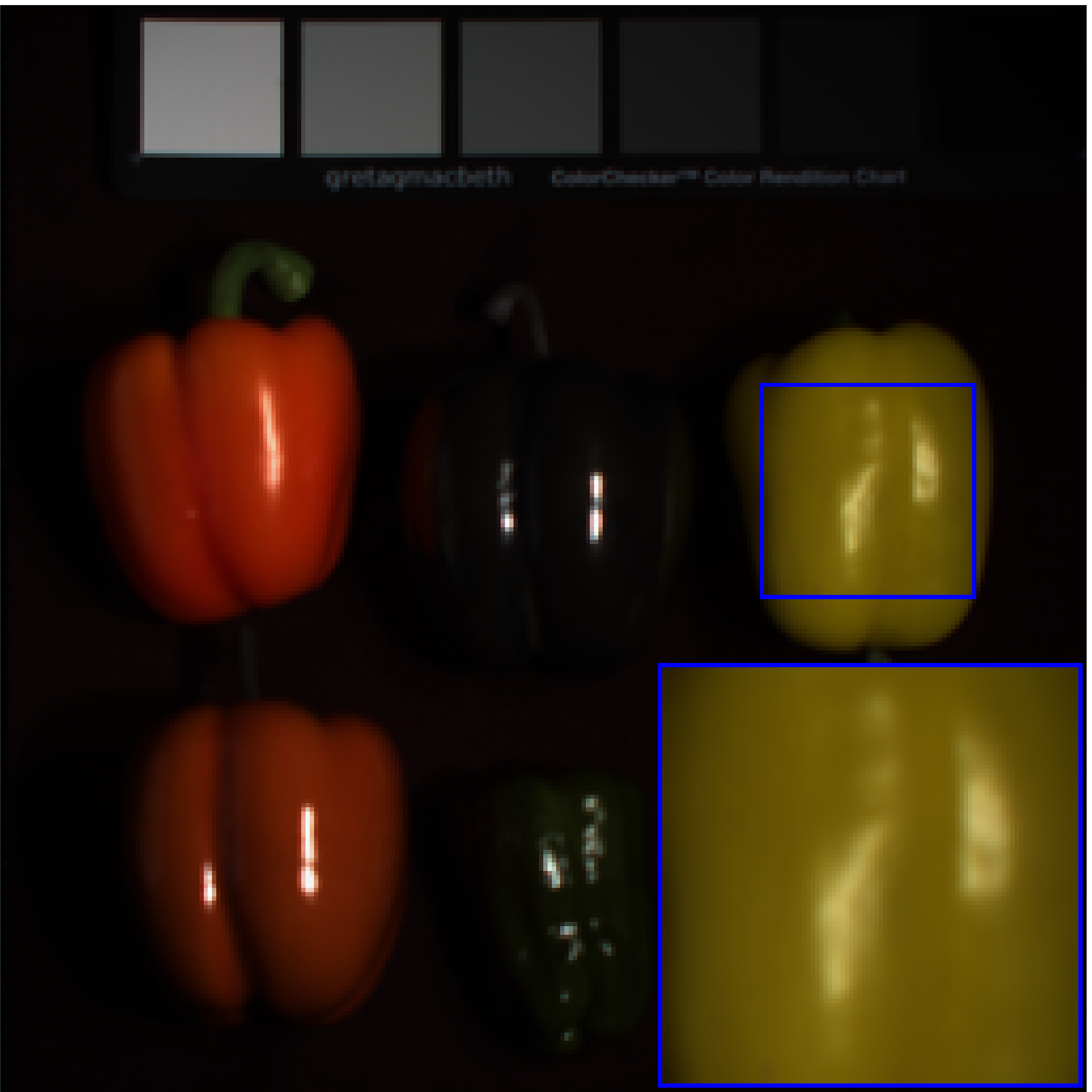}&
\includegraphics[width=0.19\textwidth]{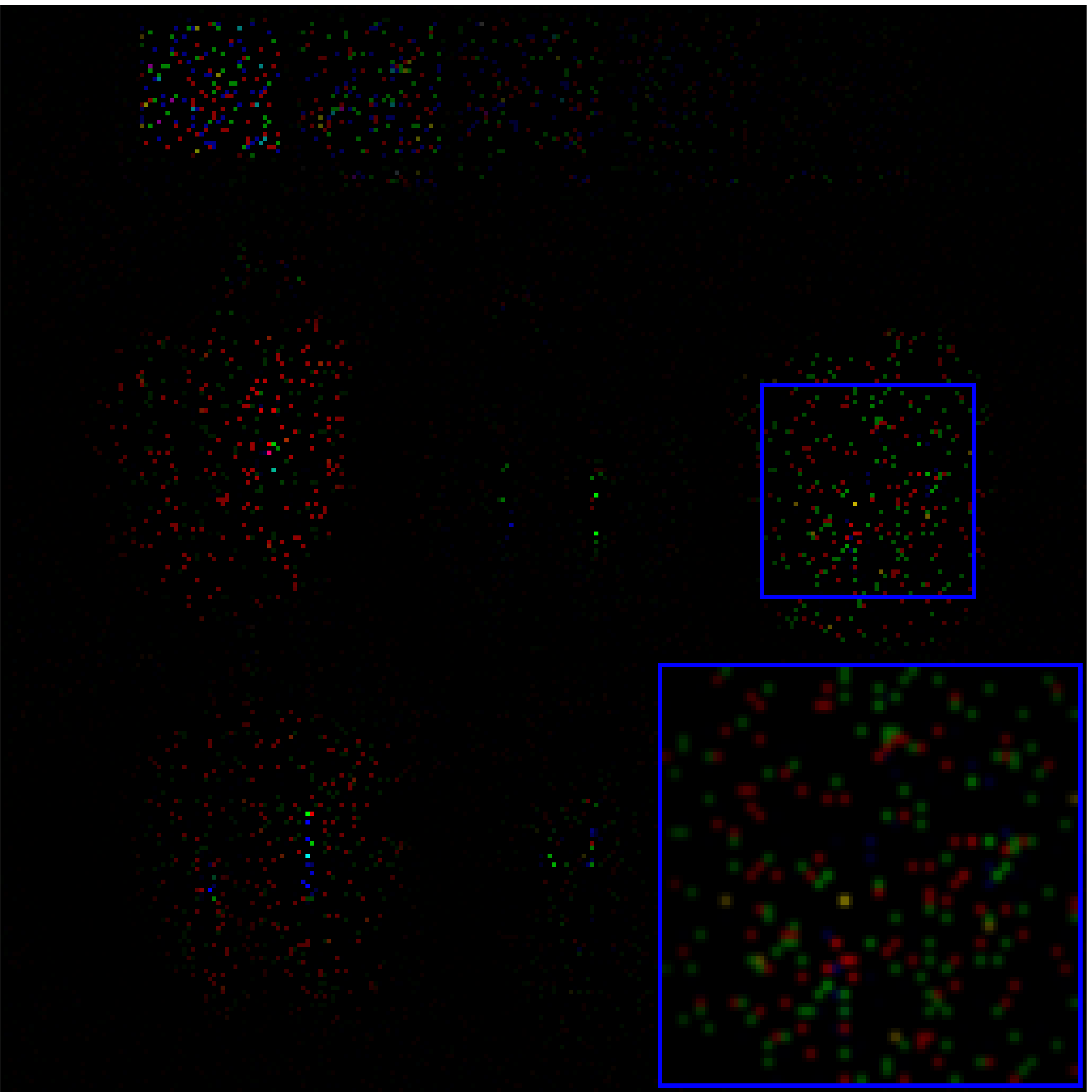}&
\includegraphics[width=0.19\textwidth]{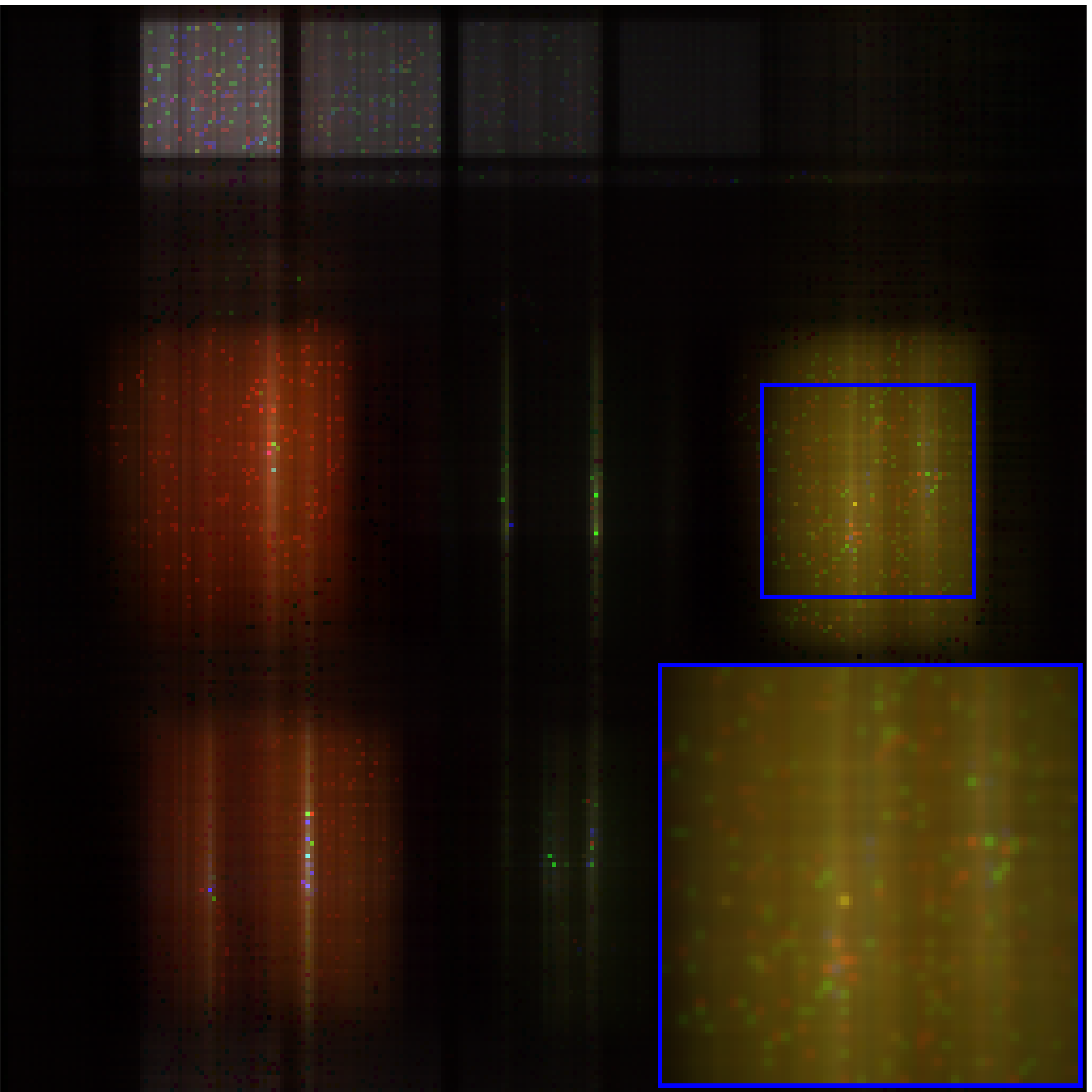}&
\includegraphics[width=0.19\textwidth]{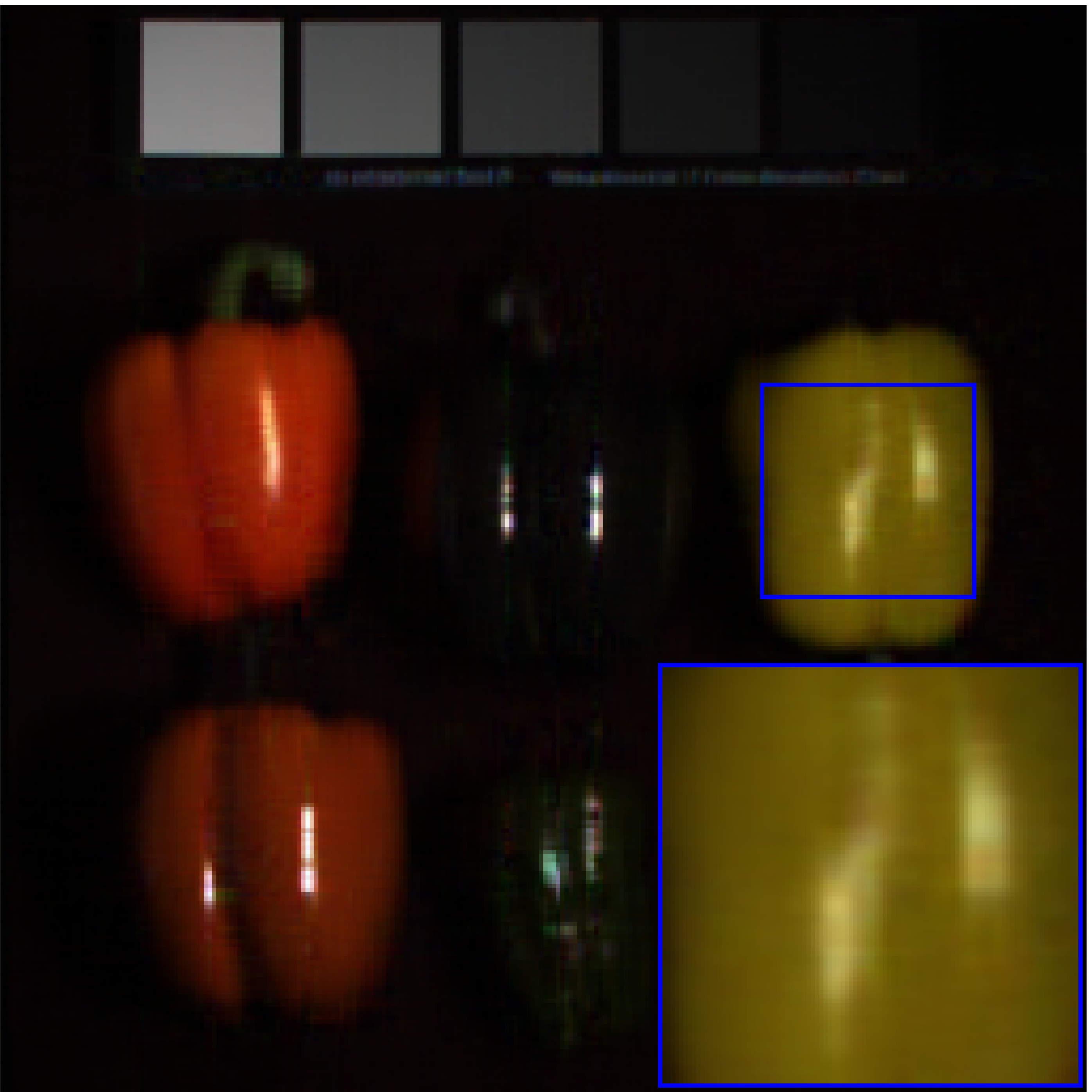}&
\includegraphics[width=0.19\textwidth]{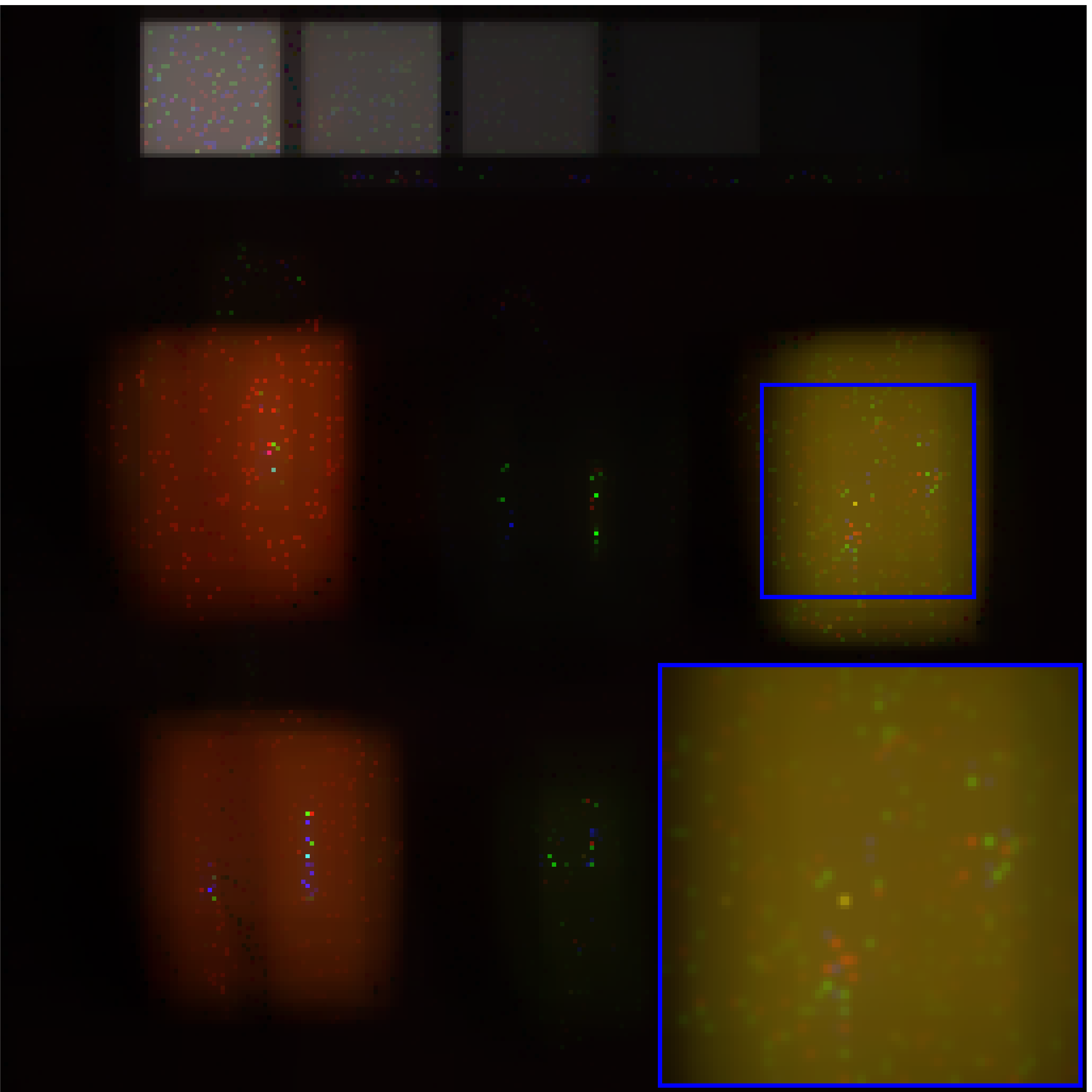}\vspace{0.01cm}\\
(a) Original & (b) Observed & (c) HaLRTC & (d) NSNN & (e) LRTC-TV\\
\includegraphics[width=0.19\textwidth]{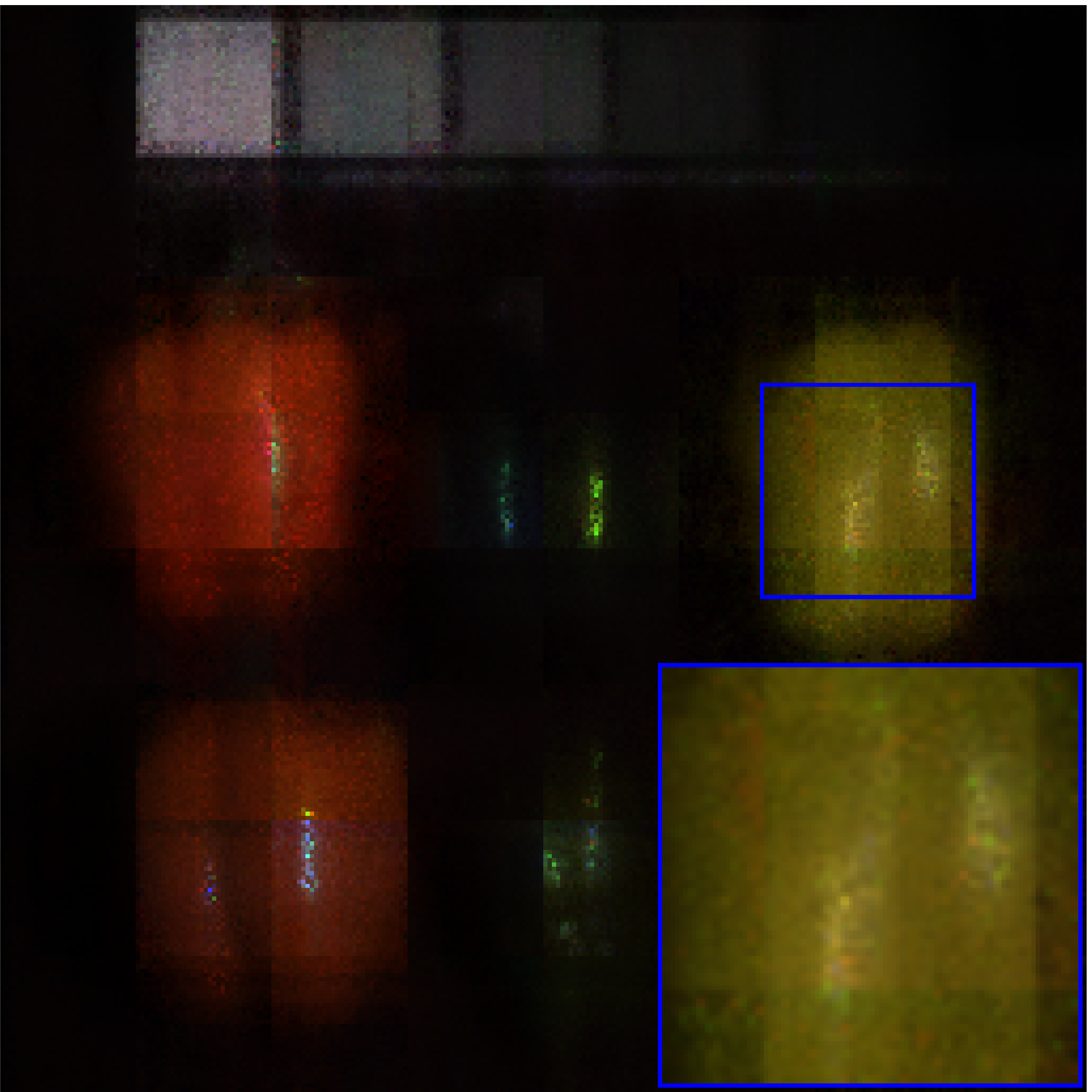}&
\includegraphics[width=0.19\textwidth]{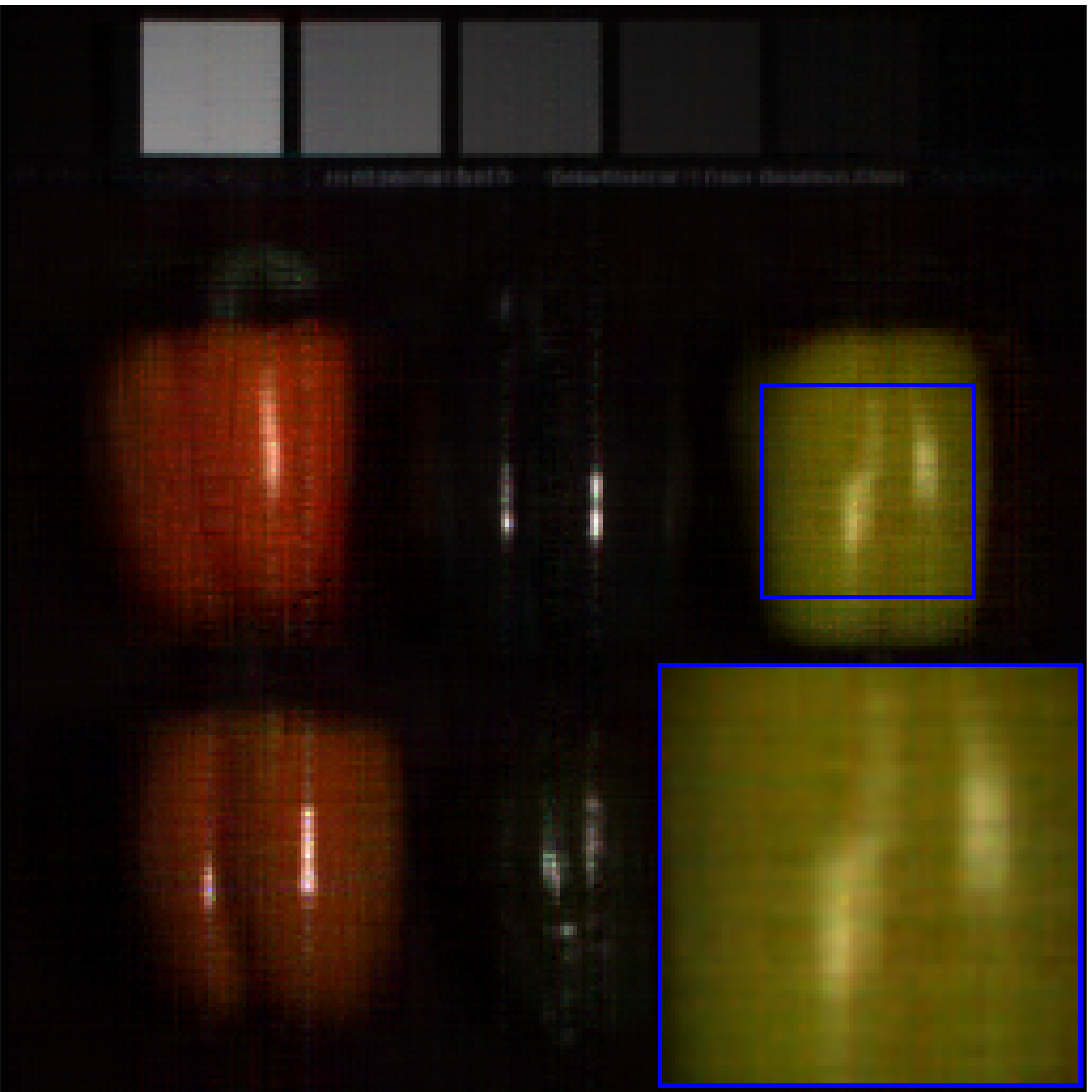}&
\includegraphics[width=0.19\textwidth]{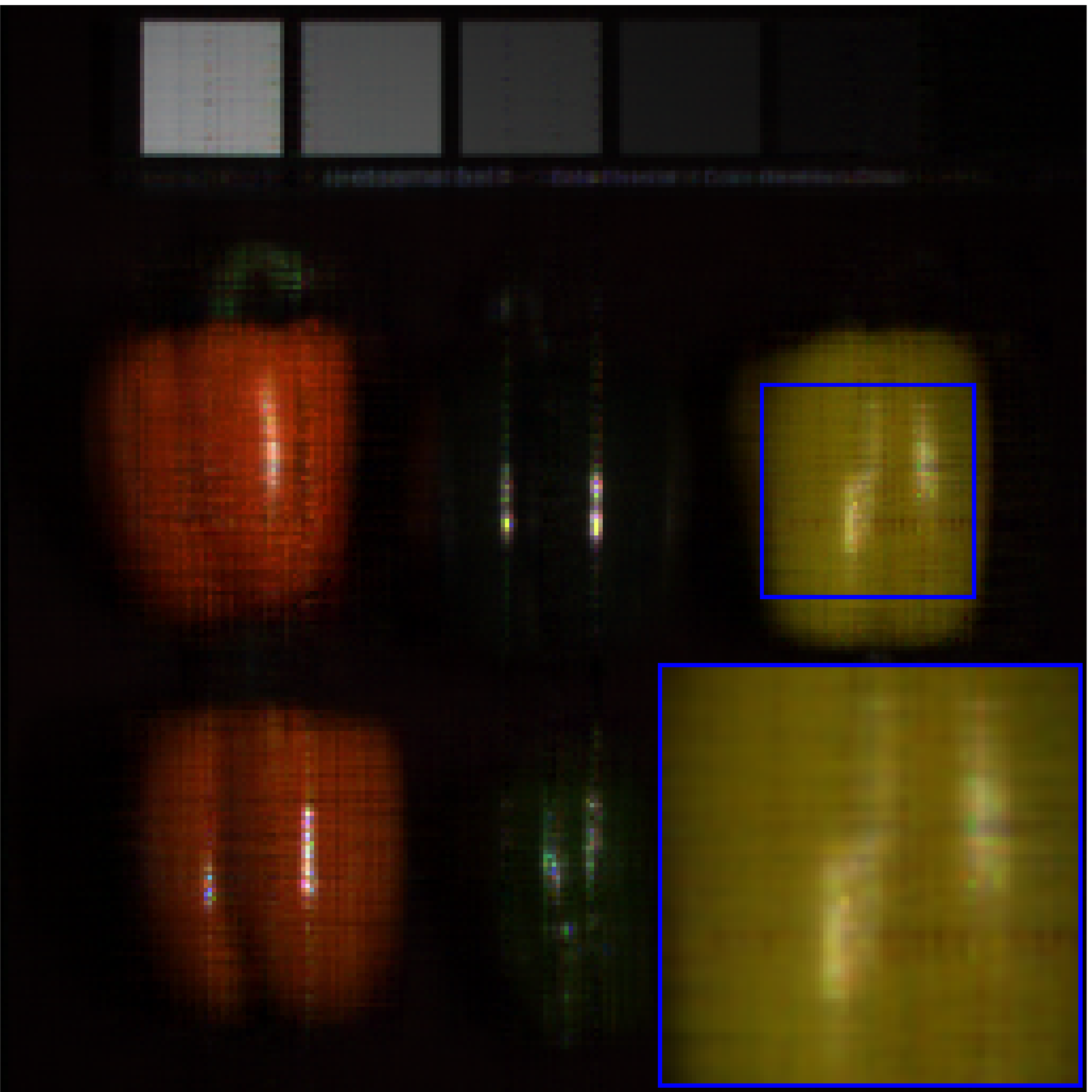}&
\includegraphics[width=0.19\textwidth]{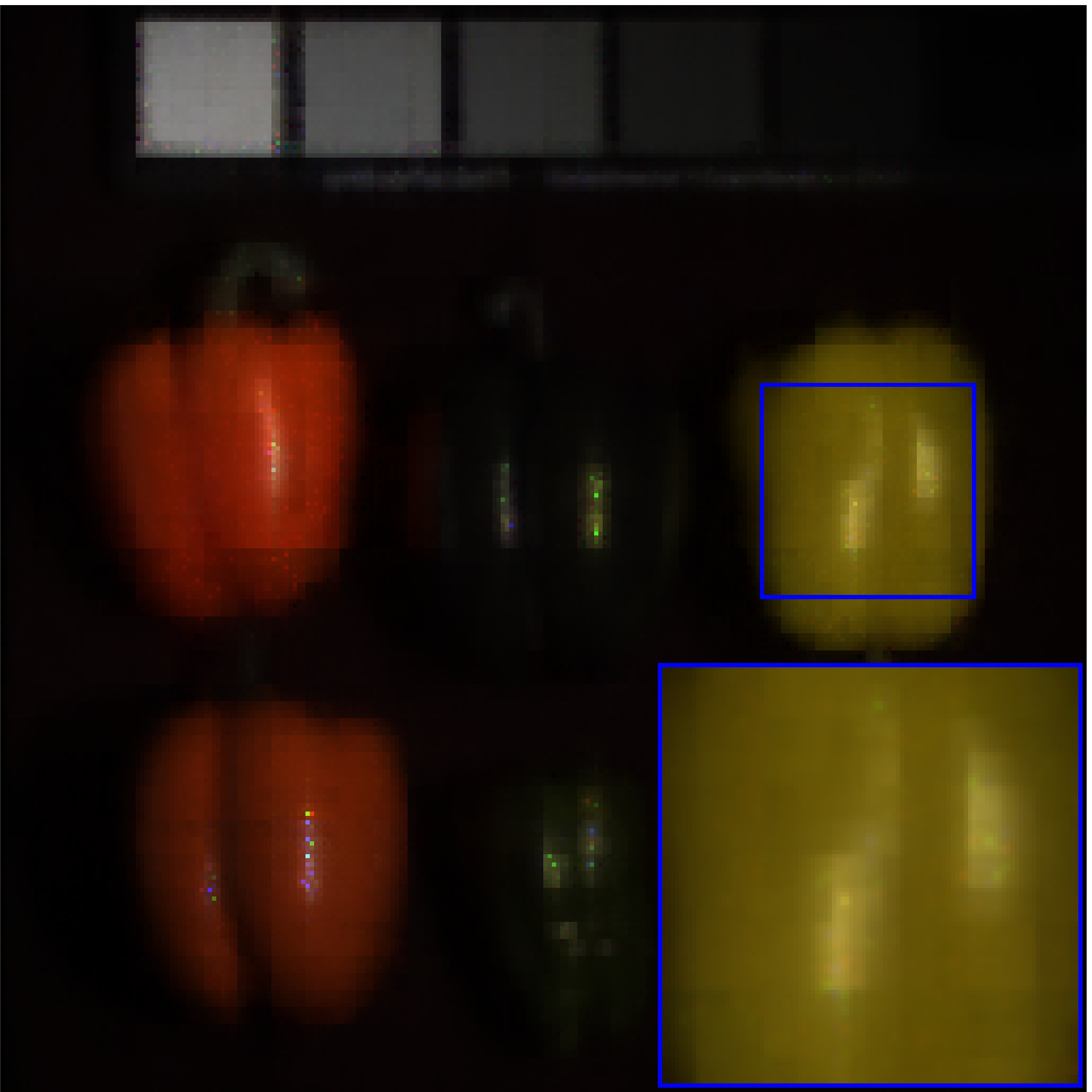}&
\includegraphics[width=0.19\textwidth]{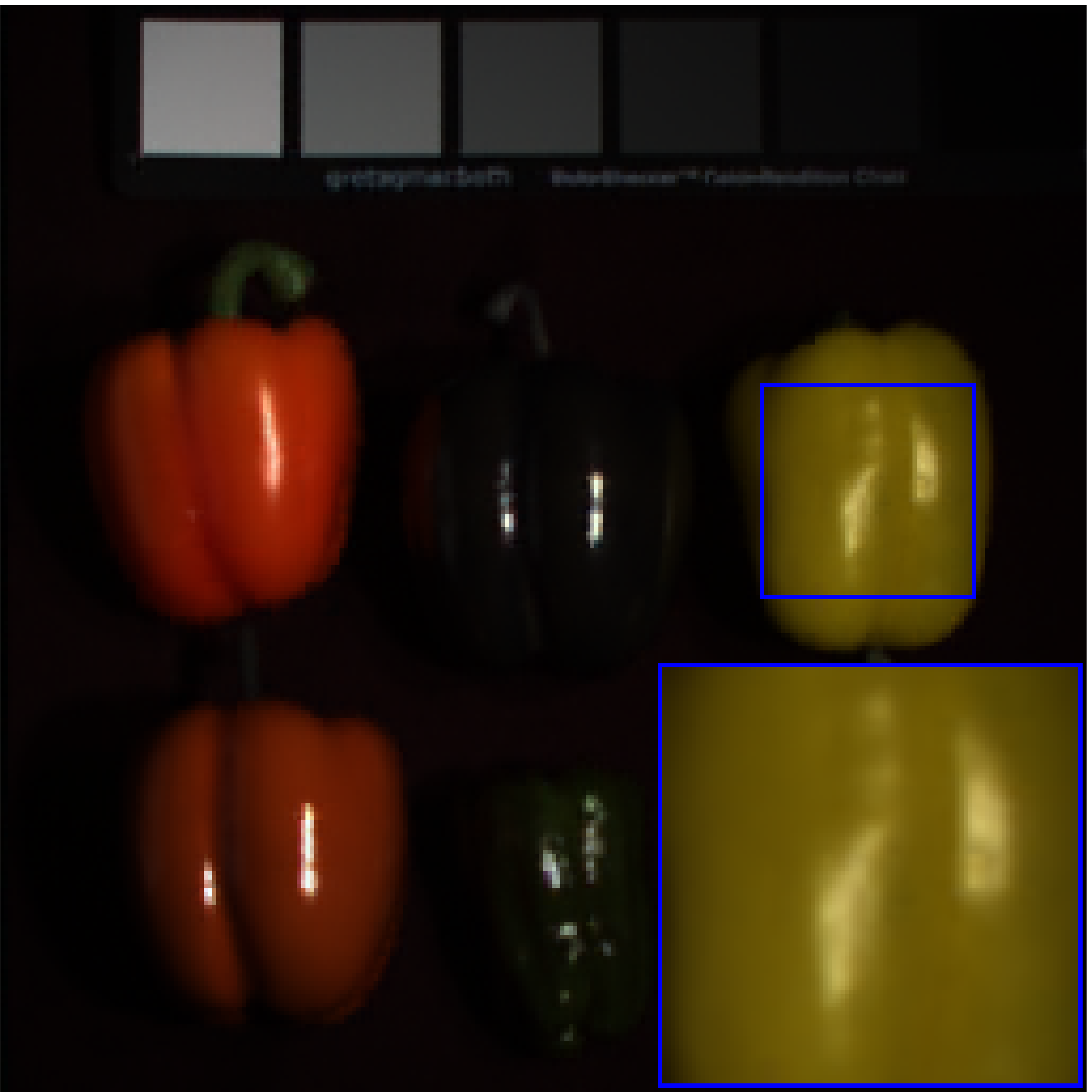}\vspace{0.01cm}\\
(f) SiLRTC-TT & (g) tSVD & (h) KBR & (i) TRNN & (j) LogTR
\end{tabular}
\caption{\small{Recovered MSIs \emph{Feathers}, \emph{Toy}, and \emph{Peppers} for random missing entries with $SR=0.05$. The color image is composed of bands 30, 20, and 10.}}
  \label{fig:msi_05}
  \end{center}\vspace{-0.3cm}
\end{figure}

\textbf{Structural missing.} We test recovering color images with structural missing entries, including random curves missing for \emph{House}, random stripes missing for \emph{Barbara}, and random texts missing for \emph{Airplane}; the results are shown in Figure \ref{fig:image_structural}. Obviously, for \emph{Barbara}, tSVD fails to recover the missing slices; HaLRTC, NSNN, and KBR recover the horizontal stripes, but remain clear vertical traces; there are ``shadows" still retained in the images recovered by LRTC-TV, SiLRTC-TT, and TRNN. For \emph{House} and \emph{Airplane}, it is clear that the outlines of curves and texts can be seen on images recovered by all compared methods. In contrast, the proposed method fills most missing areas without outlines and performs well in preserving structures and edges, which can be seen from the enlarged subregions marked by a blue box of each image. Besides, from the quality indexes reported below the recovered images, the proposed method obtains the highest PSNR and SSIM values.

\begin{table}[!htp]\scriptsize
\renewcommand\arraystretch{1.2}
\caption{PSNR and SSIM values of different methods on MSIs completion with different SRs.}
\vspace{-0.5cm}
\begin{center}
\begin{tabular}{c|c|ccccccccc}
\hline \hline
\multirow{1}{*}{MSIs} &\multicolumn{1}{c|}{SR} &Method &HaLRTC &NSNN &LRTC-TV & SiLRTC-TT & tSVD &KBR & TRNN & LogTR \\  \hline
\multirow{6}{*}{\emph{Feathers}}
          &\multirow{2}{*}{0.05}& PSNR    &22.27	  &30.05     &20.70   &23.75   &27.46	    &28.57     &27.00       &\textbf{36.35}  \\
          &                     & SSIM    &0.7058	  &0.8526    &0.7408  &0.7610  &0.7720	    &0.8471    &0.8742      &\textbf{0.9730} \\ \cline{2-11}

          &\multirow{2}{*}{0.1} & PSNR    &25.22	  &32.99     &25.43   &29.12   &31.66       &37.63     &32.34       &\textbf{42.89}  \\
          &                     & SSIM    &0.8056	  &0.9034    &0.8499  &0.8831  &0.8792      &0.9606    &0.9521      &\textbf{0.9914} \\ \cline{2-11}

          &\multirow{2}{*}{0.2} & PSNR    &29.31      &37.29     &29.84   &37.40   &36.70      	&44.35     &38.81       &\textbf{50.37}  \\
          &                     & SSIM    &0.9012     &0.9529    &0.9309  &0.9738  &0.9505      &0.9882    &0.9863      &\textbf{0.9977} \\
          \hline
\multirow{6}{*}{\emph{Toy}}
          &\multirow{2}{*}{0.05}& PSNR    &20.20	  &30.49     &18.21   &23.25   &28.66	    &28.85     &27.67       &\textbf{36.91}  \\
          &                     & SSIM    &0.6691	  &0.8867    &0.6799  &0.7492  &0.8413	    &0.8593    &0.8862      &\textbf{0.9771} \\ \cline{2-11}

          &\multirow{2}{*}{0.1} & PSNR    &24.72	  &33.63     &24.96   &29.10   &32.55       &37.60     &33.30       &\textbf{43.96}  \\
          &                     & SSIM    &0.8021	  &0.9345    &0.8294  &0.8890  &0.9164      &0.9718    &0.9601      &\textbf{0.9943} \\ \cline{2-11}

          &\multirow{2}{*}{0.2} & PSNR    &29.78      &37.45     &30.39   &37.59   &37.71      	&45.29     &40.85       &\textbf{52.88}  \\
          &                     & SSIM    &0.9107     &0.9682    &0.9329  &0.9783  &0.9665      &0.9930    &0.9913      &\textbf{0.9989} \\
          \hline
\multirow{6}{*}{\emph{Peppers}}
          &\multirow{2}{*}{0.05}& PSNR    &24.56	  &38.13     &25.30   &26.95   &32.52	    &32.55     &29.66       &\textbf{42.24}  \\
          &                     & SSIM    &0.7464	  &0.9576    &0.8602  &0.8335  &0.8558	    &0.9105    &0.9381      &\textbf{0.9930} \\
          \cline{2-11}
          &\multirow{2}{*}{0.1} & PSNR    &30.54	  &42.77     &29.66   &33.07   &37.60       &44.73     &33.95       &\textbf{48.28}  \\
          &                     & SSIM    &0.9008	  &0.9815    &0.9381  &0.9418  &0.9384      &0.9905    &0.9790      &\textbf{0.9982} \\
          \cline{2-11}
          &\multirow{2}{*}{0.2} & PSNR    &36.74      &47.48     &36.41   &41.71   &43.50      	&52.71     &40.32       &\textbf{55.77}  \\
          &                     & SSIM    &0.9680     &0.9931    &0.9812  &0.9894  &0.9803      &0.9979    &0.9953      &\textbf{0.9995} \\
          \hline \hline
\end{tabular}\label{table:msi}
\end{center}
\end{table}
\subsection{MSIs}
We test the proposed method on the CAVE MSI database containing 32 real-world scenes. For MSIs, we test the random sampling case and set $SR= 0.01, 0.05, 0.1, \textrm{and} \ 0.2$. Before applying LogTR to fill the missing entries, we transform the MSI data of size $256\times 256 \times 31$ to a ninth-order tensor of size $4 \times 4 \times 4 \times 4 \times 4 \times 4  \times 4 \times 4 \times 31$.

Figures \ref{fig:msi_01} and \ref{fig:msi_05} show the visual results for \emph{Toy}, \emph{Feather}, and \emph{Peppers} MSIs at $SR=0.01$ and $0.05$. In the extreme case $SR=0.01$ in Figure \ref{fig:msi_01}, all compared methods hardly restore the outline of the original images, while LogTR can obtain promising visual results for the testing data. From Figure \ref{fig:msi_05}, we observe that LogTR can maintain the smoothness area of \emph{Peppers} and preserve the textures and details of \emph{Feathers} and \emph{Toy}.

Table \ref{table:msi} summarizes the PSNR and SSIM values of \emph{Feathers}, \emph{Peppers}, and \emph{Toy} reconstructed by eight LRTC methods for three SRs. Figure \ref{fig:MSI_all_psnr} lists the comparison of the PSNR values by different methods on the whole CAVE dataset with $SR=0.1$. Form these quality indexes, LogTR achieves superior performance as before.

\begin{figure}[!t]
\scriptsize\setlength{\tabcolsep}{0.9pt}
\begin{center}
\begin{tabular}{cc}
\includegraphics[width=0.99\textwidth]{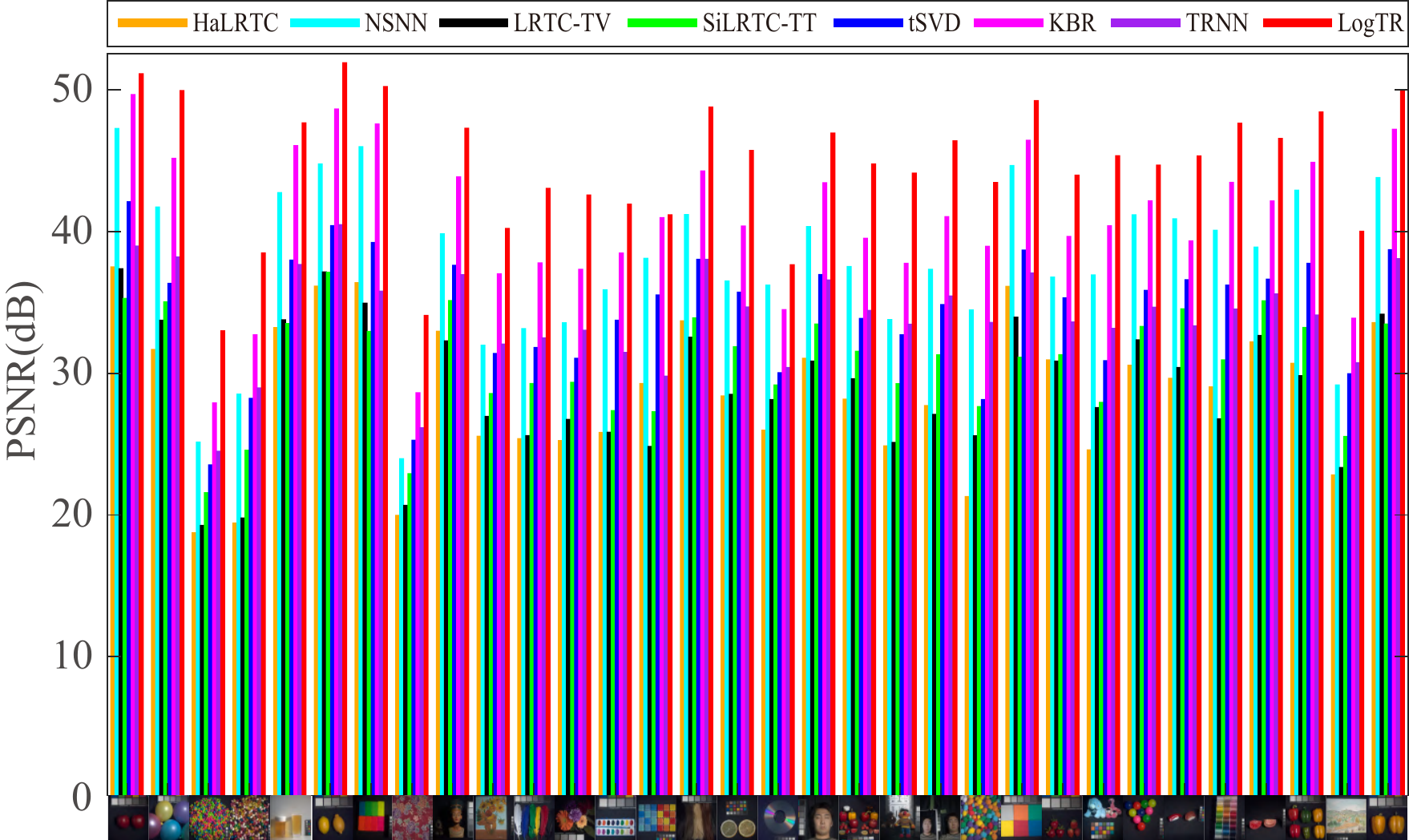} \vspace{0.1cm}\\
\end{tabular}
\caption{\small{PSNR values of different methods on the dataset CAVE with $SR=0.1$.}}
  \label{fig:MSI_all_psnr}
  \end{center}\vspace{-0.3cm}
\end{figure}

\begin{figure}[!t]
\scriptsize\setlength{\tabcolsep}{0.5pt}
\begin{center}
\begin{tabular}{ccccc}
\includegraphics[width=0.19\textwidth]{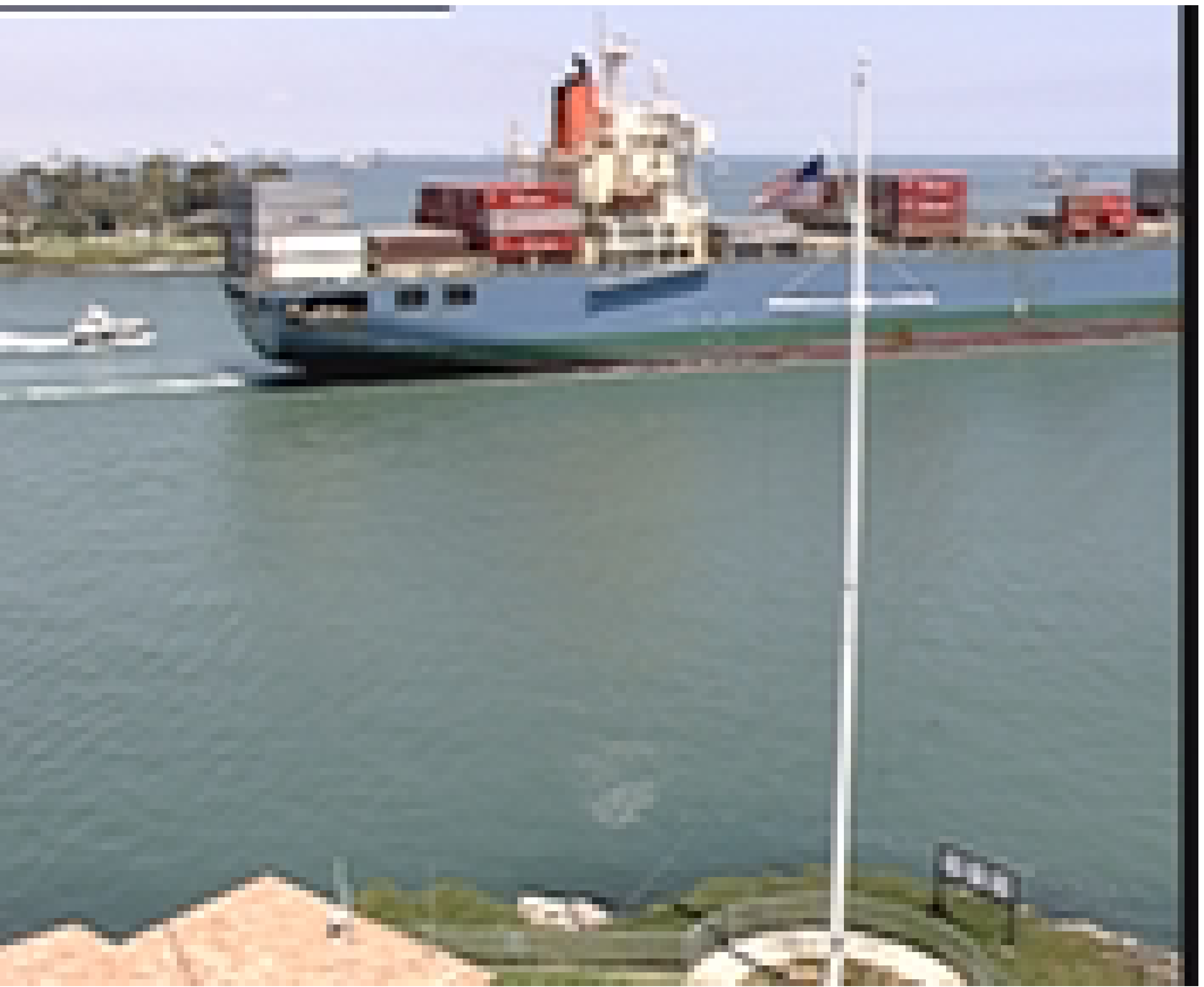}&
\includegraphics[width=0.19\textwidth]{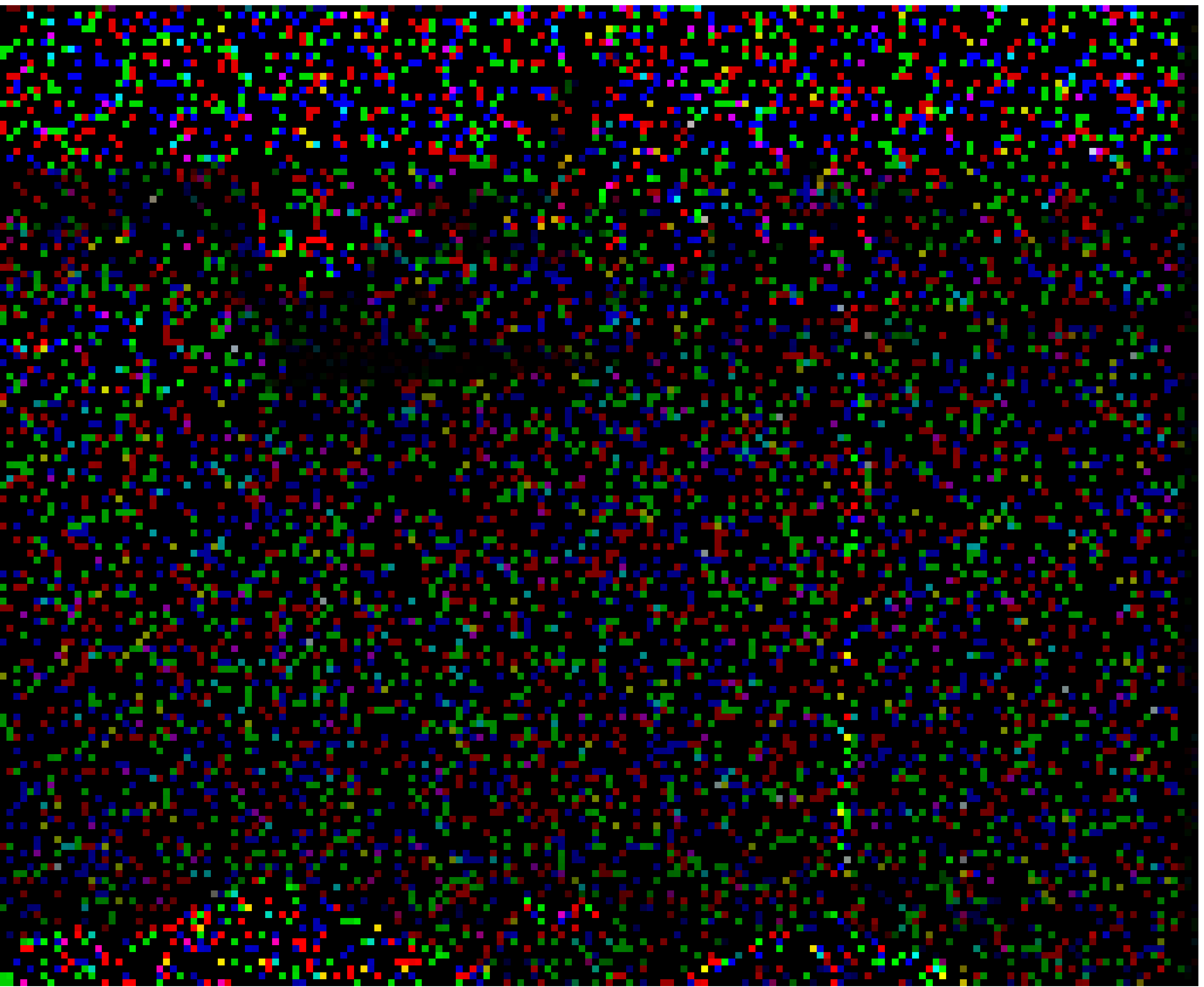}&
\includegraphics[width=0.19\textwidth]{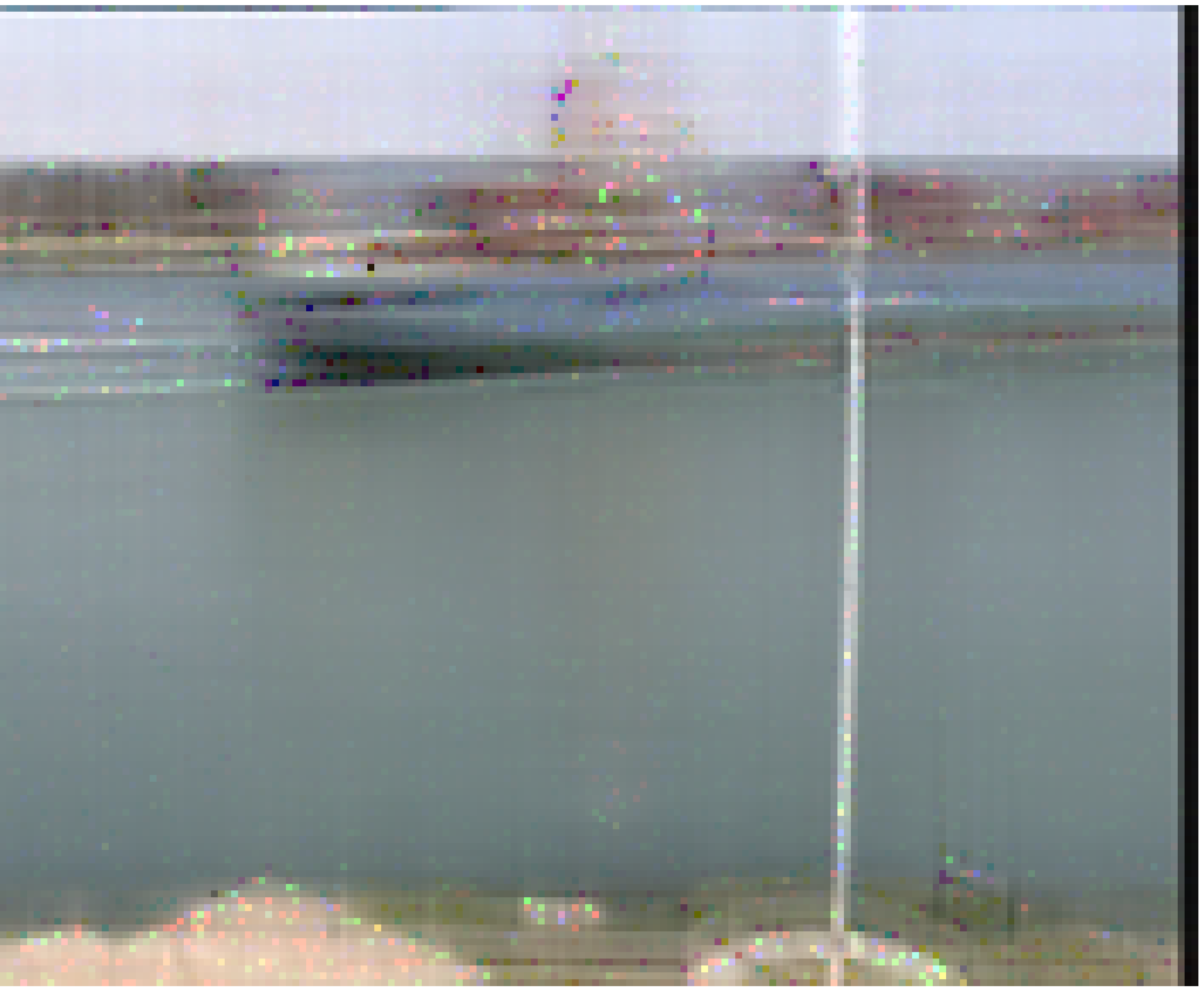}&
\includegraphics[width=0.19\textwidth]{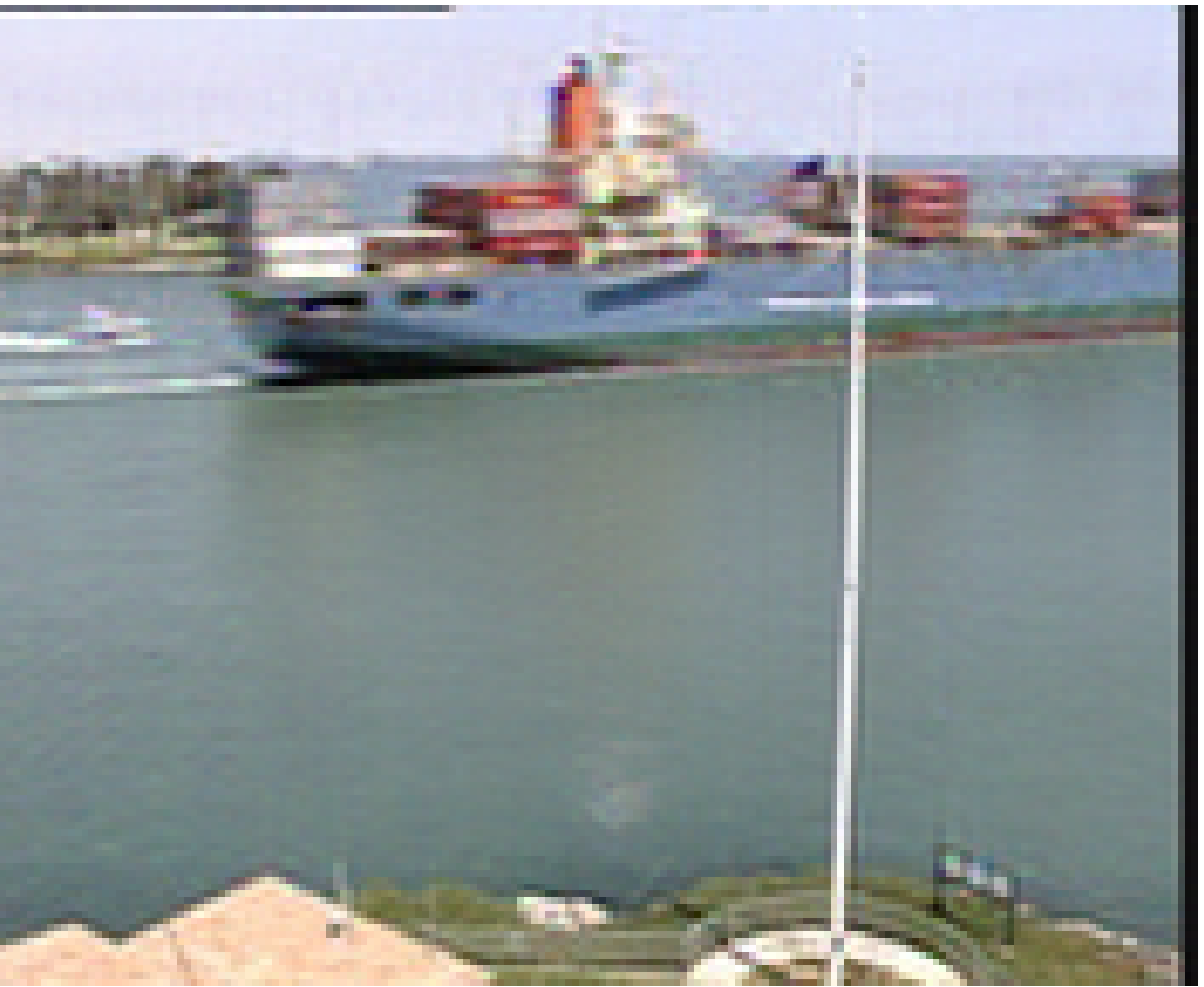}&
\includegraphics[width=0.19\textwidth]{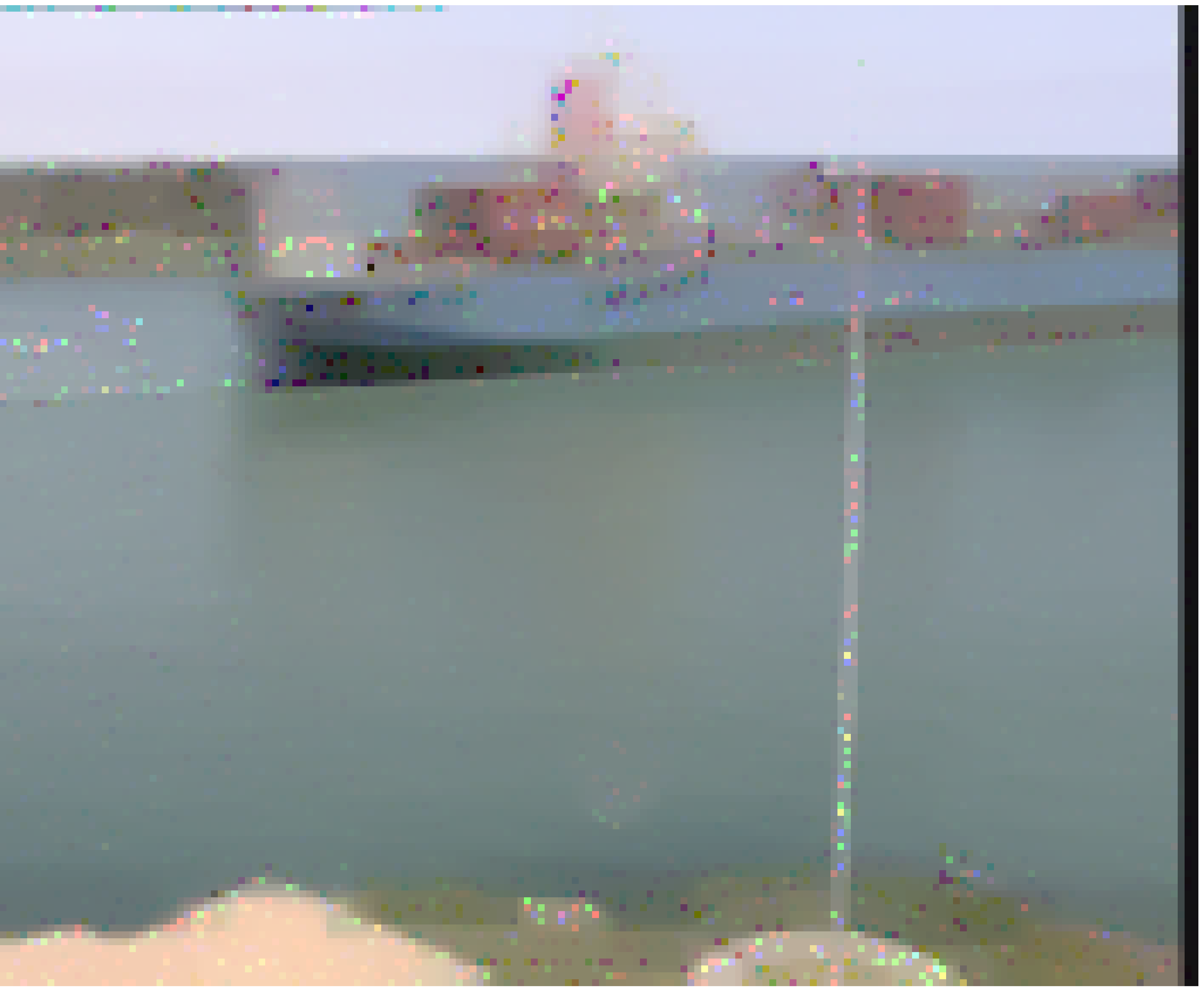}\\
(a) Original & (b) Observed & (c) HaLRTC & (d) NSNN & (e) LRTC-TV \\
\includegraphics[width=0.19\textwidth]{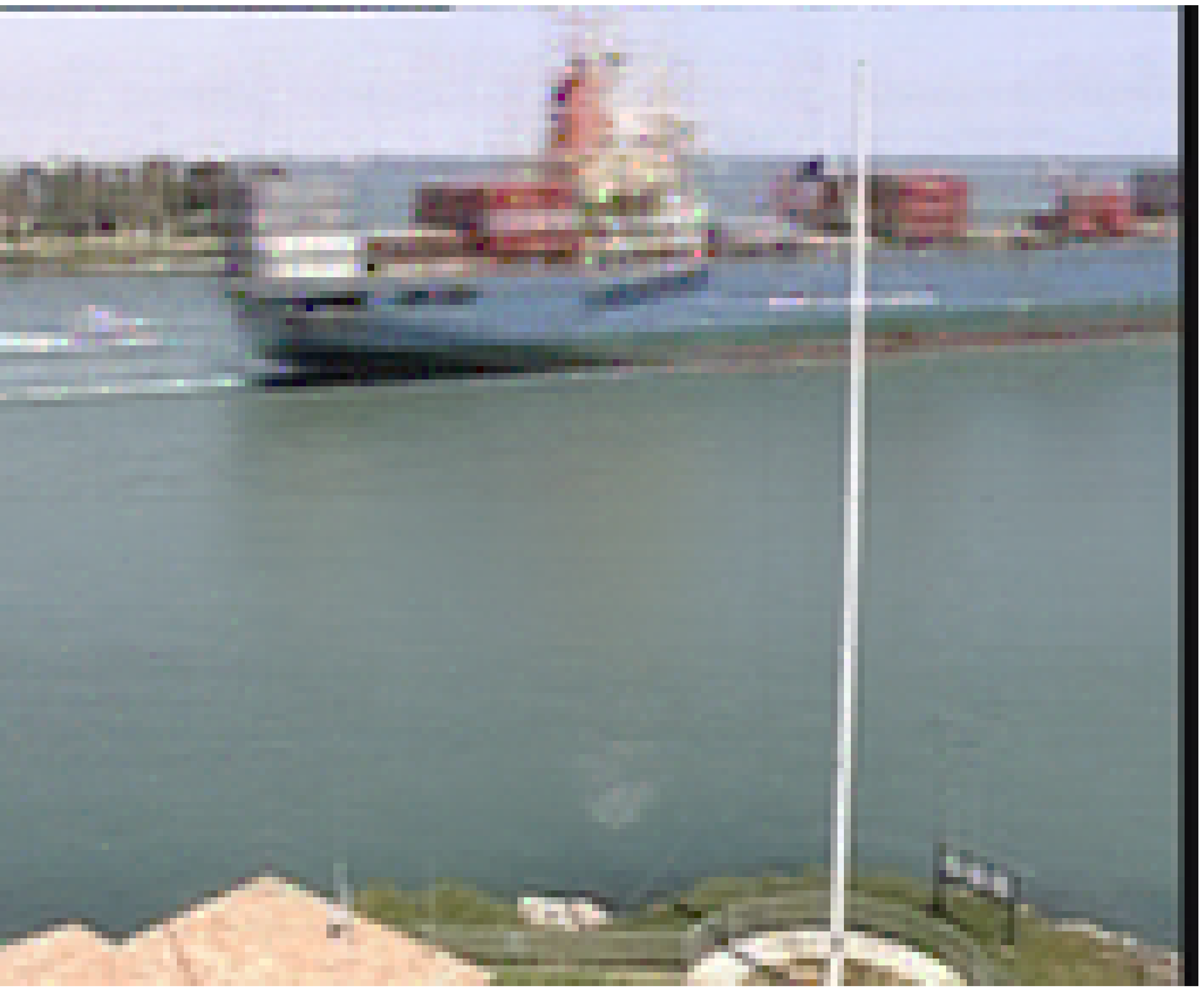}&
\includegraphics[width=0.19\textwidth]{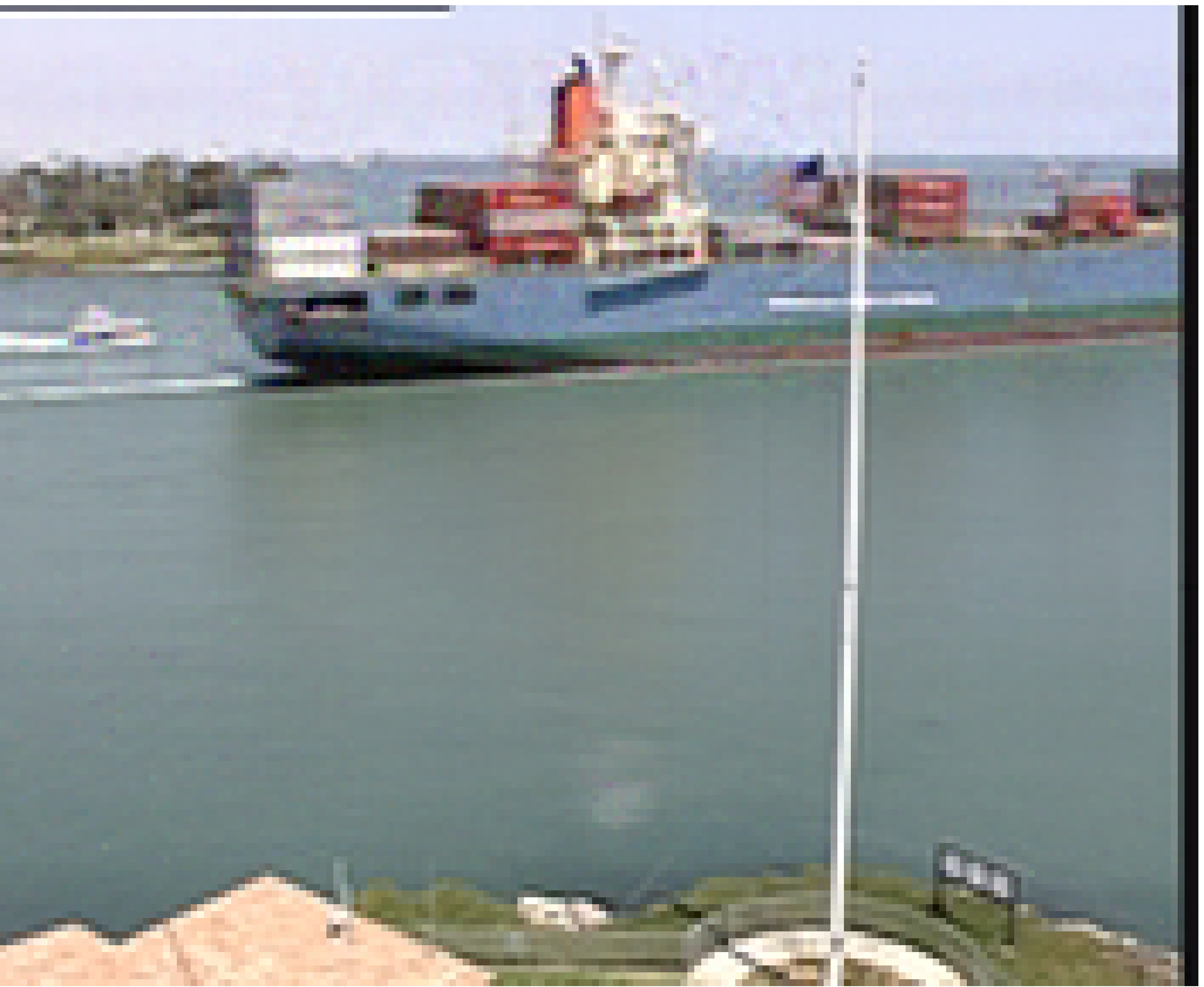}&
\includegraphics[width=0.19\textwidth]{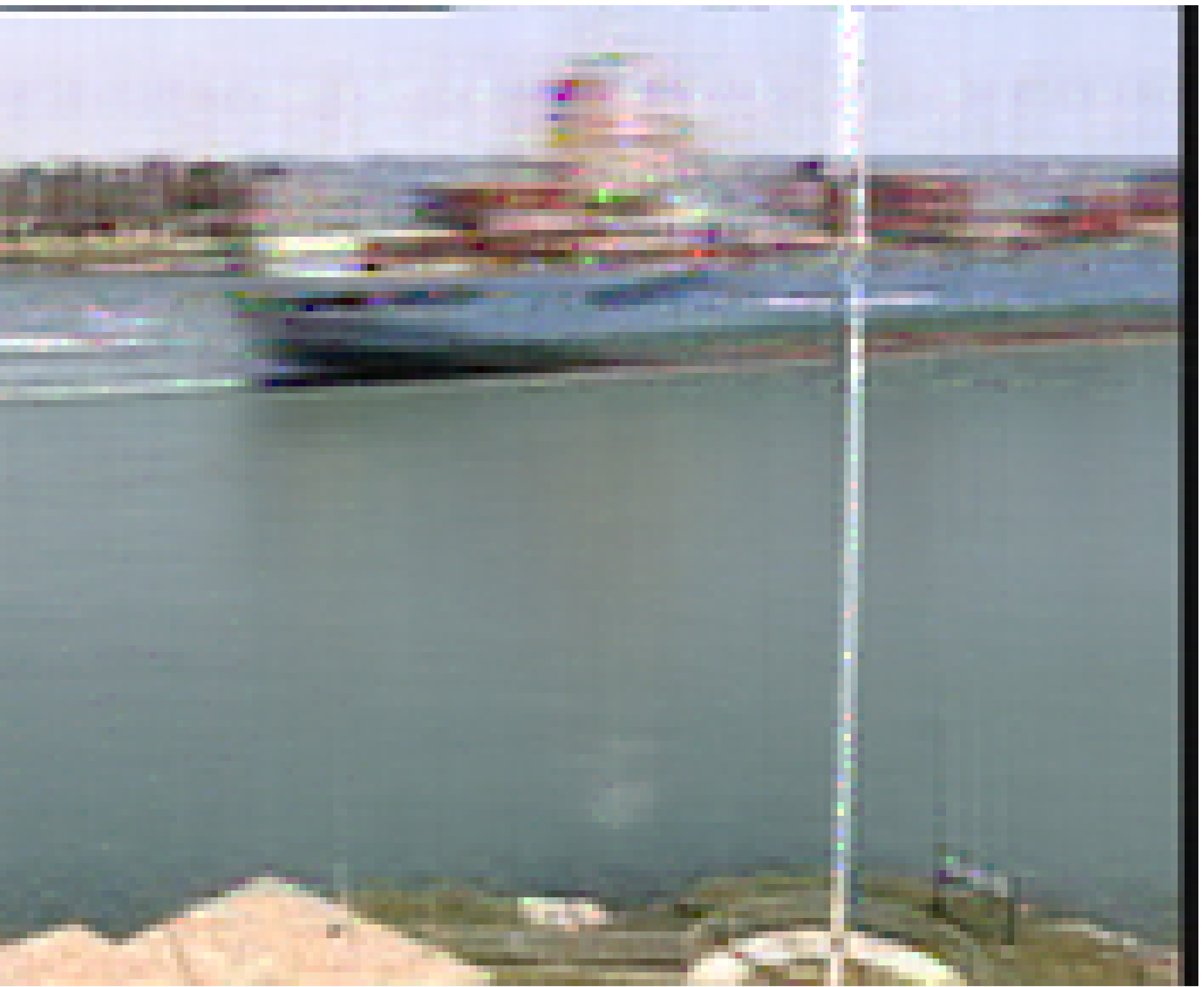}&
\includegraphics[width=0.19\textwidth]{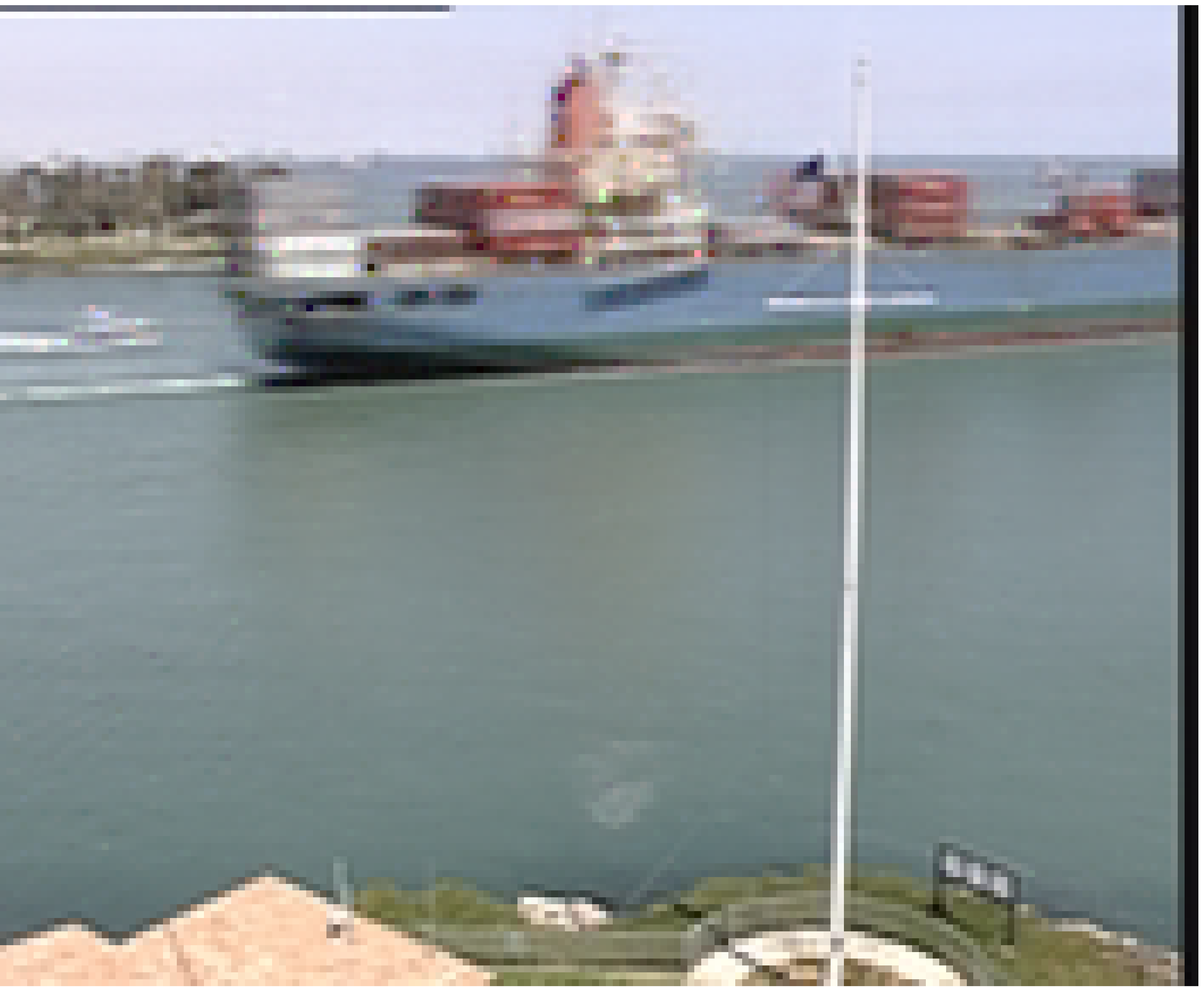}&
\includegraphics[width=0.19\textwidth]{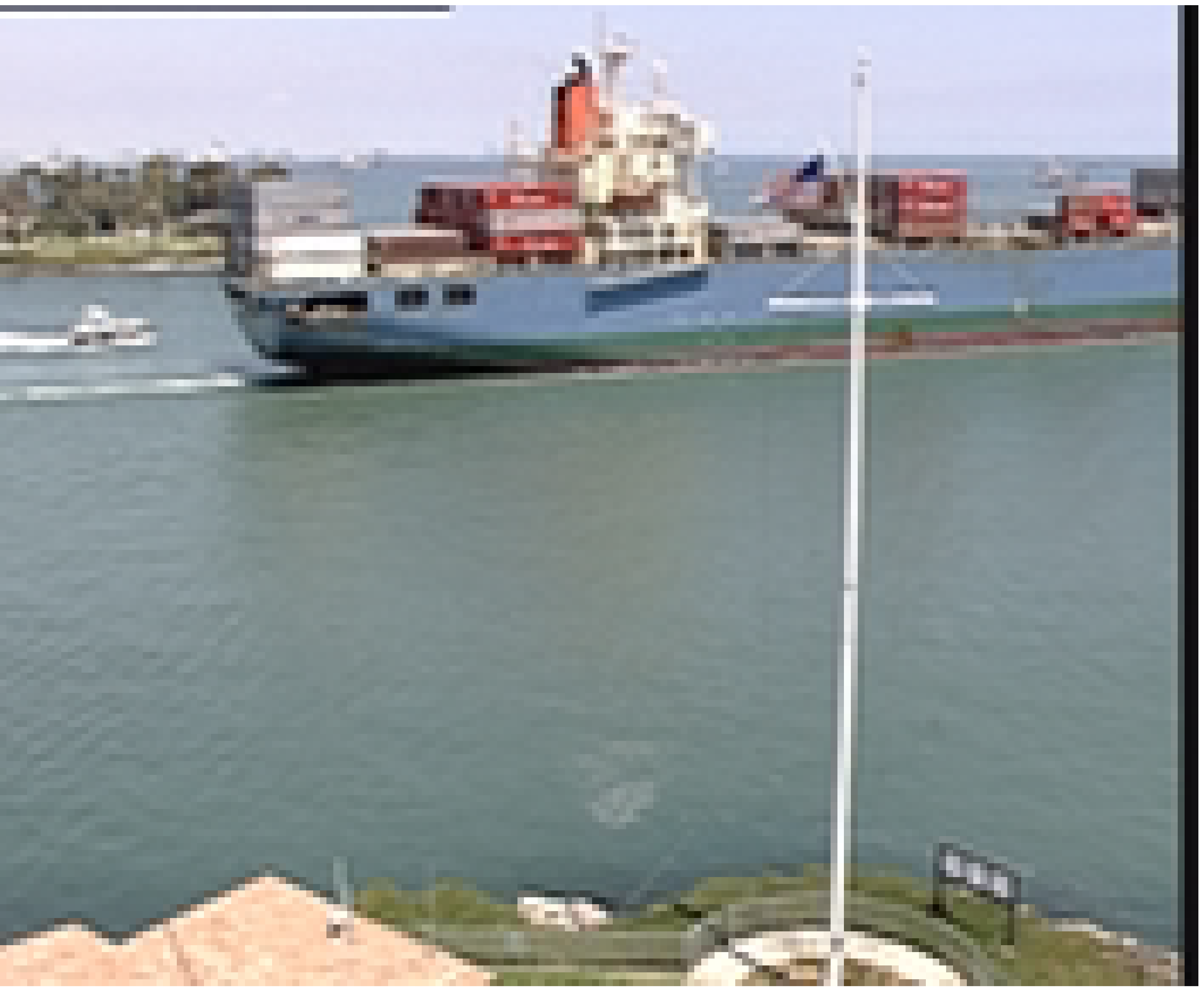}\\
(f) SiLRTC-TT & (g) tSVD & (h) KBR & (i) TRNN & (j) LogTR\\
\includegraphics[width=0.19\textwidth]{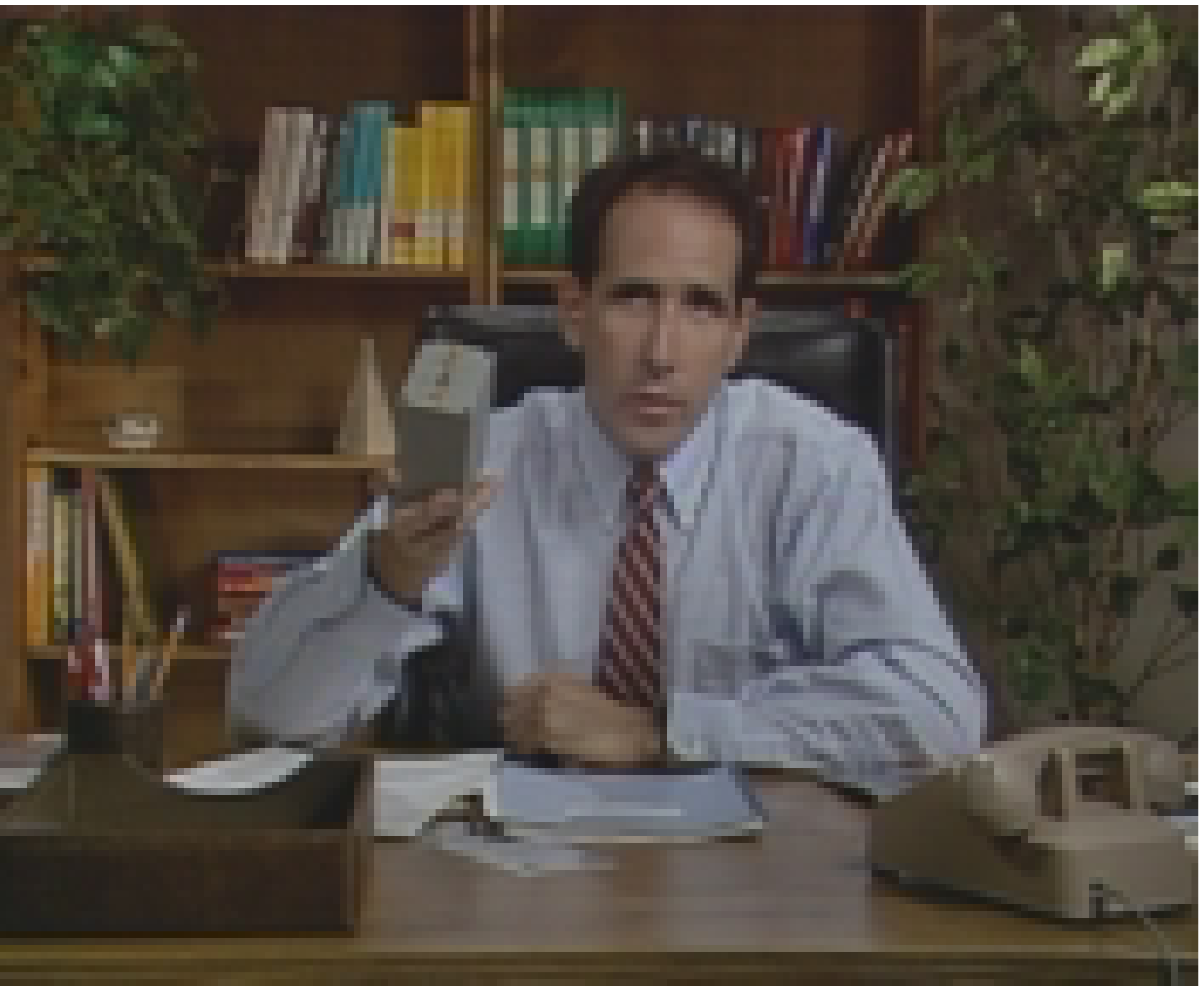}&
\includegraphics[width=0.19\textwidth]{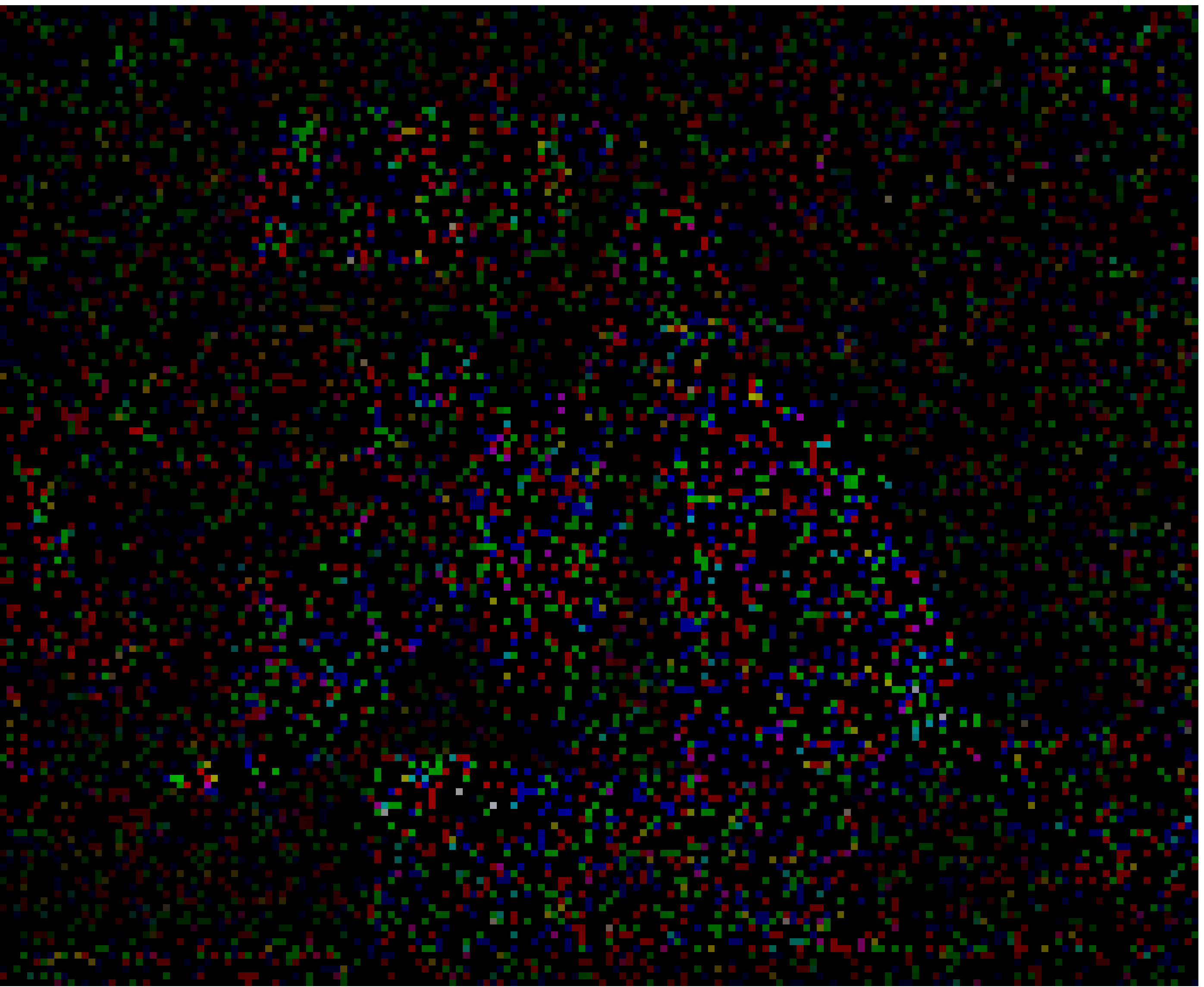}&
\includegraphics[width=0.19\textwidth]{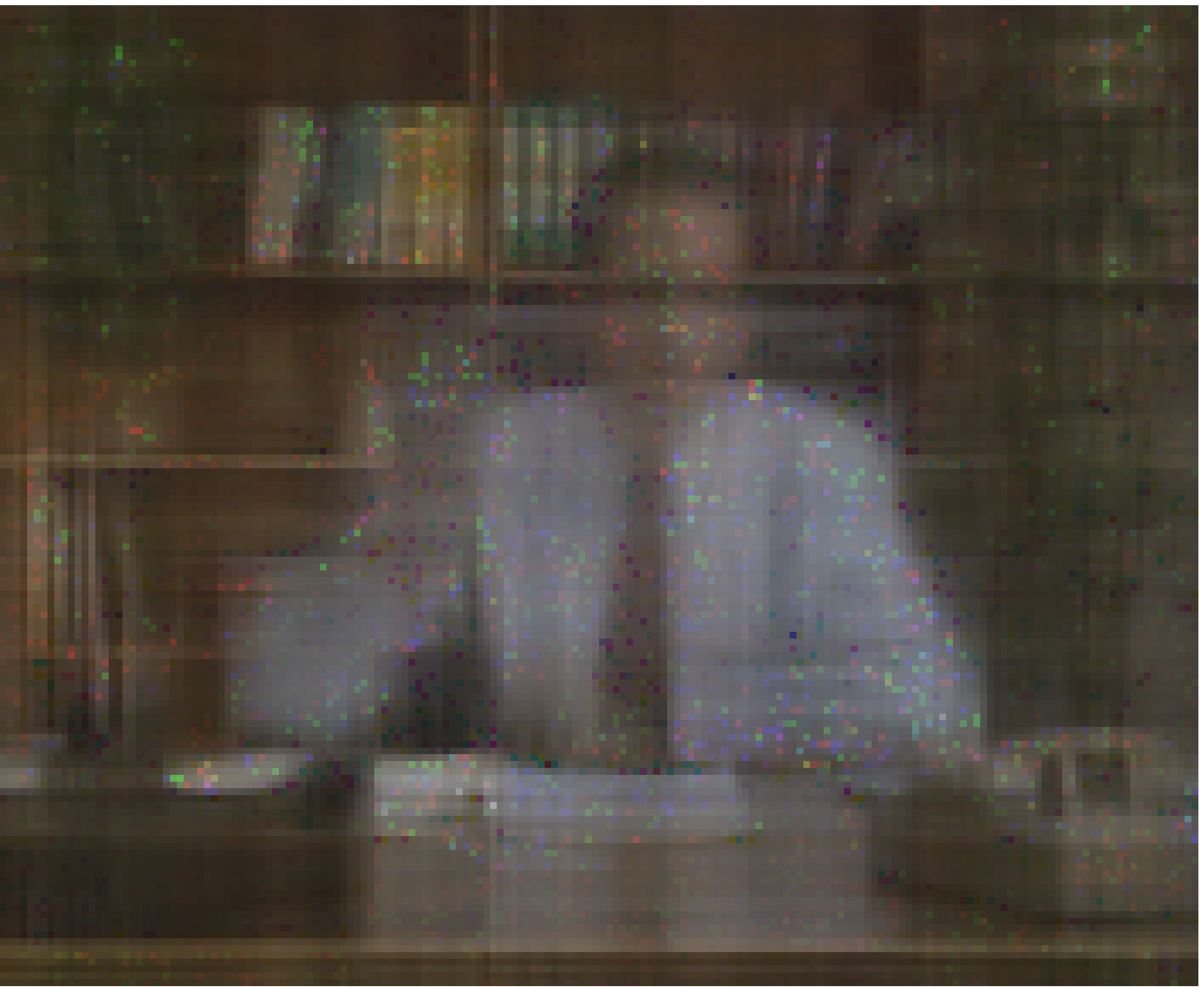}&
\includegraphics[width=0.19\textwidth]{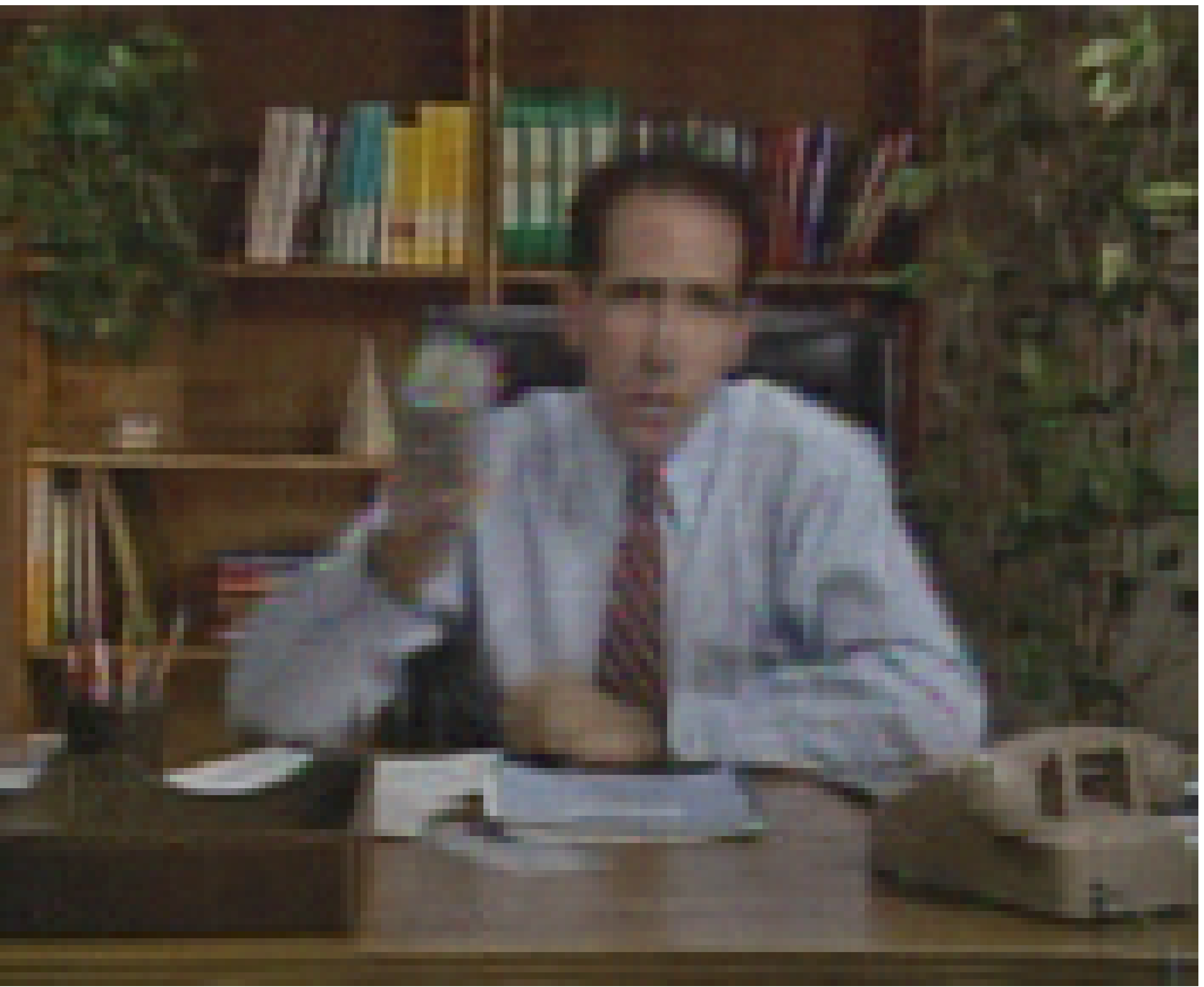}&
\includegraphics[width=0.19\textwidth]{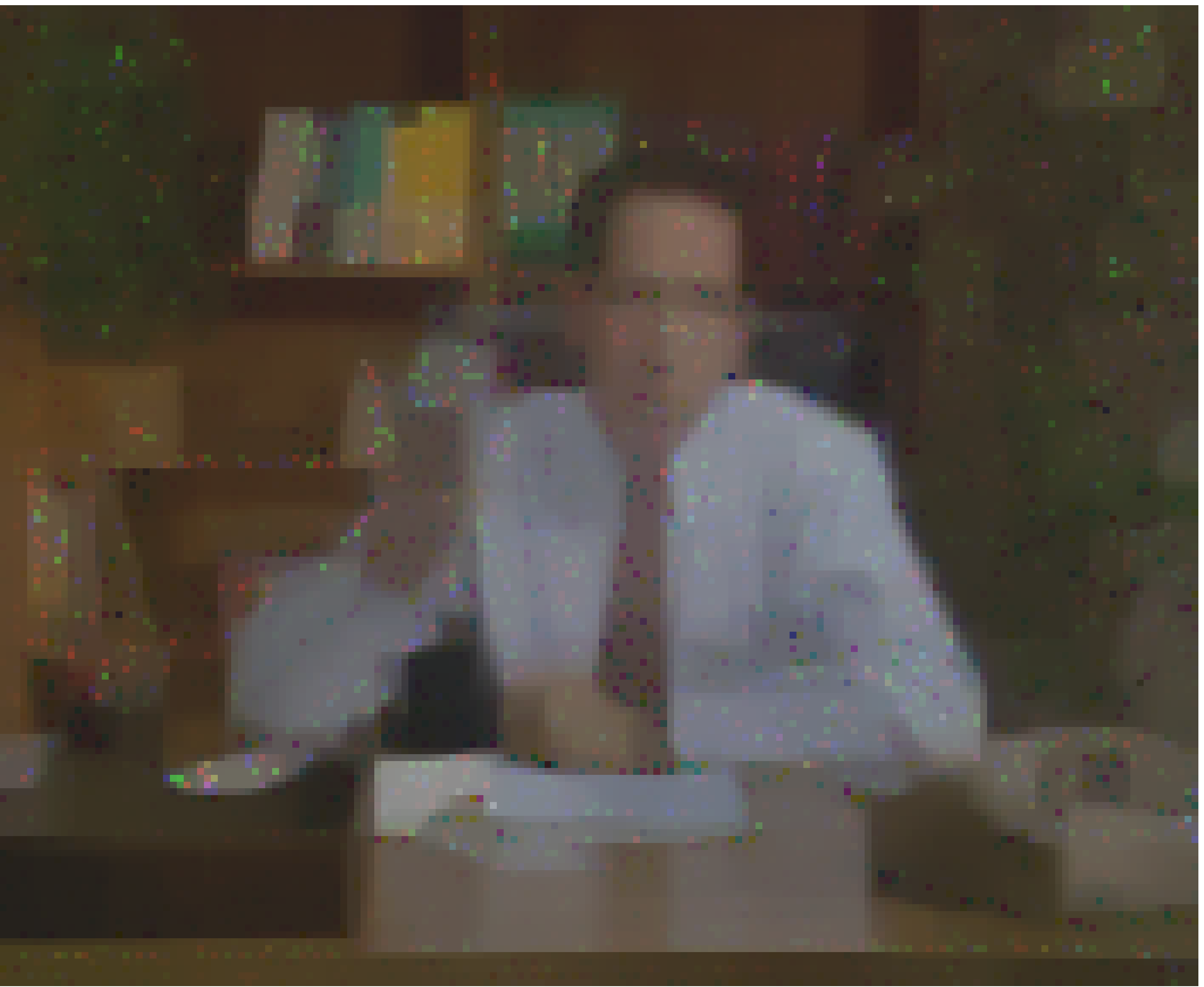}\\
(a) Original& (b) Observed & (c) HaLRTC & (d) NSNN & (e) LRTC-TV \\
\includegraphics[width=0.19\textwidth]{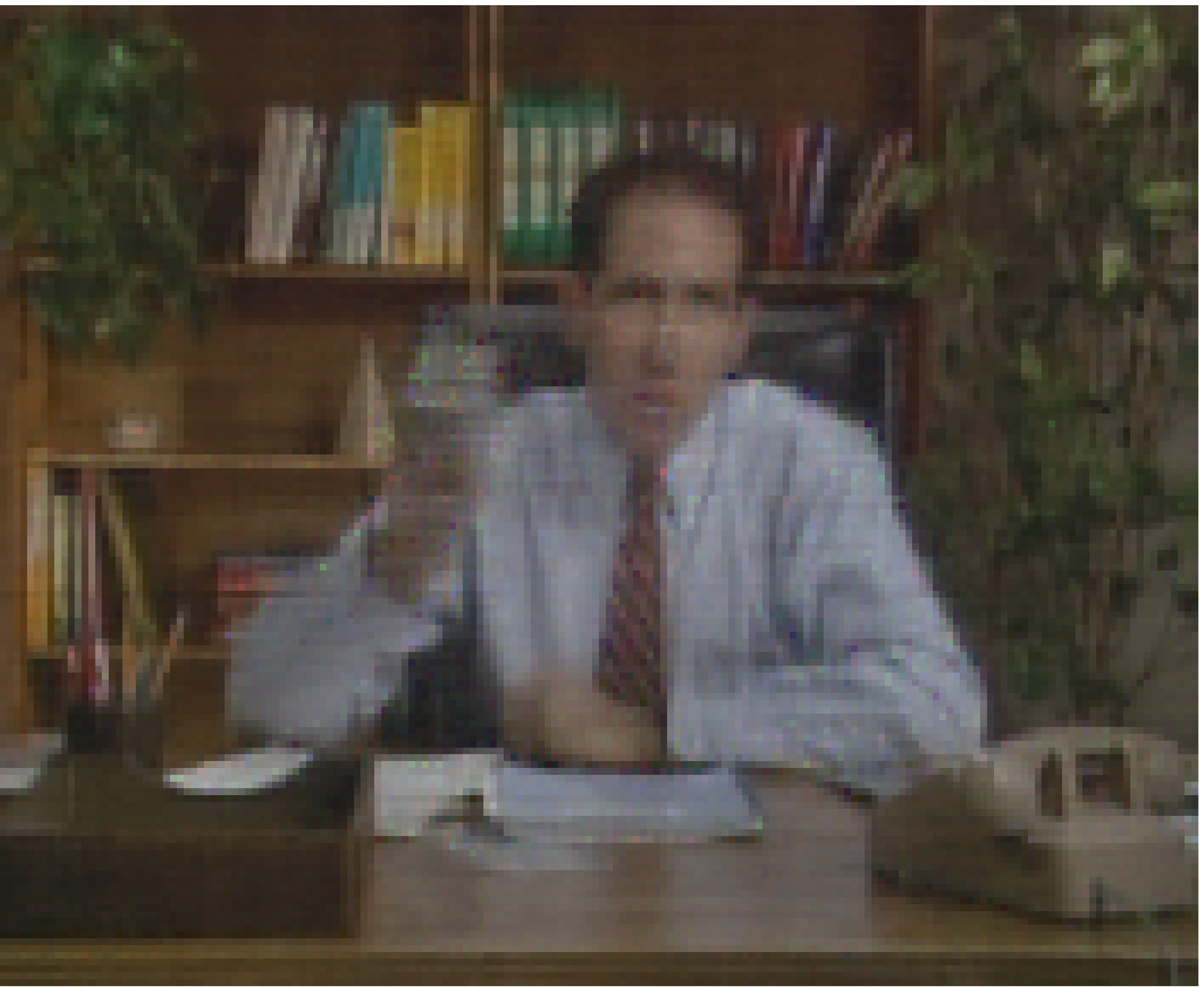}&
\includegraphics[width=0.19\textwidth]{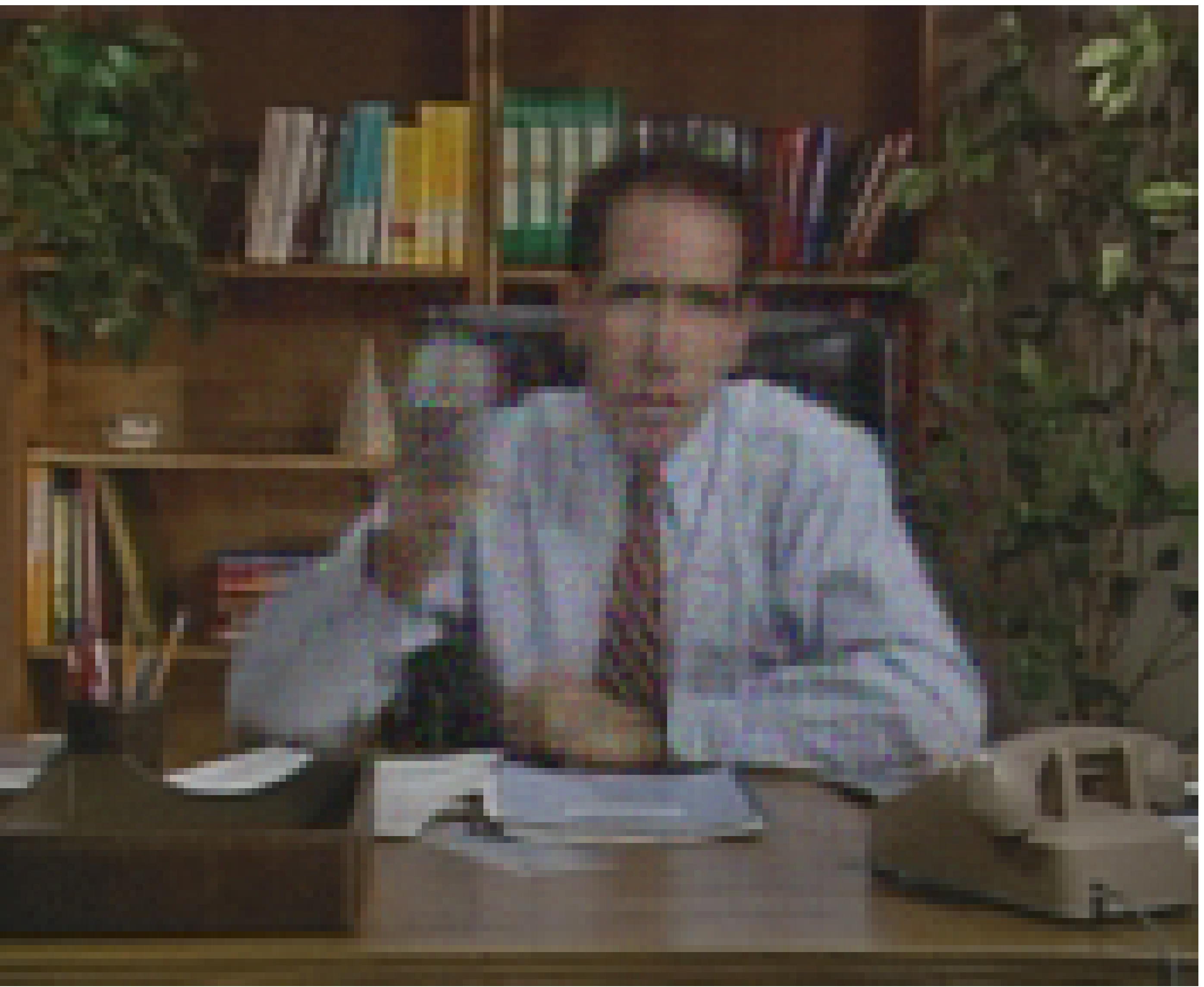}&
\includegraphics[width=0.19\textwidth]{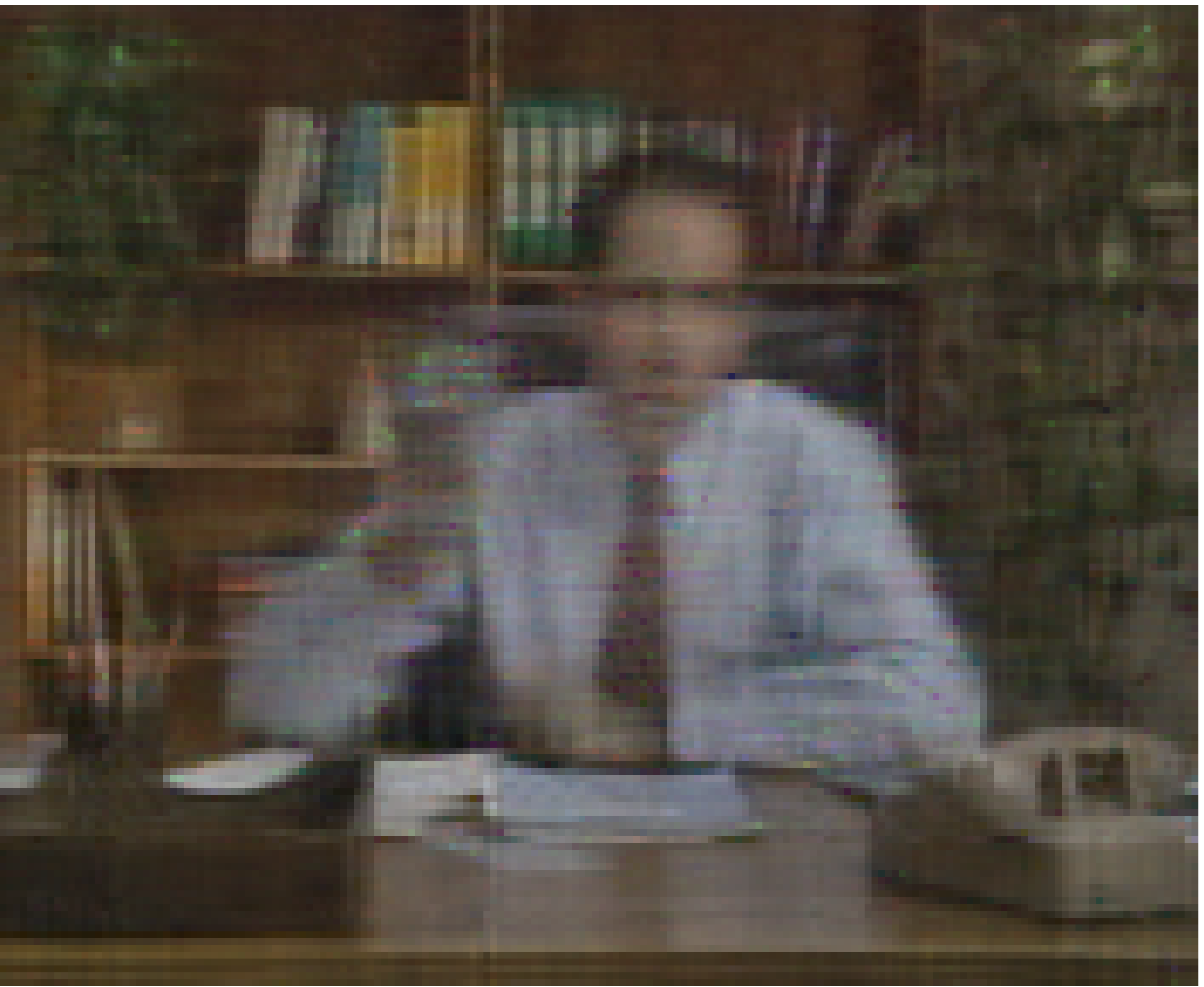}&
\includegraphics[width=0.19\textwidth]{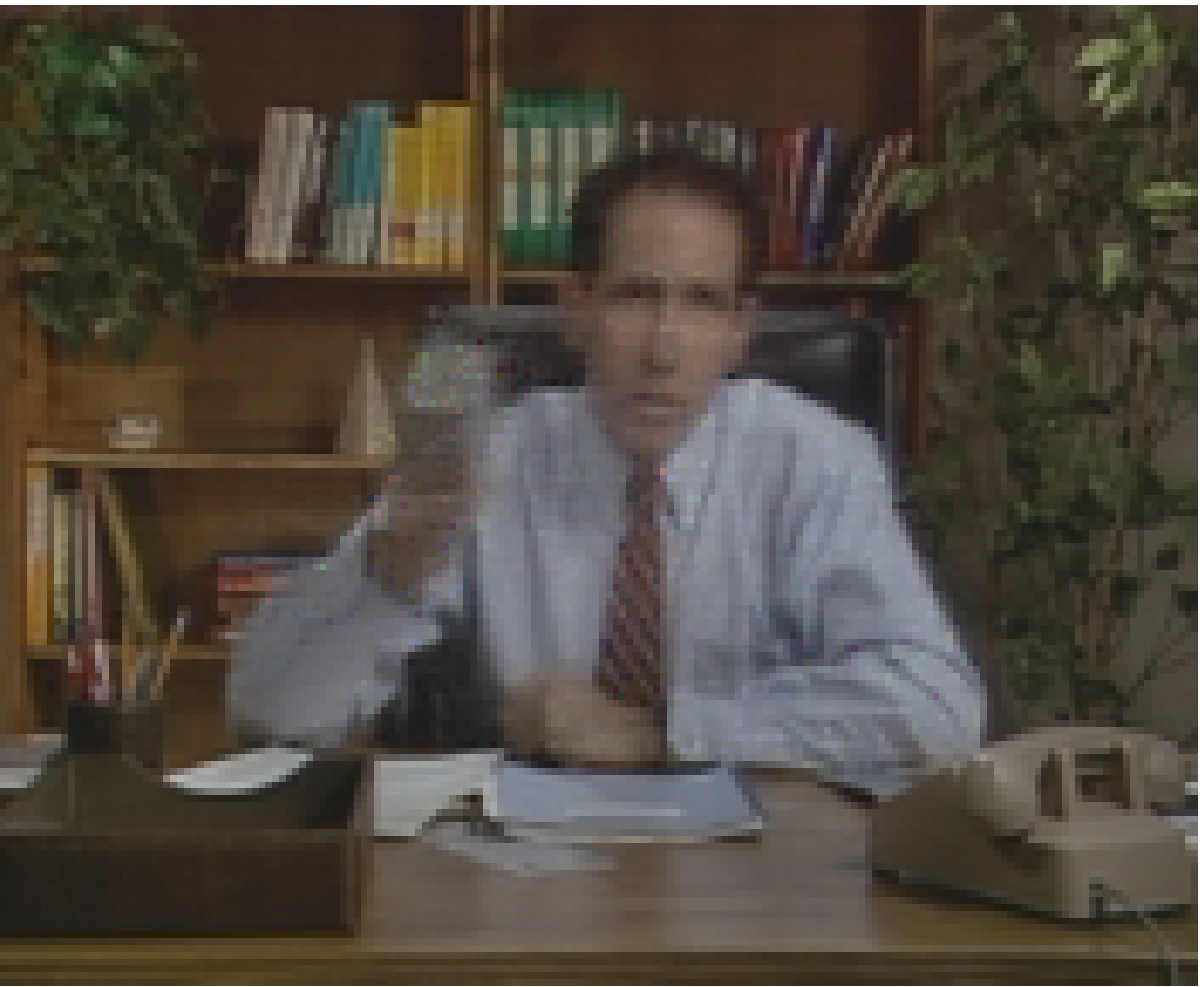}&
\includegraphics[width=0.19\textwidth]{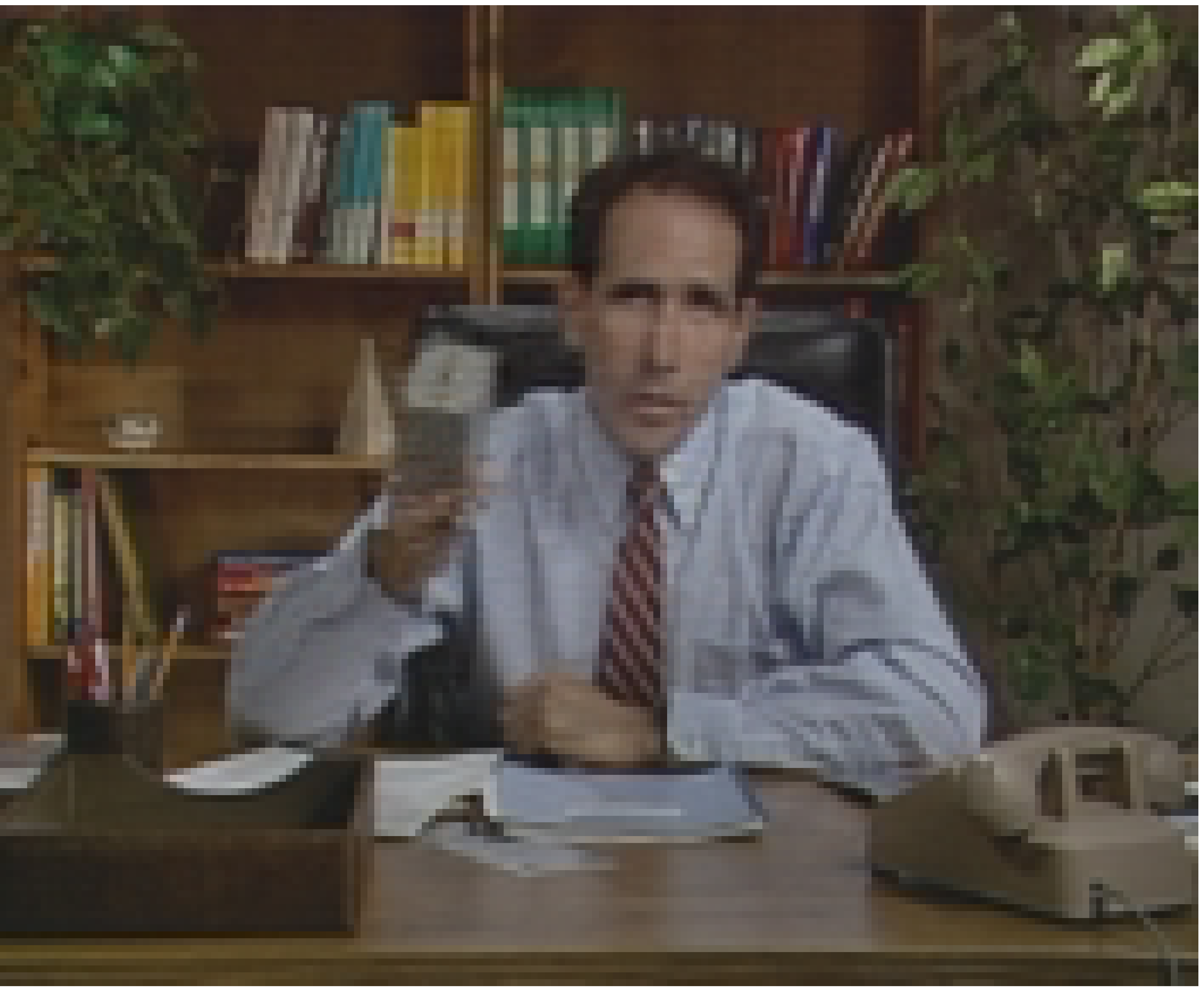}\\
(f) SiLRTC-TT & (g) tSVD & (h) KBR & (i) TRNN & (j) LogTR \\
\includegraphics[width=0.19\textwidth]{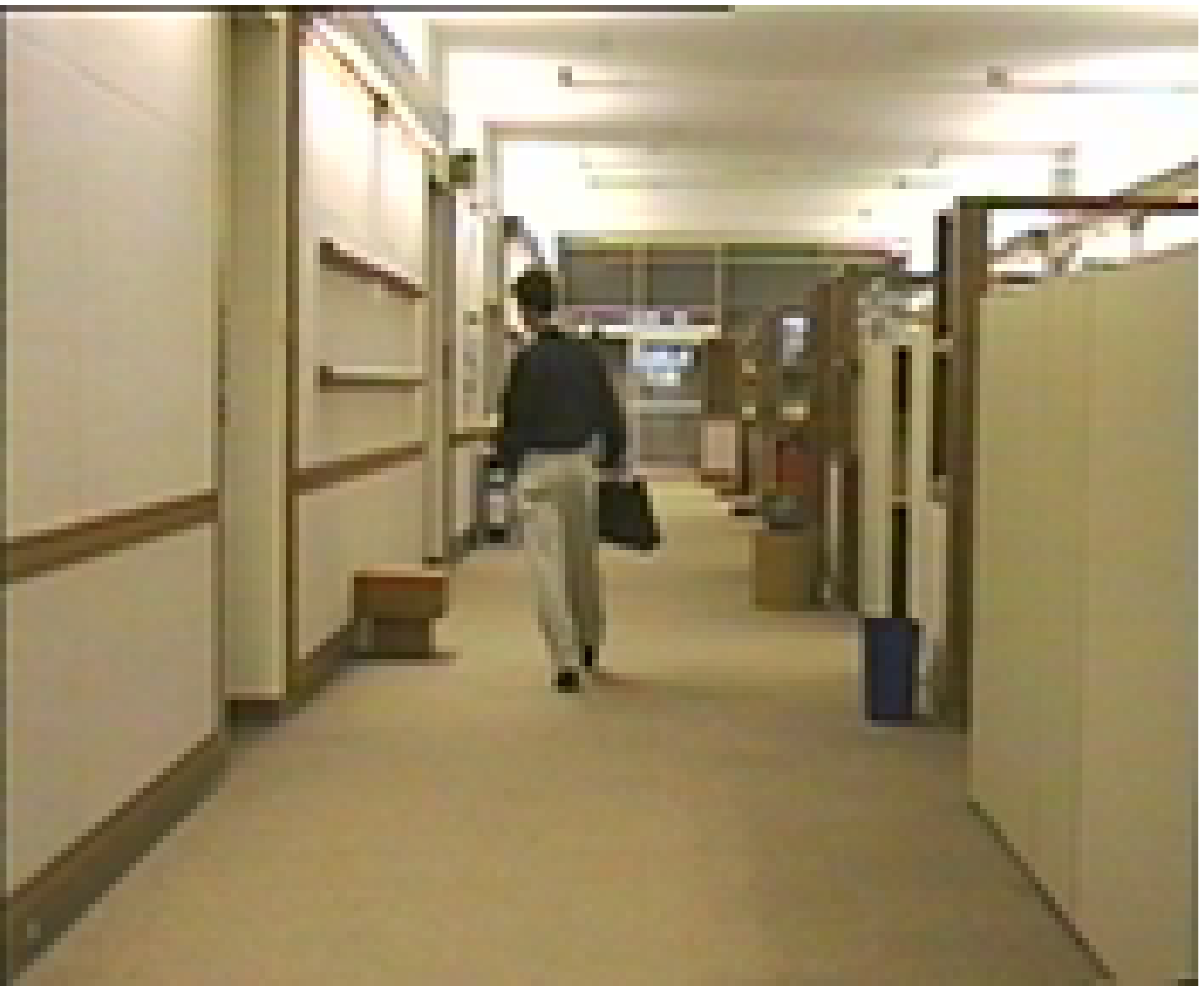}&
\includegraphics[width=0.19\textwidth]{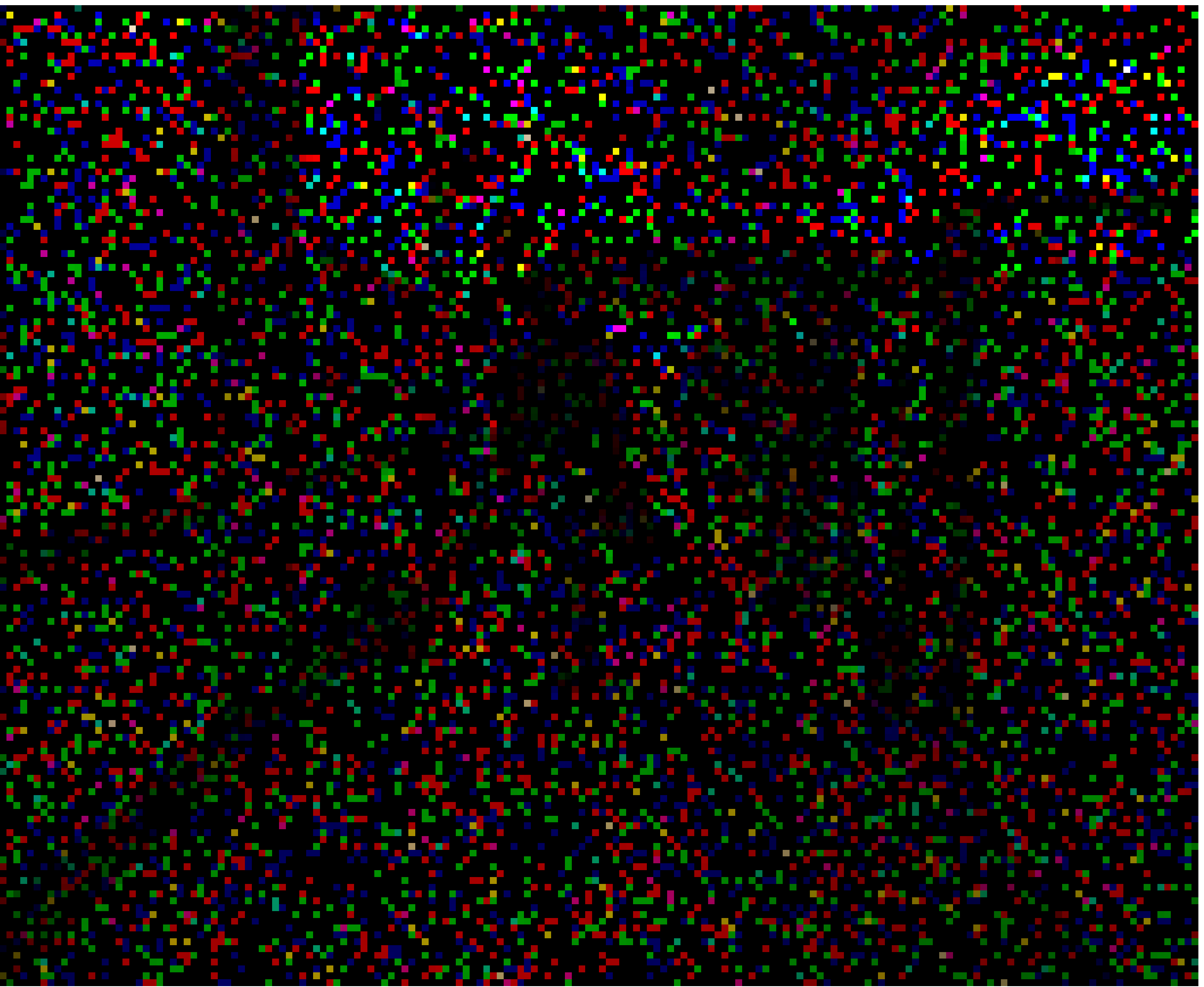}&
\includegraphics[width=0.19\textwidth]{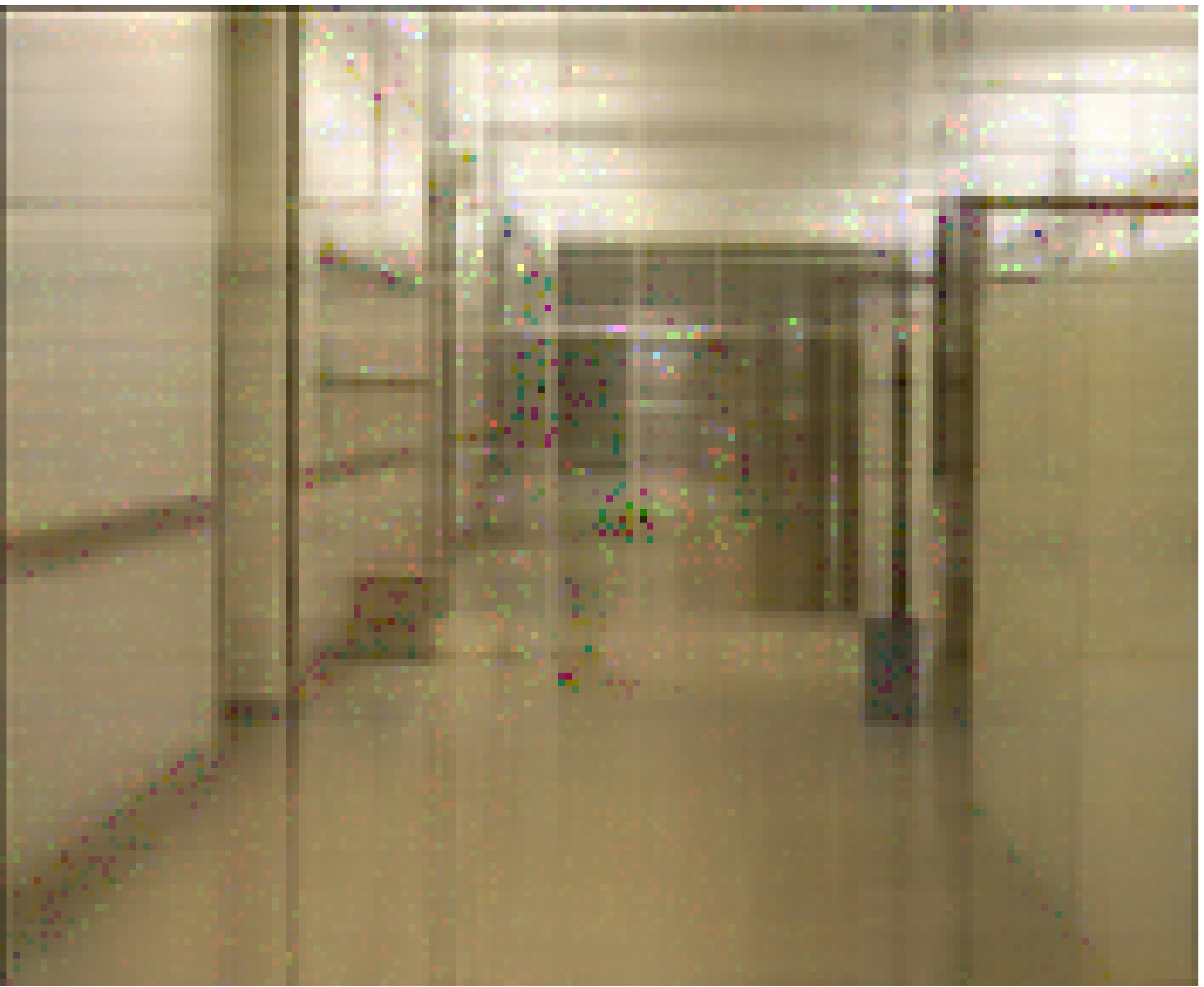}&
\includegraphics[width=0.19\textwidth]{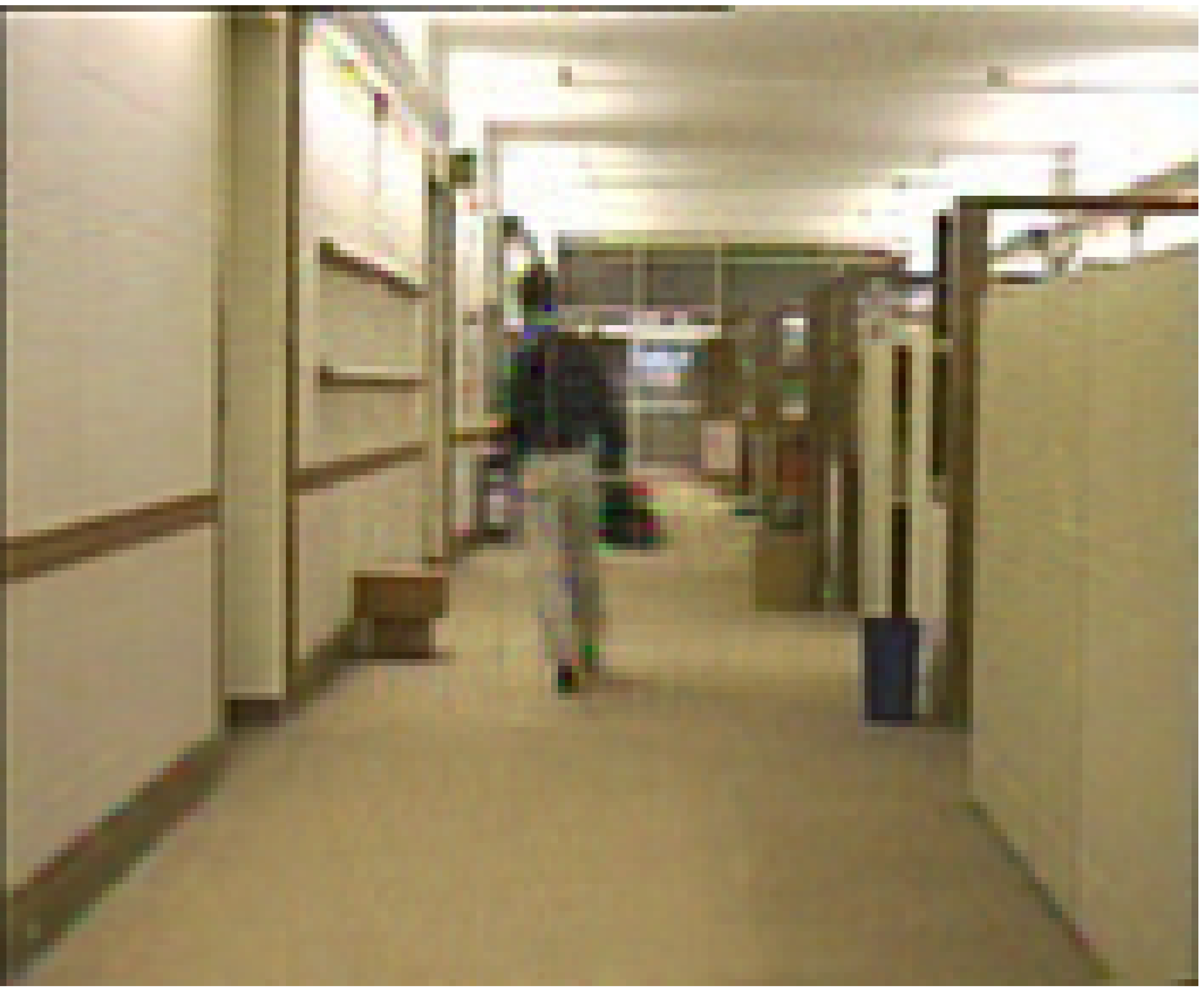}&
\includegraphics[width=0.19\textwidth]{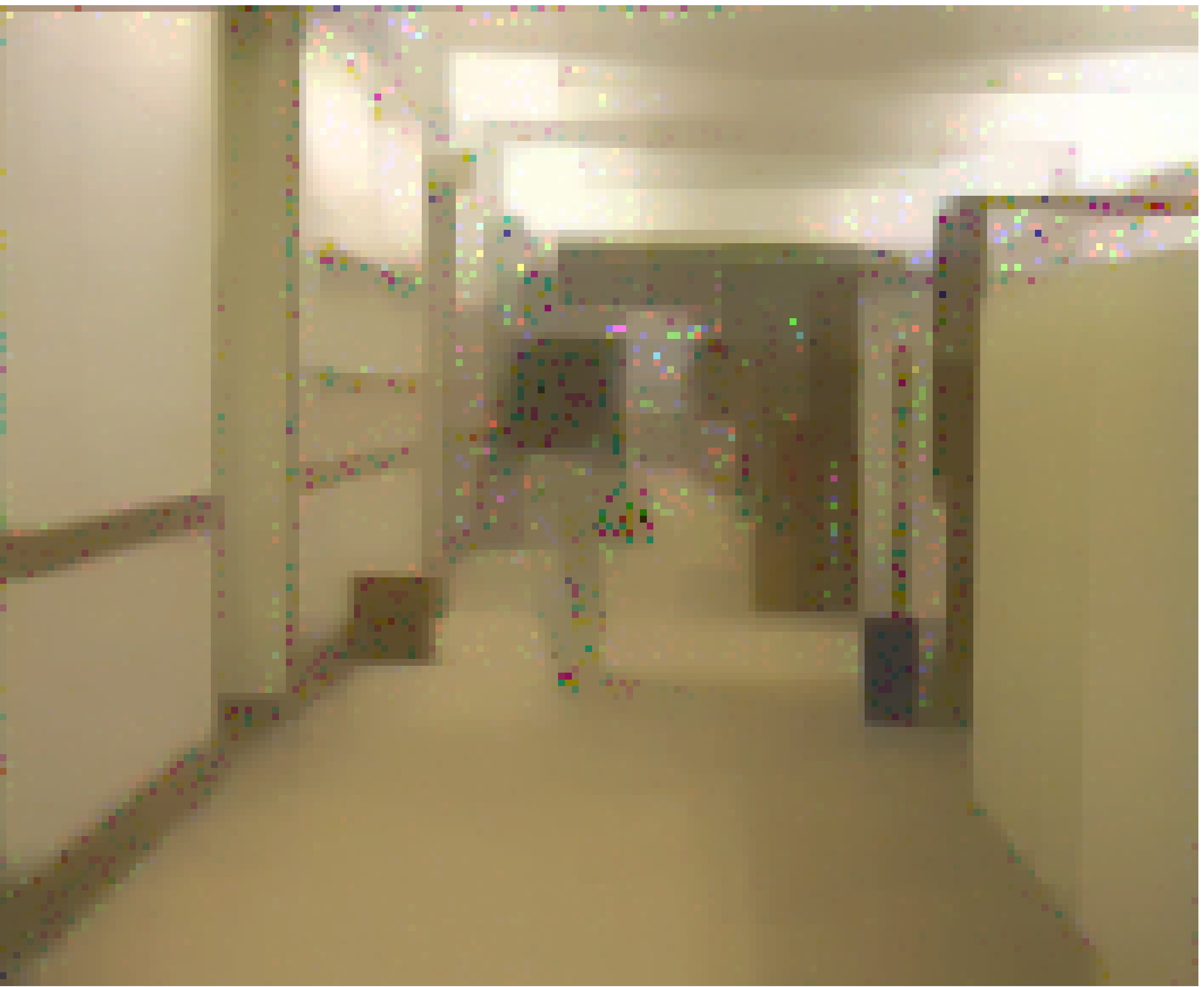}\\
(a) Original & (b) Observed & (c) HaLRTC & (d) NSNN & (e) LRTC-TV  \\
\includegraphics[width=0.19\textwidth]{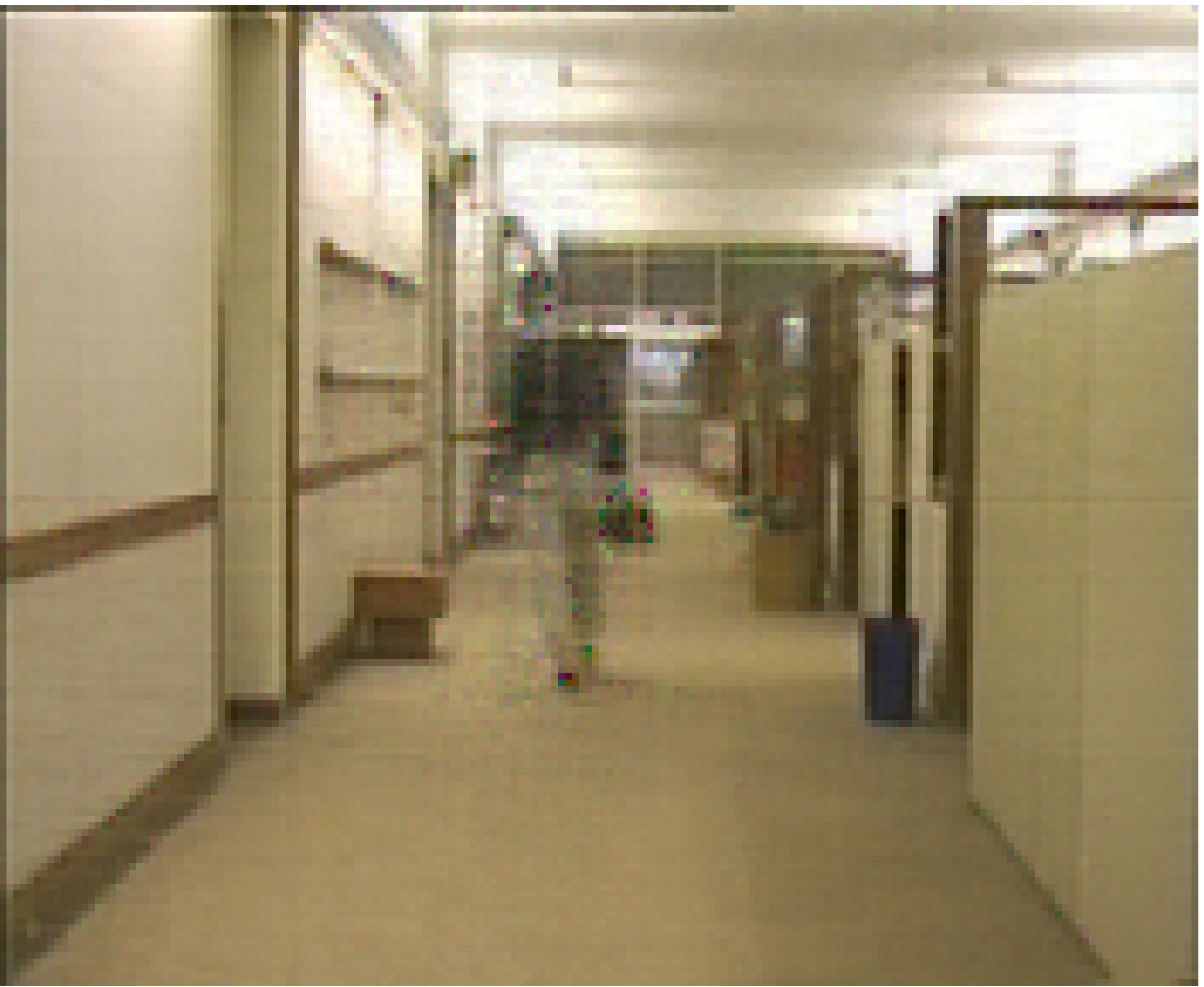}&
\includegraphics[width=0.19\textwidth]{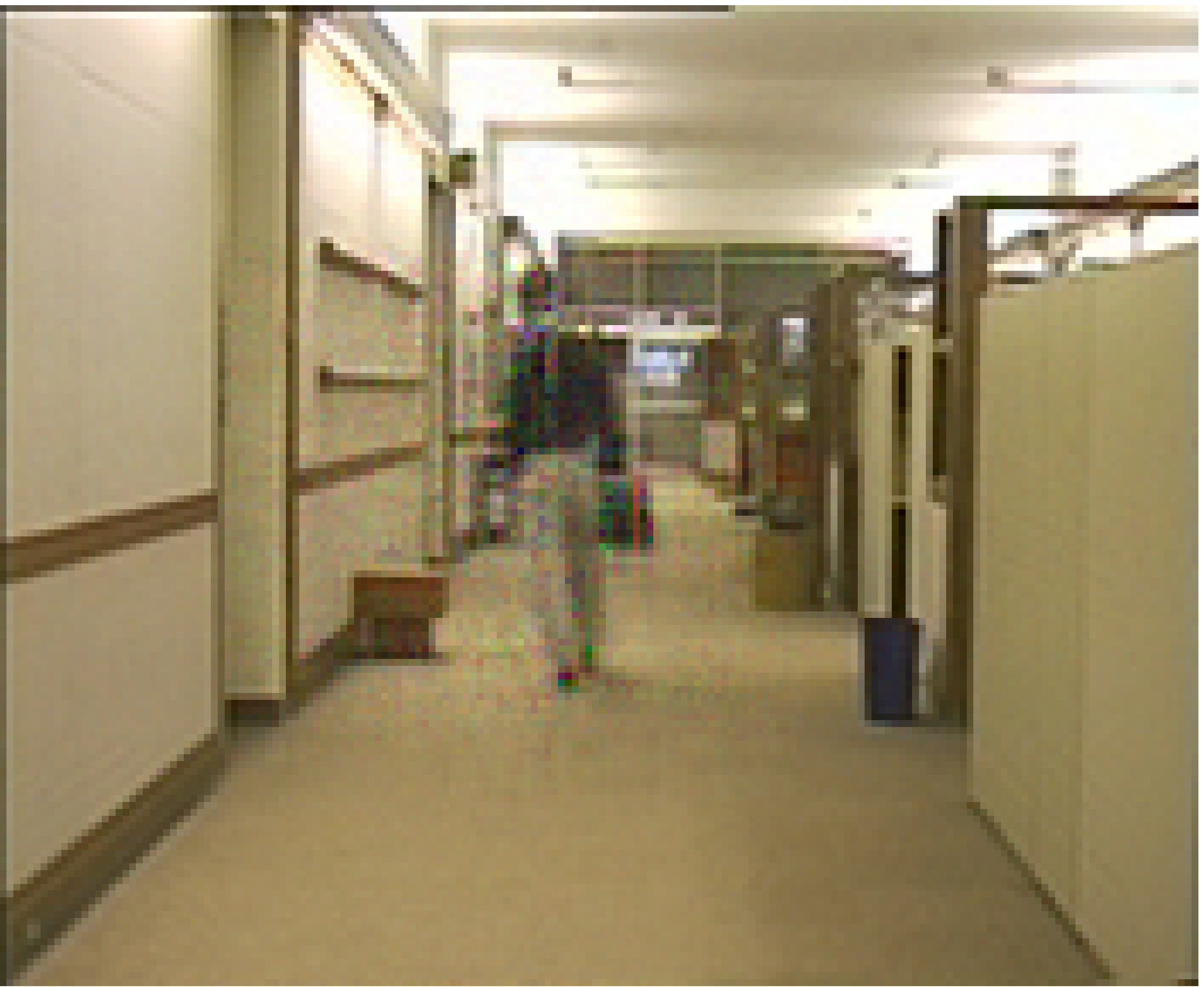}&
\includegraphics[width=0.19\textwidth]{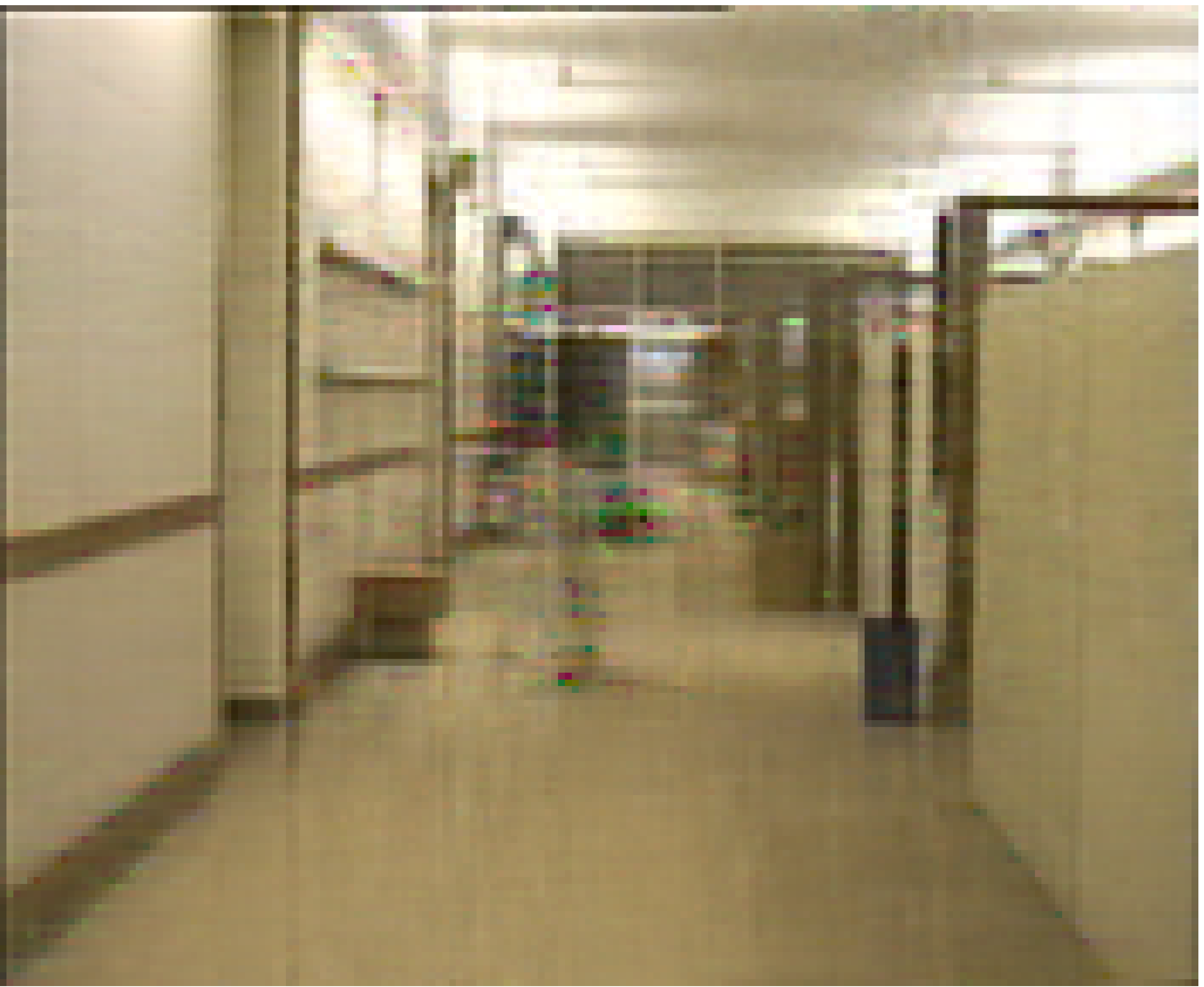}&
\includegraphics[width=0.19\textwidth]{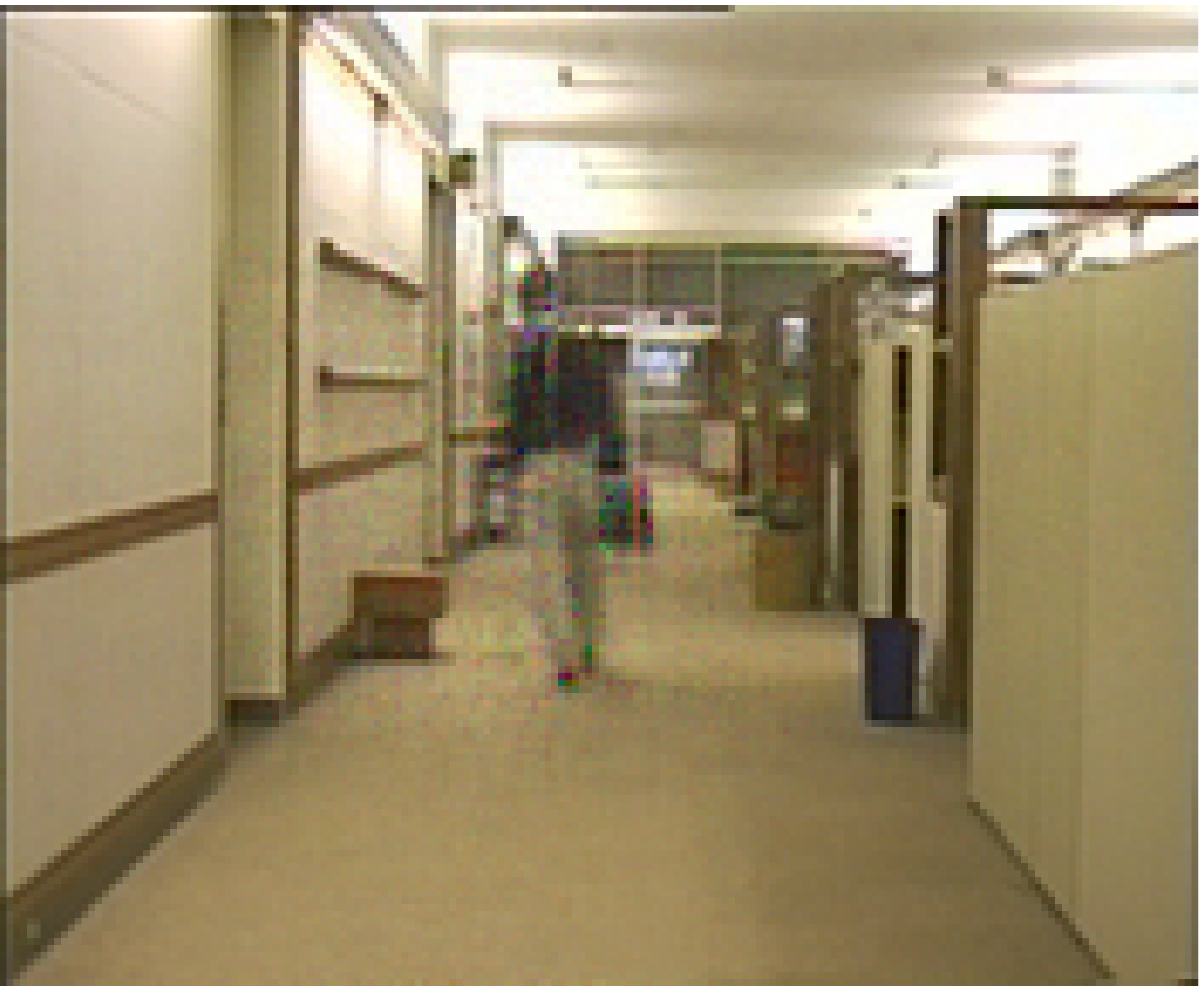}&
\includegraphics[width=0.19\textwidth]{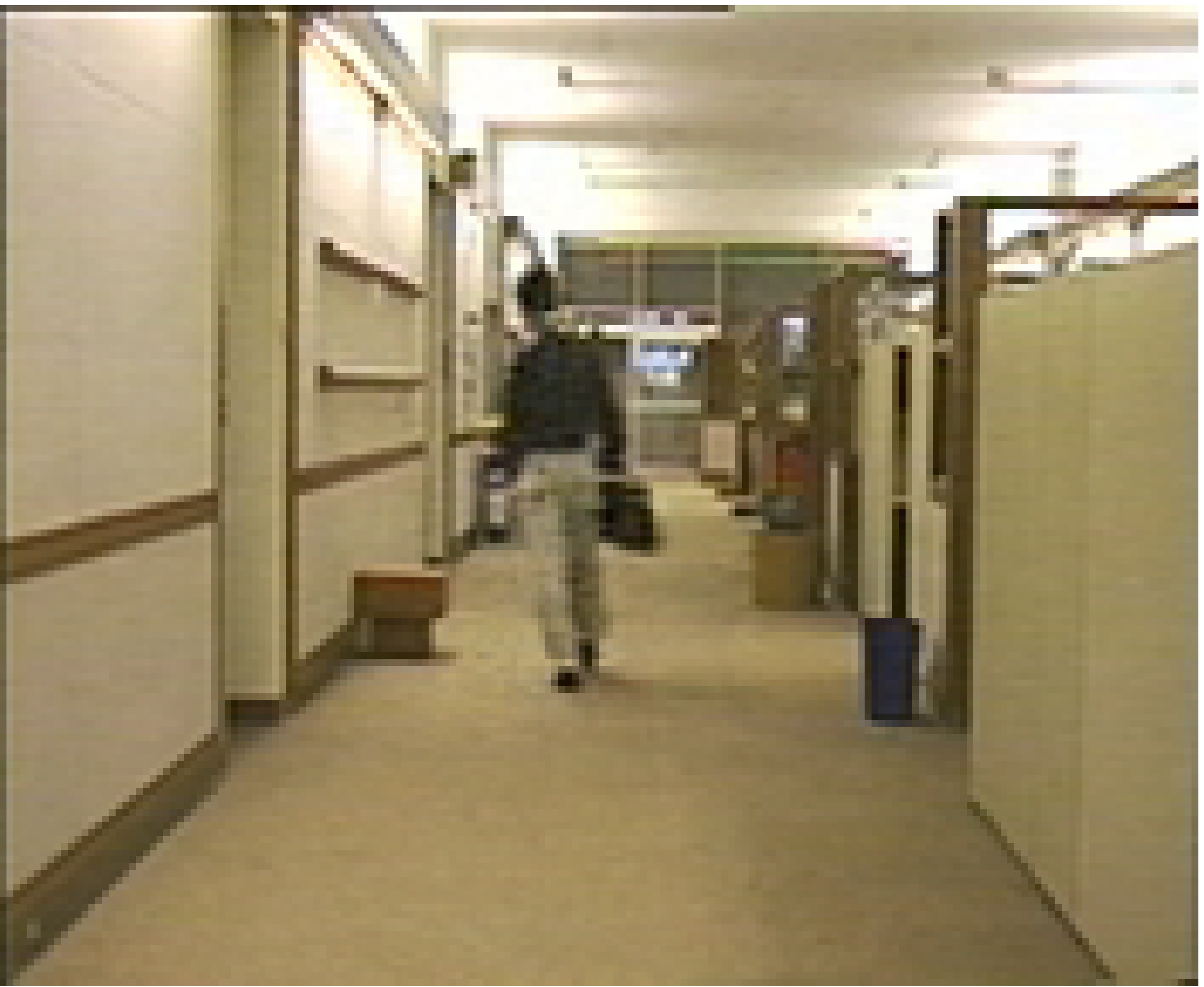}\\
(f) SiLRTC-TT & (g) tSVD & (h) KBR & (i) TRNN & (j) LogTR
\end{tabular}
\caption{\small{Recovered color videos \emph{Container}, \emph{Salesman}, and \emph{Hall} for random missing entries with $SR=0.1$.}}
  \label{fig:video}
  \end{center}\vspace{-0.3cm}
\end{figure}

\subsection{Color videos}
We test six color videos\footnote{https://media.xiph.org/video/derf/}, including \emph{Container}, \emph{Salesman}, \emph{Hall}, \emph{Foreman}, \emph{Claire}, and \emph{Suzie}. All testing videos are with size $144\times 176 \times 3 \times 144$. The SRs are set as 0.1, 0.2, and 0.3. To obtain balanced unfolding matrices, we first permute videos with order $[1,4,3,2]$ and then transform it into a tenth-order tensor, whose size is $6 \times 6 \times 4 \times 4 \times 6  \times 6 \times 3 \times 4 \times 4 \times 11$.

\begin{table}[!t]\scriptsize
\renewcommand\arraystretch{1.2}
\caption{PSNR and SSIM values of different methods on color videos completion with different SRs.}
\vspace{-0.5cm}
\begin{center}
\begin{tabular}{c|c|ccccccccc}
\hline \hline
\multirow{1}{*}{Videos}&\multicolumn{1}{c|}{SR} &Method &HaLRTC &NSNN &LRTC-TV & SiLRTC-TT & tSVD &KBR & TRNN & LogTR \\  \hline
\multirow{6}{*}{\emph{Container}}
          &\multirow{2}{*}{0.1} & PSNR    &22.68	  &30.70     &22.40    &28.30    &33.40	    &26.12     &30.51       &\textbf{42.00}  \\
          &                     & SSIM    &0.7651	  &0.9164    &0.7531   &89.24    &0.9272	&0.8565    &0.9389      &\textbf{0.9864} \\ \cline{2-11}

          &\multirow{2}{*}{0.2} & PSNR    &25.85	  &35.53     &25.76    &33.77    &37.83     &29.88     &36.75       &\textbf{46.21}  \\
          &                     & SSIM    &0.8614	  &0.9596    &0.8592   &0.9565   &0.9604    &0.9253    &0.9800      &\textbf{0.9925} \\ \cline{2-11}

          &\multirow{2}{*}{0.3} & PSNR    &28.58      &39.93     &28.41    &38.17    &40.85     &33.13     &41.85       &\textbf{48.53}  \\
          &                     & SSIM    &0.9146     &0.9805    &0.9147   &0.9786   &0.9816    &0.9606    &0.9909      &\textbf{0.9951} \\
          \hline
\multirow{6}{*}{\emph{Salesman}}
          &\multirow{2}{*}{0.1} & PSNR    &23.76	  &31.56     &25.21    &29.29    &31.33	    &27.90     &32.29       &\textbf{36.74}  \\
          &                     & SSIM    &0.6477	  &0.8910    &0.6593   &0.8640   &0.9013	&0.7995    &0.9312      &\textbf{0.9678} \\ \cline{2-11}

          &\multirow{2}{*}{0.2} & PSNR    &27.10	  &34.62     &28.80    &33.74    &34.34     &31.65     &36.48       &\textbf{40.48}  \\
          &                     & SSIM    &0.7970	  &0.9405    &0.8230   &0.9444   &0.9449    &0.9075    &0.9699      &\textbf{0.9847} \\ \cline{2-11}

          &\multirow{2}{*}{0.3} & PSNR    &29.82      &36.86     &31.20    &36.85    &36.52     &34.55     &39.39       &\textbf{43.23}  \\
          &                     & SSIM    &0.7970     &0.9622    &0.8962   &0.9708   &0.9643    &0.9511    &0.9835      &\textbf{0.9913} \\
          \hline
\multirow{6}{*}{\emph{Hall}}
          &\multirow{2}{*}{0.1} & PSNR    &22.61	  &30.92     &22.55    &27.89    &30.75	    &26.82     &30.51       &\textbf{34.92}  \\
          &                     & SSIM    &0.7469	  &0.9145    &0.7664   &0.8911   &0.9170	&0.8603    &0.9125      &\textbf{0.9532} \\
          \cline{2-11}
          &\multirow{2}{*}{0.2} & PSNR    &26.11	  &33.99     &27.02    &32.18    &33.34     &30.76     &34.83       &\textbf{38.30}  \\
          &                     & SSIM    &0.8560	  &0.9454    &0.8903   &0.9445   &0.9442    &0.9278    &0.9656      &\textbf{0.9734} \\
          \cline{2-11}
          &\multirow{2}{*}{0.3} & PSNR    &28.89      &36.04     &29.75    &35.09    &35.28     &33.73     &35.12       &\textbf{41.13}  \\
          &                     & SSIM    &0.9114     &0.9622    &0.9333   &0.9645   &0.9584    &0.9572    &0.9564      &\textbf{0.9841} \\
          \hline
\multirow{6}{*}{\emph{Foreman}}
          &\multirow{2}{*}{0.1} & PSNR    &19.89	  &28.87     &22.08    &23.98    &24.01	    &24.44     &26.21       &\textbf{30.77}  \\
          &                     & SSIM    &0.5082	  &0.8624    &0.6934   &0.6889   &0.6113	&0.7260    &0.8130      &\textbf{0.8928} \\
          \cline{2-11}
          &\multirow{2}{*}{0.2} & PSNR    &23.21	  &32.28     &26.71    &27.87    &26.90     &28.12     &30.61       &\textbf{35.72}  \\
          &                     & SSIM    &0.6770	  &0.9235    &0.8504   &0.8356   &0.7477    &0.8592    &0.9169      &\textbf{0.9569} \\
          \cline{2-11}
          &\multirow{2}{*}{0.3} & PSNR    &25.99      &34.84     &29.48    &31.21    &29.26     &30.94     &34.16       &\textbf{39.41}  \\
          &                     & SSIM    &0.7959     &0.9524    &0.9122   &0.9111   &0.8300    &0.9185    &0.9584      &\textbf{0.9788} \\
          \hline
\multirow{6}{*}{\emph{Claire}}
          &\multirow{2}{*}{0.1} & PSNR    &26.76	  &36.52     &30.07    &31.70    &33.26	    &29.39     &34.35       &\textbf{39.62}  \\
          &                     & SSIM    &0.8668	  &0.9636    &0.9182   &0.9377   &0.9380	&0.9061    &0.9649      &\textbf{0.9803} \\
          \cline{2-11}
          &\multirow{2}{*}{0.2} & PSNR    &30.70	  &40.01     &34.08    &35.84    &36.83     &33.23     &38.63       &\textbf{43.57}  \\
          &                     & SSIM    &0.9285	  &0.9794    &0.9581   &0.9697   &0.9662    &0.9542    &0.9828      &\textbf{0.9883} \\
          \cline{2-11}
          &\multirow{2}{*}{0.3} & PSNR    &33.81      &42.59     &36.64    &38.88    &39.38     &36.16     &41.67       &\textbf{46.36}  \\
          &                     & SSIM    &0.9589     &0.9867    &0.9742   &0.9822   &0.9780    &0.9741    &0.9895      &\textbf{0.9924} \\
          \hline
\multirow{6}{*}{\emph{Suzie}}
          &\multirow{2}{*}{0.1} & PSNR    &23.60	  &32.26     &27.42    &28.22    &27.99	    &28.13     &29.86       &\textbf{33.17}  \\
          &                     & SSIM    &0.6825	  &0.8810    &0.7975   &0.8058   &0.7431	&0.7895    &0.8563      &\textbf{0.9003} \\
          \cline{2-11}
          &\multirow{2}{*}{0.2} & PSNR    &27.38	  &34.60     &31.29    &31.79    &30.67     &31.32     &33.62       &\textbf{36.84}  \\
          &                     & SSIM    &0.7957	  &0.9223    &0.8847   &0.8877   &0.8298    &0.8755    &0.9254      &\textbf{0.9504} \\
          \cline{2-11}
          &\multirow{2}{*}{0.3} & PSNR    &30.14      &36.44     &33.57    &34.51    &32.65     &33.76     &36.42       &\textbf{39.78}  \\
          &                     & SSIM    &0.8654     &0.9464    &0.9233   &0.9312   &0.8801    &0.9220    &0.9569      &\textbf{0.9724} \\
          \hline \hline
\end{tabular}\label{table:video}
\end{center}
\end{table}

Figure \ref{fig:video} shows one frame of videos \emph{Container}, \emph{Salesman}, and \emph{Hall} recovered by all methods with $SR=0.1$. We observe that NSNN, SiLRTC-TT, tSVD, and TRNN can not keep structures of the recovered videos, such as the ripples of water in \emph{Container} and the tie of \emph{Salesman}, and HaLRTC, LRTC-TV, and KBR over-smooth the moved subjects, leading to obvious detail missing. In contrast, LogTR visually outperforms compared methods in keeping details and edges.

Table \ref{table:video} summaries the PNSR and SSIM values for different SRs. Figure \ref{fig:video_ps_ss} plots the PSNR and SSIM values corresponding to the frame number with $SR=0.1$. Again, for different SRs and all frames, LogTR achieves higher PSNR and SSIM values than compared methods.

\begin{figure}[!h]
\scriptsize\setlength{\tabcolsep}{0.5pt}
\begin{center}
\begin{tabular}{ccc}
\includegraphics[width=0.95\textwidth]{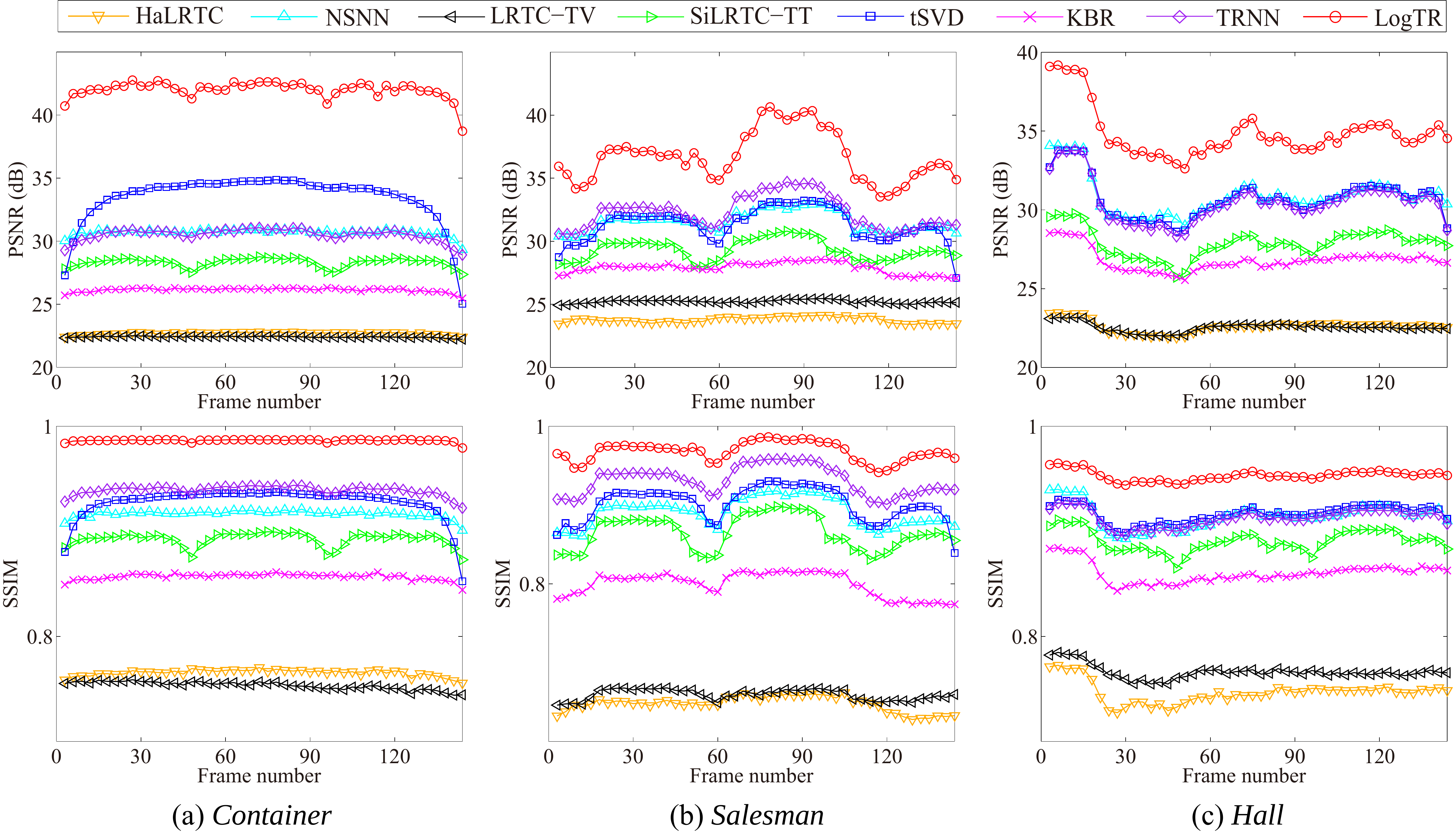}
\end{tabular}
\caption{PSNR and SSIM values of different methods on color videos completion with $SR=0.1$.}
  \label{fig:video_ps_ss}
  \end{center}\vspace{-0.3cm}
\end{figure}
%
%
%

\section{Conclusion}
\label{section:Conclusion}
In this paper, we propose a new nonconvex relaxation based on logdet function of the TR rank to more accurately depict the global low-rank prior of tensors for LRTC. We develop the ADMM algorithm to solve the nonconvex optimization problem with convergence analysis. Experiments on color images, MSIs, and color videos show that the proposed method can not only flexibly adapt to different completion tasks, but also achieve better performance than some state-of-the-art methods. In future work, we will try to apply the proposed nonconvex low-rank approximation to other tasks, such as denoising \cite{Li2016Multiplicative,Ma2017Truncated,Wang2018Speckle} and rain streaks removal \cite{Jiang2018FastDeRain}.
\section*{Acknowledgments}
This research is supported by NSFC (61876203, 61772003, 11901450), Project funded by China Postdoctoral Science Foundation (2018M643611), National Postdoctoral Program for Innovative Talents (BX20180252), and Science Strength Promotion Programme of UESTC. The authors would like to thank the authors \cite{Bengua2017Efficient,Ji2017A,Li2017LRTCTV,Liu2013tensor,Xie2018Kronecker,Zhang2017tSVD} for providing the free download of the source code.

\bibliographystyle{plain}
\bibliography{NTR}

\begin{thebibliography}{10}

\bibitem{Bengua2017Efficient}
J.~A. Bengua, H.~N. Phiem, H.~D. Tuan, and M.~N. Do.
\newblock Efficient tensor completion for color image and video recovery:
  Low-rank tensor train.
\newblock {\em IEEE Transactions on Image Processing}, 26(5):2466--2479, 2017.

\bibitem{Chang2019DeNet}
Y.~Chang, L-X. Yan, H-Z. Fang, S.~Zhong, and W-S. Liao.
\newblock H{SI-D}e{N}et: Hyperspectral image restoration via convolutional
  neural network.
\newblock {\em IEEE Transactions on Geoscience and Remote Sensing},
  57(2):667--682, 2019.

\bibitem{Chang2017Transformed}
Y.~Chang, L-X. Yan, and S.~Zhong.
\newblock Transformed low-rank model for line pattern noise removal.
\newblock In {\em IEEE International Conference on Computer Vision}, pages
  1726--1734, 2017.

\bibitem{Chen2014STDC}
Y.~Chen, C.~Hsu, and H.~M. Liao.
\newblock Simultaneous tensor decomposition and completion using factor priors.
\newblock {\em IEEE Transactions on Pattern Analysis and Machine Intelligence},
  36(3):577--591, 2014.

\bibitem{Chiantini2012CPDIdentifiability}
L.~{Chiantini} and G.~{Ottaviani}.
\newblock On generic identifiability of 3-tensors of small rank.
\newblock {\em SIAM Journal on Matrix Analysis and Applications},
  33(3):1018--1037, 2012.

\bibitem{Ding2019TTTV}
M.~Ding, T-Z. Huang, T-Y. Ji, X-L. Zhao, and J-H. Yang.
\newblock Low-rank tensor completion using matrix factorization based on tensor
  train rank and total variation.
\newblock {\em Journal of Scientific Computing}, 81:941--964, 2019.

\bibitem{Fazel2003logDet}
M.~Fazel, H.~Hindi, and S.~P. Boyd.
\newblock Log-det heuristic for matrix rank minimization with applications to
  hankel and euclidean distance matrices.
\newblock In {\em American Control Conference}, volume~3, pages 2156--2162,
  2003.

\bibitem{Fu2016Unmixing}
X.~Fu, W.~Ma, J.~M. Bioucas-Dias, and T.~Chan.
\newblock Semiblind hyperspectral unmixing in the presence of spectral library
  mismatches.
\newblock {\em IEEE Transactions on Geoscience and Remote Sensing},
  54(9):5171--5184, 2016.

\bibitem{Gandy2011Tensor}
S.~Gandy, B.~Recht, and I.~Yamada.
\newblock Tensor completion and low-$n$-rank tensor recovery via convex
  optimization.
\newblock {\em Inverse Problems}, 27(2):025010, 2011.

\bibitem{Gong2013General}
P-H. Gong, C-S. Zhang, Z-S. Lu, J-Z. Huang, and J-P. Ye.
\newblock A general iterative shrinkage and thresholding algorithm for
  non-convex regularized optimization problems.
\newblock In {\em International Conference on International Conference on
  Machine Learning}, pages II--37--II--45, 2013.

\bibitem{Grasedyck2015TT}
L.~Grasedyck, M.~Kluge, and S.~Kr\"{a}mer.
\newblock Alternating least squares tensor completion in the {TT}-format.
\newblock {\em arXiv preprint arXiv:1509.00311}.

\bibitem{Gu2017Weighted}
S-H. Gu, Q.~Xie, D-Y. Meng, W-M. Zuo, X-C. Feng, and L.~Zhang.
\newblock Weighted nuclear norm minimization and its applications to low level
  vision.
\newblock {\em International Journal of Computer Vision}, 121(2):183--208,
  2017.

\bibitem{Hitchcock1927CPD}
F.~L. {Hitchcock}.
\newblock The expression of a tensor or a polyadic as a sum of products.
\newblock {\em Journal of Mathematics and Physics}, 6(1-4):164--189, 1927.

\bibitem{Huang2019TRNN}
H-Y. Huang, Y-P. Liu, J-N. Liu, and C.~Zhu.
\newblock Provable tensor ring completion.
\newblock {\em Signal Processing}, 171:107486, 2020.

\bibitem{Ishteva2009Tucker_rank}
M.~{Ishteva}, L.~D. {Lathauwer}, P.~A. {Absil}, and S.~V. {Huffel}.
\newblock Differential-geometric newton method for the best
  rank-$(r_{1},r_{3},r_{2})$ approximation of tensors.
\newblock {\em Numerical Algorithms}, 51:179--194, 2009.

\bibitem{Ji2017A}
T-Y. Ji, T-Z. Huang, X-L. Zhao, T-H. Ma, and L-J. Deng.
\newblock A non-convex tensor rank approximation for tensor completion.
\newblock {\em Applied Mathematical Modelling}, 48:410--422, 2017.

\bibitem{Jiang2018FastDeRain}
T-X. Jiang, T-Z. Huang, X-L. Zhao, L-J. Deng, and Y.~Wang.
\newblock Fast{D}e{R}ain: A novel video rain streak removal method using
  directional gradient priors.
\newblock {\em IEEE Transactions on Image Processing}, 28(4):2089--2102, 2019.

\bibitem{Kilmer2013Third-Order}
M.~E. Kilmer, K.~Braman, N.~Hao, and R.~C. Hoover.
\newblock Third-order tensors as operators on matrices: {A} theoretical and
  computational framework with applications in imaging.
\newblock {\em SIAM Journal on Matrix Analysis and Applications},
  34(1):148--172, 2013.

\bibitem{Komodakis2006Image-Completion}
N.~Komodakis.
\newblock Image completion using global optimization.
\newblock In {\em IEEE Conference on Computer Vision and Pattern Recognition},
  volume~1, pages 442--452, 2006.

\bibitem{Kreimer2012seismic}
N.~{Kreimer} and M.~D. {Sacchi}.
\newblock Tensor completion via nuclear norm minimization for 5{D} seismic data
  reconstruction.
\newblock {\em SEG Technical Program Expanded Abstracts}, pages 1--5, 2012.

\bibitem{Li2012Coupled}
F.~Li, M.~K. Ng, and R.~J. Plemmons.
\newblock Coupled segmentation and denoising/deblurring models for
  hyperspectral material identification.
\newblock {\em Numerical Linear Algebra with Applications}, 19(1):153--173,
  2012.

\bibitem{Li2017LRTCTV}
X-T. Li, Y-M. Ye, and X-F. Xu.
\newblock Low-rank tensor completion with total variation for visual data
  inpainting.
\newblock In {\em AAAI Conference on Artificial Intelligence}, pages
  2210--2216, 2017.

\bibitem{Li2016Multiplicative}
Z.~{Li}, Y-F. {Lou}, and T-Y. {Zeng}.
\newblock Variational multiplicative noise removal by {DC} programming.
\newblock {\em Journal of Scientific Computing}, 68:1200--1216, 2016.

\bibitem{Liu2013tensor}
J.~Liu, P.~Musialski, P.~Wonka, and J.~Ye.
\newblock Tensor completion for estimating missing values in visual data.
\newblock {\em IEEE Transactions on Pattern Analysis and Machine Intelligence},
  35(1):208--220, 2013.

\bibitem{Liu2019Image}
Y-P. Liu, Z.~Long, and C.~Zhu.
\newblock Image completion using low tensor tree rank and total variation
  minimization.
\newblock {\em IEEE Transactions on Multimedia}, 21(2):338--350, 2019.

\bibitem{Lu2020TRPCA}
C-Y. Lu, J-S. Feng, Y-D. Chen, W.~Liu, Z-C. Lin, and S-C. Yan.
\newblock Tensor robust principal component analysis with a new tensor nuclear
  norm.
\newblock {\em IEEE Transactions on Pattern Analysis and Machine Intelligence},
  42:925--938, 2020.

\bibitem{Lu2018Exact}
C-Y. Lu, J-S. Feng, Z-C. Lin, and S-C. Yan.
\newblock Exact low tubal rank tensor recovery from gaussian measurements.
\newblock In {\em International Joint Conference on Artificial Intelligence},
  2018.

\bibitem{Ma2017Truncated}
T-H. Ma, Y.~Lou, and T-Z. Huang.
\newblock Truncated $l_{1-2}$ models for sparse recovery and rank minimization.
\newblock {\em SIAM Journal on Imaging Sciences}, 10(3):1346--1380, 2017.

\bibitem{Oseledets2011Tensor-Train-Decomposition}
I.~V. Oseledets.
\newblock Tensor-train decomposition.
\newblock {\em SIAM Journal on Scientific Computing}, 33(5):2295--2317, 2011.

\bibitem{Tucker1966Tucker}
L.~R. {Tucker}.
\newblock Some mathematical notes on three-mode factor analysis.
\newblock {\em Psychometrika}, 31(3):279--311, 1966.

\bibitem{Wang2018Speckle}
S.~{Wang}, T-Z. {Huang}, X-L. {Zhao}, J-J. {Mei}, and J.~{Huang}.
\newblock Speckle noise removal in ultrasound images by first-and second-order
  total variation.
\newblock {\em Numerical Algorithms}, 78(2):513--533, 2018.

\bibitem{Wang2016TT}
W.~Wang, V.~Aggarwal, and S.~Aeron.
\newblock Tensor completion by alternating minimization under the tensor train
  ({TT}) model.
\newblock {\em arXiv preprint arXiv:1609.05587}.

\bibitem{Wang2017TR}
W.~Wang, V.~Aggarwal, and S.~Aeron.
\newblock Efficient low rank tensor ring completion.
\newblock In {\em IEEE International Conference on Computer Vision}, pages
  5698--5706, 2017.

\bibitem{Wang2019Global}
Y.~Wang, W-T. Yin, and J-S. Zeng.
\newblock Global convergence of {ADMM} in nonconvex nonsmooth optimization.
\newblock {\em Journal of Scientific Computing}, 78(1):29--63, 2019.

\bibitem{Wen2008Restoration}
Y.~{Wen}, M.~K. {Ng}, and Y.~{Huang}.
\newblock Efficient total variation minimization methods for color image
  restoration.
\newblock {\em IEEE Transactions on Image Processing}, 17(11):2081--2088, 2008.

\bibitem{Wu2010Augmented}
C-L. Wu and X-C. Tai.
\newblock Augmented lagrangian method, {D}ual methods, and split {B}regman
  iteration for {ROF}, vectorial {TV}, and high order models.
\newblock {\em SIAM Journal on Imaging Sciences}, 3(3):300--339, 2010.

\bibitem{Xie2018Kronecker}
Q.~Xie, Q.~Zhao, D-Y. Meng, and Z-B. Xu.
\newblock Kronecker-basis-representation based tensor sparsity and its
  applications to tensor recovery.
\newblock {\em IEEE Transactions on Pattern Analysis and Machine Intelligence},
  40(8):1888--1902, 2018.

\bibitem{Xie2016Multispectral}
Q.~Xie, Q.~Zhao, D-Y. Meng, Z-B. Xu, S-H. Gu, W-M. Zuo, and L.~Zhang.
\newblock Multispectral images denoising by intrinsic tensor sparsity
  regularization.
\newblock In {\em IEEE Conference on Computer Vision and Pattern Recognition},
  pages 1692--1700, 2016.

\bibitem{Xing2012Dictionary}
Z-M. Xing, M-Y. Zhou, A.~Castrodad, G.~Sapiro, and L.~Carin.
\newblock Dictionary learning for noisy and incomplete hyperspectral images.
\newblock {\em SIAM Journal on Imaging Sciences}, 5(1):33--56, 2012.

\bibitem{Yang2017TTRNN}
Y-C. {Yang}, D.~{Krompass}, and V.~{Tresp}.
\newblock Tensor-train recurrent neural networks for video classification.
\newblock In {\em International Conference on Machine Learning}, volume~70,
  pages 3891--3900, 2017.

\bibitem{Yu2019TRNN}
J-S. Yu, C.~Li, Q.~Zhao, and G-X. Zhou.
\newblock Tensor-ring nuclear norm minimization and application for visual data
  completion.
\newblock In {\em IEEE International Conference on Acoustics, Speech and Signal
  Processing}, pages 3142--3146, 2019.

\bibitem{Yuan2019TR}
L-H. Yuan, C.~Li, J-T. Cao, and Q-B. Zhao.
\newblock Rank minimization on tensor ring: an efficient approach for tensor
  decomposition and completion.
\newblock {\em Machine Learning}, 2019.

\bibitem{Yuan2018High}
L-H. Yuan, Q-B. Zhao, and J-T. Cao.
\newblock High-dimension tensor completion via gradient-based optimization
  under tensor-train format.
\newblock {\em Signal Processing: Image Communication}, 73:53--61, 2019.

\bibitem{Zhang2018Nonconvex}
X-J. Zhang.
\newblock A nonconvex relaxation approach to low-rank tensor completion.
\newblock {\em IEEE Transactions on Neural Networks and Learning Systems},
  30(6):1659--1671, 2019.

\bibitem{Zhang2017tSVD}
Z.~Zhang, G.~Ely, and S~Aeron.
\newblock Exact tensor completion using t-{SVD}.
\newblock {\em IEEE Transactions on Signal Processing}, 65(6):1511--1526, 2017.

\bibitem{Zhao2015Bayesian}
Q-B. Zhao, L-Q. Zhang, and A.~Cichocki.
\newblock Bayesian {CP} factorization of incomplete tensors with automatic rank
  determination.
\newblock {\em IEEE Transactions on Pattern Analysis and Machine Intelligence},
  37(9):1751--1763, 2015.

\bibitem{Zhao2016TR}
Q-B. Zhao, G-X. Zhou, S-L. Xie, L-Q. Zhang, and A.~Cichocki.
\newblock Tensor ring decomposition.
\newblock {\em arXiv preprint arXiv:1606.05535}.

\bibitem{Zhao2013Unmixing}
X-L. Zhao, F.~Wang, T-Z. {Huang}, M.~K. {Ng}, and R.~J. {Plemmons}.
\newblock Deblurring and sparse unmixing for hyperspectral images.
\newblock {\em IEEE Transactions on Geoscience and Remote Sensing},
  51(7):4045--4058, 2013.

\bibitem{Zheng2018Ntubal}
Y-B. {Zheng}, T-Z. {Huang}, X-L. {Zhao}, T-X. {Jiang}, T-Y. {Ji}, and T-H.
  {Ma}.
\newblock Tensor {N}-tubal rank and its convex relaxation for low-rank tensor
  recovery.
\newblock {\em arXiv.org}, arXiv:1812.00688, 2018.

\bibitem{Zheng2019Mixed}
Y-B. Zheng, T-Z. Huang, X-L. Zhao, T-X. Jiang, T-H. Ma, and T-Y. Ji.
\newblock Mixed noise removal in hyperspectral image via low-fibered-rank
  regularization.
\newblock {\em IEEE Transactions on Geoscience and Remote Sensing},
  58(1):734--749, 2020.

\end{thebibliography}
\end{document}